\documentclass{article}

\usepackage{natbib}
\usepackage{microtype}
\usepackage{graphicx}
\usepackage{blindtext}
\usepackage{hyperref}
\usepackage{caption}
\usepackage{subcaption}
\usepackage{booktabs} 
\usepackage{tcolorbox}
\usepackage{multicol}
\usepackage{multirow}
\usepackage{bbm}

\usepackage{hyperref}

\usepackage[accepted]{arxiv}

\usepackage{minitoc}
\usepackage{amsmath}
\usepackage{amssymb}
\usepackage{mathtools}
\usepackage{amsthm}
\usepackage{thm-restate}
\usepackage[capitalize,noabbrev]{cleveref}

\usepackage{lipsum}
\usepackage{algorithm}
\usepackage{algorithmic}
\usepackage{nicematrix}
\usepackage{eqparbox}
\renewcommand\algorithmiccomment[1]{
  \hfill\#\ \eqparbox{COMMENT}{#1}
}
\usepackage{paralist}

\usepackage{etoolbox}  
\usepackage{dashbox}

\makeatletter
\patchcmd{\algorithmic}{\addtolength{\ALC@tlm}{\leftmargin} }{\addtolength{\ALC@tlm}{\leftmargin}}{}{}
\makeatother

\theoremstyle{plain}
\newtheorem{theorem}{Theorem}[section]

\newtheorem{lemma}[theorem]{Lemma}
\newtheorem{corollary}[theorem]{Corollary}
\theoremstyle{definition}
\newtheorem{definition}[theorem]{Definition}
\newtheorem{assumption}[theorem]{Assumption}
\theoremstyle{remark}
\newtheorem{remark}[theorem]{Remark}

\usepackage[textsize=tiny]{todonotes}

\icmltitlerunning{On the Power of Context-Enhanced Learning in LLMs}

\usepackage{amsmath,amsfonts,bm}

\def\eqref#1{equation~\ref{#1}}

\def\ceil#1{\lceil #1 \rceil}
\def\floor#1{\lfloor #1 \rfloor}
\def\1{\bm{1}}

\def\ra{{\textnormal{a}}}

\def\rmV{{\mathbf{V}}}

\def\va{{\bm{a}}}
\def\vb{{\bm{b}}}

\def\ve{{\bm{e}}}

\def\vg{{\bm{g}}}

\def\vk{{\bm{k}}}
\def\vl{{\bm{l}}}

\def\vo{{\bm{o}}}

\def\vq{{\bm{q}}}

\def\vs{{\bm{s}}}

\def\vv{{\bm{v}}}

\def\vx{{\bm{x}}}
\def\vy{{\bm{y}}}

\def\mA{{\bm{A}}}

\def\mC{{\bm{C}}}

\def\mH{{\bm{H}}}
\def\mI{{\bm{I}}}

\def\mM{{\bm{M}}}

\def\mP{{\bm{P}}}
\def\mQ{{\bm{Q}}}

\def\mS{{\bm{S}}}

\def\mV{{\bm{V}}}
\def\mW{{\bm{W}}}

\DeclareMathAlphabet{\mathsfit}{\encodingdefault}{\sfdefault}{m}{sl}
\SetMathAlphabet{\mathsfit}{bold}{\encodingdefault}{\sfdefault}{bx}{n}

\newcommand{\cJ}{\mathcal{J}}

\newcommand{\cL}{\mathcal{L}}

\newcommand{\cU}{\mathcal{U}}

\newcommand{\cY}{\mathcal{Y}}

\newcommand{\R}{\mathbb{R}}

\DeclareMathOperator*{\argmax}{arg\,max}
\DeclareMathOperator*{\argmin}{arg\,min}

\newcommand{\norm}[1]{\left\lVert#1\right\rVert}
\newcommand{\snorm}[1]{\lVert#1\rVert}

\newcommand{\abs}[1]{\left\lvert#1\right\rvert}

\newcommand{\pdr}[2]{\frac{\partial #1}{\partial #2}}

\newcommand{\pr}[1]{\mathbb{P}\left[#1\right]}

\newcommand{\var}[1]{Var\left[#1\right]}

\newcommand{\rbr}[1]{\left(#1\right)}
\newcommand{\srbr}[1]{(#1)}
\newcommand{\sbr}[1]{\left[#1\right]}
\newcommand{\cbr}[1]{\left\{#1\right\}}
\newcommand{\scbr}[1]{\{#1\}}

\newcommand{\NN}{\mathbb{N}}

\newcommand{\diag}{\text{diag}}
\newcommand{\T}{\top}

\newtcolorbox{thmbox}{
  colback=blue!5,        
  colframe=blue!75!black,
}
\newtcolorbox{defnbox}{
  colback=black!5,        
  colframe=black!75!black,
}
\newtcolorbox{promptbox}{
  colback=brown!5,        
  colframe=brown!75!black,
  fontupper=\linespread{1}\selectfont
}

\newtcolorbox{defnboxmaintext}{
  colback=black!5,        
  colframe=black!75!black,
  left=4pt,right=4pt,top=1pt,bottom=1pt
}

\newtcolorbox{remarkbox}{
  colback=orange!5,        
  colframe=orange!75!black,
}

\newtcolorbox{remarkboxmaintext}{
    colback=orange!5,        
  colframe=orange!75!black,
  left=4pt,right=4pt,top=4pt,bottom=4pt
}

\def\functionclass{\mathcal{F}}
\def\distribution{\mathcal{D}}
\def\sqdim{\text{SQ-dim}}

\def \task {g}
\def \loss {\ell}
\def \autoloss {\ell_{\text{auto}}}
\def \taskdesc {\textsc{desc}}
\def \str {\textsc{str}}

\def \curriculumtext {curriculum-text}

\def \langmodel {f}
\def \lmparam {\theta}

\def \taskcapable {$\task$-capable}

\def \dataset {D}
\def \maskcurr {\textsc{curr}}
\def \numdatapoints {N}
\def \numDPs {\numdatapoints}

\def \strcodebook {set of phrasebooks}
\def \strdicts {phrasebooks}
\def \strdict {phrasebook}

\def \var {x}
\def \ind {\mathbbm{1}}

\def \depth {d}
\def \numchar {n}
\def \dictionary {\mathbf{\pi}}
\def \specialdictionary {\mathbf{\Delta}}

\def \vardictionary {\mathbf{\pi}}
\def \codebook {\mathbf{\Pi}}

\def \transfunc {T}
\def \totaltransfunc {\translationtask}

\def \vardict {\vardictionary}

\def \alphabetset {A}

\def \translationtask {\text{\textbf{MLT}}}

\def \inputlength {L}
\def \inputseq {\vs_1}
\def \intermseq {\vs }
\def \shiftintermseq {\Tilde{\vs}}

\def \ICM {context-enhanced learning}

\def \split {\textit{Split}}
\def \translate {\textit{Translate}}
\def \circularshift {\textit{Circular shift}}
\def \merge{\textit{Merge}}

\def \copyop{\text{\textbf{copy}}}
\def \xorop{\text{\textbf{xor}}}
\def \notop{\text{\textbf{not}}}

\def \uniform {\mathcal{U}}
\def \corr {\text{Correlation}}

\def \mirror {\text{Mirrorset}}

\def \targetcodebook {{\codebook^*}}
\def \targetCB {\targetcodebook}
\def \Ttoken {\texttt{<THINK>}}
\def \Tstring {\texttt{<THINK>,...}}

\def \surrogateModel {\surrogate}
\def \surrogateWeights {{\cbr{\W_i}_{i=1}^\depth}}
\def \Q {\shiftmat}
\def \QT {\shiftmat^\T}
\def \zeromat {\mathbf{0}}

\def \vecrep {\vv}
\def \svecrep {\ve}
\def \matrep {\mV}
\def \embed {\mH}
\def \rmatrep {\rmV}
\def \shiftmatrep {\tilde{\mV}}
\def \rshiftmatrep {\tilde{\rmV}}
\def \shiftfunc {\mathop{\texttt{Shift}}}
\def \maxfunc {\texttt{HardMax}}
\def \softmaxfunc {\texttt{SoftMax}}

\def \matricize {\textbf{Mat}}
\def \shiftmat {\mQ}
\def \transmat {\texttt{Matrix}}
\def \contextmat {\mC}
\def \weightmat {\mW}
\def \surrogate {\textsc{Surr-MLT}}

\def \onehot {\ve}
\def \levelembedding {\vl}
\def \startendembedding {\vb}
\def \segment {\vg}
\def \constructmodule {\textsc{MLT-module}}
\def \constructcircularshiftmodule {\textsc{Circular Shift-module}}
\def \constructtranslatemodule {\textsc{Translate-module}}

\def \think {\Ttoken}
\def \pad {\texttt{<P>}}

\def \W {\weightmat}
\def \P {\mP}
\def \CM {\contextmat}

\def \Vmat {\matrep}
\def \gtW {\weightmat^*}
\def \gdloss {\cL}

\def \len {\inputlength}
\def \lonehot {\bar{\onehot}}
\newcommand{\ai}[1]{a_{(#1)}}
\newcommand{\rai}[1]{{\ra_{(#1)}}}
\newcommand{\bi}[1]{b_{(#1)}}

\newcommand{\V}[2]{\Vmat_{#1}^{(#2)}}
\newcommand{\RV}[2]{\rmatrep_{#1}^{(#2)}}
\newcommand{\VT}[2]{\Vmat_{#1}^{*(#2)}}
\newcommand{\RVS}[2]{\rshiftmatrep_{#1}^{*(#2)}}
\newcommand{\SV}[2]{\tilde{\Vmat}_{#1}^{(#2)}}
\newcommand{\SVT}[2]{\tilde{\Vmat}_{#1}^{*(#2)}}
\newcommand{\transarrow}[1]{\overset{#1}{\longrightarrow}}

\definecolor{newgreen}{RGB}{17,200,214}
\definecolor{newpink}{RGB}{254,149,208}
\definecolor{newred}{RGB}{254,1,143}

\newcommand{\red}[1]{{\color{newred}{#1}}}
\newcommand{\orange}[1]{{\color{newpink}{#1}}}
\newcommand{\green}[1]{{\color{newgreen}{#1}}}

\def\shownotes{1}  \ifnum\shownotes=1
\newcommand{\authnote}[2]{{[#1: #2]}}
\else
\newcommand{\authnote}[2]{}
\fi

\newcommand{\circular}[1]{{}^{\circlearrowleft}{#1}}

\def \family {\textsc{family}}

\def \masking {\textsc{Random-Drop}}

\def \batchsize {B}

\def \fixdpstrat {{\bf Fixed Dropout}}
\def \annealdpstrat {{\bf Annealing Dropout}}
\begin{document}
\doparttoc 
\faketableofcontents 
\twocolumn[
\icmltitle{On the Power of Context-Enhanced Learning in LLMs}

\icmlsetsymbol{equal}{*}

\begin{icmlauthorlist}
\icmlauthor{Xingyu Zhu}{princeton university,equal}
\icmlauthor{Abhishek Panigrahi}{princeton university,equal}
\icmlauthor{Sanjeev Arora}{princeton university}
\end{icmlauthorlist}

\icmlaffiliation{princeton university}{Princeton Language and Intelligence, Princeton University}

\icmlcorrespondingauthor{}{$\{$xingyu.zhu, ap34$\}$@princeton.edu}

\icmlkeywords{Machine Learning, ICML}

\vskip 0.3in
]

\printAffiliationsAndNotice{\icmlEqualContribution} 


\begin{abstract}
We formalize a new concept for LLMs, \textbf{context-enhanced learning}. It involves standard gradient-based learning on text except that the context is enhanced with additional data on which no auto-regressive gradients are computed. This setting is a gradient-based analog of usual in-context learning (ICL) and appears in some recent works.

Using a multi-step reasoning task, we prove in a simplified setting that context-enhanced learning can be \textbf{exponentially more sample-efficient} than standard learning when the model is capable of ICL. At a mechanistic level, we find that the benefit of context-enhancement arises from a more accurate gradient learning signal.
We also experimentally demonstrate that it appears hard to detect or recover learning materials that were used in the context during training. This may have implications for data security as well as copyright.
\end{abstract}

\section{Introduction}
\label{sec:intro}

Pre-trained LLMs \citep{brown2020language,touvron2023llama,team2023gemini} show strong capability to learn new material at inference time, for instance via in-context-learning (ICL). There is also emerging evidence that gradient-based learning on a piece of text (say, math Q\&A) can be enhanced if additional helpful text is placed in the context, even though no auto-regressive loss is computed on this helpful text ~\cite{liao2024textit,zou2024promptintern,choi2025teaching}. Such strategies have also been shown to benefit pre-training as prepending source URLs to documents can enhance the model's training efficiency and memorization capacity \citep{allen2024physics,gao2025metadata}. 

In this paper we seek to formally study this phenomenon, whereby LLMs' gradient-based learning is {\em enhanced} via placement of additional helpful material in the context,  but without actual auto-regressive gradient updates on this  material. We will call this form of learning {\em \ICM{}}.
Because the material used for context-enhancement can evolve over the course of training, this approach naturally aligns with the idea of a {\em curriculum}.

Context-enhanced learning intuitively mirrors how humans learn: when solving problems, they refer to textbooks or demonstrations for guidance, yet they do not seek to memorize these resources {\em per se}. An analogous concept {\em Learning using Privileged Information} (or LUPI), has been well-studied in the context of kernel SVMs~\cite{vapnik2009new} and classification models. Our work adapts this concept for LLMs, and surfaces the following questions:

\looseness-1\hypertarget{q1}{\textbf{Q1:}}\quad  Even though autoregressive loss is computed on the same set of tokens, can \ICM{} be significantly more powerful than usual auto-regressive learning that has no additional in-context materials? If so, can we theoretically characterize and understand the mechanism behind such improvement?

\looseness-1\hypertarget{q2}{\textbf{Q2:}}\quad Do models need a certain capability level to benefit from \ICM{}? 
This is a natural question, since
leveraging in-context information (e.g., ICL) likely requires a minimum capability level or model size~\cite{brown2020language,wei2022emergent}.

\looseness-1\hypertarget{q3}{\textbf{Q3:}}\quad Is \ICM{} a viable way to use  privileged/private information during learning? Providing such privileged information in the context could conceivably enhance the model's learning, but since no auto-regressive gradient updates happen on the privileged/private information, there might be a lower risk of leakage of such information via  API calls. 

\looseness-1
\textbf{Paper Overview:}
 \cref{sec:setup-contextAidedLearning}  formally defines \ICM{}. To allow rigorous understanding of the power of \ICM{}, \cref{sec:setup-translationTask}  introduces a multi-step reasoning task 
called {\em Multi-layer translation}.  This is a synthetic setting involving  $d+1$ languages $L_1, L_2, \dots, L_{d+1}$ over finite alphabets. For each $i$, there is a simple {\em \strdict{}} that describes how to translate from $L_i$ to $L_{i+1}$, and the mapping from $L_1$ to $L_{d+1}$ is a sequential application of the set of \strdicts{}.

The goal is to learn how to translate text from $L_1$ into $L_{d+1}$ without explicitly writing down intermediate steps. The learner is provided excerpts from these \strdicts{} as helpful information in the context during training, but allowed no auto-regressive gradient updates on these tokens. 

If we train with auto-regressive loss on translation output conditioning on the \strdicts{}' excerpts and the input, a model with certain ICL capacity level may quickly learn the translation task by leveraging the in-context \strdicts{}.
However, this learning could be brittle, that the model becomes reliant on having the \strdicts{}' excerpts in context.
This reliance can be weaned off by use of probabilistic {\em dropout} on \strdicts{} tokens in context. 
Intuitively, this curriculum forces the model to not only read \strdicts{}' excerpts, but also gradually internalize the \strdicts{}' contents. Over time, the model's ability to translate from $L_1$ to $L_{\depth+1}$ will become robust to the dropout of \strdicts{}' excerpts, and eventually, their complete removal.

\looseness-1Experiments show that this training strategy indeed works when the learner is a pre-trained LLM that is capable of ICL (but fails when LLM is incapable of ICL). Even when training with $20\%$ dropout rate, the model can perfectly translate strings from $L_1$ to $L_{\depth+1}$ without any \strdicts{}' excerpts at test time. The rest of the paper is structured as follows:

\looseness-1$\bullet$ \cref{sec:experiments} details our experiments and the findings sketched above. Experiments show that an ICL-capable model follows an intuitive sequential processing of the \strdicts{} provided in-context, whereby transformer layers approach stages of translation in an intuitive way; e.g., $L_3 \rightarrow L_4$  is done after $L_2 \rightarrow L_3$
(\cref{sec:mechanistic}).

\looseness-1$\bullet$ \cref{sec:memorization}, shows
that after \ICM{}, the output probabilities of the model reveal little about the  \strdicts{} rules that were seen during training.

\looseness-1$\bullet$ In \cref{sec:theory}, we propose a theoretical framework using a surrogate/simplified model that represents an ideal LLM for the translation task (\cref{sec:theory-surrogateModel}). This framework shows an exponential gap in sample complexity depending on whether the model is trained with or without in-context information on \strdicts{} (\cref{sec:theory-hard,sec:theory-converge}).  Experiments reveal that the mechanism behind the increased sample efficiency of context-enhanced learning is an improved gradient signal, measured by gradient prediction accuracy (\cref{sec:theory-snr}).

\section{Setup}

\subsection{Context-Enhanced Learning}
\label{sec:setup-contextAidedLearning}

Let $X$ be the space of all possible text strings and let $\mathcal{Y}$ be the space of all possible distributions over texts. Let $\task$ be a language task mapping inputs $x\in X_g\subset X$ to a distribution $Y \in \mathcal{Y}$. Let $\langmodel_\lmparam: X \to \mathcal{Y}$ be a general auto-regressive language model. We characterize $\langmodel_\lmparam$'s capability on task $g$ as follows:
\begin{defnboxmaintext}
\begin{definition}[\taskcapable{} model, informal]
\label{def:task-capable-model}
   A language model $\langmodel_\lmparam$ is \taskcapable{} for a language task $\task$ if  $\langmodel_\lmparam$ is close to $\task$, as measured by a suitable metric on $X_g$.
\end{definition}
\end{defnboxmaintext}

\looseness-1Vanilla supervised fine-tuning (SFT) aims to create a \taskcapable{} model by minimizing auto-regressive loss $\autoloss$ on a supervised dataset $\dataset_\task = \cbr{(x_i, y_i)}_{i=1}^{\numDPs}$, where the label $y_i$ for each $x_i\in X_g$ is sampled from $\task(x_i)$. 

\looseness-1 \emph{Context-enhanced learning} involves augmenting the supervision with additional \curriculumtext{} that depends on the task $\task$, input $x$, and training step $t$. We denote \curriculumtext{} as $\maskcurr_\task(x, t)$, which could be anything (helpful explanations, excerpts from textbooks,   worked-out examples, etc.).

\begin{algorithm}[H]
   \caption{\textbf{Context-Enhanced Learning} \\In contrast to standard SFT, it relies on \color{blue}{curriculum-text} in context on which \textbf{no auto-regressive loss is computed}.}
   \label{alg:contextEnhancedLearning}
\begin{algorithmic}
   \STATE {\bfseries Input:} Supervised dataset $\dataset_\task$, \curriculumtext{} $\maskcurr_\task$, initialization $\lmparam$, total steps $T$
   \FOR{$t=1$ {\bfseries to} $T$}
   \STATE Sample $(x, y)\sim \dataset_\task$
   \STATE Compute loss $l \gets \autoloss \rbr{\langmodel_\lmparam([{\color{blue}\maskcurr_\task(x, t)}, x,y]), y}.$
   \STATE Update parameters $\lmparam$ with gradient $\nabla_\lmparam l$ 
   \ENDFOR
   \STATE Return $\lmparam$
\end{algorithmic}
\end{algorithm}
\vspace{-0.2in}

\looseness-1On sample $(x, y)$ drawn from the supervised dataset, we use auto-regressive loss for model's prediction on $y$ conditioned on $[\maskcurr_\task(x, t), x]$ to train our models. Note that \emph{no loss is computed for \curriculumtext{} tokens.} We denote this loss as $\autoloss \rbr{\langmodel_\lmparam([\maskcurr_\task(x, t), x, y]), y}$. 

\subsection{Multi-level Translation (\translationtask)}
\label{sec:setup-translationTask}

To study the  power of context-enhanced learning, we introduce a multi-step translation task that is easy to learn with a straightforward curriculum in the context, but 
very difficult to learn with just input-output examples.

The task is inspired by encryption methods\footnote{Note that there is no proof that current cryptographic tasks are fundamentally difficult for computers. Here the task makes sense as an example of multi-step reasoning and we are interested in sample-complexity lower bounds.}  such as the Feistel cipher 
\cite{knudsen1993Feistelcipher}. 
The multi-level translation (\translationtask) task involves a bijective mapping from strings to strings that is a composition of $2d$ simpler bijections, each involving simple shift by $1$, or transforming bigrams ($2$-tuples of characters) via a bijection.  The depth-$d$ translation can be described by $O(d)$ bits.  But we will show that learning the task only from input-output pairs would require $e^{\Omega(d)}$ sample complexity in the SQ-learning framework \citep{kearns1998efficient} (see~\cref{thm:theory-sqdim}).

\vspace{-0.1in}
\begin{table}[H]
\caption{Important notations for defining \translationtask{}}
    \label{tab:notation_table}
\scalebox{0.8}{
    \centering
    \begin{tabular}{c|c}
    \toprule
       $\depth$  & Depth of translation task \\
       $\numchar$  & Number of characters in each alphabet \\
       $\alphabetset$ & An alphabet set \\
       $\dictionary$ & A \strdict{} between 2-tuples in two alphabets\\
       $\codebook$ & A set of \strdicts{} $\{\dictionary_i\}$ defining a translation task \\
       $\translationtask_\codebook$ & Translation task with a \strcodebook{} $\codebook$\\
       $\translationtask(\depth, \numchar)$ & Family of translation tasks of depth $\depth$ and $\numchar$ characters \\
    \bottomrule
    \end{tabular}
}
\end{table}
\vspace{-0.1in}

Concretely, let $\alphabetset_1, \dots, \alphabetset_{\depth+1}$ be $\depth+1$ alphabets all of the same size with $\numchar$ characters. For every consecutive pair of alphabets $\alphabetset_i$ and $\alphabetset_{i+1}$, we fix a \emph\strdict{} $\vardictionary_i: \alphabetset_i^2\to \alphabetset_{i+1}^2$ as a bijective mapping from 2-tuples in $\alphabetset_i$ to 2-tuples in $\alphabetset_{i+1}$. Each \strdict{} $\vardict_i$ can be represented by a binary stochastic matrix $\transmat({\vardict_i})$ with rules represented as one-hot columns (see \cref{def:apdx-surrogate-translationMatrixRepresentation}).

\begin{figure}[H]
    \centering
    \vspace{-0.05in}
    \includegraphics[width=0.85\linewidth]{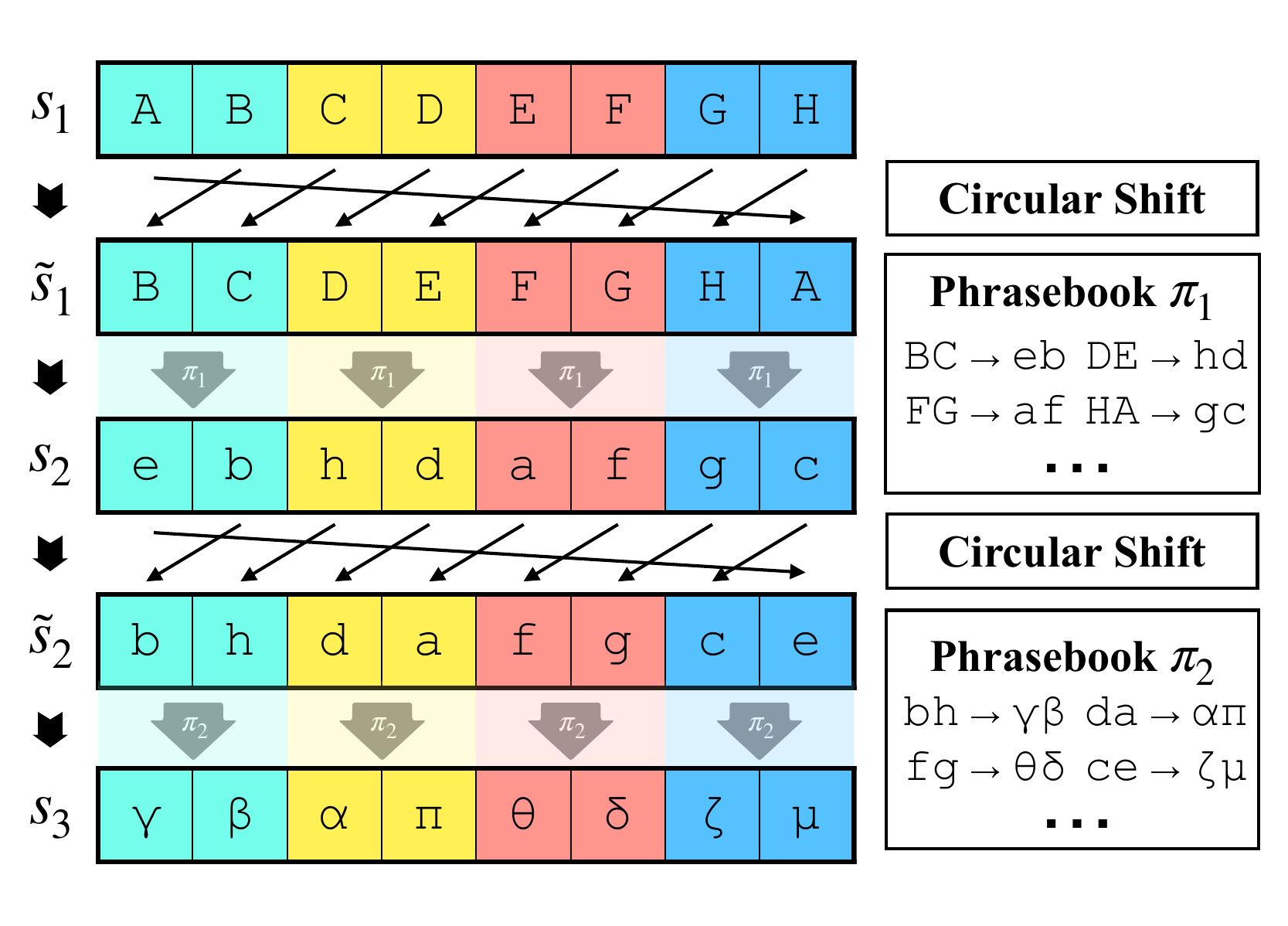}
    \vspace{-0.13in}
    \caption{Illustration of an \translationtask$(2,8)$ instance with input sequence of $8$ tokens. Input sequence $\inputseq$ went through 2 translation steps to $\intermseq_3$.
    Each output character depends on $4$ input characters, for example, character $\mu$ is derived from \texttt{c,e} in $\intermseq_2$, which, in turn, are computed from \texttt{A,B,C,H} in $\intermseq_1$.
    }
    \vspace{-0.1in}
    \label{fig:translation-demo}
\end{figure}

The input of the translation process is an even-length sequence, referred to as $\inputseq \in \alphabetset_1^{\inputlength}$ where $\inputlength$ is the sequence length. The translation process modifies $\inputseq$ recursively. For every $i\in[\depth]$, $\intermseq_i \in \alphabetset_i^{\inputlength}$ will be transformed to $\intermseq_{i+1}\in \alphabetset_{i+1}^{\inputlength}$ using \emph\strdict{} $\vardict_i$ through the following 2 sub-processes:

$\bullet$ \circularshift{}: The characters in $\intermseq_i \in \alphabetset_i^{\inputlength}$ are shifted by $1$ character leftward (and wrapped around to the end if necessary) to give sequence $\shiftintermseq_i \in \alphabetset_i^{\inputlength}$. Formally, for each $j \in [1, \inputlength]$ we have $\shiftintermseq_{i, j} = \intermseq_{i, (j+1) \% \inputlength}$.

 $\bullet$  \translate{}: Using the \emph\strdict{} $\vardictionary_i: \alphabetset_i^2\to \alphabetset_{i+1}^2$, we translate 2-tuples (bigrams) of consecutive characters in sequence $\shiftintermseq_i$ to create $\intermseq_{i+1}$. That is, for every odd $j  \in [1, \inputlength]$, $(\intermseq_{i+1, j}, \intermseq_{i+1, j+1}) = \vardict_i (\shiftintermseq_{i, j}, \shiftintermseq_{i, j+1})$.

We denote the mapping from $\intermseq_i$ to $\intermseq_{i+1}$ as $\intermseq_{i+1} = \transfunc_{\vardict_i}(\intermseq_i)$.  
The $\depth$-step translation is defined as the composition $\intermseq_{\depth+1} = \transfunc_{\vardict_\depth}\circ \transfunc_{\vardict_{\depth-1}}\circ\dots\circ \transfunc_{\vardict_{1}}\rbr{s_1}$
which converts the input sequence $\inputseq$ to $\intermseq_{d+1}$ through $\depth$ translation steps. Please see \cref{fig:translation-demo} for a visual illustration with $d=2, n=8$. 

We denote $\codebook=\cbr{\dictionary_i}_{i=1}^\depth$ the \strcodebook{} of all levels, and denote $\totaltransfunc_\codebook: \inputseq\mapsto \totaltransfunc_\codebook(\inputseq):=\transfunc_{\vardict_\depth}\circ\dots\circ \transfunc_{\vardict_{1}}\rbr{s_1}$ the mapping from input $\inputseq$ to output $\intermseq_{\depth + 1}$ using $\codebook$. We use $\translationtask(\depth,\numchar)$ to refer to the family of translation tasks involving $\depth$ step and $\numchar$ characters in each alphabet.

We note that $\translationtask(\depth,\numchar)$ has two key properties: 
\textbf{1.} Once the \strdicts{} are fixed, the translation task defines a bijection between input and output strings since \circularshift{} and \translate{} are invertible (see \cref{lemma:setting-MLTinvertability}).
\textbf{2.} Each character in the output string depends on $2 \depth$ characters from the input text string (see caption of \cref{fig:translation-demo}), making learning from input-output pairs very difficult (\cref{thm:theory-sqdim}). 

\looseness-1For each \strdict{} $\dictionary_i$, we compose its textual representation $\str\rbr{\dictionary_i}$ to be of the form \texttt{... a b -> C D; e d -> B A; ...}
which lists (insensitive of ordering) \strdict{} rules between 2-tuples in the previous alphabet and the next alphabet. Moreover, we denote the concatenation $[\str\rbr{\dictionary_1},\dots, \str\rbr{\dictionary_d}]$ as $\str\rbr{\codebook}$, and will be used to define \curriculumtext{}.

\subsection{Needed: Curriculum without Explicit CoT}
\looseness-1To teach the model a particular translation task $\translationtask_\codebook$ from input-output pairs of the form $(\inputseq, \totaltransfunc_\codebook(\inputseq))$, we can train it with relevant sections of the phrasebooks $\str\rbr{\codebook}$ in context as curriculum, but at test time it would not have access to the phrasebook so it is important not to teach it explicit chain-of-thought (CoT) containing in-context information. (Another consideration is data privacy, with the phrasebook being considered privileged information.) However, a dual use of CoT is to provide the model extra compute at inference time \cite{goyal2023thinkToken}, which is needed here since the translation task has $d$ stages.  To facilitate such silent computation we teach the model to output a fixed number of {\tt <THINK>} tokens, sometimes refered to as {\em silent CoT} or {\em internalized CoT}.

 \subsection{ICL-capablity for $\translationtask(\depth, \numchar)$}

To learn from books, one needs to know how to read. The analogous notion under study here is whether context-enhanced learning requires  capability to sort-of ``understand'' the in-context material (\hyperlink{q2}{\textbf{Q2}}).  In the context of \translationtask, we formalize such capability as being able to achieve low loss on the translation task when provided with the relevant phrasebook sections in context while allowing silent CoT.

\begin{defnboxmaintext}
\begin{definition}[$\translationtask(\depth, \numchar)$-ICL-capability, informal]
\label{def:icl-capable-model}
    A language model $\langmodel_\lmparam$ is $\translationtask(\depth, \numchar)$\emph{-ICL-capable} if for any set of phrasebooks $\codebook$ in $\translationtask(\depth, \numchar)$, $\langmodel_\lmparam([\str\rbr{\codebook}, \inputseq])$ is close to $\intermseq_{\depth + 1} = \translationtask_{\codebook}(\inputseq)$ disregarding the \think{} tokens, when measured by a discrepancy metric over all valid input strings $\inputseq$.
\end{definition}
\end{defnboxmaintext}

\section{Experiments and Observations}
\label{sec:experiments}

\looseness-1In this section, we fix a \strcodebook{} $\targetCB$ and study \ICM{} on $\translationtask_\targetCB$.

We first introduce the preparation of an $\translationtask(\depth, \numchar)$-ICL-capable model. We then introduce a \ICM{} curriculum involving random dropping of phrasebook rules in context. We then present empirical evidence for significant sample efficiency of \ICM{}. We conclude the section with mechanistic insights into \ICM{} concerning internal representations and  evolution of parameters.

We use the \textbf{Llama 3.2-3B instruction-tuned model} \citep{dubey2024llama} as the base model and fix $\depth=5$ with $n=8$ or $10$.
Detailed configurations are available in \cref{sec:apdx-expconfig}.

\subsection{Experimental Setup}\label{sec:experiments-stage1}

\textbf{(i) Preparing an $\translationtask(\depth, \numchar)$-ICL-Capable Model:}\\
The Llama 3.2B model is $\translationtask(\depth, \numchar)$-ICL-capable as it has not seen the  task during training. To make it ICL-capable for our purpose, we use SFT on other random translation tasks with random \strdicts{} $\codebook_1,\dots, \codebook_M$, following common CoT internalization pipeline \citep{deng2024internalizeCoT,pfau2024let, hao2024coconut}. We use one training example per \strcodebook{} to prevent memorization of specific phrasebooks. At the end of training, given  input $[\str(\codebook),\inputseq]$ for any $\inputseq$ and $\codebook$ the model can generate $\Tstring,\translationtask_\codebook(\inputseq)$ correctly. 
Details on the first stage of training are described in \cref{sec:apdx-expconfig-stage1}.

\textbf{(ii) Setting up \ICM{} for $\translationtask_\targetCB$:}
We use the $\translationtask(\depth, \numchar)$-ICL-capable model above as initialization and train for $\translationtask_\targetCB$. Supervised dataset $D_\targetCB$ is curated with input-label pairs of the form $(\inputseq, [\Tstring,\translationtask_\targetCB(\inputseq)])$, where $\inputseq$ is a random string sampled from $\alphabetset_1$, with length between 20 and 40.

We define curriculum-text $\maskcurr_\targetCB(\inputseq, t)$ using excerpts from phrasebooks $\str(\targetCB)$ (selected based on $\inputseq$) with random dropout of rules (parameterized by training step $t$). We explore the following curriculum and study their impact:

$\bullet$ {\bf No Context} (vanilla SFT): Empty \curriculumtext{}.

$\bullet$ \fixdpstrat{}: A simple strategy independent of step $t$; given $\inputseq$, only curate rules in $\targetCB$ used in the translation of $\inputseq$, then randomly drop 20\% of the curated rules.

$\bullet$ \annealdpstrat{}: A better strategy: for $\inputseq$, select the necessary rules from $\targetCB$ plus $25\%$ unused rules. Apply random dropout on these rules, increasing linearly from $0\%$ to $100\%$ over the first $60\%$ of training, then maintain $100\%$.

$\bullet$ {\bf No Dropout} (ablation): Given $\inputseq$, always provide all rules in $\targetCB$ used in the translation of $\inputseq$ in \curriculumtext{}.

$\bullet$ {\bf Wrong Context} (ablation): Equivalent to \annealdpstrat{} but the rules in the curriculum are incorrect.

\subsection{Experiment Results}
\label{sec:experiments-results}

To check the sample efficiency benefit of \ICM{} (\hyperlink{q1}{\textbf{Q1}}), we construct supervised datasets $D_\targetCB$ with $10^4$ to $10^6$ unique samples and train the models for one epoch on each.\footnote{We fix 1 epoch for fair comparison on sample efficiency.} We report the next-token prediction accuracy on the final answer tokens (ignoring thought tokens) for held-out samples when conditioning on no \curriculumtext{} (100\% dropped-out) and compare against the  supervised dataset size. To check the necessity of proper ICL capability (\hyperlink{q2}{\textbf{Q2}}), we ablate with \annealdpstrat{} but starting from non-$\translationtask(\depth, \numchar)$-ICL-Capable 3B base model.

\begin{figure}[!htbp]
    \centering
    \includegraphics[width=1\linewidth]{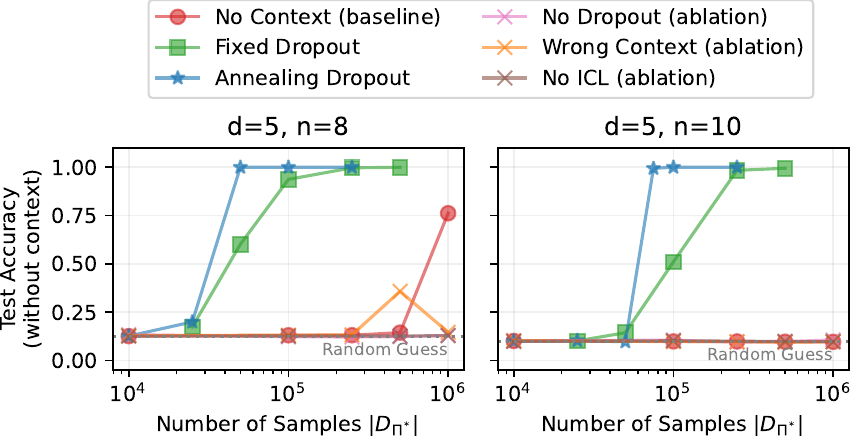}
    \caption{Context-enhanced learning of $\translationtask_\targetCB$ when $n=8$ and $n=10$. 
    \annealdpstrat{} learns the   fastest followed by \fixdpstrat{}, requiring 10x less samples compared to vanilla SFT. Ablations show the necessity of correct context, proper dropout, and sufficient ICL capability as a good initialization. Note that there are multiple ablations overlapping around random guesses.}
    \label{fig:experiments-main-n10}
\end{figure}

\looseness-1\cref{fig:experiments-main-n10} demonstrates the significant sample efficiency of \ICM{}. 
Moreover, models trained with subsets of \strdicts{} and just 20\% dropout give perfect heldout test-time accuracy with 100\% dropout rate. Thus they are able to effectively use \strdict{} rules from subsets of the phrasebook that did not co-occur in the same training sample. Clearly, the model has learned the \strdict{} atomically, and can combine the rules as needed at test time.

In an ablation (\cref{sec:apdx-experiments}), we show that \ICM{} only internalizes the rules whose dropout from  \curriculumtext{} leads to an increase in loss on training data.

The experiment results can be summarized as follows:
\begin{remarkboxmaintext}
\textbf{(i)} Context-enhanced learning from an ICL-capable model greatly improves training sample efficiency.\\
\textbf{(ii)} The \strdict{} rules are internalized atomically, and only when missing them incurs an increased loss.
\end{remarkboxmaintext}

\begin{figure*}[h]
    \centering
    \includegraphics[width=0.95\linewidth]{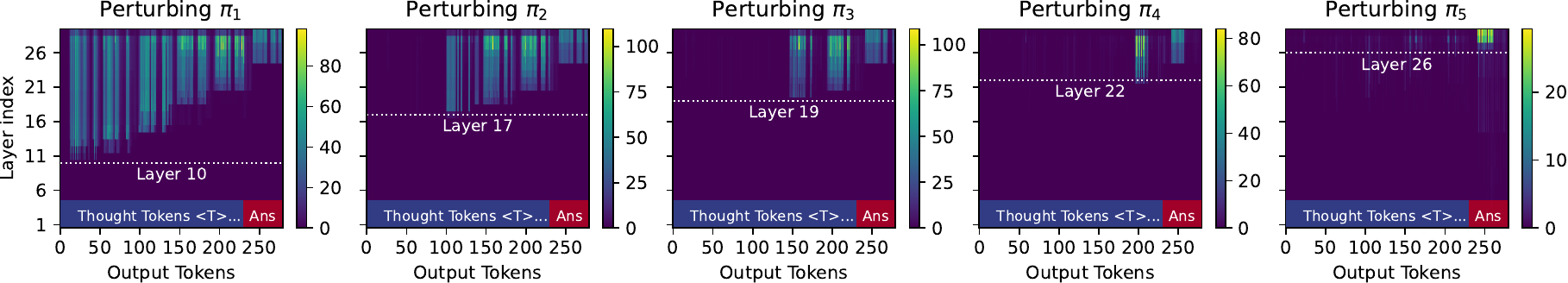}
    \caption{Evidence for sequential processing in a $\translationtask(5, 8)$-ICL-capable model. Each entry is the $l_2$ norm between the latent representation pre and post-perturbing $\dictionary_i$ (after certain layer at certain token position). Perturbing later \strdicts{} in the context changes representations in the later layers, suggesting that later translation steps happens in later layers. To rule out the possibility that the affected depth is only dependent on position of the perturbation instead of the semantic content, we conduct the same set of experiment except that we only perturb rules that are \emph{not} used when translation $\inputseq$, which yields negligible representation difference (see comparison in \cref{fig:mechanistic-icl-additional}).}
    \label{fig:mechanistic-icl}
\end{figure*}

\begin{figure*}[!htbp]
    \centering
    \includegraphics[width=0.97\linewidth]{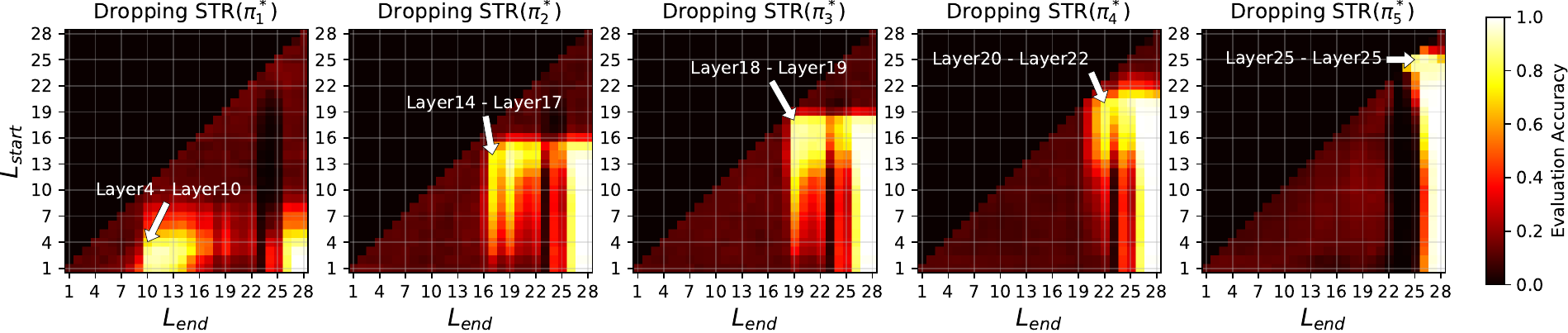}
    \caption{Starting from the $\translationtask(5,8)$-ICL-capable model $\langmodel_\lmparam$ evaluated in \cref{fig:mechanistic-icl}, we fix a set of \strdicts{} $\targetCB$, train with \annealdpstrat{} curriculum on 100k samples, and obtain an $\translationtask_\targetCB$-capable model $\langmodel_{\lmparam^*}$. We construct ``stitched models'' by selectively substituting layers from layer $L_{\text{start}}$ to layer $L_{\text{end}}$ in $\langmodel_\lmparam$ by $\langmodel_{\lmparam^*}$ and evaluate on $\translationtask_\targetCB$ with certain phrasebooks dropped from context. Bright colors close to the diagonal reflects that knowledge required to compensate the dropped phrasebook in context can be localized in a small subset of layers in $\langmodel_{\lmparam^*}$. It is particularly worth noting that for any level $i$, the end of the group of layers responsible for storing $\dictionary^*_i$ matches the start of the layers responsible for reading $\dictionary_i$ (see \cref{fig:mechanistic-icl}) before context-enhanced learning.} 
    \label{fig:mechanistic-localizedStorage}
    \vspace{-0.1in}
\end{figure*}

\subsection{Mechanistic Insights: Layer by Layer Mapping of the Translation Task}\label{sec:mechanistic}

\looseness-1To understand the behavior of \ICM{}, we probe into the hidden representations and weights of models before and after context-enhanced learning. 

\textbf{Sequential Processing in ICL-capable model}

First, we look into how \strdicts{} in \curriculumtext{} are used by an $\translationtask(\depth, \numchar)$-ICL-capable model during the translation task via the following controlled experiment.
Fix a random \strcodebook{} $\codebook=\cbr{\pi_1,\dots,\pi_\depth}$ and an input sequence $\inputseq$, we pass $[\str(\codebook),\intermseq_1,\Tstring, \intermseq_{d+1}]$ into the ICL-capable model, and record the model's hidden representations (output of each transformer block) for the tokens $\Tstring,\intermseq_{\depth+1}$.

Next, we apply independent perturbations to each phrasebook $\pi_i$ in the \curriculumtext{} and measure the change in the model's representations. That is, for each level $1 \leq i \leq \depth$, we randomly sample 10 rules in $\pi_i$ used in the translation of $\inputseq$ and replace their image by other randomly chosen tuples. We denote the perturbed phrasebook $\str(\hat\vardict_i)$, and compute the norm difference between the model's representations at tokens $\Tstring,\intermseq_{\depth+1}$ before and after replacing $\str(\vardict_i)$ with $\str(\hat\vardict_i)$. The first layer showing a significant difference in its output representation is identified as the layer where the model begins processing the \strdict{}.

In \cref{fig:mechanistic-icl}, we show that an ICL-capable model processes \strdicts{} in \curriculumtext{} sequentially, with earlier \strdicts{} read by earlier layers. Formal details are in \cref{sec:hid_rep_model} and additional experiments are in \cref{sec:apdx-experiments-mechanistic}.

\textbf{Localized Storage after Context-Enhanced Learning}

\looseness-1
\looseness-1Here, we verify whether a similar sequential pattern for internalizing \strdicts{} is used in a $\translationtask_\targetCB$-capable model. We denote the ICL-capable model as $\langmodel_{\lmparam}$ and its post \ICM{} counterpart as $\langmodel_{\lmparam^*}$. As $\langmodel_{\lmparam^*}$'s capability no longer depends on textual representation of the \strdicts{}, our analysis focuses on the model's parameters.

\looseness-1We assess the importance of each layer in model $\langmodel_{\lmparam^*}$ for each \strdict{} by constructing ``stitched'' models and measuring their output behavior. For every $1 \leq L_{\text{start}} \leq L_{\text{end}} \leq 28$ (total layers in Llama3.2 3B), we replace layers $L_{\text{start}}$ to $L_{\text{end}}$ of the ICL model $\langmodel_\lmparam$ with corresponding layers from $\langmodel_{\lmparam^*}$ to create a stitched model. This strategy has been used in prior works on localizing information after SFT \citep{gong2022finding,panigrahi2023task,wei2024assessing}.

For each level $1 \le i \le \depth$, we evaluate the stitched models on $\translationtask_\targetCB$ using the following in-context information: $[\str(\vardict^*_1), \cdots, \str(\vardict^*_{i-1}),\str(\vardict^*_{i+1}), \cdots, \str(\vardict^*_{\depth})]$. Note here the textual representation of $\vardict^*_i$ has been dropped. If a stitched model shows high accuracy on inputs that use $\vardict^*_i$ for translation but contain only the above in-context information, then the layers selected from $\langmodel_{\lmparam^*}$ to create the stitched model can be deemed responsible for storing information on $\vardict^*_{i}$ in their parameter space. 

\looseness-1\cref{fig:mechanistic-localizedStorage} demonstrates that information of all \strdicts{} can be localized to a few mutually disjoint layers in $\langmodel_{\lmparam^*}$. Moreover, the end of the group of layers where  $\dictionary^*_i$ is localized in $\langmodel_{\lmparam^*}$ marks the start of the layers that begin processing level $i$ \strdict{} in the ICL-capable model $\langmodel_{\lmparam}$ (e.g. compare the role of layer $17$ in \cref{fig:mechanistic-localizedStorage,fig:mechanistic-icl}). Formal details and additional experiments are in \cref{sec:stich_model,sec:apdx-experiments-mechanistic}.

This suggests that instead of storing \strdicts{} as a single chunk in its parameters, the model re-learns each translation step locally to compensate for missing information when rules are dropped. Thus, we conjecture that \ICM{} leverages \curriculumtext{} to improve training by localizing learning in the parameter space. We build a surrogate model in \cref{sec:theory} to show how such localized learning can prove beneficial for faster training.

\section{Curriculum-text: Detectable from Queries?}\label{sec:memorization}

Context-enhanced learning for MLT uses the \strdicts{} in its curriculum-text. Can rules of the \strdicts{} be recovered post-hoc  querying the model (\hyperlink{q3}{\textbf{Q3}}) using likelihood-based methods for detecting training data?

\looseness-1We take an $\translationtask_\targetCB$-capable model (with $d=5, n=10$) trained with \ICM{} from \cref{sec:experiments-results}. Let $\str\rbr{\targetCB}$ denote the concatenation of all phrasebooks. For each rule in $\str\rbr{\targetCB}$, which are of the form ``\texttt{a b -> C D}''\footnote{$\texttt{a,b,C,D}$ are generic tokens for presentation}, we measure the probability of the model generating the ground truth tokens ``\texttt{C D}'' after observing ``\texttt{a b ->}''. We compute (1) the fraction of cases where the model’s top-1 prediction matches the ground truth (greedy decoding) and (2) the probability of generating ground truth tokens via random sampling (with temperature set as 1). These are standard tests on textual description, and we would expect high probabilities for the correct tokens.

\looseness-1We also test two adversarial strategies with additional {\em token filtering}\footnote{See detailed characterization in \cref{sec:apdx-experiments-memorization}.} in generation: (i) setting probability of \think{} tokens to zero (ii) when querying for a rule in $\dictionary_i$ with output alphabet $\alphabetset_{i+1}$, setting probability of all tokens outside $\alphabetset_{i+1}$ to zero\footnote{We did not explore the full spectrum of attacks (e.g. adversarial prompt engineering) and leave that as interesting future work.}. Note that this corresponds to a  strong adversary which knows the alphabet set of the intermediate \strdicts{}, but not the rules.

\looseness-1As shown in \cref{tab:memorization-query-succ}, recovering rules from intermediate phrasebooks ($\vardict_1, \dots, \vardict_4$) is nearly impossible without token filtering. For the final phrasebook ($\vardict_5$), results remain near-random, where random guess probability is $1\%$ for a 2-tuple when $\numchar=10$. Even with filtering (ii), recovery success rates remain only slightly above random. Setup details and more results are available in \cref{sec:apdx-experiments-memorization}.

\begin{table}[t]
\caption{Recovery Success Rate (Rounded to 2 decimals)}

    \centering
    \scalebox{0.95}{
    \begin{tabular}{l|cc|cc}
\toprule
\multirow{2}{*}{Token Filter} & \multicolumn{2}{c|}{Greedy Decoding} & \multicolumn{2}{c}{Sampling} \\
                              & $\pi^*_1-\pi^*_4$        & $\pi^*_5$       & $\pi^*_1-\pi^*_4$        & $\pi^*_5$       \\ \midrule
None & $0.00\%$ & $0.20\%$ & $0.00\%$ & $0.89\%$ \\
No \texttt{<THINK>} & $0.00\%$ & $0.20\%$ & $0.00\%$ & $0.90\%$ \\
Only from $\alphabetset_{i+1}$ & $1.66\%$ & $0.20\%$ & $1.28\%$ & $0.94\%$ \\ \bottomrule
\end{tabular}}
\label{tab:memorization-query-succ}
\end{table}

\section{Mathematical Analysis}\label{sec:theory}

\looseness-1 Having empirically demonstrated the sample efficiency benefits of \ICM{}, we now formalize them through a mathematical lens. We first define the sample complexity of an algorithm for learning the task. 

\begin{definition}[informal]
    Sample complexity for an algorithm to learn $\translationtask_{\codebook}$ is defined as the minimum total length of all (possibly repeated) input sequences required by the algorithm to return an $\translationtask_{\codebook}$-capable model $\langmodel_{\lmparam}$.
\end{definition}

Analysis of gradient-based learning on a multilayer transformer (let alone a $28$-layer model like Llama 3.2-3B) is an open mathematical question.   We use our mechanistic findings (i.e., translation layers map onto transformer layers) to propose a surrogate model to think about how transformers learn $\translationtask(\depth, \numchar)$. In this surrogate model we demonstrate that learning a task $\translationtask_{\codebook}$ via vanilla SFT will require $\numchar^{\Omega(\depth)}$ samples. Then, we prove that, with \ICM{}, the surrogate model can learn $\translationtask_{\targetCB}$ with a sample complexity of $\mathcal{O}(\mathsf{poly}(n)d\log d)$.

\subsection{Surrogate Model ($\surrogateModel$)}
\label{sec:theory-surrogateModel}
Formalization of the surrogate model, in short $\surrogateModel$, relies on observations from \cref{sec:mechanistic}, which reveal that an ICL-capable model (\cref{def:icl-capable-model}) performs the translation task step-by-step, with earlier \strdict{} processed by lower layers. This aligns with how the model stores internalized knowledge  in layers after \ICM{}. Our $\surrogateModel$ represents an idealized and simplified transformer that has already been ``pre-conditioned'' to solve $\translationtask$ in this sequential fashion. Using $\surrogateModel$, we will show the benefits of \ICM{}.

\looseness-1Without loss of generality we assume the alphabet sets as $\alphabetset_1,\dots,\alphabetset_{d+1} := \alphabetset = \cbr{1, 2, \dots, \numchar}$.  $\surrogateModel$ will represent a length-$L$ sequence $\intermseq_i = \srbr{s_{i,1}, \dots, s_{i,\inputlength}}$ as an embedding matrix $\matrep_i\in \R^{\numchar^2\times \inputlength/2}$ that uses  $\{\vecrep\rbr{s_{i,1}, s_{i,2}}, \cdots, \vecrep\rbr{s_{i,\inputlength-1}, s_{i,\inputlength}}\}$ as columns.
Here, for any $2$-tuple $(a,b)$, $\vecrep(a,b)\in\R^{\numchar^2}$ represents a one-hot vector with $1$ at dimension $a\numchar+b$.
\looseness-1$\surrogateModel$ operates on embedding matrices, transforming $\matrep_1 \to \matrep_2 \to \cdots \to \matrep_{\depth+1}$. Each layer $i$ will be primarily defined by two matrices, $\contextmat_i, \weightmat_i \in \R^{\numchar^2 \times \numchar^2}$. $\contextmat_i$ presents a (possibly partial or completely dropped)  \strdict{} that is provided in-context, and $\weightmat_i$ is a trainable parameter storing \strdict{} information during \ICM{}\footnote{For simplicity of presentation $\contextmat_i$'s are provided directly to the corresponding layers. We note that $\surrogateModel$ can be exactly re-parameterized that $\contextmat_i$'s are provided in context (i.e. concatenated with $\matrep_1$). Please see \cref{defn:apdx-surrogate-surrogateModel-context-augmented} for more discussions.}. 

\begin{defnboxmaintext}
\begin{definition}
\label{eqn:theory-surrogateRep-recurrance}
    $\surrogateModel$ with trainable parameters $\{\weightmat_i\}_{i=1}^{\depth}$ and in-context representation $\{\contextmat_i\}_{i=1}^{\depth}$,  is represented by its operation on an input $\inputseq$ as 
    \begin{align*}
        &\matrep_{\depth+1} = \surrogate_{ \{\weightmat_i\}_{i=1}^{\depth} } \rbr{ \{\contextmat_i\}_{i=1}^{\depth}, \matrep_{1} },  \quad \text{where}\\
        &\matrep_{i+1} = \maxfunc(\contextmat_i + \weightmat_i)\shiftfunc(\matrep_i), \text{ for } i \ge 1,  \nonumber
    \end{align*}
    and $\matrep_1$ is the embedding matrix for the input string $\inputseq$. 
\end{definition}
\end{defnboxmaintext}

Here $\shiftfunc$ represents \circularshift{} operation and is defined as a Hadamard product on the embedding matrices (details in \cref{defn:apdx-surrogate-shift-operator}). $\maxfunc$ represents hard-max function converting $\contextmat_i + \weightmat_i$ to a binary column stochastic matrix. In the following discussion, we show $2$ examples where the surrogate model can perfectly represent an ICL-capable model and $\translationtask_{\targetCB}$-capable model.

\looseness-1\textbf{Case 1 (Representing $\translationtask(\depth, \numchar)$-ICL-capable): } The trainable matrices  are all $0$'s as the model hasn't undergone \ICM{}. To produce the output for a task $\translationtask_{\codebook}$, $\surrogateModel$ takes \strdicts{} into in-context representations by setting each $\contextmat_i$ as a stochastic matrix $\transmat(\vardict_i)$, which represents rules of $\vardict_i$ as one-hot columns (please see \cref{def:apdx-surrogate-translationMatrixRepresentation,lemma:translate-equiv}).

\looseness-1\textbf{Case 2 (Representing $\translationtask_{\targetCB}$-capable): } For a model that has performed \ICM{} on a translation task $\translationtask_{\targetCB}$, no in-context information will be provided to the surrogate model and so $\{\contextmat_i\}_{i=1}^{\depth}$ will be all $0$s. A  $\translationtask_{\targetCB}$-capable-model should contain the \strdicts{} as $\{ \transmat(\vardict^*_i)\}_{i=1}^{\depth}$  in its trainable parameters $\{\weightmat_i\}_{i=1}^{\depth}$.

\textbf{$\surrogateModel$ as an Ideal Transformer for $\translationtask$:} The following theorem constructs a transformer that can simulate $\surrogateModel$. Thus, while our discussions and proofs focus on the surrogate model for simplicity, they remain fully applicable to the transformer architecture.
\begin{theorem}[cf \cref{lem:construction}]
    There exists a transformer that can simulate $\surrogateModel$  with $2\depth$ self-attention and $2\depth$ MLP layers with embedding dimension $2\numchar^2 + 2\depth + 4$. 
\end{theorem}

\subsection{Sample Complexity for Vanilla SFT} \label{sec:theory-hard}
\looseness-1 For the surrogate model, vanilla SFT corresponds to always setting in-context representations $\{\contextmat_i\}_{i=1}^{\depth}$ to $0$s when training for $\translationtask_{\targetCB}$. While we use the surrogate model for consistency in our discussion, the argument generalizes to any model learning $\translationtask_{\targetCB}$ with vanilla SFT.

\looseness-1Our analysis is built on the Statistical Query (SQ) framework \citep{kearns1998efficient}, which measures the difficulty of learning tasks using algorithms that rely on expectation estimates of specific functions to approximate the true solution. Gradient-based methods, such as Stochastic Gradient Descent (SGD), fall under this framework as they compute gradients by estimating expectations of loss functions and their derivatives.

The complexity of learning a task is quantified by the SQ dimension, which measures the number of candidate functions that are pairwise uncorrelated under the input distribution and difficult to distinguish with limited samples. A higher SQ dimension implies a richer hypothesis class, that requires more samples to identify the correct function. We show in the following theorem that the SQ dimension of $\translationtask(\depth, \numchar)$ grows exponentially with task parameters.

\begin{theorem}
    \label{thm:theory-sqdim}
    \looseness-1
    SQ dimension of $\translationtask(\depth, \numchar)$ under uniform input distribution is at least $\numchar^{\Omega(\depth)}$. 
\end{theorem}

\begin{remarkboxmaintext}
Informally this implies that any algorithm that tries to learn a $\translationtask_{\targetCB}$-capable $\surrogateModel$ with trainable parameters $\{\weightmat_i\}_{i=1}^{\depth}$, and in-context information $\{\contextmat_i\}_{i=1}^{\depth}$ always fixed at $\zeromat$s,  by minimizing loss across samples will require at least $\numchar^{\Omega(\depth)}$ sample complexity.
\end{remarkboxmaintext}

\looseness-1A corollary is that for vanilla SFT with SGD, sample complexity to learn $\translationtask_{\targetCB}$ can be at least $\numchar^{\Omega(\depth)}$. This is adapted from \citet{edelman2023pareto}, who analyse for sparse parity that has similar SQ dimension (\cref{cor:sgd}).

\textbf{Informal proof for SQ dimension:} Our proof extensively analyses the case where number of characters is $2$ (lem. \ref{lem:sqn2}), on which we will build proof for general $\numchar$ (lem. \ref{lem:sqn}).
The proof for $\numchar=2$ leverages uncorrelations between two randomly selected translation tasks; i.e. we show that for two random \strcodebook{} $\codebook^{\alpha}, \codebook^{\beta}$, the translation tasks $\translationtask_{\codebook^{\alpha}}, \translationtask_{\codebook^{\beta}}$ will have $0$ output correlation with probability at least $1-2^{-\Omega(\depth)}$ w.r.t. random choice of $\codebook^{\alpha}, \codebook^{\beta}$ (\cref{lem:corr_random_map_n2}). We then show that we can pick exponentially many such random \strcodebook{} for which the translation tasks will be pairwise uncorrelated. This translates to a high SQ dimension.

\subsection{Sample Complexity of Context-Enhanced Learning} \label{sec:theory-converge}

Here, we show that \ICM{} substantially improves the sample complexity of learning in $\surrogateModel$.
\looseness-1An $\translationtask(\depth, \numchar)$-ICL-capable model gets a \strcodebook{} $\targetCB$ by setting $\{\transmat(\vardict^*_i)\}_{i=1}^{\depth}$ as in-context representations $\{\contextmat_i\}_{i=1}^{\depth}$. When a curriculum is followed such that a translation rule is dropped from a \strdict{}, say $\vardict^{*}_i$, $\surrogateModel$ will set the corresponding column in its in-context representation $\contextmat_i$ as $\zeromat$'s.
Denote the zero-ed out column by $\contextmat_i^{(j)}$ for some generic column index $j$, we note that the loss will be low if and only if the corresponding column in the learnable parameters $\weightmat_i^{(j)}$ exactly matches $\transmat(\vardict^*_i)^{(j)}$.

A heuristic search algorithm, that searches among $\numchar^2$ possibilities for the dropped rule and stores its one-hot representation in the corresponding column in $\weightmat_i$, can be used to minimize the loss. By sequentially dropping rules, followed by a search and store process, the algorithm achieves polynomial sample complexity.
\begin{theorem}[Informal; cf \cref{cor:apdx-surrogate-search-learning-random}]

    For any task $\emph{\translationtask}_\targetCB$, there is a heuristic search algorithm paired with a curriculum of iteratively dropping rules from \strdicts{},
    that can learn a $\emph{\translationtask}_\targetCB$-capable $\surrogateModel$ with sample complexity $\mathcal{O}(\numchar^6 \depth \log \depth)$ with high probability.
\end{theorem}
The enumerative step in the heuristic search algorithm requires $\Theta(\numchar^2)$ steps on average, as the algorithm needs to search over $\Theta(\numchar^2)$ possibilities when a rule is dropped from a \strdict{}. Instead, we show that gradient descent requires only a few steps per dropped rule. Dropping a rule sets the respective column in the in-context representation to $0$, causing the gradient for the corresponding column in the trainable parameters to strongly align with the missing column.
We formally present results for $\depth=2$; due to exponentially growing number of terms to analyse with higher $\depth$, we 
keep the result for general $\depth$ as a conjecture.

\begin{theorem}[Informal; cf \cref{thm:apdx-surrogate-GD-learning-random}] 
\label{thm:surrogate-GD-learning-coverable}
When $d=2$, there is a gradient descent based algorithm, paired with a curriculum of iteratively dropping a random rule from \strdicts{},  that can return $\emph{\translationtask}_\targetCB$-capable $\surrogateModel$ with sample complexity $\mathcal{O}(\numchar^4)$ with high probability.
\end{theorem}

While theoretical analysis of GD dynamics beyond $d=2$ is challenging, empirically we show that when training with SGD, the trainable parameters for much deeper \surrogateModel{} can quickly learn the set of \strdicts{} in $\targetCB$.
In \cref{fig:surrogate-multilayer-GD-dynamics}, we show \ICM{} results for a \surrogateModel{} on a randomly selected set of \strdicts{} $\targetCB$ in $\translationtask(10, 10)$. Regardless of whether we learn each layer separately with layer-wise SGD or all layers simultaneously with SGD in the surrogate model, the trainable parameters learn the \strdicts{} in $\targetCB$ very quickly. More discussions are deferred to \cref{sec:apdx-experiments-linear-surrogate}.

\begin{figure}[H]
    \centering
    \includegraphics[width=1\linewidth]{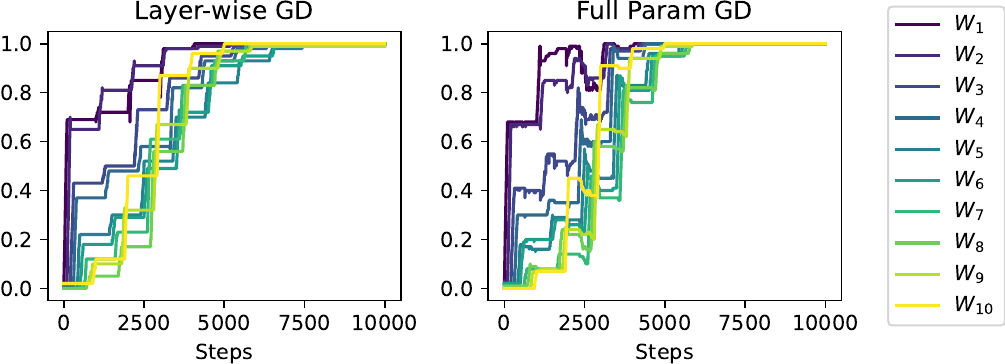}
    \vspace{-0.2in}
    \caption{Percentage of columns in $\maxfunc(\W_i)$ matching $\transmat(\vardict^*_i)$ when optimizing a \surrogateModel{} for $\translationtask(10,10)$. Learnable parameters of all layers quickly converge to the correct values, regardless of whether layer-wise updates are imposed.}
    \label{fig:surrogate-multilayer-GD-dynamics}
\end{figure}
\vspace{-0.2in}

\subsection{Insight from $\surrogateModel$: Gradient Quality}
\label{sec:theory-snr}

In this section, we show the major difference between \ICM{} and vanilla SFT to be the amount of predictive information in gradients for the trainable parameters. Our theoretical analysis in \cref{thm:surrogate-GD-learning-coverable} on $\surrogateModel$ primarily shows that when a single translation rule in a \strdict{} $\vardict^*_i$ is dropped, equivalently a column in the in-context representation, say $\contextmat_i^{(j)}$, is zero-ed out, the gradient of the corresponding column in trainable parameter $\weightmat_i^{(j)}$ aligns strongly with the one-hot vector representation of the dropped rule, $\transmat(\vardict^*_i)^{(j)}$.
However, this strong alignment heavily relies on the presence of other rules in-context. When more rules are dropped, the gradient signal gets increasingly ``inaccurate''. We quantitatively characterize such degradation as follows:

Consider an ICL-capable $\surrogateModel$, we focus on a column (with index $j$, denoted as superscript) $\weightmat_1^{(j)}$ in the first layer with ground truth $\transmat(\vardict^*_1)^{(j)}$. We define  \textbf{gradient prediction accuracy} (see formal definition in \cref{defn:apdx-grad-acc}) for $\weightmat_1^{(j)}$ as the probability that the argmax entry of the negative stochastic batch gradient on $\weightmat_1^{(j)}$ matches the argmax entry of $\transmat(\vardict^*_1)^{(j)}$. Intuitively, a higher gradient prediction accuracy will make learning $\transmat(\vardict^*_1)$ easier. We track the accuracy metric when only $\contextmat_i^{(j)}$ is zero-ed out as well as when taking more aggressive context dropout schemes.

\begin{figure}[!htbp]
    \centering
    \includegraphics[width=\linewidth]{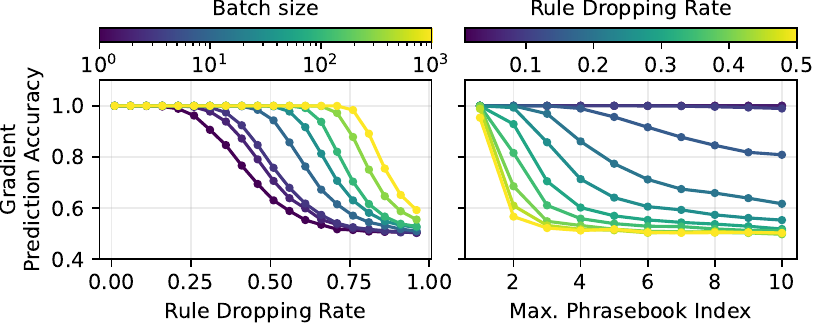}
    \vspace{-0.2in}
    \caption{
    Gradient prediction accuracy  when (left) varying dropping rates for rules in $\vardict_1$ and batch sizes of gradient computation, and (right) varying dropping rates and max \strdict{} index ($k$) for rules dropped from $\vardict_1, \dots, \vardict_k$. Higher dropping rates always lead to lower gradient prediction accuracy.
    \vspace{-0.1in}
    }
    \label{fig:grad-accuracy}
\end{figure}

In \cref{fig:grad-accuracy}, we report the gradient prediction accuracy with dropout schemes involving (1) multiple rules dropped out from the first \strdict{} and (2) rules dropped from multiple \strdicts{}. We can see that higher dropping rates significantly degrade the gradient prediction accuracy, highlighting the necessity of proper in-context information in order to obtain the optimization benefit. More details on the metric and experiments are deferred to \cref{apdx:gpa}.
We note that finding above is based on observations on the simplified surrogate model. A quantitative characterization of the optimization benefit for context-enhanced learning in the LLM regime is left as an interesting future work.

\section{Related Works}

In this section, we discuss related works to this paper. More related works regarding (1) compositional and OOD generalization and (2) mechanistic understanding of transformers are discussed more extensively in \cref{sec:addition_work}.

\subsection*{Learning using Privileged Information (LUPI)} LUPI was formally introduced by \citet{vapnik2009new} in kernel SVMs, and it is heavily related to concepts of Learning with Side Information \cite{LSIkuusela,jonschkowski2015patterns,zhang2018joint}.
Primarily, both concepts refer to training a model with additional information that could help training but may not be available at test time. 
This framework has been extended theoretically for classification tasks (e.g. \citet{pechyony2010theory,momeni2018understanding}) and has been used to explain benefits of knowledge distillation \citep{vapnik2015learning,lopez2015unifying}. To name a few applications, this concept has been heavily studied for improving boosting for classification tasks \citep{chen2012boosting}, visual and video encoders \cite{hoffman2016learning,cheng2020weakly,xu2017end}, human preference predictions \cite{farias2019learning}, multi-agent games \citep{sessa2020contextual}, speech recognition \citep{acousticLSI}, and medical recognition \citep{ceccarelli2008improving,sabeti2020learning}.

While LUPI has rich application in classification, extending LUPI to LLMs introduces unique challenges due to their auto-regressive training. Such a concept raises questions on whether applying auto-regressive loss on the additional information is necessary to get the benefits of additional supervision information in context, and how the additional information changes the training behavior of LLMs.
Hence, our work is a nontrivial generalization of this framework to LLMs that connects to their in-context learning strengths.

\looseness-1\subsection*{In-context learning and memorization} Recent works have studied
emergence of ICL and competition with in-weights learning (IWL) (equivalent to knowledge memorization) during pre-training of language models
\citep{chan2022data,reddymechanistic,singh2024transient,singh2024needs,nguyen2024differential}. These works show that data distribution properties affect the behavior of the model during training. Our work can be thought of as contemporary to the works above, where we show that strong ICL capabilities can be utilized for improving knowledge on a task by \ICM{}.  

\textbf{Benefits of in-context learning:} In-context learning has been primarily studied in the context of few-shot prompting of large language models. Better supervision with in-context supervision can help in improved performance (e.g. some representative works \citep{arora2022ask,si2022prompting,wu2022self,lu2021fantastically,su2022selective}), OOD generalization and factuality (reduced hallucination) \citep{yang2023supervised,dhuliawala2023chain,chen2023skills,didolkar2024metacognitive}, and more structured latent representations \citep{park2024iclr} for large language models. On the other hand, we show that improved supervision with in-context supervision can also help a model learn faster in SFT, while seemingly not leaking the in-context information in its output probabilities.

\section{Discussion, Limitations, and Future work}
\label{sec:conclusion}

\looseness-1Some experimental works have implicitly used the notion of \ICM{} but the current paper formalized this notion  for auto-regressive models and showed, using $\translationtask$, that this form of learning can be exponentially more sample-efficient than standard SFT.  At the end of training it is hard to recover the in-context information seen during training from the model's output probabilities. 
We note that this finding appears to have implications about copyright law (e.g., whether or not LLM training amounts to ``transformative use'' of text \cite{carlini2021extracting,carlini2022quantifying,karamolegkou2023copyright}) whose further study is left for future work.

\looseness-1Our experiments focus on a synthetic $\translationtask$ task for a few reasons: (1) to ensure that the task is absent from LLM pre-training, which allows precise quantification of benefits of \ICM{}, including  not revealing the curriculum text at inference time. (2) the task is too difficult  (at least for Llama 3.2 3B model) to learn via vanilla SFT, but is learnable via \ICM{}. Extending these findings to real-world complicated tasks (e.g., in math and coding) is left for future work.

\looseness-1Our convergence analysis for \ICM{} relies on a surrogate model, and extending it to an actual transformer remains an open challenge for theory of deep learning. Extending formalization of \ICM{} to explore LLM training in multi-agent settings--where models collaborate and learn from each other to discover novel concepts--would be an exciting avenue for future research.

\section*{Acknowledgment}
We thank Yun Cheng, Simon Park, Tianyu Gao, Yihe Dong, Zixuan Wang, Haoyu Zhao, and Bingbin
Liu for discussions, suggestions, and proof-reading at various stages of the paper. AP, SA acknowledge funding from NSF, PLI, DARPA, ONR, and OpenAI. XZ is additionally supported by a Gordon Y.S. Wu Fellowship in Engineering. 

\section*{Impact statement}
We formulate a basic notion ``\ICM{}'', and study how it is different from usual learning. One consequence of our study is that enhancing the context with good quality data can enhance the quality of the learner.

This enhancement could be used with privileged information and it appears that the training can be done in such a way that the model does not leak this privileged information. This can be seen as enhancing privacy, since there is a lower chance of leaking privileged training data.

The flip side is that it suggests --- albeit in very toy and synthetic setting, with full study left for future work --- that  model training could use off-limits data (albeit with no gradient updates on it) and this use might not be  detectable from querying  the trained model. But this is hypothetical at this point since the paper concerns a very toy setting.

\bibliography{reference}
\bibliographystyle{arxiv}

\newpage
\onecolumn

\appendix
\addcontentsline{toc}{section}{Appendix} 
\part{Appendix} 
\parttoc 

\section{Overview of the Appendix and Common Notations}
Here, we outline the structure for the appendix for easier readability. Due to space constraints, we had to defer a lot of details from the main paper. 
\cref{apdx:gpa} contains details on the experiments on gradient prediction accuracy from \cref{sec:theory-snr}. \cref{sec:apdx-expconfig-stage1} shows the details on CoT internalization pipeline that we followed to get an $\translationtask(\depth, \numchar)$-ICL-capable model. \cref{sec:apdx-expconfig} gives more details on the \ICM{} experiments conducted in \cref{sec:experiments}.

\cref{sec:apdx-experiments} shows mechanistic experiments, designed in \cref{sec:mechanistic}, for additional experimental settings. \cref{sec:apdx-experiments-memorization} present additional details and results for information recovery through querying experiments conducted in \cref{sec:memorization}.
\cref{sec:addition_work} presents additional relevant related works, including discussion on OOD generalization of language models and mechanistic understanding of large transformer architectures.
\cref{sec:properties_mlt} discusses few properties on the $\translationtask$. \cref{sec:apdx-lowerBound} presents the formal theorems on the computational hardness of learning the translation task and their proofs, which were informally outlined in \cref{sec:theory-hard} in the main paper. \cref{sec:apdx-surrogateModel} then presents the formal statements and proofs for \ICM{} in the surrogate model, which were informally outlined in \cref{sec:theory-surrogateModel} in the main paper. Finally, \cref{app:construction_transformer} presents the theoretical construction of an ideal transformer that can simulate the surrogate model. We present extensive details on prompts, and examples of training sequences, that we used for \ICM{} in \cref{sec:training_data_setups}.

\begin{table}[H]
\caption{Important notations}
    \label{tab:notation_table_apendix}
    \centering
    \begin{tabular}{l|l|l}
    \toprule
    \textbf{Scope of Notation} & \textbf{Symbol} & \textbf{Description}\\
    \midrule

    \multirow{4}{*}{General Notations} & $A$ & A set \\
    & $\mA$ & A matrix \\
    & $\mA^{(j)}$ & The $j$-th column of matrix $\mA$\\
    & $\mathbf{e}_k$ & One-hot embedding vector with 1 at dimension $k$ \\
    \midrule    
    \multirow{6}{*}{Language Modeling} & $X$ & All possible text strings (input space for a causal LM) \\
    & $\cY$ & All possible distribution over texts (output space for a causal LM) \\
    & $\langmodel_\lmparam$ & Causal LM parameterized by $\theta$ \\
    & $g$ & Language task mapping inputs $x \in X_g \subset X$ to a distribution $Y \in \mathcal{Y}$ \\ 
    & $\ell_{\text{auto}}$ & Auto-regressive cross entropy loss \\
    & $\maskcurr_\task(x, t)$ & In-context curriculum for learning task $\task$ with input $x$ at step $t$ \\
    
    \midrule

       \multirow{10}{*}{MLT Translation Task} & $\depth$  & Depth of translation task \\
       & $\numchar$  & Number of characters in each alphabet \\
       & $\alphabetset$ & An alphabet set \\
       & $\intermseq$ & A sequence in alphabet $\alphabetset$\\
       & $\dictionary$ & A \strdict{} between 2-tuples in two alphabets, e.g. $(\dictionary_1:\alphabetset_1^2\to \alphabetset_2^2)$.\\
       & $\mathcal{B}^{\dictionary}$ & Space of all possible \strdict{} on $\alphabetset_1^2\to \alphabetset_2^2$\\
       & $\codebook$ & A set of \strdicts{} $\{\dictionary_i\}$ defining a translation task \\
       & $\translationtask_\codebook$ & Translation task with a \strcodebook{} $\codebook$\\
       & $\translationtask(\depth, \numchar)$ & Family of translation tasks of depth $\depth$ and $\numchar$ characters \\
       & $\str\rbr{\dictionary}$ & Descriptive text for a phrasebook $\dictionary$ (see \cref{sec:setup-translationTask}) \\

    \midrule
    \multirow{7}{*}{Surrogate Model} 
    & $\matrep_i$ & Embedding matrix representing sequence $\intermseq_i$ (\cref{defn:apdx-seq-matrep}) \\
    & $\contextmat_i$ & In-context information matrix for level $i$ \\
    & $\weightmat_i$ & Trainable parameter matrix for level $i$ \\
    & $\P_i$ & Effective translation matrix for level $i$ (\cref{defn:apdx-surrogate-effectivetransmat}) \\
    & $\maxfunc$ & Column-wise hard-max function \\
    & $\transmat(\pi)$ & Matrix representation of a phrasebook $\pi$ (\cref{def:apdx-surrogate-translationMatrixRepresentation}) \\
    & $\surrogate_{\{\weightmat_i\}_{i=1}^d}$ & Surrogate model parameterized by $\weightmat_i$'s (\cref{defn:apdx-surrogate-surrogateModel})\\

    \bottomrule
    \end{tabular}
\end{table}

\newpage

\section{Deferred definitions, and experimental details from the main paper}
\label{apdx:deferred}

\subsection{Gradient Prediction Accuracy}\label{apdx:gpa}

Our theoretical analysis in \cref{thm:surrogate-GD-learning-coverable} was built on the fact that when a translation rule in a \strdict{} is dropped by zeroing out a column in an in-context representation $\contextmat_i$, the gradient for the corresponding column in $\weightmat_i$ points to the direction of the dropped rule. 

\looseness-1On the other hand, we show that when multiple rules are simultaneously dropped from the \strdicts{}, the gradients for the trainable parameters become increasingly noisy. To quantify this degradation, we compute gradient for each column of the trainable parameters for an ICL-capable model, when the corresponding rule is dropped from \strdicts{} by zeroing out the relevant column in the in-context representations. We then compute whether the computed gradient points to the right rule. By progressively increasing the number of simultaneously dropped rules, we measure the resulting degradation in the accuracy of the gradient's predictions.

More formally, denote $\masking$ as an operation that takes in column dropping rates per layer $p_1, \cdots, p_{\depth}$, \strcodebook{} $\targetCB$, and returns in-context representations $\{\contextmat_i\}_{i=1}^{\depth}$, such that  $\contextmat_{i}^{(j)} = \transmat\rbr{{\vardict^*_i}}^{(j)}$ ($j$th columns of $\contextmat_{i}$ and $\transmat\rbr{{\vardict^*_i}}$ are equal) with probability $1-p_i$ and $0$ otherwise. Then,

\begin{defnbox}
\begin{definition}
\label{defn:apdx-grad-acc}
    For a set of column dropping rates per layer $p_1, \cdots, p_{\depth}$, \strcodebook{} $\targetCB$ and $\translationtask(\depth, \numchar)$-ICL-capable model, predictive accuracy of gradients is defined as 
    \begin{align*}
        &\mathbb{E}_{ \{\contextmat_i\}_{i=1}^{\depth} = \masking (p_1, \cdots, p_{\depth}, \targetCB) } \\&  \mathbb{E}_{j \in [1, \numchar^2] \mid \contextmat_1^{(j)} = 0} \mathbb{I} \left[ \maxfunc \rbr{-\nabla_{ \weightmat_1^{(j)} }  \mathcal{L} } = \rbr{\transmat\rbr{\vardict^*_i}}^{(j)} \right] 
        \\& \mathcal{L} = \mathbb{E}_{\inputseq} \loss \rbr{ \surrogate_{ \{\weightmat_i\}_{i=1}^{\depth} } \rbr{ \{\contextmat_i\}_{i=1}^{\depth}, \matrep_1 }, \translationtask_{\targetCB} (\inputseq) }, 
    \end{align*}    
    where $\loss$, adapted from \cref{sec:setup-contextAidedLearning}, computes cross-entropy loss on the predicted output embeddings of surrogate model using true output string and $\mathbb{I}$ denotes the indicator function. 
\end{definition}
\end{defnbox}

The above definition computes gradients on the expected loss of the model. We primarily focus on predicting the trainable parameters of the first layer, i.e. $\weightmat_1$, as that is the deepest layer in the surrogate model and intuitively should suffer the most with noise accumulation from dropped rules.
On the other hand, we can further adapt the definition to compute the accuracy for batched gradients, where the gradients are computed using average loss on a randomly sampled batch of input sequences.

In \cref{fig:grad-accuracy}, we report the predictive accuracy of gradient for an $\translationtask(\depth, \numchar)$-ICL-capable model, and its behavior with varying batch size and the column dropping rates. We report for two cases, one where column dropping rates is non-zero only for the first \strdict{}, and one where we increase the number of \strdicts{} for which rules are independently and uniformly dropped. In both cases,

\begin{itemize}
    \item Increased column dropping rates leads to noisier gradients and reduced prediction accuracy.
    \item Larger batch sizes improve gradient accuracy but cannot fully compensate for high dropout rates.
    \item Dropping rules from multiple phrasebooks significantly degrades gradient prediction accuracy.
\end{itemize}

\newpage\subsection{Hidden Representations of a Model} \label{sec:hid_rep_model}

A transformer $\langmodel_{\lmparam}$ with embedding dimension $p$ and $K$ layers takes any input sequence $\vx$, say of length $\inputlength$, converts to an embedding matrix $\embed_1 \in \mathbb{R}^{\inputlength \times p}$, and modifies the embeddings using a succession of $K$ transformer layers; which we will denote by $\langmodel_\lmparam^{(1)}, \langmodel_\lmparam^{(2)}, \cdots, \langmodel_\lmparam^{(K)}$. 
We refer to the hidden representations for the input $\vx$, with embedding matrix $\embed_1$, as the output of the model after every layer. We will denote them as $\embed_{i+1} \in \mathbb{R}^{\inputlength \times p}$ for the output of layer $\langmodel_\lmparam^{(i)}$. That is,
\begin{align*}
    \embed_{i+1} = \langmodel_\lmparam^{(i)} \circ \cdots \circ \langmodel_\lmparam^{(2)} \circ \langmodel_\lmparam^{(1)} \rbr{\embed_1}, \quad \text{ for all } i \ge 1.
\end{align*}

\paragraph{$\ell_2$-norm in change in hidden representation with perturbation in in-context information} 

For an $\translationtask(\depth, \numchar)$-ICL-capable model, we supply in-context information for the textual description of a \strcodebook{} $\codebook$ as $\str(\codebook) = [\str(\vardict_1), \cdots, \str(\vardict_{\depth})]$. Our inputs to the transformer for an input string $\inputseq$ will be of the form $[\str\rbr{\codebook}, \inputseq, \Tstring, \intermseq_{\depth+1}]$, where $\intermseq_{\depth+1} = \translationtask_{\codebook}(\inputseq)$. By the definition of hidden representations, $\embed_2, \cdots, \embed_{K+1}$ will denote the output of the transformer layers for this input string. However, we will be only interested in the hidden representations for the tokens involved in the tokens for $\Tstring, \intermseq_{\depth+1}$; and we will refer to the corresponding subsets of $\embed_2, \cdots, \embed_{K+1}$ that represent these specific tokens as  $\matrep_2, \cdots, \matrep_{K+1}$.

Now, suppose we randomly take a \strdict{} $\vardict_i$ in $\codebook$ and change to a random \strdict{} $\Tilde{\vardict}_i$. The corresponding  textual description that will augment the context for an input string will then be $[\str(\vardict_1), \cdots, \str(\vardict_{i-1}), \str(\Tilde{\vardict}_i),  \str(\vardict_{i+1}), \cdots, \str(\vardict_{\depth})].$ If $\Tilde{\matrep}_2, \cdots, \Tilde{\matrep}_{K+1}$ now denote the hidden representations that represent the tokens for $\Tstring, \intermseq_{\depth+1}$, then the $\ell_2$-norm in the change of the hidden representation after layer $j$ (for any $1 \le j \le K$) with the perturbation in $\codebook$ will be given by $\norm{ \Tilde{\matrep}_{j} - \matrep_j }_2.$

\subsection{Definition of a ``stitched'' model} \label{sec:stich_model}

We reuse notations from \cref{sec:hid_rep_model}. Suppose we have an $\translationtask(\depth, \numchar)$-ICL-capable model $\langmodel_{\lmparam}$ and an $\translationtask_\targetCB$-capable model $\langmodel_{\lmparam^*}$. Their corresponding transformer layers are denoted by $\langmodel_\lmparam^{(1)}, \langmodel_\lmparam^{(2)}, \cdots, \langmodel_\lmparam^{(K)}$ and $\langmodel_{\lmparam^*}^{(1)}, \langmodel_{\lmparam^*}^{(2)}, \cdots, \langmodel_{\lmparam^*}^{(K)}$. Each model takes in an input sequence and processes them with their $K$ transformer layers.

Formally, we will write for the ICL-capable model. It takes in input sequence $\vx$, and converts to an embedding matrix, say $\embed_1$, and the output after the $K$layers are given by:
\begin{align*}
     \embed_{K+1} = \langmodel_\lmparam^{(\depth)} \circ \cdots \circ \langmodel_\lmparam^{(2)} \circ \langmodel_\lmparam^{(1)} \rbr{\embed_1}.
\end{align*}

\paragraph{Process of ``stitching'':} The process of stitching takes in two parameters $L_{\text{start}}$ and $L_{\text{end}}$ and replaces all layers from $L_{\text{start}}$ to $L_{\text{end}}$ in $\langmodel_{\lmparam}$ with the corresponding layers in $\langmodel_{\lmparam^*}$ to give a ``stitched'' model, say $f_{\lmparam, \lmparam^*, L_{\text{start}}, L_{\text{end}}}$. The output of the ``stitched'' model $f_{\lmparam, \lmparam^*, L_{\text{start}}, L_{\text{end}}}$ on an input sequence $\vx$ will be given by
\begin{align*}
    \embed_{K+1} = \langmodel_\lmparam^{(\depth)} \circ \cdots \circ \langmodel_{\lmparam}^{(L_{\text{end}+1})} \circ \underbrace{\langmodel_{\lmparam^*}^{(L_{\text{end}})} \circ \cdots \langmodel_{\lmparam^*}^{(L_{\text{start}})}}_{\text{Layers are replaced by layers from } \langmodel_{\lmparam^*} } \circ  \langmodel_\lmparam^{(L_{\text{start}}-1)} \circ  \langmodel_\lmparam^{(2)} \circ \langmodel_\lmparam^{(1)} \rbr{\embed_1}.
\end{align*}

\subsection{Pipeline on CoT Internalization}
\label{sec:apdx-expconfig-stage1}

We randomly sample $M$ sets of \strdicts{} $\pi_1,\dots, \pi_M$ not equal to $\targetCB$. For each set of \strdicts{} $\pi_i$, we randomly sample a single input sequence $\inputseq$ and compute all the intermediate translation steps $\intermseq_2,\dots,\intermseq_{d+1}$. We first train the model to do robust explicit CoT by auto-regressive training on sequences $[\str(\pi_i),\intermseq_1,\intermseq_2,\dots, \intermseq_{d},\intermseq_{d+1}]$ with loss computed over $\intermseq_2,\dots,\intermseq_{d+1}$.
Then we follow common CoT internalization strategies \citep{deng2024internalizeCoT,hao2024coconut,yu2024distilling,su2024dualformer} and gradually replace the intermediate sequences by \think{} tokens in training. 
After all intermediate sequences have been replaced, the model has low loss on $\intermseq_{d+1}$ with input $[\str(\codebook_1),\intermseq_1,\think,\dots, \think,\intermseq_{d+1}]$, satisfying \cref{def:icl-capable-model}. Since we only sample on sequence per set of \strdicts{}, there is little memorization on particular \strdicts{}.

\textbf{Details on training hyperparameters:} We use $M=3 \times 10^5$ random sets of \strdicts{}
with length between $20$ and $40$, for getting $\translationtask(5, 8)$-ICL-capable model, and $M=10^6$ random sets of \strdicts{} with length between $20$ and $40$, for getting $\translationtask(5, 10)$-ICL-capable model. We use cosine learning rate schedule \citep{loshchilov2016sgdr}, with peak learning rate $10^{-4}$ and a $6\%$ warmup phase, where learning rate is linearly increased from $0$ to the peak. We use AdamW optimizer \citep{loshchilov2018decoupled} with weight decay fixed at $10^{-4}$. We use a batch size of $64$ for training.

\textbf{CoT internalization curriculum:} For the first $10\%$ fraction of training, we train the model with explicit CoT tokens that contain the intermediate steps in translation. Then between $10\%$ to $60\%$ fractions of training, CoT tokens are gradually replaced by $\think$ tokens, with the rate of replacement increasing linearly from $0\%$ to $100\%$. We follow a deterministic first-to-last order for replacing CoT tokens; earlier CoT tokens are replaced first with $\think$ tokens.
After that, the model is trained with the $\think$ CoT tokens till the end of training.

\subsection{Experiment configuration for \ICM{}}
\label{sec:apdx-expconfig}

For \ICM{}, we create supervised datasets $\dataset_{\targetCB}$ of different sizes; each containing between $10^4$ to $10^6$ samples. When performing \annealdpstrat{} or \fixdpstrat{}, at each step of training, we randomly perform dropout on the rules of all \strdicts{} or apply dropout to the rules of a randomly sampled \strdict{} to define \curriculumtext{}. That is, if $\vardict^*_1, \dots, \vardict^*_5$ represent the \strdicts{}, then at each step of training, we either randomly drop rules uniformly from all of $\vardict^*_1, \dots, \vardict^*_5$, or just drop from one of the  \strdicts{} randomly selected from $\vardict^*_1, \dots, \vardict^*_5$, while keeping the rules of all other \strdicts{} intact, to create \curriculumtext{}.

\textbf{Hyperparameters:} Training hyperparameters are set equal to the optimization hyperparameters used in preparation of ICL-capable training phase (\cref{sec:apdx-expconfig-stage1}), except we set weight decay to $0$ in all experiments. We report the performance of the trained model after single epoch of training on each $\dataset_{\targetCB}$ and plot against the size of the dataset in \cref{fig:experiments-main-n10}.





\newpage

\section{Additional Experiments Results}
\label{sec:apdx-experiments}

\subsection{Selective Internalization of Context}
\label{sec:apdx-experiments-selectiveMemorization}

In this experiment, we test whether the model can internalize rules that don't incur an increase in loss when dropped during training. To do so, we ablate on \annealdpstrat{}. We select a \strdict{} and create 2 splits of rules in the \strdict{}; one set of rules will be utilized by training samples for performing the translation task (which we call the training split), while other set of rules will appear in \curriculumtext{} during training but never utilized for the translation task (which we call the heldout split). 

At test time, we measure the performance of the trained model on 2 sets of evaluation examples, one that only use rules from the training split for their translation (equal to training distribution), and other that uses rules only from the heldout split for their translation (different from training distribution). The model is being measured without any \strdicts{} information at evaluation. We conduct the above experiment for each \strdict{}, i.e. we create $5$ sets of experiments where we only create heldout split for one specific level of \strdict{}. In \cref{fig:apdx-experiments-selectiveMemorization}, we show that the model fails to perform any translation that use the rules from the held-out split in all settings. This shows that the model only internalizes those rules that are important for the translation task for the training samples, and which incurs an increase in training loss when dropped.

\begin{figure}[H]
    \vspace{-0.1in}
    \centering
    \includegraphics[width=0.9\linewidth]{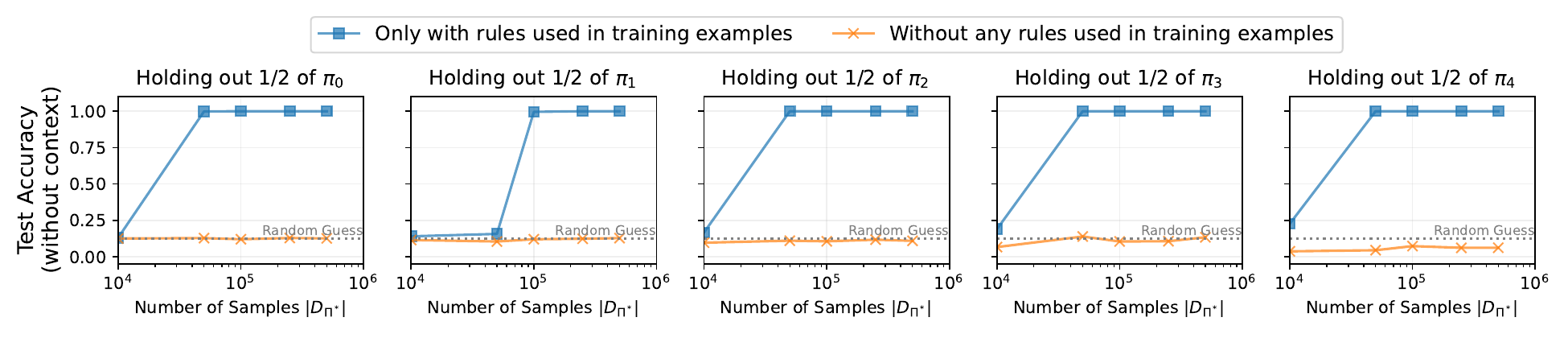}
    \vspace{-0.1in}
    \caption{Ablations on \annealdpstrat{} for $\translationtask(5,10)$ assess whether models internalize rules not used in training samples. (Left to right) $1 \leq i \leq 5$: Five independent experiments where a held-out split is created for \strdict{} $\vardict^*_i$, and training excludes samples that use the translation rules in the heldout split of $\vardict^*_i$. Evaluation is conducted on two sets: one where samples do not use held-out rules from $\vardict^*_i$ and one where only held-out rules from $\vardict^*_i$ are used. The model’s random performance on the latter indicates that it internalizes only rules used during training, particularly those whose removal increases loss.
 }
    \label{fig:apdx-experiments-selectiveMemorization}
\end{figure}

\subsection{Mechanistic Insights}

Here we provide additional figures corresponding to \cref{fig:mechanistic-icl} and \cref{fig:mechanistic-localizedStorage}, but in more settings ($n=10$ vs $n=8$, \annealdpstrat{} vs \fixdpstrat{}).
\label{sec:apdx-experiments-mechanistic}
\begin{figure}[H]
    \centering
    \includegraphics[width=0.8\linewidth]{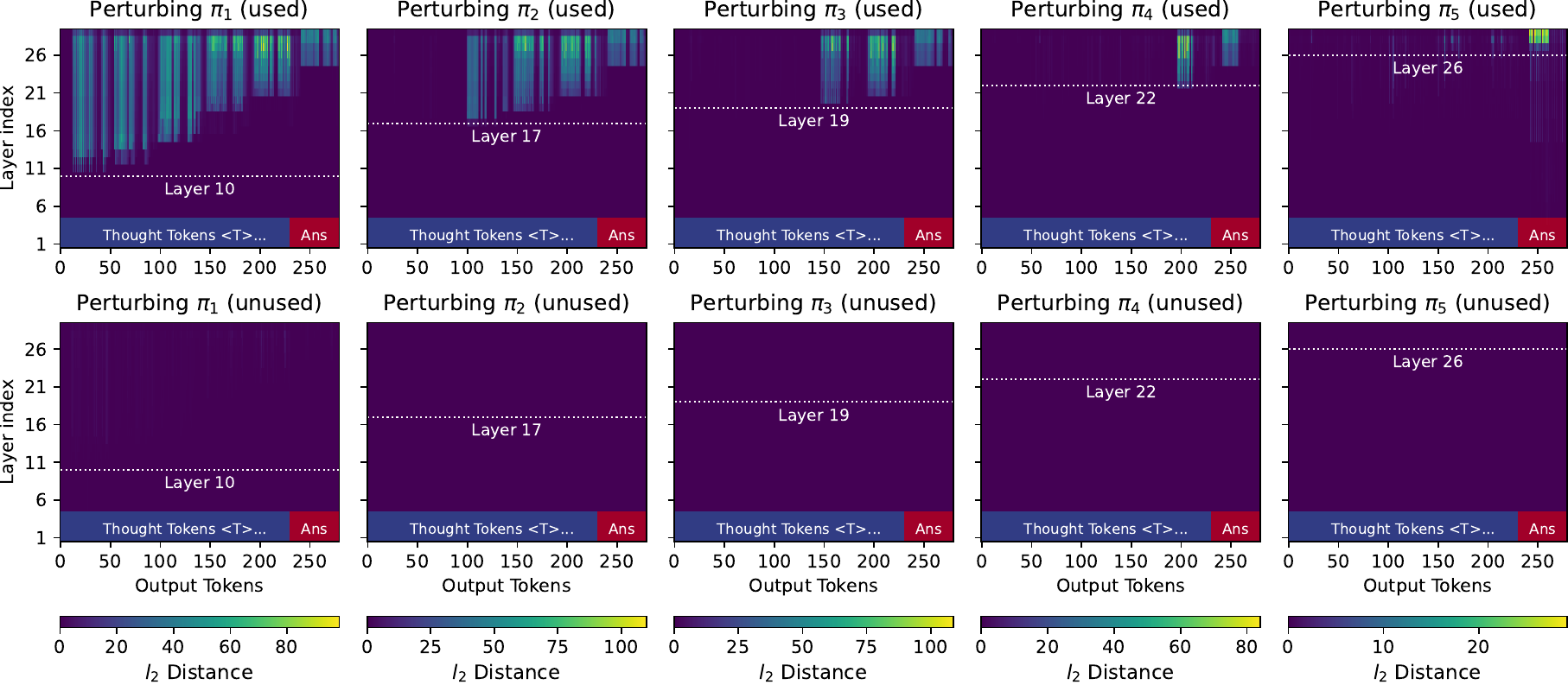}
    \caption{Comparison between perturbing used rules (top row, identical to \cref{fig:mechanistic-icl}) and unused rules (bottom row) in context. When perturbing rules used in the translation at a later phrasebook, representation changes at later layers. However only perturbing unused rules leads to negligible representation changes. This experiment rules out the possibility that the affected depth is only dependent on position of the perturbation instead of the semantic content. }
    \label{fig:mechanistic-icl-additional}
\end{figure}

\begin{figure}[H]
     \centering
     \begin{subfigure}[b]{0.495\linewidth}
         \centering
         \includegraphics[width=\textwidth]{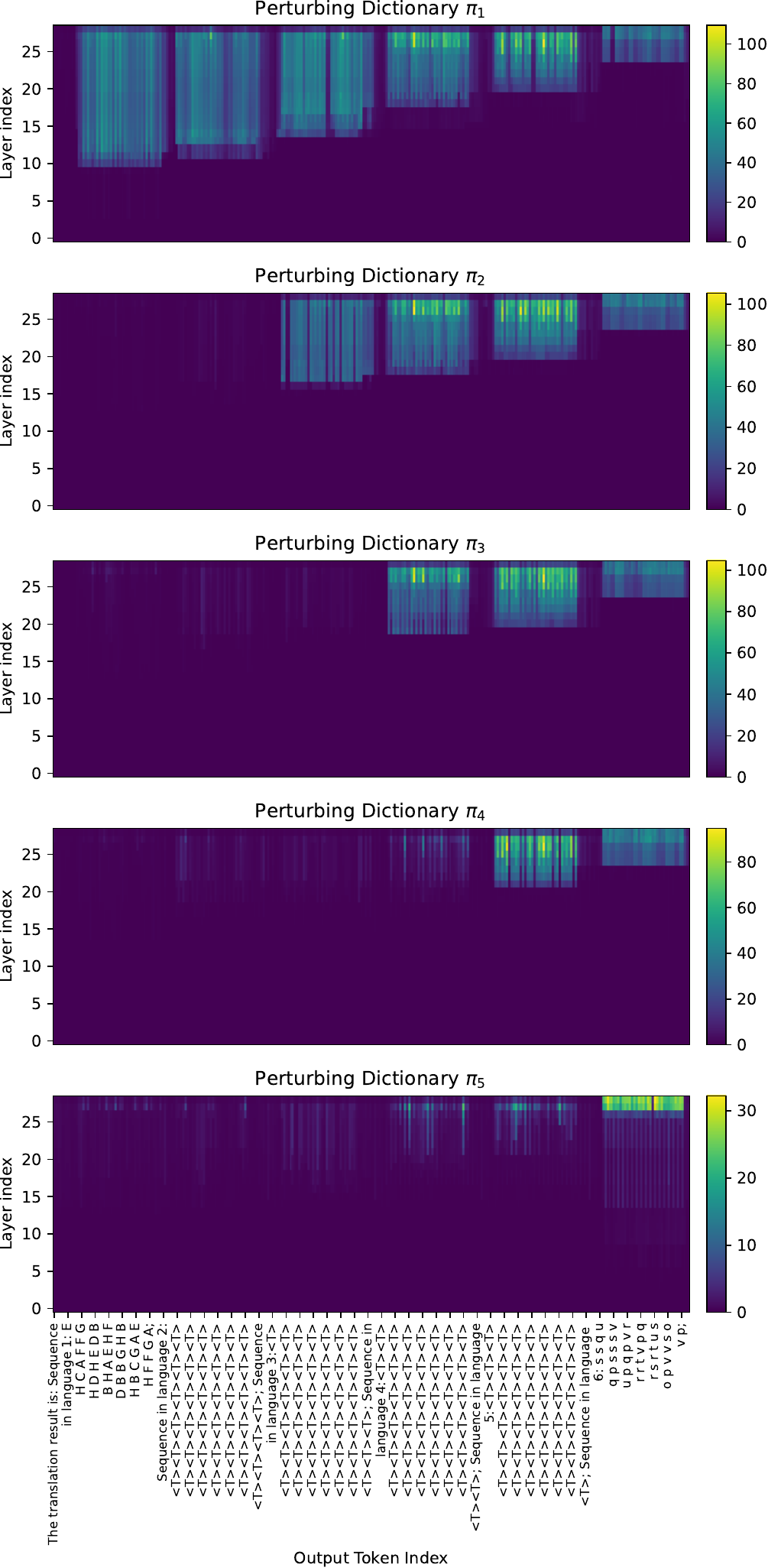}
         \caption{$\translationtask(5,8)$-ICL-capable Model}
     \end{subfigure}
     \hfill
     \begin{subfigure}[b]{0.495\linewidth}
         \centering
         \includegraphics[width=\textwidth]{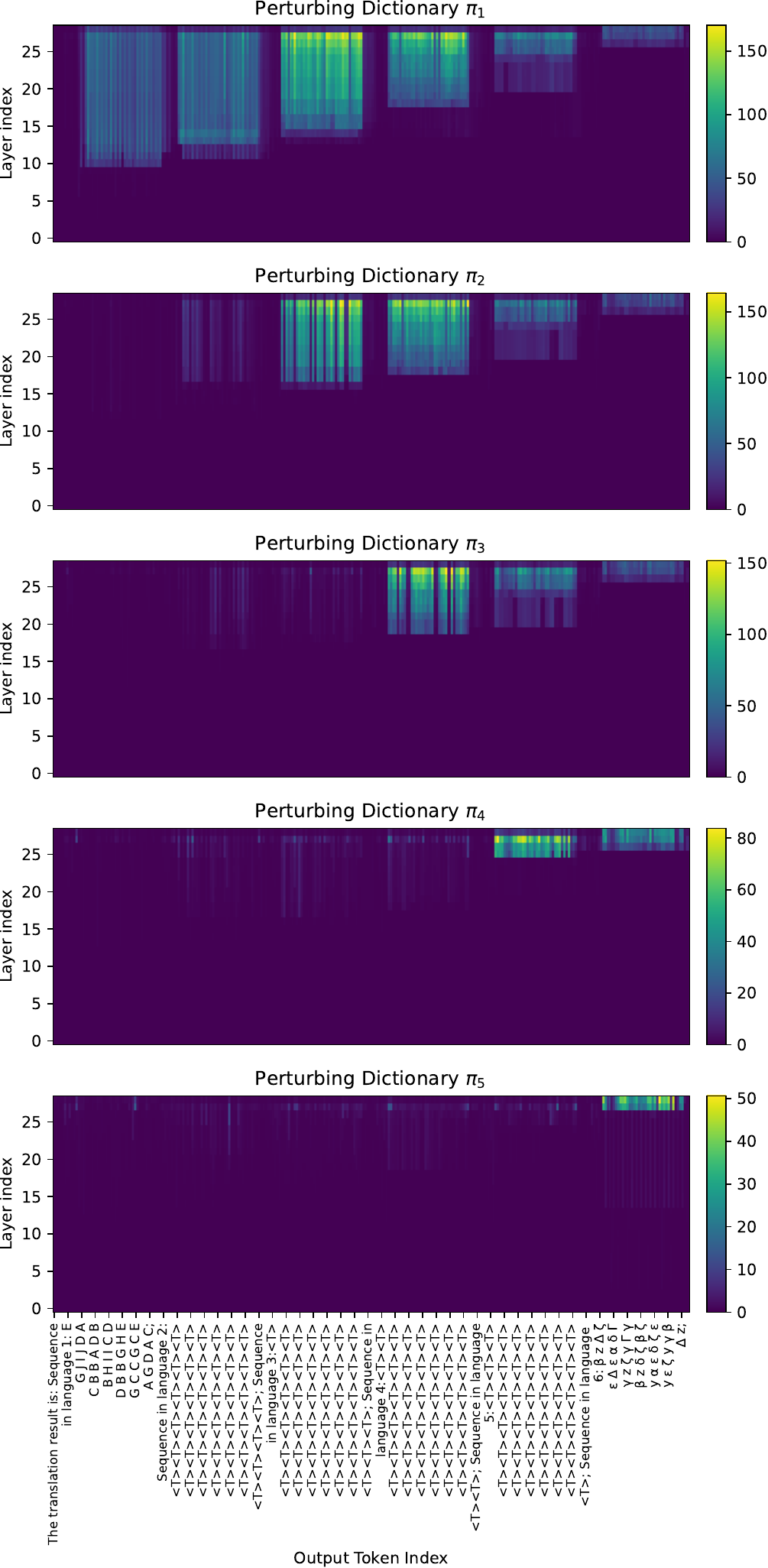}
         \caption{$\translationtask(5,10)$-ICL-capable Model}
     \end{subfigure}
        \caption{Evidence for sequential processing in \translationtask-ICL capable models ($n=8, n=10$). The left figure is identical to \cref{fig:mechanistic-icl}, except we substitute the entire phrasebook $\pi_i$ with another random phrasebook of the same level $\hat\pi_i$. We observe the same behavior for $\translationtask(5, 10)$-ICL-capable model:  perturbing later \strdicts{} in the context changes output representations in the later layers.}
\end{figure}
\newpage

\begin{figure}[H]
     \centering
     \begin{subfigure}[b]{0.85\linewidth}
         \centering
         \includegraphics[width=\textwidth]{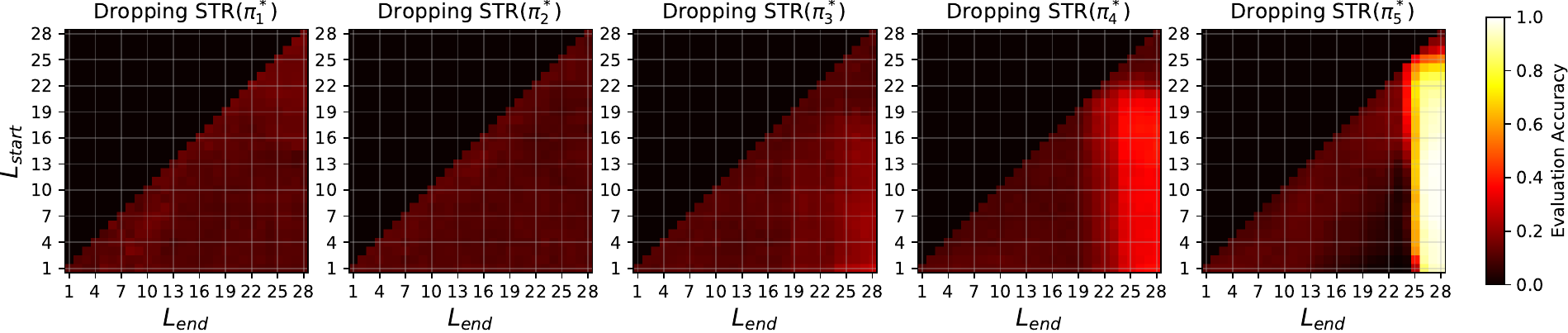}
         \vspace{-0.23in}
         \caption{$D_\targetCB$ with 10000 samples}\vspace{0.2in}
     \end{subfigure}
     \begin{subfigure}[b]{0.85\linewidth}
         \centering
         \includegraphics[width=\textwidth]{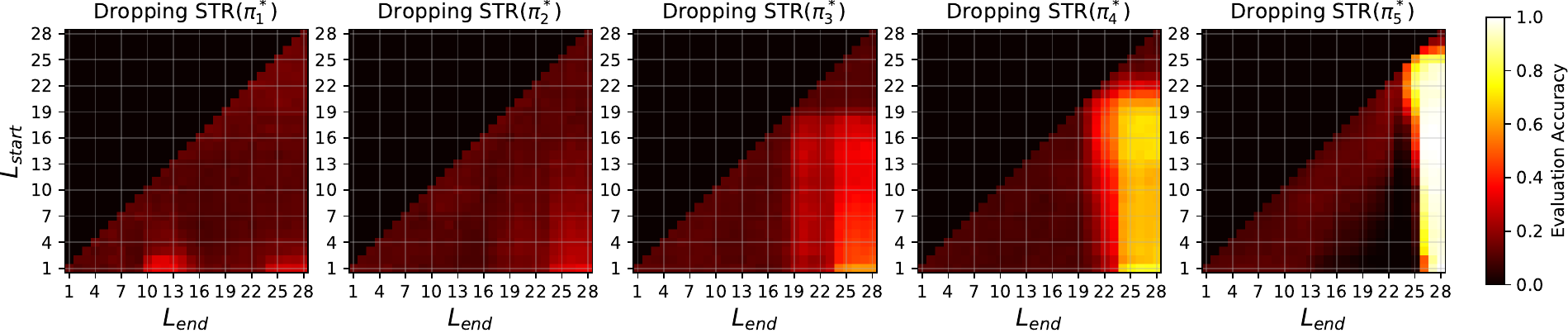}
         \vspace{-0.23in}
         \caption{$D_\targetCB$ with 25000 samples}\vspace{0.2in}
     \end{subfigure}
     \begin{subfigure}[b]{0.85\linewidth}
         \centering
         \includegraphics[width=\textwidth]{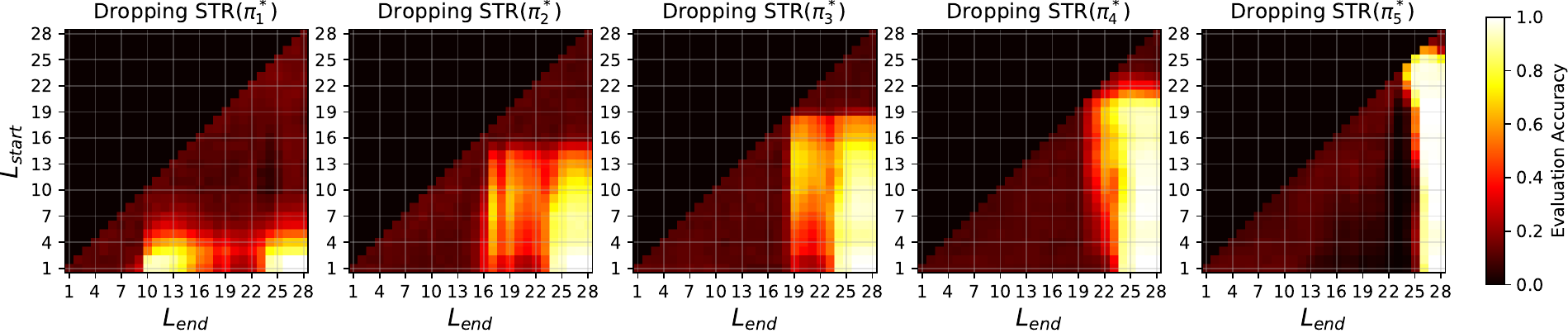}
         \vspace{-0.23in}
         \caption{$D_\targetCB$ with 50000 samples}\vspace{0.2in}
     \end{subfigure}
     \begin{subfigure}[b]{0.85\linewidth}
         \centering
         \includegraphics[width=\textwidth]{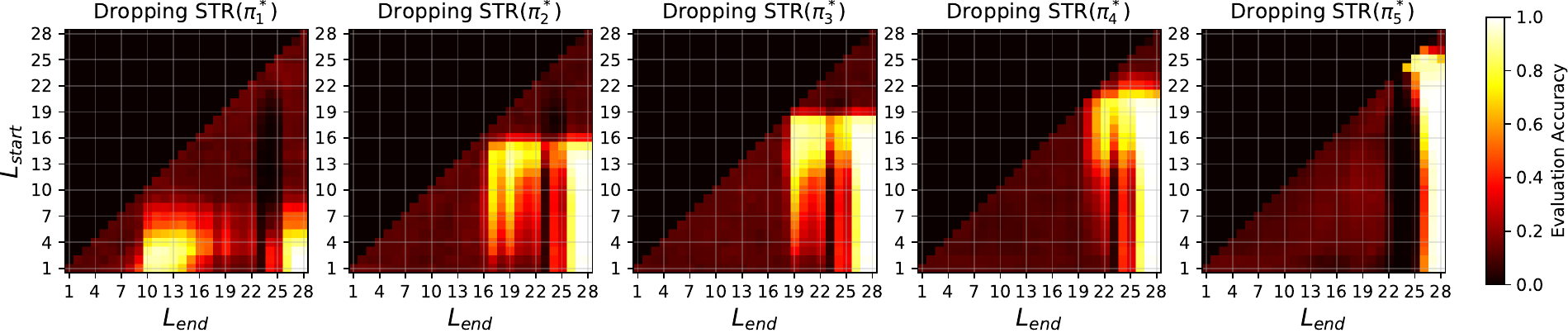}
         \vspace{-0.23in}
         \caption{$D_\targetCB$ with 100000 samples}\vspace{0.2in}
     \end{subfigure}
     \begin{subfigure}[b]{0.85\linewidth}
         \centering
         \includegraphics[width=\textwidth]{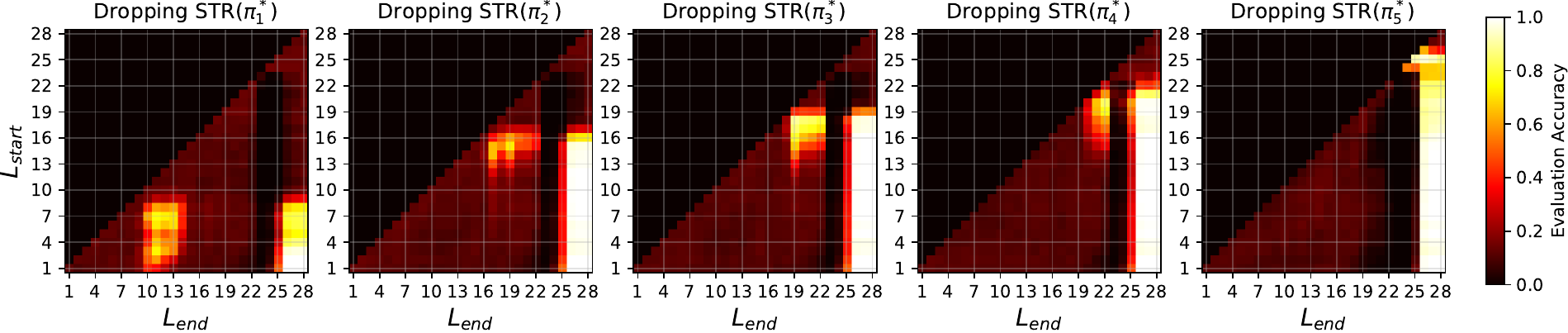}
         \vspace{-0.23in}
         \caption{$D_\targetCB$ with 250000 samples}
     \end{subfigure} 
        \caption{Repeated experiments from \cref{fig:mechanistic-localizedStorage} for models trained with \ICM{} at varying supervised dataset sizes (Plots for row (d) are identical to the plots in \cref{fig:mechanistic-localizedStorage}). We observe that \strdicts{} are progressively internalized with the number of training samples available during \ICM{}; later \strdicts{} are internalized with fewer samples than the earlier ones. We observe similar localization patterns across layers from different \strdicts{} across the models, however, we also observe that the localization patterns get increasingly sparser as training continues. }
        \label{fig:apdx-experiments-mechanistic-stitchingAnneal}
\end{figure}

\begin{figure}[H]
     \centering
     \begin{subfigure}[b]{0.85\linewidth}
         \centering
         \includegraphics[width=\textwidth]{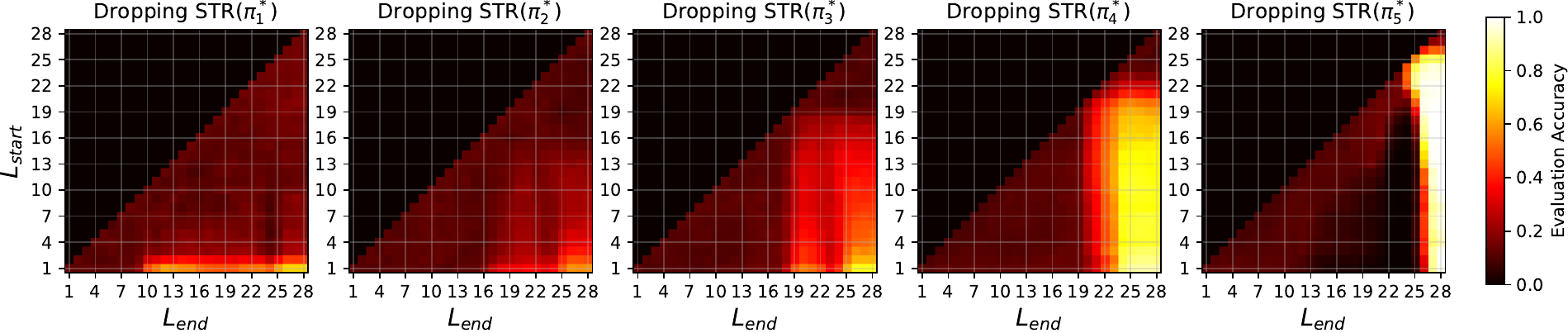}
         \vspace{-0.23in}
         \caption{$D_\targetCB$ with 25000 samples}\vspace{0.2in}
     \end{subfigure}
     \begin{subfigure}[b]{0.85\linewidth}
         \centering
         \includegraphics[width=\textwidth]{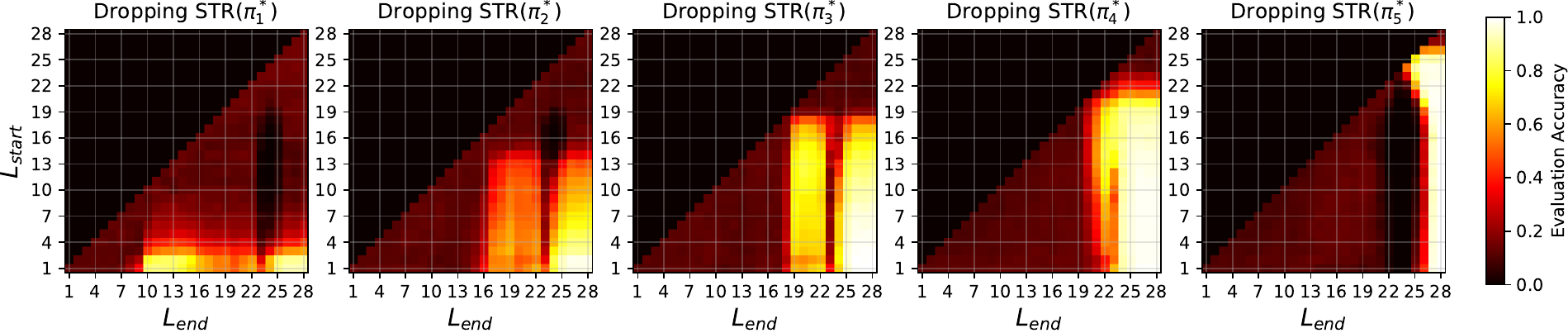}
         \vspace{-0.23in}
         \caption{$D_\targetCB$ with 50000 samples}\vspace{0.2in}
     \end{subfigure}
     \begin{subfigure}[b]{0.85\linewidth}
         \centering
         \includegraphics[width=\textwidth]{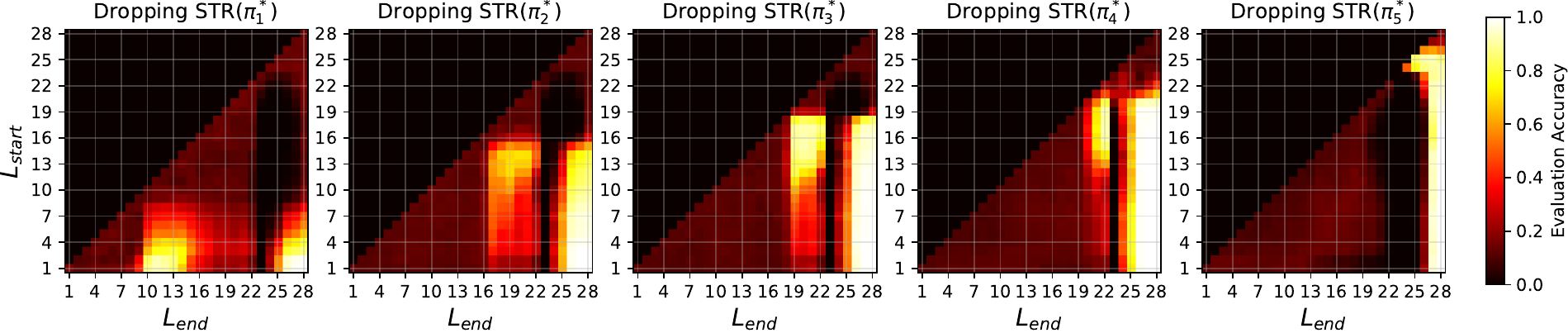}
         \vspace{-0.23in}
         \caption{$D_\targetCB$ with 100000 samples}\vspace{0.2in}
     \end{subfigure}
     \begin{subfigure}[b]{0.85\linewidth}
         \centering
         \includegraphics[width=\textwidth]{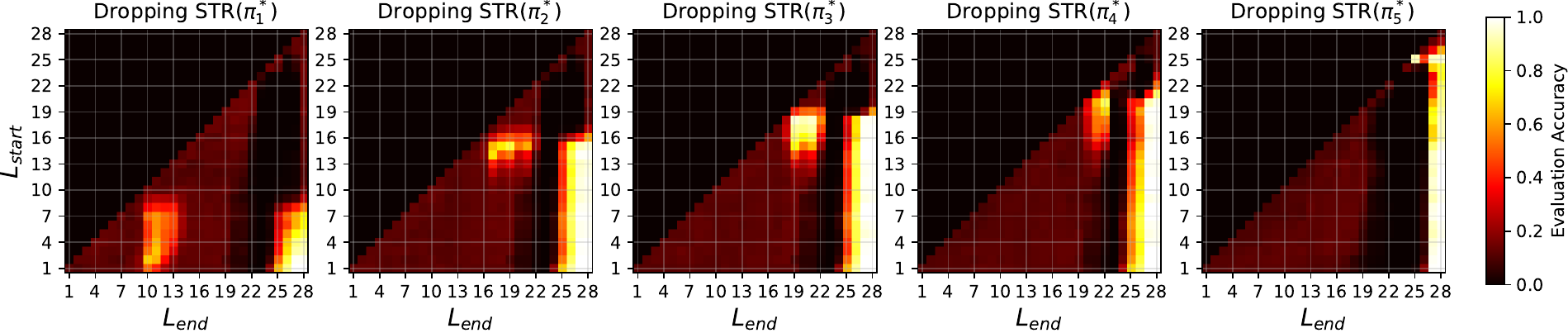}
         \vspace{-0.23in}
         \caption{$D_\targetCB$ with 250000 samples}\vspace{0.2in}
     \end{subfigure}
     \begin{subfigure}[b]{0.85\linewidth}
         \centering
         \includegraphics[width=\textwidth]{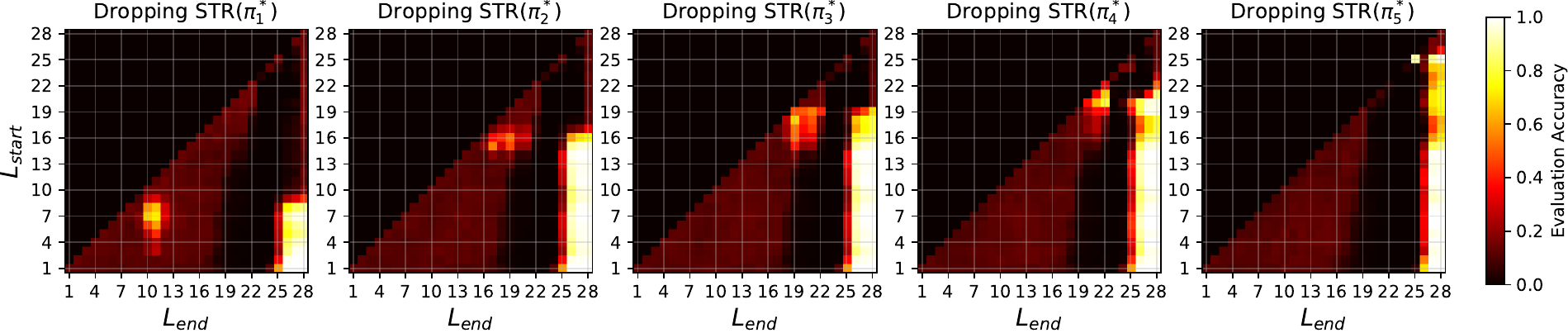}
         \vspace{-0.23in}
         \caption{$D_\targetCB$ with 500000 samples}
     \end{subfigure}
        \caption{Repeated experiments from \cref{fig:apdx-experiments-mechanistic-stitchingAnneal} for models trained with \fixdpstrat{} instead of \annealdpstrat{}. Observations remain similar. This suggests that the position of phrasebook internalization is primarily decided by the $\translationtask(\depth, \numchar)$-ICL-capable initialization, and less dependent on the specific dropout curriculum we use for context-enhanced learning.}
        \label{fig:apdx-experiments-mechanistic-stitchingfix20}
\end{figure}

\newpage
\subsection{(Non)Verbatim Memorization of Phrasebook Rules}
\label{sec:apdx-experiments-memorization}

In this subsection, we provide a more detailed explanation for the evaluation on the feasibility of recovering phrasebook rules from models trained with \ICM{} with phrasebook excerpts in context.

Given a $\translationtask_\targetCB$-capable model $\langmodel_{\lmparam^*}$ trained with \ICM{} where phrasebook rules from $\str(\targetCB)$ are provided within the context during training, we test if the model retains explicit memory of the textual format of rules. Namely, for each of the phrasebook rule of the form ``\texttt{a b -> C D}'' in $\str(\targetCB)$, whether it can complete the ground truth output tokens ``\texttt{C D}'' providing ``\texttt{a b ->}'' and additional context of the phrasebook $\str(\targetCB)$ before ``\texttt{a b -> C D}''.
We conduct the test by computing the forward pass through $\langmodel_{\lmparam^*}$ with $\str(\targetCB)$ as the input (see exact input format in \cref{fig:apdx-expconfig-inputformat-internalizedcot}).

Now we formally define the query strategies and metrics we used for creating \cref{tab:memorization-query-succ}. Suppose there are $h$ unique tokens and $\str(\targetCB)$ contains $L$ tokens, let $\mS\in \R^{h\times L}$ denote the corresponding output logit score matrix when we pass $\str(\targetCB)$ into $\langmodel_{\lmparam^*}$. For a translation rule with ground truth output tokens of index $a_1, a_2\in \alphabetset_{i+1}\subset [h]$, let $\mS^{(k)}$ and $\mS^{(k+1)}\in\R^h$ denote the logit vector corresponding to predicting these two entries.

If we are doing greedy decoding, then we will recover the correct phrasebook rule if and only if we greedily select both $a_1$ from $\mS^{(k)}$ and $a_2$ from $\mS^{(k+1)}$, so the probability of correctly recovering $(a_1,a_2)$ conditioned on $\mS$ is \begin{align*}
    \pr{\text{Greedy-Recover} (a_1, a_2)} = \ind\sbr{\argmax_{i\in[h]}\mS^{(k)} = a_1}\cdot\ind\sbr{\argmax_{i\in[h]}\mS^{(k+1)} = a_2}.
\end{align*}
Meanwhile, if we apply random sampling with softmax temperature of 1, the probability of correctly recovering $(a_1,a_2)$ conditioned on $\mS$ is just then \begin{align*}
    \pr{\text{Sampling-Recover} (a_1, a_2)} = \rbr{\frac{\exp\rbr{\mS^{(k)}_{a_1}}}{\sum_{i\in[h]}\exp\rbr{\mS^{(k)}_{i}}}}\cdot \rbr{\frac{\exp\rbr{\mS^{(k+1)}_{a_1}}}{\sum_{i\in[h]}\exp\rbr{\mS^{(k+1)}_{i}}}}.
\end{align*}
Recall from \cref{sec:memorization} that we have also introduced two stronger adversaries with additional \emph{token filtering} when doing decoding: (1) setting probability of \think{} tokens to zero (2) when querying for a rule in $\dictionary_i$ with output alphabet $\alphabetset_{i+1}$, setting probability of all tokens outside $\alphabetset_{i+1}$ to zero.

Denoting the token index of \think{} as $a_{\think}$, then filter 1 corresponds to
\begin{align*}
\begin{split}
    \pr{\text{Greedy-Filter1-Recover} (a_1, a_2)} &= \ind\sbr{\argmax_{i\in[h]\backslash\cbr{a_{\think}}}\mS^{(k)} = a_1}\cdot\ind\sbr{\argmax_{i\in[h]\backslash\cbr{a_{\think}}}\mS^{(k+1)} = a_2},\\
    \pr{\text{Sampling-Filter1-Recover} (a_1, a_2)} &= \rbr{\frac{\exp\rbr{\mS^{(k)}_{a_1}}}{\sum_{i\in[h]\backslash\cbr{a_{\think}}}\exp\rbr{\mS^{(k+1)}_{i}}}}\cdot \rbr{\frac{\exp\rbr{\mS^{(k)}_{a_1}}}{\sum_{i\in[h]\backslash\cbr{a_{\think}}}\exp\rbr{\mS^{(k+1)}_{i}}}}.
\end{split}
\end{align*}
Similarly, for filter 2, the probabilities are 
\begin{align*}
\begin{split}
    \pr{\text{Greedy-Filter2-Recover} (a_1, a_2)} &= \ind\sbr{\argmax_{i\in\alphabetset_{i+1}}\mS^{(k)} = a_1}\cdot\ind\sbr{\argmax_{i\in\alphabetset_{i+1}}\mS^{(k+1)} = a_2},\\
    \pr{\text{Sampling-Filter2-Recover} (a_1, a_2)} &= \rbr{\frac{\exp\rbr{\mS^{(k)}_{a_1}}}{\sum_{i\in\alphabetset_{i+1}}\exp\rbr{\mS^{(k+1)}_{i}}}}\cdot \rbr{\frac{\exp\rbr{\mS^{(k)}_{a_1}}}{\sum_{i\in\alphabetset_{i+1}}\exp\rbr{\mS^{(k+1)}_{i}}}}.
\end{split}
\end{align*}
Note that the second filter is a very strong adversarial assumption which assumes the user already has side information on the set of tokens contained in alphabet $\alphabetset_i$ and $\alphabetset_{i+1}$. To compute the final statistics, we sample 20 permutations of $\str(\targetCB)$ for forward passes and compute the mean of the above statistics over all atomic phrasebook rules appearing in the context. That is 1280 entries for each phrasebook in the case of $n=8$ and 2000 entries for each phrasebook in the case of $n = 10$.

\begin{table}[H]
\label{tab:apdx-experiments-memorization-c8}
\caption{Recovery Success Rate For $n=8, d=5$ Cases (Rounded to 2 decimals) Random Guess Baseline: $1.56\%$}
\begin{center}
    
\begin{tabular}{cclcccc}
\toprule
\multirow{2}{*}{Curriculum} & \multirow{2}{*}{\# Training Samples} & \multirow{2}{*}{Query Method} & \multicolumn{2}{c}{Greedy Decoding} & \multicolumn{2}{c}{Sampling $(T=1)$} \\
                                    & &                               & $\pi_1-\pi_4$       & $\pi_5$       & $\pi_1-\pi_4$        & $\pi_5$       \\\midrule
\multirow{3}{*}{\annealdpstrat{}} & \multirow{3}{*}{50000} & Base & $0.00\%$ & $0.00\%$ & $0.00\%$ & $0.01\%$ \\
& & Rule out \texttt{<THINK>}  &   $0.06\%$ & $1.95\%$ & $0.00\%$ & $0.54\%$ \\
& & Only keeping $\alphabetset_i$ & $3.18\%$ & $1.95\%$ & $1.82\%$ & $1.49\%$ \\\midrule
\multirow{3}{*}{\annealdpstrat{}} & \multirow{3}{*}{100000} & Base & $0.00\%$ & $0.00\%$ & $0.00\%$ & $0.64\%$ \\
& & Rule out \texttt{<THINK>}     & $0.00\%$ & $0.08\%$ & $0.00\%$ & $1.25\%$ \\
& & Only keeping $\alphabetset_i$ & $1.37\%$ & $0.08\%$ & $1.69\%$ & $1.80\%$ \\\midrule
\multirow{3}{*}{\annealdpstrat{}} & \multirow{3}{*}{250000} & Base & $0.00\%$ & $0.23\%$ & $0.00\%$ & $1.25\%$ \\
& & Rule out \texttt{<THINK>}     & $0.00\%$ & $0.23\%$ & $0.00\%$ & $1.28\%$ \\
& & Only keeping $\alphabetset_i$ & $2.95\%$ & $0.23\%$ & $2.53\%$ & $1.38\%$ \\\midrule
\multirow{3}{*}{\fixdpstrat{}} & \multirow{3}{*}{50000} & Base & $0.00\%$ & $0.00\%$ & $0.00\%$ & $0.00\%$ \\
& & Rule out \texttt{<THINK>}  &   $0.08\%$ & $0.78\%$ & $0.00\%$ & $0.09\%$ \\
& & Only keeping $\alphabetset_i$ &  $0.82\%$ & $0.86\%$ & $1.48\%$ & $1.43\%$ \\\midrule
\multirow{3}{*}{\fixdpstrat{}} & \multirow{3}{*}{100000} & Base & $0.00\%$ & $0.00\%$ & $0.00\%$ & $0.21\%$\\
& & Rule out \texttt{<THINK>}     & $0.43\%$ & $0.00\%$ & $0.03\%$ & $0.80\%$ \\
& & Only keeping $\alphabetset_i$ & $2.36\%$ & $0.00\%$ & $1.95\%$ & $1.32\%$ \\\midrule
\multirow{3}{*}{\fixdpstrat{}} & \multirow{3}{*}{250000} & Base & $0.00\%$ & $0.47\%$ & $0.02\%$ & $0.51\%$ \\
& & Rule out \texttt{<THINK>}     & $0.68\%$ & $1.09\%$ & $0.11\%$ & $0.73\%$ \\
& & Only keeping $\alphabetset_i$ & $2.23\%$ & $1.09\%$ & $2.19\%$ & $1.10\%$ \\
\bottomrule

\end{tabular}
\end{center}
\end{table}
\vspace{-0.2in}
\begin{table}[H]
\label{tab:apdx-experiments-memorization-c10}
\caption{Recovery Success Rate For $n=10, d=5$ Cases (Rounded to 2 decimals) Random Guess Baseline: $1\%$}
\begin{center}
    
\begin{tabular}{cclcccc}
\toprule
\multirow{2}{*}{Curriculum} & \multirow{2}{*}{\# Training Samples} & \multirow{2}{*}{Query Method} & \multicolumn{2}{c}{Greedy Decoding} & \multicolumn{2}{c}{Sampling $(T=1)$} \\
                                    & &                               & $\pi_1-\pi_4$       & $\pi_5$       & $\pi_1-\pi_4$        & $\pi_5$       \\\midrule
\multirow{3}{*}{\annealdpstrat{}} & \multirow{3}{*}{100000} & Base & $0.00\%$ & $0.20\%$ & $0.00\%$ & $0.89\%$ \\
& & Rule out \texttt{<THINK>}     & $0.00\%$ & $0.20\%$ & $0.00\%$ & $0.90\%$ \\
& & Only keeping $\alphabetset_i$ & $1.66\%$ & $0.20\%$ & $1.28\%$ & $0.94\%$ \\\midrule
\multirow{3}{*}{\annealdpstrat{}} & \multirow{3}{*}{250000} & Base & $0.00\%$ & $1.80\%$ & $0.00\%$ & $1.12\%$ \\
& & Rule out \texttt{<THINK>}   & $0.00\%$ & $1.80\%$ & $0.00\%$ & $1.13\%$   \\
& & Only keeping $\alphabetset_i$ & $3.05\%$ & $1.80\%$ & $1.72\%$ & $1.23\%$ \\\midrule
\multirow{3}{*}{\fixdpstrat{}} & \multirow{3}{*}{250000} & Base & $0.00\%$ & $1.60\%$ & $0.00\%$ & $1.07\%$ \\
& & Rule out \texttt{<THINK>}  &   $0.05\%$ & $1.60\%$ & $0.01\%$ & $1.07\%$ \\
& & Only keeping $\alphabetset_i$ &  $2.46\%$ & $2.15\%$ & $1.99\%$ & $1.39\%$ \\\midrule
\multirow{3}{*}{\fixdpstrat{}} & \multirow{3}{*}{500000} & Base & $0.26\%$ & $2.05\%$ & $0.05\%$ & $1.13\%$\\
& & Rule out \texttt{<THINK>}     & $0.33\%$ & $2.05\%$ & $0.07\%$ & $1.13\%$ \\
& & Only keeping $\alphabetset_i$ & $2.11\%$ & $2.05\%$ & $1.91\%$ & $1.34\%$ \\
\bottomrule

\end{tabular}
\end{center}
\end{table}
\vspace{-0.1in}

Here we report the query success rate for a variaty of models trained with \ICM{} using different curriculum on different datasets sizes. All models reaches nearly perfect test accuracy when no context is provided (see \cref{fig:experiments-main-n10}), i.e., the in context phrasebook rules played significant role in the learning of the models.

For all runs, we can see that it is nearly impossible (with recovery probability $<0.1\%$ in most cases) to recover the correct phrasebook rules for hidden steps ($\pi_1,\dots, \pi_4$) even when we provide the correct partial phrasebooks in context (see columns corresponding to Query Method ``Base''). Ruling out \think{} token when sampling also did not significantly increase the recovery rate.
We note that the phrasebook knowledge \textbf{are not memorized in an completely undetectable manner}, as the strongest token filtering gives non-random probability of outputting the correct target 2-tuple. However the probability is still very low ($<3\%$), which can be considered as negligible to recover the full correct phrasebooks $\str(\targetCB)$.

\section{Additional Related Works}
\label{sec:addition_work}

\subsection*{Differences with Masked Language Modeling (MLM) and Language Infilling} 
MLM models like BERT, RoBERTa, and T5 \citep{kenton2019bert,liu1907roberta,raffel2020exploring} train models by either masking or removing tokens from a sequence and compute loss on the model's prediction on the missing tokens. This concept has been adapted for training auto-regressive models via language infilling task ~\citep{bavarian2022efficient,li2022tip,donahue2020enabling,li2022tip}. The primary difference from these works is that \ICM{} does not take loss on the context tokens when they are removed from our \curriculumtext{}.

\subsection*{Compositional and OOD generalization} Measuring generalization for a transformer beyond training distribution has been a study of interest in many prior works. OOD generalization is measured by training a transformer on simpler examples and measuring its performance on harder ones. Prominent studies include length generalization, informally defined as the ability of the model to reason longer than what it has been trained on \citep{zhou2023algorithms,anil2022exploring}, and compositional generalization on concepts, defined as the ability of the model to reason on composition of the concepts that it has seen during training \citep{press2022measuring,allen2023physics3point2,ramesh2023compositional,yu2023skill,zhao2024can,wang2024grokked,yang2024large}. Our experiments on \fixdpstrat{} in \cref{sec:experiments}, where we train with $20\%$ dropout on the \strdicts{} information but measure performance with $100\%$ dropout at test time, measures compositional OOD generalization behavior of the language model. The results show that the model internalizes the rules from the \strdicts{} in an atomic way, and re-compose them together as necessary at test time.

\subsection*{Mechanistic behavior of transformers with synthetic datasets:} Our work builds on a growing body of research exploring the behavior of transformers trained on synthetic datasets. Prior studies have examined tasks such as modular addition \citep{nanda2023progress,zhong2024clock}, context-free grammars \citep{zhao2023transformers,allen2023physics}, regular and $n$-gram languages \citep{bhattamishra2020ability,yao2021self,akyurek2024context,li2023transformers}, and synthetic article-style datasets \citep{allen2023physics3point2,allen2024physics,eldan2023tinystories}. While our work is structurally similar to these studies, it investigates mechanistic study on \ICM{} that has not been explored in previous works.

\section{Properties of $\translationtask$}
\label{sec:properties_mlt}

$\translationtask(\depth, \numchar)$ is defined by the \strdicts{} $\dictionary_1, \cdots, \dictionary_{\depth}$ at each of its translation layers. We use $\mathcal{B}^{\dictionary}$ as all possible set of bijective maps that can be used to define the \strdicts{}. We will use variable $\vardictionary$ to refer to an arbitrary bijective map from the set $\mathcal{B}^{\dictionary}$. For simplicity of proof, we will refer to any alphabet set $\alphabetset$ of size $\numchar$ as $\{ 0, 1, \cdots, \numchar-1 \}$.

Here, we formally mention some of the properties of $\translationtask(\depth, \numchar)$.

\begin{lemma}[Invertibility of sequence translation]
\label{lemma:setting-MLTinvertability}
    For any level $i\in[d]$, fixing $\cbr{\dictionary_i, \dictionary_{i+1},\dots,\dictionary_\depth}$ gives a bijection between $\intermseq_{i}$ and $\intermseq_{\depth+1}$.
\end{lemma}

\begin{proof}
    All of the four operations involved in the mapping from $\intermseq_i$ to $\intermseq_{i+1}$ are invertible. 
\end{proof}

\paragraph{Structure of the mappings:}
The mappings $\dictionary_i$ are selected as random bijective maps between 2-tuples of characters in $\alphabetset_i^2$ and 2-tuples of characters in $\alphabetset_{i+1}^2$. A combinatorial argument can then give the number of such possible \strdicts{} to be $\numchar^2!$.

\begin{lemma}[Number of mappings in each level]
\label{lem:num_bijective}
    The number of possible bijective maps between 2-tuples of characters from alpbhabet sets $\alphabetset_i, \alphabetset_{i+1}$, each being of size $\numchar$, is $\numchar^2 !$, i.e. $|\mathcal{B}^{\dictionary}| = \numchar^2!$.
\end{lemma}

\begin{proof}
Each alphabet set contains $\numchar^2$ possible 2-tuples of characters. If we fix an order in which the 2-tuples appear in $\alphabetset^2_{i+1}$ , then the number of bijective maps between $\alphabetset^2_{i}$  and $\alphabetset^2_{i+1}$ can be reduced to the number of possible ordering of the 2-tuples in $\alphabetset^2_{i}$. The number of possible orderings is $\numchar^2!$.

\end{proof}

\paragraph{Importance of \circularshift{}:} The composition of the $\depth$ random bijective maps can be demonstrated to result in another random bijective map. Consequently, without the \circularshift{}, each character in the output sequence $\intermseq_{\depth+1}$ depends on only two characters from the input sequence $\inputseq$ via a shared random map across all the 2-tuples.

Incorporating \circularshift{} on the other hand enables each character in the output sequence to depend on $2\depth$ characters from the input sequence. This is because the character positions are shifted to the right at each step, causing the input 2-tuples to the bijective map to also shift to the right at each stage. As we will elaborate later, incorporating \circularshift{} increases the required number of training samples to learn the \strcodebook{} from input and label pairs to $\numchar^{\Omega(\depth)}$, whereas without \circularshift{}, this requirement is only $\mathcal{O}(n^2)$.

\paragraph{Representing the mappings on 2-tuples in-context:} Each map can be defined using $\mathcal{O}(\numchar^2)$ characters, as it can be simply defined by the $\numchar^2$ relations each connecting $2$ unique random 2-tuples from their corresponding alphabet sets. Thus, defining $\depth$ maps in-context will require $\mathcal{O}(\numchar^2 \depth)$ characters in \curriculumtext{}. In contrast, describing a completely random bijective mapping that maps $\depth$-tuples of  characters in input sequence to a character in output sequence will require $\Omega(\depth {\numchar \choose \depth})$ bits. Thus, $\translationtask$ with $\depth$ translation steps involving random mappings on 2-tuples and \circularshift{} helps define a mapping where each character in the output sequence can depend on $\depth$ character in the input sequence, and the \strcodebook{} can be described using $\mathcal{O}(\numchar^2 \depth)$ characters in \curriculumtext{}.

\section{Lower Bound: Hardness of Learning $\translationtask( \depth, \numchar)$ without Context}
\label{sec:apdx-lowerBound}
\subsection{Brief introduction to SQ framework}\label{app:definition}

\paragraph{Statistical query (SQ) bounds}:
The statistical query (SQ) framework measures the computational hardness of learning a task in the presence of noise. It measures the hardness of learning a task by the number of statistical queries needed by a learning algorithm to learn the true function. Statistical queries are defined by some polynomially-computable property $Q$ of labeled instances and a tolerance parameter $\tau \in [0, 1]$ over $(x, y) \sim D$ where D is the data distribution. For a query, the algorithm receives a response from the oracle within $\tau$ error of the true value. The statistical dimension, or $\sqdim$, is measured in terms of the number of functions in the hypothesis class that the learning algorithm needs to distinguish and the number of queries necessary to do the same. Correlation between two functions is used to define the statistical query dimension.

\begin{definition}\label{def:corr}
    Correlation of two functions $f_1, f_2$ on a domain $\mathcal{X}$ with respect to a distribution $\distribution$ is given by
    \begin{align*}
        \corr(f_1, f_2, \distribution) := \abs{\Pr_{x \in \distribution} [f_1 (x) = f_2 (x)] - \Pr_{x \in \distribution} [f_1(x) \ne f_2(x)]}. 
    \end{align*}
    For functions $f_1, f_2: \mathcal{X} \to \{0, 1\}$, the above definition is also equivalent to
    \begin{align*}
        \corr(f_1, f_2, \distribution) := \abs{1 - 2\mathbb{E}_{x \in \distribution} [f_1 (x) \oplus f_2 (x)]}. 
    \end{align*}
\end{definition}

\begin{remarkbox}
\begin{remark}
    Two functions $f_1, f_2: \mathcal{X} \to \{0, 1\}$ are said to be uncorrelated w.r.t. $\distribution$ if
    \begin{align*}
        &\Pr_{x \in \distribution} [f_1 (x) = f_2 (x)] = \Pr_{x \in \distribution} [f_1(x) \ne f_2(x)] \\
        \text{(alternately) } & \mathbb{E}_{x \in \distribution} [f_1 (x) \oplus f_2 (x)] = \frac{1}{2}.
    \end{align*}
    On the other hand, if
    \begin{align*}
        &Pr_{x \in \distribution} [f_1 (x) = f_2 (x)] = 1 \quad (\text{or } 0) \\
        \text{(alternately) } & \mathbb{E}_{x \in \distribution} [f_1 (x) \oplus f_2 (x)] = 0  \quad (\text{or } 1), \\
    \end{align*}
    then $\corr(f_1, f_2, \distribution) = 1$.
\end{remark}
\end{remarkbox}

We take the following formal definitions of $\sqdim$ and its relation to computational hardness in the SQ framework from \cite{blum1994weakly}.

\begin{definition}[Definition 2 in \citet{blum1994weakly}]\label{def:sqdim}
    For a function class $\functionclass$ of boolean functions over $\{0, 1\}^{n}$ and $\distribution$ a distribution over $\{0, 1\}^{n}$, $\sqdim(\functionclass, \distribution)$, the statistical query dimension of $\functionclass$ with respect to $\distribution$, is defined to be the largest natural number $\mu$ such that $\functionclass$ contains $\mu$ functions $f_1, \cdots, f_\mu$ with the property that for all $i \ne j$ we have:
    \begin{align*}
        \corr(f_i, f_j, \distribution) := \abs{\Pr_{x \in \distribution} [f_i = f_j] - \Pr_{x \in \distribution} [f_i \ne f_j]} \le \frac{1}{\mu^3}. 
    \end{align*}
\end{definition}

\begin{theorem}[Theorem 12 in \citet{blum1994weakly}]
    Let $\functionclass$ be a class of functions $\{0, 1\}^{n}$ and $\mathcal{D}$ a distribution such that $\sqdim(\functionclass, \distribution) \ge \mu \ge 16$. Then if all queries are made with a tolerance of atleast $\frac{1}{\mu^{1/3}}$, at least $\mu^{1/3}/2$ queries are required to learn $\mathcal{F}$ with error less than $1/2 - 1/\mu^3$ in the statistical query model.
\end{theorem}

\subsection{Lower bound lemma}

Here, we mention the $2$ main theorems that study the SQ dimension of the $\translationtask$ task at hand. The first theorem shows the SQ dimension bound when number of characters $\numchar=2$, which is then adapted to get the SQ dimension bound for general $\numchar$. Proofs of both the theorems are given in \cref{sec:proof_sqn2} and \cref{sec:proof_sqn} respectively.

\begin{restatable*}[SQ dimension for $\numchar=2$]{theorem}{sqntwo}
\label{lem:sqn2}
The family of translation task $\translationtask(\depth, 2)$ on input distribution $\uniform(\{0,1\}^{2\depth})$ has statistical query dimension $\sqdim(\translationtask(\depth,2))$  atleast $2^{\Omega(\depth)}$.
\end{restatable*}

\begin{restatable*}[SQ dimension for general $\numchar$]{theorem}{sqn}
\label{lem:sqn}
    For the translation task $\translationtask(\depth, \numchar)$ that has depth $\depth$ and $\numchar$ characters per level, the statistical query dimension $\sqdim(\translationtask( \depth, \numchar))$ is atleast $\numchar^{\Omega(\depth)}$.
\end{restatable*}

\paragraph{Notations:} We will require the following notations for proving the above theorems. First, we will denote a $2$-tuple that contains arbitrary characters $a$ and $b$ as $(a, b)$.
For a bijective map $\vardict$, $\vardict((a, b))_{i}$ will represent $i$th character in the output of $\vardict$ on any 2-tuple input $(a, b)$ and any $i \in \{1, 2\}$. Similarly, for any \strcodebook{} $\codebook$, we will use $\translationtask_\codebook(\inputseq)_i$ to denote the  $i$th  character in the output of $\translationtask_\codebook$ on any $\inputlength$ length input sequence $\inputseq$ and any $i\in[1, \inputlength]$.

\textbf{Recall that} for any input sequence $\inputseq$, the output of the $i$th level of a $\translationtask$ task will be denoted as $\intermseq_i$. Furthermore, to denote the $j$th character (arbitrary) in the sequence $\intermseq_i$, we will use $\intermseq_{i, j}$.

\subsubsection{Bounds for SGD}

The following corollary measures the sample complexity needed to learn $\translationtask_{\targetCB}$ by SGD. It has been adapted from proposition 3 in \citet{edelman2023pareto}, who study the sample complexity for learning $\depth$-sparse parity task on $\numchar$ length sequences, whose SQ dimension is ${\numchar \choose \depth} = \Theta(\numchar^{\depth})$. We simply state the corollary without specifying the proof.

\textbf{Loss function and SGD updates:} Consider training of a model $\langmodel_\lmparam$  with $r$ parameters that is trained with mean squared error, i.e. $\autoloss \rbr{\langmodel_\lmparam([\maskcurr_\task(x, t), x, y]), y} = \norm{ \langmodel_\lmparam([\maskcurr_\task(x, t), x, y]) - y }^2$ \footnote{The results holds for any loss that satisfies  $\frac{\partial \autoloss (y', y) }{ \partial y'} = -y + \frac{\partial \autoloss (y', 0) }{ \partial y'}$ .}.
The empirical loss on a batch $S$ of batch size $\batchsize$ from the supervised dataset $\dataset_{\targetCB}$ will be denoted by $L_S(\langmodel_\lmparam, \translationtask_{\targetCB}) = \mathbb{E}_{(x, y) \sim S} \autoloss \rbr{\langmodel_\lmparam([\maskcurr_\task(x, t), x, y]), y} = \norm{ \langmodel_\lmparam([\maskcurr_\task(x, t), x, y]) - y }^2$; the population loss will be denoted by $L_{ \mathcal{U}\rbr{\{0, 1\}^{\inputlength}} }(\langmodel_\lmparam, \translationtask_{\targetCB})$. SGD updates are of the form:
\begin{align*}
    \lmparam_{t+1} = \lmparam_{t} - \eta_t ( \nabla_{\lmparam} L_{S_t} (\langmodel_{\lmparam_{t}}, \translationtask_{\targetCB}) + R(\lmparam_t) + \zeta_t ).
\end{align*}
for some sample $S_t$, step size $\eta_t$, regularizer $R(\cdot)$, and adversarial noise $\zeta_t \in [-\tau, \tau]^{r}$. For simplicity, we assume the gradient $\nabla_{\lmparam} L_{S_t} (\lmparam_t)$ is bounded on all parameters $\lmparam$ in parameter space of the model. 

\textbf{Fake trajectory:} Suppose $\mathbf{0}$ denote the constant function that maps all inputs to $0$. Then, with the contemporary losses $L_S(\langmodel_\lmparam, \mathbf{0})$ and $L_{ \mathcal{U}\rbr{\{0, 1\}^{\inputlength}} }(\langmodel_\lmparam, \mathbf{0})$ that computes difference from this constant function, consider the following trajectory $\Tilde{\lmparam}_1, \cdots, \Tilde{\lmparam}_t, \cdots$ starting from the same initiation $\lmparam_0$:
\begin{align*}
    \Tilde{\lmparam}_{t+1} = \Tilde{\lmparam}_{t} - \eta_t ( \nabla_{\lmparam} L_{S_t} (\langmodel_{\Tilde{\lmparam}_{t}}, \mathbf{0}) + R(\Tilde{\lmparam}_t)).
\end{align*}

\begin{assumption}
     For all $t$, suppose $\norm{ \nabla_{\lmparam} L_{ \mathcal{U}\rbr{\{0, 1\}^{\inputlength}} } (\langmodel_{\Tilde{\lmparam}_{t}}, \mathbf{0}) - \nabla_{\lmparam} L_{S_t} (\langmodel_{\Tilde{\lmparam}_{t}}, \mathbf{0}) }_2 \le \tau/2$.
\end{assumption}

\begin{corollary}[Sample complexity bounds for SGD]\label{cor:sgd}
    Fix an initialization $\theta_0$ such that $f_{\theta_0}$ is statistically independent (correlation 0) from all possible $\emph{\translationtask}_{\targetCB}$. Under this assumption, if $\frac{\inputlength T \batchsize}{\tau^2}  \le \numchar^{\Omega(\depth)}/r$, then there exists at least one task $\emph{\translationtask}_{\targetCB} \in \emph{\translationtask}(\depth, \numchar)$ for which the functions obtained after the first $T$ SGD updates, $f_{\theta_1}, \dots, f_{\theta_T}$, remain statistically independent (correlation 0) from $\emph{\translationtask}_{\targetCB}$ despite training on it.
\end{corollary}

$\inputlength T \batchsize$ is the product of input length, total training steps, and batch size, which represents sample complexity used by the SGD algorithm.

\subsection{Useful lemmas for $\numchar=2$} \label{app:useful_lemmas}

Here, we mention some useful lemmas for charaterizing the bijective maps on $\{0, 1\}^2 \to \{0, 1\}^2$ that we will regularly use for the proof of \cref{lem:sqn2}.
The first lemma will show that each bijective map can be represented by three operations on input characters, $\copyop{},\xorop{},$ and $\notop{}$ operations. The second lemma measures correlation on the outputs of two randomly sampled bijective maps. The proofs are given in \cref{sec:Proofs_of_Useful_lemmas}.

We define the following necessary operations to describe the random bijective maps on $\{0, 1\}^2 \to \{0, 1\}^2$:
\begin{enumerate}
    \item \copyop{}: Given a tuple $(\var_1, \var_2)$, and a position argument $i \in \{1, 2\}$, the operation returns value of $\var_i$ as output. We will use $\copyop{}_1$ and $\copyop{}_2$ to indicate the \copyop{} operation on positions $1$ and $2$ respectively.
    \item \notop{}: Given a variable $\var_i$, this operation returns flipped value of $\var_i$. That is, if $\var_i = 1$, then it returns $0$ and vice-versa. 
    \item \xorop{}: Given a tuple $(\var_1, \var_2)$, this operation returns $\var_1 \oplus \var_2$.
\end{enumerate}

\begin{restatable*}[Formulation of bijective maps for $\numchar=2$]{lemma}{mapformntwo}
\label{lem:map_form_n2}
Any bijective map $\vardictionary:  \{0, 1\}^2 \to \{0, 1\}^2$ can be expressed using \copyop{}, \notop{}, and \xorop{} operations. Furthermore, from $24$ possible maps for $\vardictionary$,
    \begin{enumerate}
        \item There are $6$ maps $\specialdictionary_1, \specialdictionary_2, \cdots, \specialdictionary_6$ for which characters in the output tuple can be defined by \copyop{} and \xorop{} operations on the characters in the input tuple.
        \item \textbf{Mirror maps:} For each map $\vardictionary \in \{\specialdictionary_1, \specialdictionary_2, \cdots, \specialdictionary_6\}$, there  exist mirror maps $\vardictionary_{(1)}, \vardictionary_{(2)}, \vardictionary_{(3)}$ whose output on each input tuple can be defined by selective \notop{} operations on either or both characters of the output tuple of $\vardictionary$. We call $\{\vardictionary, \vardictionary_{(1)}, \vardictionary_{(2)}, \vardictionary_{(3)} \}$ as a \textbf{mirror map set} of $\vardictionary$, in short, $\mirror(\vardictionary).$ 
    \end{enumerate}
\end{restatable*}
\begin{remark}
    Mirror map set definition is general and isn't restricted to maps $\vardictionary \in \{ \specialdictionary_1, \specialdictionary_2, \cdots, \specialdictionary_6 \}$. Because mirror maps are defined in terms of \notop{} operation, a mirror map set $\mirror(\vardictionary)$ for $\vardictionary \in \{ \specialdictionary_1, \specialdictionary_2, \cdots, \specialdictionary_6 \}$ is also equivalent to $\mirror(\vardictionary_{(i)})$ for $i \in \{1, 2, 3\}$.
\end{remark}

\begin{remarkbox}
\begin{remark}
    \cref{lem:map_form_n2} can be re-stated as follows: the $24$ possible bijective maps $\{0, 1\}^2 \to \{0, 1\}^2$ can be grouped into $6$ family of maps, each containing $4$ maps and represented by a unique map from $\{ \specialdictionary_1, \specialdictionary_2, \cdots, \specialdictionary_6 \}$. Each family is defined  by \textbf{mirror map set} of their corresponding representative map.
\end{remark}
\end{remarkbox}

\begin{restatable*}[Correlation of bijective maps for $\numchar=2$]{lemma}{corrntwo}
\label{lem:corr_n2}
 For two randomly selected maps $\vardictionary^{\alpha},\vardictionary^{\beta}: \{0,1\}^2 \to \{0,1\}^2$, the following hold true.
    \begin{enumerate}
        \item \textbf{Correlation at atleast one output character pair:} Fix an $i, j \in \{0, 1\}$. With probability $\frac{1}{3}$ w.r.t. the random selection of the maps, the $i${th} character in the output of $\vardictionary^{\alpha}$ has perfect correlation to $j$th character in the output of $\vardictionary^{\beta}$, i.e.
        \begin{align*}
            \corr(\vardictionary^{\alpha}(\cdot)_i, \vardictionary^{\beta}(\cdot)_j, \uniform(\{0, 1\}^2)) := \abs{  \Pr_{\var \sim \uniform(\{0, 1\}^2)} [\vardictionary^{\alpha}(\var)_i \ne  \vardictionary^{\beta}(\var)_j] - \Pr_{\var \sim \uniform(\{0, 1\}^2)} [\vardictionary^{\alpha}(\var)_i = \vardictionary^{\beta}(\var)_j] } = 1.
        \end{align*}

        \item  \textbf{Correlation at both output character pairs:} With probability $\frac{1}{3}$ w.r.t. the random selection of the maps, both the characters in the outputs of $\vardictionary^{\alpha},\vardictionary^{\beta}$ have perfect correlation, i.e. either one of the cases hold true
        \begin{align*}
            &\corr(\vardictionary^{\alpha}(\cdot)_1, \vardictionary^{\beta}(\cdot)_1, \uniform(\{0, 1\}^2)) = 1. \\
            &\corr(\vardictionary^{\alpha}(\cdot)_2, \vardictionary^{\beta}(\cdot)_2, \uniform(\{0, 1\}^2)) = 1. 
        \end{align*}
        or
        \begin{align*}
            &\corr(\vardictionary^{\alpha}(\cdot)_1, \vardictionary^{\beta}(\cdot)_2, \uniform(\{0, 1\}^2)) = 1. \\
            &\corr(\vardictionary^{\alpha}(\cdot)_2, \vardictionary^{\beta}(\cdot)_1, \uniform(\{0, 1\}^2)) = 1. 
        \end{align*}
        Other cases are not possible, i.e. for any $i \in \{1, 2\}$, both $\corr(\vardictionary^{\alpha}(\cdot)_i, \vardictionary^{\beta}(\cdot)_1, \uniform(\{0, 1\}^2)) = 1$ and $\corr(\vardictionary^{\alpha}(\cdot)_i, \vardictionary^{\beta}(\cdot)_2, \uniform(\{0, 1\}^2)) = 1$ can't hold true.
        
    \end{enumerate}
\end{restatable*}

\begin{remarkbox}
\begin{remark}
    Implication of \cref{lem:corr_n2} is as follows: With probability at most $\frac{1}{3}$, output of two random maps can stay correlated at either one output character pair or both pairs of output characters. This would suggest that the probability of the output of \translationtask{} under two random \strcodebook{} staying correlated should decay exponentially with the depth of the task.
\end{remark}
\end{remarkbox}

\subsection{Proof for Statistical dimension lower bound for $\numchar=2$}
\label{sec:proof_sqn2}

We repeat the lemma of interest for presentation.
\begin{thmbox}
    \sqntwo
\end{thmbox}

\begin{proof}
    We narrow our argument to the first character output of the task. 
    The proof goes through $2$ major steps. 
    \begin{enumerate}
    \item First, we show that two random instances of  $\translationtask(\depth, 2)$, defined by $2$ \strcodebook{} $\{\vardictionary^{\alpha}_{1}, \cdots, \vardictionary^{\alpha}_{\depth}\}$ and $\{\vardictionary^{\beta}_{1}, \cdots, \vardictionary^{\beta}_{\depth}\}$, will be uncorrelated with probability at least $1-(1/3)(7/9)^{\depth-1}$ w.r.t. the selection of the random \strcodebook{}. The lemma is formally given in \cref{lem:corr_random_map_n2}.
    \item Then, we can apply Lovász local lemma (\cref{lem:lovasz}) to show that we can create $2^{\Omega(\depth)}$ instances of $\translationtask(\depth, 2)$ that will have zero pairwise correlation of their output (\cref{cor:local}).
    \end{enumerate}
    By the definition of $\sqdim$ from \cref{def:sqdim}, the above observations suggest that $\sqdim(\translationtask(\depth, 2)) \ge 2^{\Omega(\depth)}$.
\end{proof}

\begin{theorem}[Lovász local lemma, theorem 1.5 in \citet{spencer1977asymptotic}]\label{lem:lovasz}
    Let $A_1, \cdots, A_k$ be events in some probability space with $\Pr(A_i) \leq p$, $1 \leq i \leq k$, such that each event is dependent on atmost $k_0$ other events. If $ep(k_0+1) < 1$, then $\Pr(  \lnot A_1 \land  \lnot A_2 \land \lnot A_3 \cdots \land \lnot A_k ) > 0$.
\end{theorem}

\begin{corollary}[Number of uncorrelated instances of $\translationtask(\depth, 2)$] \label{cor:local}
    There exists a set of $2^{\Omega(\depth)}$ instances in $\translationtask(\depth, 2)$ that are pairwise uncorrelated to each other on input distribution $\mathcal{U}(\{0,1\}^{2\depth})$.
\end{corollary}

\begin{proof}
   Suppose $\codebook_1, \cdots, \codebook_k$ represent $k$ randomly sampled \strcodebook{}. For each $1 \leq i < j \leq k$, we will denote event $A_{ij}$ as the event that the outputs of $\translationtask_{\codebook_i}$ and $\translationtask_{\codebook_j}$ are uncorrelated on input distribution $\mathcal{U}(\{0,1\}^{2\depth})$. This happens with probability $p = 1-(1/3)(7/9)^{\depth-1}$ from \cref{lem:corr_random_map_n2}. Because each event can atmost depend on atmost $k(k+1)/2 < k^2$ events (total number of events), by \cref{lem:lovasz}, if
   \begin{align*}
       e p ( k^2 + 1) < 1,
   \end{align*}
   then there exists a list of such \strcodebook{} which are pairwise uncorrelated w.r.t. the outputs of their corresponding translation tasks on input distribution $\mathcal{U}(\{0,1\}^{2\depth})$. Solving the above for $k$, we can set $k = ((ep)^{-1} - 1)^{1/2} = 2^{\Omega(\depth)}$.
\end{proof}

\subsubsection{Auxiliary lemmas}
\label{sec:proof_sqn2_auxiliary}

Here, we will prove the following primary lemma, that is used to prove \cref{lem:sqn2}.
\begin{thmbox}
\begin{lemma}\label{lem:corr_random_map_n2}
    With probability at least $1 - \frac{1}{3}\left(\frac{7}{9}\right)^{\depth-1}$ w.r.t. random map selection, the following holds true for $2$ sets of random \strdicts{} $\codebook^\alpha = \{\vardictionary^{\alpha}_{1}, \cdots, \vardictionary^{\alpha}_{\depth}\}$ and $\codebook^{\beta} = \{\vardictionary^{\beta}_{1}, \cdots, \vardictionary^{\beta}_{\depth}\}$:
    \begin{align*}
        \corr(\translationtask_{ \codebook^\alpha = \{ \vardictionary^{\alpha}_{\ell} \}_{\ell=1}^{\depth}}(\cdot)_1, \translationtask_{\codebook^\beta = \{ \vardictionary^{\beta}_{\ell} \}_{\ell=1}^{\depth}}(\cdot)_1, \uniform(\{0, 1\}^{2\depth})) = 0.
    \end{align*}
\end{lemma}
\end{thmbox}

\begin{proof}
  
    We first dive into the dependencies between characters in input and output sequences in the translation process. In \cref{lem:dependencies}, we show that at any level $1 \leq \ell \leq \depth$ of $\translationtask_\codebook$ with \strdicts{} $\codebook = \{ \dictionary_{\ell} \}_{\ell=1}^{\depth}$, the following relation holds true on an input $\inputseq \in \{0, 1\}^{2\depth}$:
    \begin{align*}
        (\intermseq_{\ell+1, 1}, \intermseq_{\ell+1, 2}) = \dictionary_{\ell} ( (\intermseq_{\ell+1, 2}, \intermseq_{\ell+1, 3}) ).
    \end{align*}

    We will use superscripts $\alpha$ and $\beta$ to differentiate the intermediate outputs of $\translationtask$ when the \strdicts{} are set as $\codebook^\alpha=\{\vardictionary^{\alpha}_{1}, \cdots, \vardictionary^{\alpha}_{\depth}\}$ and $\codebook^\beta=\{\vardictionary^{\beta}_{1}, \cdots, \vardictionary^{\beta}_{\depth}\}$ respectively.
    We will use an induction strategy to find the correlation of $\intermseq^{\alpha}_{\depth+1, 1}$ and $\intermseq^{\beta}_{\depth+1, 1}$. To do so, we will require the following variables:

    \begin{enumerate}
        \item $p_{\ell, i;j}$: For one pair of $i, j \in \{2,3\}$, this represents the probability at a level $2 \le \ell \le 1+\depth$ w.r.t. the randomness of  $\{\vardictionary^{\alpha}_{1}, \cdots, \vardictionary^{\alpha}_{\ell-1}\}$ and $\{\vardictionary^{\beta}_{1}, \cdots, \vardictionary^{\beta}_{\ell-1}\}$, that $\intermseq^{\alpha}_{\ell, i}$ and $\intermseq^{\beta}_{\ell, j}$ are perfectly correlated. Additionally, we define $p_{\ell, 1;1}$ that represents the probability at a level $1 \le \ell \le 1+\depth$ w.r.t. the randomness of  $\{\vardictionary^{\alpha}_{1}, \cdots, \vardictionary^{\alpha}_{\ell-1}\}$ and $\{\vardictionary^{\beta}_{1}, \cdots, \vardictionary^{\beta}_{\ell-1}\}$, that $\intermseq^{\alpha}_{\ell, 1}$ and $\intermseq^{\beta}_{\ell, 1}$ are perfectly correlated.

        \item  $p_{\ell, \text{both}}$: This represents the probability at a level $2 \le \ell \le 1+\depth$ w.r.t. the randomness of  $\{\vardictionary^{\alpha}_{1}, \cdots, \vardictionary^{\alpha}_{\ell-1}\}$ and $\{\vardictionary^{\beta}_{1}, \cdots, \vardictionary^{\beta}_{\ell-1}\}$, that either $\intermseq^{\alpha}_{\ell, 2}$, $\intermseq^{\beta}_{\ell, 2}$ are correlated and  $\intermseq^{\alpha}_{\ell, 3}$, $\intermseq^{\beta}_{\ell, 3}$ are correlated, or $\intermseq^{\alpha}_{\ell, 2}$, $\intermseq^{\beta}_{\ell, 3}$ are correlated and  $\intermseq^{\alpha}_{\ell, 3}$, $\intermseq^{\beta}_{\ell, 2}$ are perfectly correlated.

    \end{enumerate}
    
    There are three relations that we need to keep in mind, before we proceed with the induction proof.

    \begin{enumerate}
        \item First, the correlations are circular in nature, i.e., for $j \in \{2, 3\}$ and any integer $k$, $p_{\ell, j;j} = p_{\ell, j; (j+2k) \% 2\depth}$ (\cref{lem:circ_invar}). We will use this relation to connect $p_{\ell, 1;1}$ with $p_{\ell,3;3}$, i.e.
        \begin{align}
        \label{eq:connect11to33}
            p_{\ell, 1;1} = p_{\ell, 3;3}. 
        \end{align}
        
        \item Second, $p_{\ell, \text{both}} \leq \max_{ij} p_{\ell, i;j}.$ Intuitively, this is because getting correlations at both positions must be at most as probable as getting correlations at one of the positions. 
        \begin{align}
            p_{\ell, \text{both}} \leq \max_{i, j} p_{\ell, i;j}. \label{eq:connectbothtoindv}
        \end{align}
    
        \item Third, cross-correlation probabilities given by $p_{\ell, 2;3}$ and $p_{\ell, 3;2}$ must be $\leq \frac{1}{2} (p_{\ell, 2;2} + p_{\ell, 3;3})$ (\cref{lem:cross_corr}).
        \begin{align}
            p_{\ell, i;j}  \leq \frac{1}{2} (p_{\ell, i;i} + p_{\ell, j;j}), \text{ for any } i, j \in \{2, 3\}, i \ne j. \label{eq:connectcrosstosame}
        \end{align}
        
    \end{enumerate}

    \paragraph{Base condition:} We will use the values of the variables at $\ell=2$ as our base condition. The input for first layer of translation is same to both $\translationtask_{\codebook^\alpha}$, $\translationtask_{\codebook^\beta}$. Then, for an input $\inputseq$,
    \begin{align*}
        &(\intermseq^{\alpha}_{2, 1}, \intermseq^{\alpha}_{2, 2}) = \vardictionary^{\alpha}_1((\intermseq_{1, 2}, \intermseq_{1, 3})) \\
        &
        (\intermseq^{\beta}_{2, 1}, \intermseq^{\beta}_{2, 2}) = \vardictionary^{\beta}_1((\intermseq_{1, 2}, \intermseq_{1, 3})).
    \end{align*}
    We can then use \cref{lem:corr_n2} to show that
    \begin{align*}
        &p_{2, i;j} = \frac{1}{3}, \text{for all } i, j \in \{1, 2\},\\
        &p_{2, \text{both}} = \frac{1}{3}.
    \end{align*}

    \paragraph{Connecting the variables at $\ell$ and $\ell+1$:}

    We connect $p_{\ell+1, 2;2}$ to $p_{\ell, 2;3}$, $p_{\ell, 3;2}$, $p_{\ell, 3;3}$ and $p_{\ell, \text{both}}$. This will undergo a case by case analysis.

    \begin{enumerate}
        \item First, if under \strdicts{} $\vardictionary^{\alpha}_1, \cdots, \vardictionary^{\alpha}_{\ell-1}$ and $\vardictionary^{\beta}_1, \cdots, \vardictionary^{\beta}_{\ell-1}$, if 
        \begin{align}
             \text{(either) } & \intermseq^{\alpha}_{\ell, 2} \text{ and } \intermseq^{\beta}_{\ell, 2} \text{ are perfectly correlated } \text{, and } \intermseq^{\alpha}_{\ell, 3} \text{ and } \intermseq^{\beta}_{\ell, 3} \text{ are perfectly correlated } \nonumber\\
             \text{(or) } & \intermseq^{\alpha}_{\ell, 2} \text{ and } \intermseq^{\beta}_{\ell, 3} \text{ are perfectly correlated } \text{, and } \intermseq^{\alpha}_{\ell, 3} \text{ and } \intermseq^{\beta}_{\ell, 2} \text{ are perfectly correlated }, \label{cond:induct_plboth}
        \end{align}
        we can use \cref{lem:corr_compose} 
        to show that with probability $\frac{1}{3}$ w.r.t. the selection of $\vardictionary^{\alpha}_{\ell}$ and $\vardictionary^{\beta}_{\ell}$, $\intermseq^{\alpha}_{\ell+1, 2}$ and $\intermseq^{\beta}_{\ell+1, 2}$ will be correlated. Conditions necessary for \cref{lem:corr_compose}, i.e. uniform distribution of the characters in the input sequence, are shown to hold true in \cref{lem:uniform_n2_dist}.
        The probability w.r.t. the selection of \strdicts{} $\vardictionary^{\alpha}_1, \cdots, \vardictionary^{\alpha}_{\ell-1}$ and $\vardictionary^{\beta}_1, \cdots, \vardictionary^{\beta}_{\ell-1}$ such  that \cref{cond:induct_plboth} holds true is given by the variable $p_{\ell, \text{both}}$.
        
        \item On the other hand, if there exists a pair $i, j \in \{2, 3\}$, such that under \strdicts{} $\vardictionary^{\alpha}_1, \cdots, \vardictionary^{\alpha}_{\ell-1}$ and $\vardictionary^{\beta}_1, \cdots, \vardictionary^{\beta}_{\ell-1}$, $\intermseq^{\alpha}_{\ell, i}$ and $\intermseq^{\beta}_{\ell, j}$ are perfectly correlated, then with probability $\frac{1}{9}$ w.r.t. the selection of $\vardictionary^{\alpha}_{\ell}$ and $\vardictionary^{\beta}_{\ell}$, $\intermseq^{\alpha}_{\ell+1, 2}$ and $\intermseq^{\beta}_{\ell+1, 2}$ will be correlated. We again refer to \cref{lem:corr_compose} for this statement. Probability that this condition happens under \strdicts{} $\vardictionary^{\alpha}_1, \cdots, \vardictionary^{\alpha}_{\ell-1}$ and $\vardictionary^{\beta}_1, \cdots, \vardictionary^{\beta}_{\ell-1}$ is given  by the variable $p_{\ell, i;j}$.

        \item If none of the above situation occurs, then for all pairs $i, j \in \{2, 3\}$,  $\intermseq^{\alpha}_{\ell, i}$ and $\intermseq^{\beta}_{\ell, j}$ are uncorrelated, and so, by \cref{lem:corr_n2}, correlation between $\intermseq^{\alpha}_{\ell+1, 2}$ and $\intermseq^{\beta}_{\ell+1, 2}$ will stay $0$ for any choice of $\vardictionary^{\alpha}_{\ell}$ and $\vardictionary^{\beta}_{\ell}$. 
        
    \end{enumerate}

    Thus, combining the $3$ cases, we must have
    \begin{align}
        &p_{\ell+1, 2;2} \leq \frac{1}{3} p_{\ell, \text{both}} + \frac{1}{9}\sum_{i, j \in \{2, 3\}} p_{\ell, i;j} \label{eq:step1_conditional_induction}\\
        &\leq \frac{1}{3} p_{\ell, \text{both}} + \frac{4}{9} \max_{i, j \in \{2, 3\}} p_{\ell, i;j} \nonumber \\
        &\leq \frac{1}{3} \max_{i, j \in \{2, 3\}} p_{\ell, i;j} + \frac{4}{9} \max_{i, j \in \{2, 3\}} p_{\ell, i;j} = \frac{7}{9} \max_{i, j \in \{2, 3\}} p_{\ell, i;j} \label{eq:step2_connectbothtoind}
    \end{align}
     
    Reasoning for each step is as follows:
    \begin{enumerate}
        \item \cref{eq:step1_conditional_induction} follows from conditional probability computations using the $3$ cases that we discussed before. 
        \item \cref{eq:step2_connectbothtoind} uses \cref{eq:connectbothtoindv} to connect $p_{\ell, \text{both}}$ to $p_{\ell, i;j}$.
    \end{enumerate}

    We can give the same inequality for $p_{\ell+1, 1;1}$. As $p_{\ell+1, 1;1} = p_{\ell+1, 3;3}$ from \cref{eq:connect11to33} and $p_{\ell, 2;3}$ and $p_{\ell, 3;2}$ must be atmost the average of $p_{\ell, 2;2}$ and $p_{\ell, 3;3}$ (\cref{eq:connectcrosstosame}), we can write
    \begin{align*}
        \max_{i, j \in \{2, 3\}} p_{\ell+1, i;j} \leq \frac{7}{9} \max_{i, j \in \{2, 3\}} p_{\ell, i;j}. 
    \end{align*}
    This implies that the probability of correlation at any output character under the two random \strcodebook{} decays at the rate of $\frac{7}{9}$. Solving the recurrence will give:
    \begin{align*}
         \max_{i, j \in \{2, 3\}} p_{\depth+1, i;j} \leq  \rbr{\frac{7}{9}}^{\depth-1} p_{2, i;j} =  \rbr{\frac{7}{9}}^{\depth-1} \frac{1}{3}.
    \end{align*}

    As $p_{\depth+1, 3;3}$ should be equal to $p_{\depth+1, 1;1}$, this implies that with probability at least $1 - p_{\depth+1, 3;3} = 1 - \rbr{\frac{7}{9}}^{\depth-1} \frac{1}{3}$, the following must hold true:
    \begin{align*}
        \corr(\translationtask_{ \codebook^\alpha}(\cdot)_1, \translationtask_{\codebook^\beta}(\cdot)_1, \uniform(\{0, 1\}^{2\depth})) = 0.
    \end{align*}

\end{proof}

\begin{lemma}\label{lem:corr_compose}
    Suppose $f, g: \{0, 1\}^{k} \to \{0, 1\}^{2}$ denote $2$ functions for some arbitrary $k$, satisfying conditions for uniformity in output distribution, i.e.
    \begin{align*}
        \mathbb{E}_{x \sim \uniform(\{0, 1\}^k)} f(x)_j = 1/2, \quad& \mathbb{E}_{x \sim \uniform(\{0, 1\}^k)} g(x)_j = 1/2 \\
        \mathbb{E}_{x \sim \uniform(\{0, 1\}^k)} f(x)_1 \oplus f(x)_2 = 1/2,\quad& \mathbb{E}_{x \sim \uniform(\{0, 1\}^k)} g(x)_1 \oplus g(x)_2 = 1/2
    \end{align*}
    for all $j \in \{1, 2\}$. Then, the following holds true:
    \begin{enumerate}
        \item \textbf{Both output character pairs have correlations between $f$ and $g$:} If 
        \begin{align*}
            \text{(either) } &\corr(f(\cdot)_1, g(\cdot)_1, \uniform(\{0, 1\}^k)) = 1 \text{ and } \corr(f(\cdot)_2, g(\cdot)_2, \uniform(\{0, 1\}^k)) = 1 \\
            \text{(or) } &\corr(f(\cdot)_1, g(\cdot)_2, \uniform(\{0, 1\}^k)) = 1 \text{ and } \corr(f(\cdot)_2, g(\cdot)_1, \uniform(\{0, 1\}^k)) = 1,
        \end{align*}
       then for two randomly picked \strdicts{} $\vardictionary^{\alpha}, \vardictionary^{\beta}$,
        \begin{enumerate}
            \item \textbf{Correlation of an output character pair of $\vardictionary^{\alpha} ( f(\cdot) )$ and $\vardictionary^{\beta} ( g(\cdot) )$:} Fix an $i, j \in \{1, 2\}$. With probability  $1/3$ w.r.t. the random selections of $\vardictionary^{\alpha}, \vardictionary^{\beta}$, 
            \begin{align*}
                \corr( \vardictionary^{\alpha} ( f(\cdot) )_i,  \vardictionary^{\alpha} ( g(\cdot) )_j, \uniform(\{0, 1\}^k) ) = 1
            \end{align*}

            \item  \textbf{Correlation of both character pairs in output  of $\vardictionary^{\alpha} ( f(\cdot) )$ and $\vardictionary^{\beta} ( g(\cdot) )$:} With probability  $1/3$ w.r.t. the random selections of $\vardictionary^{\alpha}, \vardictionary^{\beta}$, one of the following two conditions hold true.
            \begin{align*}
                \text{(either) } & \corr( \vardictionary^{\alpha} ( f(\cdot) )_1,  \vardictionary^{\alpha} ( g(\cdot) )_1, \uniform(\{0, 1\}^k) ) = 1 \text{ and } \corr( \vardictionary^{\alpha} ( f(\cdot) )_2,  \vardictionary^{\alpha} ( g(\cdot) )_2, \uniform(\{0, 1\}^k) ) = 1 \\
                \text{(or) } & \corr( \vardictionary^{\alpha} ( f(\cdot) )_1,  \vardictionary^{\alpha} ( g(\cdot) )_2, \uniform(\{0, 1\}^k) ) = 1 \text{ and } \corr( \vardictionary^{\alpha} ( f(\cdot) )_2,  \vardictionary^{\alpha} ( g(\cdot) )_1, \uniform(\{0, 1\}^k) ) = 1 \\
            \end{align*}
          
        \end{enumerate}

        \item \textbf{Only a pair of characters have correlations between $f$ and $g$:} If there exists only one pair $i, j \in \{0, 1\}$ such that  
        \begin{align*}
            \corr(f(\cdot)_i, g(\cdot)_j, \uniform(\{0, 1\}^k) = 1, 
        \end{align*}
        and for all other $i', j'$, $\corr(f(\cdot)_{i'}, g(\cdot)_{j'}, \uniform(\{0, 1\}^k)) = 0$, then
        \begin{enumerate}
            \item Fix any $i', j' \in \{1, 2\}$. With probability $\frac{1}{9}$ w.r.t. the random selections of $\vardictionary^{\alpha}, \vardictionary^{\beta}$,
            \begin{align*}
                \corr(\vardictionary^{\alpha}(f(\cdot))_{i'}, \vardictionary^{\beta}(g(\cdot))_{j'}, \uniform(\{0, 1\}^k)) = 1.
            \end{align*}
            If the above condition holds true, for all other $\bar{i}, \bar{j}$ pairs, \begin{align*}
                \corr(\vardictionary^{\alpha}(f(\cdot))_{\bar{i}}, \vardictionary^{\beta}(g(\cdot))_{\bar{j}}, \uniform(\{0, 1\}^k)) = 0.
            \end{align*}
        \end{enumerate}

        \item \textbf{No correlations between $f$ and $g$:} If for all pairs $i, j \in \{1,2\}$, $\corr(f(\cdot)_{i}, g(\cdot)_{j}, \uniform(\{0, 1\}^k)) = 0$, then for all pairs $i', j' \in \{1,2\}$,
        \begin{align*}
            \corr(\vardictionary^{\alpha}(f(\cdot))_{i'}, \vardictionary^{\beta}(g(\cdot))_{j'}, \uniform(\{0, 1\}^k)) = 0.
        \end{align*}        
        
    \end{enumerate}

\end{lemma}

\begin{proof}
    We will prove each case separately. 

    \begin{enumerate}
        \item \textbf{Both output character pairs have correlations between $f$ and $g$:} We will consider the case when $\corr(f(\cdot)_1, g(\cdot)_1, \uniform(\{0, 1\}^k)) = 1 \text{ and } \corr(f(\cdot)_2, g(\cdot)_2, \uniform(\{0, 1\}^k)) = 1$, proof for the other case is similar. Then, there are $4$ cases possible.
        \begin{align*}
            \text{Subcase 1: }& f(x)_1 = g(x)_1, f(x)_2 = g(x)_2 \quad \text{ for all } x \in \{0, 1\}^k \\
            \text{Subcase 2: }& f(x)_1 = \notop(g(x)_1), f(x)_2 = g(x)_2 \quad \text{ for all } x \in \{0, 1\}^k \\
            \text{Subcase 3: }& f(x)_1 = g(x)_1, f(x)_2 = \notop(g(x)_2) \quad \text{ for all } x \in \{0, 1\}^k \\
            \text{Subcase 4: }& f(x)_1 = \notop(g(x)_1), f(x)_2 = \notop(g(x)_2) \quad \text{ for all } x \in \{0, 1\}^k 
        \end{align*}
        Because $f(\cdot)$ and $g(\cdot)$ are inputs to the \strdicts{} $\vardictionary^{\alpha}$ and $\vardictionary^{\beta}$ respectively and $f(\cdot)$ and $g(\cdot)$ are related by selective \notop{} operations in their outputs in each of the possible 4 subcases, the output of  $\vardictionary^{\beta}(g(\cdot))$ can be replaced by  $\Tilde{\vardictionary}^{\beta}(f(\cdot))$ for a mirror map $\Tilde{\vardictionary}^{\beta}$ of $\vardictionary^{\beta}$. We showcase this formally for one subcase, say subcase 2; argument for other subcases are similar.

        \textbf{Suppose Subcase 2 is true:} Then, for any two \strdicts{} $\vardictionary^{\alpha}$ and $\vardictionary^{\beta}$, suppose $\Tilde{\vardictionary}^{\beta}$ denotes a mirror map of $\vardictionary^{\beta}$ such that 
        \begin{align*}
            \Tilde{\vardictionary}^{\beta}(x)_1 = \notop(\vardictionary^{\alpha}(x)_1), \text{ } \Tilde{\vardictionary}^{\beta}(x)_2 = \vardictionary^{\alpha}(x)_2, \quad \text{ for all } x \in \{0, 1\}^2.
        \end{align*}
        Such a map exists by \cref{lem:map_form_n2}. Then, for a pair $i, j \in \{1, 2\}$,
        \begin{align}
            &\corr(\vardictionary^{\alpha} (f(\cdot))_i, \vardictionary^{\beta} (g(\cdot))_j, \uniform(\{0, 1\}^k))  \nonumber \\
            & = \corr(\vardictionary^{\alpha} ( (f(\cdot)_1, f(\cdot)_2 ))_i, \vardictionary^{\beta} ((g(\cdot)_1, g(\cdot)_2 ))_j, \uniform(\{0, 1\}^k)) \label{eq:bothbitcorr_step1}\\
            & = \corr(\vardictionary^{\alpha} ( (f(\cdot)_1, f(\cdot)_2 ))_i, \vardictionary^{\beta} ((\notop(f(\cdot)_1), f(\cdot)_2 ))_j, \uniform(\{0, 1\}^k)) \label{eq:bothbitcorr_step2} \\
            & =  \corr(\vardictionary^{\alpha} ( (f(\cdot)_1, f(\cdot)_2 ))_i, \Tilde{\vardictionary}^{\beta} (( f(\cdot)_1, f(\cdot)_2 ))_j, \uniform(\{0, 1\}^k)) \label{eq:bothbitcorr_step3} \\
            & = \corr(\vardictionary^{\alpha} (\cdot)_i, \Tilde{\vardictionary}^{\beta} (\cdot)_j, \uniform(\{0, 1\}^2)). \label{eq:bothbitcorr_step4}
        \end{align}

        Reasoning for each step is as follows.
        \begin{itemize}
            \item Because the outputs of $f$ and $g$ are a $2$-tuple on any input, \cref{eq:bothbitcorr_step1} simply replaces $f(x)$ (similarly $g(x)$)  as $(f(x)_1, f(x)_2)$ for any input $x$.
            \item \cref{eq:bothbitcorr_step2} uses the relation between the outputs of $f$ and $g$ when subcase 2 is true.
            \item \cref{eq:bothbitcorr_step3} simply replaces $\vardictionary^{\beta}$ with $\Tilde{\vardictionary}^{\beta}$ due to their relation as mirror \strdicts{}.
            \item Finally, because of the assumption on the output of $f$, i.e. $\mathbb{E}_{x \sim \uniform(\{0, 1\}^k)} f(x)_j = 1/2, \mathbb{E}_{x \sim \uniform(\{0, 1\}^k)} f(x)_1 \oplus f(x)_2 = 1/2$, the distribution of outputs of $f$ is identical to $\uniform(\{0, 1\}^2).$  
        \end{itemize}

        Hence, $\corr(\vardictionary^{\alpha} (f(\cdot))_i, \vardictionary^{\beta} (g(\cdot))_j, \uniform(\{0, 1\}^k))$ boils down to $\corr(\vardictionary^{\alpha} (\cdot)_i, \Tilde{\vardictionary}^{\beta} (\cdot)_j, \uniform(\{0, 1\}^2))$. We can then use \cref{lem:corr_n2} to show the probabilities of each of the desired conditions.

        \item \textbf{Only a pair of characters in output have correlations between $f$ and $g$:} For typographical simplicity, consider the case when 
        \begin{align*}
            \corr(f(\cdot)_1, g(\cdot)_1, \uniform(\{0, 1\}^k)) = 1, 
        \end{align*}
        and for all pairs $i', j'$ with atleast one of them being not equal to $1$, $\corr(f(\cdot)_{i'}, g(\cdot)_{j'}, \uniform(\{0, 1\}^k)) = 0$. The argument when the condition holds for other $i, j$ pairs can be similarly handled. Then, there are $2$ cases possible.
        \begin{align*}
            \text{Subcase 1: }& f(x)_1 = g(x)_1, \text{ for all } x \in \{0, 1\}^k \\
            \text{Subcase 2: }& f(x)_1 = \notop(g(x)_1), \text{ for all } x \in \{0, 1\}^k, 
        \end{align*}
        while in each of these pairs, $(f(\cdot)_2, g(\cdot)_2)$, $(f(\cdot)_2, g(\cdot)_1)$, and $(f(\cdot)_1, g(\cdot)_2)$, the output bits of $f$ and $g$ are completely independent of each other. We only consider subcase 1; subcase 2 can be handled using mirror \strdicts{} similar to the proof for case 1 (in particular with \cref{eq:bothbitcorr_step3}). 

        For simplicity, we will argue for $i', j' = 1, 1$; the argument about other $i', j'$ pairs are similar. The proof will follow by two steps, 
        \begin{itemize}
            \item \textbf{Step (a):} We first argue about the probability with which $\corr(\vardictionary^{\alpha} (f(\cdot))_1, \vardictionary^{\beta} (g(\cdot))_1, \uniform(\{0, 1\}^k))$ is equal to $1$,
            \item \textbf{Step (b):} We  then argue that if the condition in (a) holds true, then $\corr(\vardictionary^{\alpha} (f(\cdot))_{i'}, \vardictionary^{\beta} (g(\cdot))_{j'}, \uniform(\{0, 1\}^k))$ must be $0$ for any other pair $i', j'$ where atleast one of them is not equal to $1$. 
        \end{itemize}
        
        \textbf{Step (a):} For simplicity, we will assume that $\vardictionary^{\alpha}, \vardictionary^{\beta} \in \{\specialdictionary_1, \specialdictionary_2, \cdots, \specialdictionary_6\}$ defined in \cref{tab:truthtable_permmap}; other cases can be handled similarly as \cref{lem:corr_distinctmirrormap}.

        \begin{align}
            &\corr(\vardictionary^{\alpha} (f(\cdot))_1, \vardictionary^{\beta} (g(\cdot))_1, \uniform(\{0, 1\}^k))  \nonumber \\
            & = \corr(\vardictionary^{\alpha} ( (f(\cdot)_1, f(\cdot)_2 ))_1, \vardictionary^{\beta} ((g(\cdot)_1, g(\cdot)_2 ))_1, \uniform(\{0, 1\}^k)) \label{eq:singlebitcorr_step1}\\
            & = \corr(\vardictionary^{\alpha} ( (f(\cdot)_1, f(\cdot)_2 ))_1, \vardictionary^{\beta} ((f(\cdot)_1, g(\cdot)_2 ))_1, \uniform(\{0, 1\}^k)) \label{eq:singlebitcorr_step2} \\
            \begin{split}
            &= \left| \Pr_{ x \sim \uniform(\{0, 1\}^k) } \left[ \vardictionary^{\alpha} ( (f(x)_1, f(x)_2 ))_1 = \vardictionary^{\beta} ( (f(x)_1, g(x)_2 ))_1 \right] \right.\\
            &\qquad -\left. \Pr_{ x \sim \uniform(\{0, 1\}^k } \left[ \vardictionary^{\alpha} ( (f(x)_1, f(x)_2 ))_1 \ne \vardictionary^{\beta} ( (f(x)_1, g(x)_2 ))_1 \right] \right| \label{eq:singlebitcorr_step3} 
            \end{split}\\
            &=  \abs{ \Pr_{ x \sim \uniform(\{0, 1\}) } \left[   \Pr_{ y, y'  \sim \uniform(\{0, 1\}^2) }  \left[ \vardictionary^{\alpha} ( (x, y))_1 = \vardictionary^{\beta} ( (x, y'))_1 \right]  -    \Pr_{ y, y'  \sim \uniform(\{0, 1\}^2) }  \left[ \vardictionary^{\alpha} ( (x, y))_1 = \vardictionary^{\beta} ( (x, y'))_1 \right] \right] }  \label{eq:singlebitcorr_step4}
        \end{align}
        The reasoning for each step is as follows:
        \begin{itemize}
            \item  Because the outputs of $f$ and $g$ are a $2$-tuple on any input, \cref{eq:singlebitcorr_step1} simply replaces $f(x)$ (similarly $g(x)$)  as $(f(x)_1, f(x)_2)$ for any input $x$.
            \item \cref{eq:singlebitcorr_step2} uses the relation between the outputs of $f$ and $g$ when subcase 1 is true.
            \item  \cref{eq:bothbitcorr_step3} simply writes the definition of correlation.
            \item Finally, because of the assumption on the output of $f$, i.e. $\mathbb{E}_{x \sim \uniform(\{0, 1\}^k)} f(x)_j = 1/2, \mathbb{E}_{x \sim \uniform(\{0, 1\}^k)} f(x)_1 \oplus f(x)_2 = 1/2$, the distribution of outputs of $f$ can be shown to be identical to $\uniform(\{0, 1\}^2).$  
        \end{itemize}

        From the definitions of $\vardictionary^{\alpha}, \vardictionary^{\beta} \in \{\specialdictionary_1, \specialdictionary_2, \cdots, \specialdictionary_6\}$  in \cref{tab:truthtable_permmap}, the first character in the outputs of $\vardictionary^{\alpha}$ and $\vardictionary^{\beta}$ can be expressed using $3$ operators on the input characters, which are $\copyop{}_1$, $\copyop{}_2$, and $\xorop{}$. These operators are independently selected for $\vardictionary^{\alpha}$ and $\vardictionary^{\beta}$, as they are randomly picked from these $6$ possibilities. Formally, for any tuple of variables $(x, y)$ and $(x, y')$, 
        \begin{align*}
            &\copyop{}_1(x, y) = x, \quad \copyop{}_2(x, y) = y, \quad \xorop{}(x, y) = x \oplus y, \\
            &\copyop{}_1(x, y') = x, \quad \copyop{}_2(x, y') = y', \quad \xorop{}(x, y') = x \oplus y'.
        \end{align*}
        However, because $x,y,y'$ are independent variables, only operations $\copyop{}_1(x, y)$ and $\copyop{}_1(x, y')$ for $(x,y,y') \sim \uniform(\{0, 1\}^3)$ will be correlated. This can be verified by writing the definition of correlation and using $\copyop{}_1$ for the first character output of $\vardictionary^{\alpha}$ and $\vardictionary^{\beta}$:
        \begin{align*}
            &\abs{ \Pr_{ x \sim \uniform(\{0, 1\}) } \left[   \Pr_{ y, y'  \sim \uniform(\{0, 1\}^2) }  \left[ \vardictionary^{\alpha} ( (x, y))_1 = \vardictionary^{\beta} ( (x, y'))_1 \right]  -    \Pr_{ y, y'  \sim \uniform(\{0, 1\}^2) }  \left[ \vardictionary^{\alpha} ( (x, y))_1 = \vardictionary^{\beta} ( (x, y'))_1 \right] \right] } 
            \\&= \abs{ \Pr_{ x \sim \uniform(\{0, 1\}) } \left[ \Pr_{ y, y'  \sim \uniform(\{0, 1\}^2) }  \left[ \copyop{}_1 (x, y) = \copyop{}_1 ( x, y') \right]  -   \Pr_{ y, y'  \sim \uniform(\{0, 1\}^2)  }   \left[ \copyop{}_1 (x, y) \ne \copyop{}_1 (x, y') \right] \right] }  = 1,
        \end{align*}
        as the first term is $1$ and the second term is $0$. However, if you pick any other pair of operations, the correlation will be $0$. We demonstrate by using $\copyop{}_1$ and $\xorop{}$ for $\vardictionary^{\alpha}$ and $\vardictionary^{\beta}$ respectively.
        \begin{align*}
             &\abs{ \Pr_{ x \sim \uniform(\{0, 1\}) } \left[   \Pr_{ y, y'  \sim \uniform(\{0, 1\}^2) }  \left[ \vardictionary^{\alpha} ( (x, y))_1 = \vardictionary^{\beta} ( (x, y'))_1 \right]  -    \Pr_{ y, y'  \sim \uniform(\{0, 1\}^2) }  \left[ \vardictionary^{\alpha} ( (x, y))_1 = \vardictionary^{\beta} ( (x, y'))_1 \right] \right] } \\&=
            \abs{ \Pr_{ x \sim \uniform(\{0, 1\}) } \left[ \Pr_{ y, y' \sim \uniform(\{0, 1\}^2)} \left[ \copyop{}_1 (x, y) = \xorop{} ( x, y') \right]  -   \Pr_{ y, y' \sim \uniform(\{0, 1\}^2)} \left[ \copyop{}_1 (x, y) \ne \xorop{} (x, y') \right] \right] } \\
            &= \abs{ \Pr_{ x \sim \uniform(\{0, 1\}) } \left[ \Pr_{ y, y' \sim \uniform(\{0, 1\}^2)} \left[ x = x \oplus y' \right]  -   \Pr_{ y, y' \sim \uniform(\{0, 1\}^2)} \left[ x \ne x \oplus y' \right] \right] } = 0,
        \end{align*}
        as both terms are equal to $\frac{1}{2}$ in the final step. By a counting argument, one can reason that the probability of  $corr(\vardictionary^{\alpha} (f(\cdot))_1, \vardictionary^{\beta} (g(\cdot))_1, \uniform(\{0, 1\}^k))$ being $1$ is equal to the probability of $\copyop{}_1$ being selected to define the first characters of both $\vardictionary^{\alpha}$ and $\vardictionary^{\beta}$, which will be equal to $\frac{1}{9}$.

        \textbf{Step (b):} Now, say we have selected a pair $\vardictionary^{\alpha}$ and $\vardictionary^{\beta}$ such that $\corr(\vardictionary^{\alpha} (f(\cdot))_1, \vardictionary^{\beta} (g(\cdot))_1, \uniform(\{0, 1\}^k)) = 1$. From the proof of step (a), this is only possible when $\copyop{}_1$ was selected to define the first characters in the outputs of $\vardictionary^{\alpha}$ and $\vardictionary^{\beta}$. Any other operation pairs for defining the first characters in the outputs of $\vardictionary^{\alpha}$ and $\vardictionary^{\beta}$ would have meant correlation to be $0$.

        We can use this same argument to show that for any other pair $i', j'$ where atleast one of them is not equal to $1$.  $\corr(\vardictionary^{\alpha} (f(\cdot))_{i'}, \vardictionary^{\beta} (g(\cdot))_{j'}, \uniform(\{0, 1\}^k))$ must be $0$. This is because once $\copyop{}_1$ has been used to define the operation for the first character, it can't be used to define the operation for the second character (please refer at the truth tables for $\{\specialdictionary_1, \cdots, \specialdictionary_6\}$ in \cref{tab:truthtable_permmap}). Following a similar argument, we  can then show that correlation between output characters $\vardictionary^{\alpha}(f(\cdot))_{i'}, \vardictionary^{\beta} (g(\cdot))_{j'}$ will be $0$, as $\copyop{}_1$ can't be used to define at least one of these characters.

        \item The third case is when none of the above conditions hold true. That is, for any pair $i, j \in \{1, 2\}$,  
        \begin{align*}
            \corr(f(\cdot)_i, g(\cdot)_j, \uniform(\{0, 1\}^k)) = 0. 
        \end{align*}
        In such case, following similar arguments as case 1 and 2, one can show that for any $i, j \in \{1, 2\}$, for any choice of $\vardictionary^{\alpha}$ and $\vardictionary^{\beta}$:
        \begin{align*}
            &\corr(\vardictionary^{\alpha} (f(\cdot))_i, \vardictionary^{\beta} (g(\cdot))_j, \uniform(\{0, 1\}^k)) \\&
            = \left| \Pr_{ x, y  \sim \uniform(\{0, 1\}^2) } \Pr_{ x', y'  \sim \uniform(\{0, 1\}^2) }      \left[ \vardictionary^{\alpha} ( (x, y))_i = \vardictionary^{\beta} ( (x', y'))_j \right]\right.  \\
            &\qquad -  \left.  \Pr_{ x, y  \sim \uniform(\{0, 1\}^2) } \Pr_{ x', y'  \sim \uniform(\{0, 1\}^2) } \left[ \vardictionary^{\alpha} ( (x, y))_i \ne \vardictionary^{\beta} ( (x', y'))_j \right]\right| \\&
            = 0,
        \end{align*}        
        due to independence of inputs to $\vardictionary^{\alpha}$ and $\vardictionary^{\beta}$.
    \end{enumerate}
\end{proof}

\begin{lemma}\label{lem:dependencies}
    At any step $i$ of $\translationtask(\depth, 2)$ with \strdicts{} $\{ \dictionary_{\ell} \}_{\ell=1}^{\depth}$, the following relation holds true for the intermediate outputs $\{ \intermseq_i \}_{i=2}^{\depth+1}$ on an input $\inputseq \in \{0, 1\}^{\inputlength}$:
    \begin{align*}
        (\intermseq_{i+1, 1}, \intermseq_{i+1, 2}) = \dictionary_i ( (\intermseq_{i, 2}, \intermseq_{i, 3}) ).
    \end{align*}
\end{lemma}

\begin{proof}
    $\translationtask(\depth, 2)$ has $2$ primary steps: \circularshift{} and \translate{}. By \circularshift{}, first, we first get sequence $\shiftintermseq_i$, where for any $j\in [\inputlength]$ we have $\shiftintermseq_{i, j} = \intermseq_{i, (j+1) \% \inputlength}$. After $\translate{}$ step, 
    $$
        (\intermseq_{i+1, 1}, \intermseq_{i+1, 2}) = \dictionary_i ( (\shiftintermseq_{i, 1}, \shiftintermseq_{i, 2}) ) =  \dictionary_i ( (\intermseq_{i, 2}, \intermseq_{i, 3}) ).
    $$
\end{proof}

\begin{lemma}\label{lem:uniform_n2_dist}
    At any step $i$ of $\translationtask(\depth, 2)$ with \strdicts{} $\{ \dictionary_{\ell} \}_{\ell=1}^{\depth}$, the following conditions hold true for the intermediate output $\intermseq_i$.
    \begin{align*}
        &\mathbb{E}_{\inputseq \sim \uniform(\{0, 1\}^{\inputlength})} \intermseq_{i, j} = \frac{1}{2}, \text{ for all } 1 \le j \le \inputlength \\
        &\mathbb{E}_{\inputseq \sim \uniform(\{0, 1\}^{\inputlength})} \intermseq_{i, j} \oplus \intermseq_{i, j'} = \frac{1}{2}, \text{ for all } j \ne j'.
    \end{align*}
    The above conditions are equivalent to showing that $\intermseq_{i, j}$ behaves like a uniformly random boolean variable, independent of any other character $\intermseq_{i, j'}$ for all coordinates $j' \ne j$ .
\end{lemma}

\begin{proof}
    The proof will follow by induction on the output of the translation task at each step. We will show the result for coordinate $j=1$ in  $\intermseq_{i}$; similar argument holds for other coordinates $j$.

    \paragraph{Base condition:} At layer $i=1$, the $\intermseq_{1}$ represents the input sequence from $\uniform(\{0, 1\}^{\inputlength})$. By definition of uniform distribution, the conditions hold true for the input.

    \paragraph{Induction step: Argument for general $i>1$:} Suppose the conditions are true for all layers $1 \le \ell < i$. Then for layer $i$, we will provide an argument for the condition to hold true for $j=1$ and $j=2$, arguments for other $j$s will extend similarly. By \cref{lem:dependencies},
    $$
    (\intermseq_{i, 1}, \intermseq_{i, 2}) = \dictionary_i ( (\intermseq_{i-1, 2}, \intermseq_{i-1, 3}) ).
    $$

    From \cref{lem:map_form_n2}, we have each output character can be defined in terms of \copyop{}, \xorop{}, and \notop{} operations on input characters. Then, the relations for $\intermseq_{i, 1}, \intermseq_{i, 2}$ can be computed as follows:

    \begin{itemize}
        \item If a map $\vardictionary^{\alpha}$ is selected from $\mirror(\vardictionary)$ for some $\vardictionary \in \{ \specialdictionary_1, \cdots, \specialdictionary_6 \}$, then one can show that the conditions hold true for $\vardictionary_i = \vardictionary^{\alpha}$ if the conditions hold true for $\vardictionary_i = \vardictionary$. This follows because the output of $\vardictionary^{\alpha}$ follows from the output of $\vardictionary$ by selective \notop{} operations to the output of $\vardictionary$. As \notop{} operation won't change the expected values of a variable which behaves like a random boolean variable, the argument follows. 
        
        \item Now, we show that for $\vardictionary_i = \vardictionary$ for some $\vardictionary \in \{ \specialdictionary_1, \cdots, \specialdictionary_6 \}$, the conditions hold true. For these \strdicts{}, the output characters are defined by \copyop{} and \xorop{} operations. We use the definitions of these \strdicts{} from \cref{tab:truthtable_permmap} and show the expected values  $\intermseq_{i, 1}$ in terms of expected values of $\intermseq_{i-1, 1}$ and $\intermseq_{i-1, 2}$, and the expected values  $\intermseq_{i, 1} \oplus \intermseq_{i, 2}$ in terms of expected values of $\intermseq_{i-1, 1}$ and $\intermseq_{i-1, 2}$ in \cref{tab:expectation_and_oplus}.
    \end{itemize}
    Both of the above arguments then can be combined to show that 
    \begin{align*}
        &\mathbb{E}_{\inputseq \sim \uniform(\{0, 1\}^{\inputlength})} \intermseq_{i, j} = \frac{1}{2}, \text{ for  } j \in \{1, 2\}\\&
        \mathbb{E}_{\inputseq \sim \uniform(\{0, 1\}^{\inputlength})} \intermseq_{i, 1} \oplus \intermseq_{i, 2} = \frac{1}{2}.
    \end{align*}
    We can similarly extend the argument for expectation of each individual character to other positions $j > 2$, i.e. $\mathbb{E}_{\inputseq \sim \uniform(\{0, 1\}^{\inputlength})} \intermseq_{i, j} = \frac{1}{2}$ for all other possible $j$s. We can also similarly extend the argument for joint expectation of two consecutive characters $ \intermseq_{i, 2k+1} \oplus \intermseq_{i, 2k+2} $ for any general $k$. 
    
    The remaining argument will be to show that characters that don't form a consecutive $2$-tuple are going to be independent of each other as well. Consider any two $2$-tuples $(\intermseq_{i, 2k+1}, \intermseq_{i, 2k+2})$ and $(\intermseq_{i, 2k'+1}, \intermseq_{i, 2k'+2})$, with $k \ne k'$. By adapting \cref{lem:dependencies}, one can show that
    \begin{align*}
    &(\intermseq_{i, 2k+1}, \intermseq_{i, 2k+2}) = \dictionary_i  ((\intermseq_{i-1, 2k+2}, \intermseq_{i, 2k+3}))  \\&
    (\intermseq_{i, 2k'+1}, \intermseq_{i, 2k'+2}) = \dictionary_i  ((\intermseq_{i-1, 2k'+2}, \intermseq_{i, 2k'+3}))
    \end{align*}
    The primary thing to note here is that both the tuples depend on two distinct tuples in layer $i-1$. By induction assumption, characters across these two $2$-tuples are independent of each other. As $\vardictionary_i$ is a bijective map, this will also suggest that characters across the $2$-tuples in the resulting output must also be independent of each other. This will give the final argument. 
   
\end{proof}

\begin{table}[t]
\scalebox{1}{
    \centering
    \begin{tabular}{|c|c|ccc|}
        Operator & Expected value & \multicolumn{3}{c|}{Expected $\oplus$ value with  operator} \\
        \toprule
        & & $\copyop{}_1$ & $\copyop{}_2$ & $\xorop{}$  \\
        \midrule
        $\copyop{}_1$ & $\mathbb{E}_{\inputseq} \intermseq_{i-1, 2}  = 1/2$ & - & $ \mathbb{E}_{\inputseq} \intermseq_{i-1, 2} \oplus  \intermseq_{i-1, 3} = 1/2$ & $\mathbb{E}_{\inputseq} \intermseq_{i-1, 3} = 1/2$  \\ 
        $\copyop{}_2$ &  $\mathbb{E}_{\inputseq} \intermseq_{i-1, 3}  = 1/2$ & $\mathbb{E}_{\inputseq} \intermseq_{i-1, 2} \oplus \intermseq_{i-1, 3} = 1/2$  & - & $ \mathbb{E}_{\inputseq} \intermseq_{i-1, 2} = 1/2$ \\
        $\xorop{}$ &  $\mathbb{E}_{\inputseq} \intermseq_{i-1, 2} \oplus  \intermseq_{i-1, 3} = 1/2$ & $\mathbb{E}_{\inputseq}  \intermseq_{i-1, 3} = 1/2$ & $\mathbb{E}_{\inputseq}  \intermseq_{i-1, 2} = 1/2$ & -\\
        \bottomrule
    \end{tabular}
    }
    \caption{ $\mathbb{E} \intermseq_{i, 1}$ (similarly $\mathbb{E} \intermseq_{i, 2}$) and $\mathbb{E} \intermseq_{i, 1} \oplus \intermseq_{i, 2}$ under different $\vardictionary_i$ \strdicts{}, defined by \copyop{} and \xorop{} operations on $\intermseq_{i-1, 2}$ and $\intermseq_{i-1, 3}$.}
    \label{tab:expectation_and_oplus}
\end{table}

\begin{lemma}\label{lem:circ_invar}
    Under the assumption that sequence lengths $\inputlength$ are even, for any $2 \leq i \leq \depth+1$, where $\intermseq_{i, j}^{\alpha}, \intermseq_{i, j}^{\beta}$ denote the output after $i-1$st translation step  for a random sequence $\inputseq \sim \mathcal{U}(\{0, 1\}^{\inputlength})$ at any position  $1 \le j \le \inputlength$ under the two \strcodebook{} $\codebook^{\alpha}_{:i} := \{\vardictionary^{\alpha}_{\ell}\}_{\ell=1}^{i-1}$ and $\codebook^{\beta}_{:i} := \{\vardictionary^{\beta}_{\ell}\}_{\ell=1}^{i-1}$, the following holds true for all positions $j$:
    \begin{align*}
        \corr(\intermseq_{i, j}^{\alpha}, \intermseq_{i, j}^{\beta}, \mathcal{U}(\{0, 1\}^{\inputlength})) = \corr(\intermseq_{i, (j+2k) \% \inputlength }^{\alpha}, \intermseq_{i, (j+2k) \% \inputlength}^{\beta}, \mathcal{U}(\{0, 1\}^{\inputlength})), \text{ for all integer } k. 
    \end{align*}
\end{lemma}

\begin{proof}
    Fix an integer $k$. By \cref{def:corr}, the necessary condition would be to show 
    \begin{align*}
        \abs{1 - 2\mathbb{E}_{\inputseq \sim \{0, 1\}^{\inputlength}}  \intermseq_{i, j}^{\alpha} \oplus \intermseq_{i, j}^{\beta}} = \abs{1 - 2\mathbb{E}_{\inputseq \sim \{0, 1\}^{\inputlength}}  \intermseq_{i, (j+2k) \% \inputlength}^{\alpha} \oplus \intermseq_{i, (j+2k) \% \inputlength}^{\beta}}.
    \end{align*}
    We will look at the behavior of the function $h(\inputseq) =  \intermseq_{i, j}^{\alpha} \oplus \intermseq_{i, j}^{\beta}$. Denote $\mathcal{C}_{k}$ as a circular operation that takes a sequence $\inputseq$ and returns a shifted sequence $\circular{\inputseq}$, i.e. if $\circular{\inputseq} = \mathcal{C}_k(\inputseq)$ then for all $j$, $\circular{\intermseq}_{1, j} = \intermseq_{1, (j+2k) \% \inputlength }$. Note that, we can also define an inverse function $\mathcal{C}_{k}^{-1}$ and $\inputseq = \mathcal{C}_k^{-1}(\mathcal{C}_k(\inputseq))$.
    
    From the definition of $\translationtask_{\codebook}$ for any \strdict{} $\codebook$, we have for any input $\inputseq$:
    \begin{align*}
        \intermseq_{i} = \transfunc_{\vardict_{i-1}}\circ\dots\circ \transfunc_{\vardict_{1}}\rbr{\inputseq}, 
    \end{align*}
    where for any level $i$, $\transfunc_{\vardict_{i}}$ denotes the translation step at level $i$ that includes \circularshift{} and \translate{} using $\vardictionary_i$.

    In \cref{lem:circular_invariance}, we show that for any translation task $\translationtask_{\codebook}$ and any input $\inputseq$,
    $$
    \transfunc_{\vardict_{i-1}}\circ\dots\circ \transfunc_{\vardict_{1}}\rbr{\inputseq} = \mathcal{C}_k^{-1}\left(\transfunc_{\vardict_{i-1}}\circ\dots\circ \transfunc_{\vardict_{1}} \circ \mathcal{C}_k\rbr{\inputseq}\right).
    $$
    
    On input $\circular{\inputseq} = \mathcal{C}_k(\inputseq)$, if $\circular{\intermseq}_{i}^{\alpha}$ and $\circular{\intermseq}_{i}^{\beta}$ represent the output of translations using $\codebook^{\alpha}_{:i} := \{\vardictionary^{\alpha}_{\ell}\}_{\ell=1}^{i-1}$ and $\codebook^{\beta}_{:i} := \{\vardictionary^{\beta}_{\ell}\}_{\ell=1}^{i-1}$ respectively, then  the above statement says that
    \begin{align*}
        \intermseq_{i}^{\alpha} = \mathcal{C}_k^{-1} \rbr{\circular{\intermseq}_{i}^{\alpha}}, &\quad \intermseq_{i}^{\beta} = \mathcal{C}_k^{-1} \rbr{\circular{\intermseq}_{i}^{\beta}} \\
        \text{(or)} \text{ } \mathcal{C}_k \rbr{\intermseq_{i}^{\alpha}} =  \circular{\intermseq}_{i}^{\alpha}, &\quad \mathcal{C}_k \rbr{\intermseq_{i}^{\beta}} = \circular{\intermseq}_{i}^{\beta}
    \end{align*}
    Thus, for any position $j$, we will have $\mathcal{C}_k \rbr{\intermseq_{i}^{\alpha}}_j =  \circular{\intermseq}_{i, j}^{\alpha}$ (and similarly for  $\circular{\intermseq}_{i, j}^{\beta}$). We can then compute the value of $h$ on input $\circular{\inputseq}$ as follows: 
    \begin{align*}
        h(\circular{\inputseq}) &= \circular{\intermseq}_{i, j}^{\alpha} \oplus \circular{\intermseq}_{i, j}^{\beta} \\
        &= \mathcal{C}_k \left(\intermseq_{i}^{\alpha}\right)_j \oplus \mathcal{C}_k \left(\intermseq_{i}^{\beta}\right)_j \\
        &= \intermseq_{i, (j+2k) \% \inputlength}^{\alpha} \oplus \intermseq_{i, (j+2k) \% \inputlength}^{\beta}
    \end{align*}
    The last step follows from  using the definition of $\mathcal{C}_k$.
    On the other hand,
    \begin{align*}
        \mathbb{E}_{\circular{\inputseq} \sim \uniform \rbr{ \{0, 1 \}^{\inputlength} }}  h(\circular{\inputseq}) = \mathbb{E}_{\inputseq \sim \uniform \rbr{ \{0, 1 \}^{\inputlength} }}  h(\inputseq),
    \end{align*}
    as the uniform distribution can be shown to not change under circular function $\mathcal{C}_k$.
    This will imply:
    \begin{align*}
        &\mathbb{E}_{\inputseq \sim \uniform \rbr{ \{0, 1 \}^{\inputlength} }}  \intermseq_{i, (j+2k) \% \inputlength}^{\alpha} \oplus \intermseq_{i, (j+2k) \% \inputlength}^{\beta} =  \mathbb{E}_{\inputseq \sim \{0, 1\}^{\inputlength}}  \intermseq_{i, j}^{\alpha} \oplus \intermseq_{i, j}^{\beta} \\&
        \implies \abs{1 - 2\mathbb{E}_{\inputseq \sim \{0, 1\}^{\inputlength}}  \intermseq_{i, j}^{\alpha} \oplus \intermseq_{i, j}^{\beta}} = \abs{1 - 2\mathbb{E}_{\inputseq \sim \{0, 1\}^{\inputlength}}  \intermseq_{i, (j+2k) \% \inputlength}^{\alpha} \oplus \intermseq_{i, (j+2k) \% \inputlength}^{\beta}} \\&
        \implies \corr(\intermseq_{i, j}^{\alpha}, \intermseq_{i, j}^{\beta}, \mathcal{U}(\{0, 1\}^{\inputlength})) = \corr(\intermseq_{i, (j+2k) \% \inputlength }^{\alpha}, \intermseq_{i, (j+2k) \% \inputlength}^{\beta}, \mathcal{U}(\{0, 1\}^{\inputlength})).
    \end{align*}
    
\end{proof}

\begin{lemma}\label{lem:circular_invariance}
    For any translation task $\translationtask_{\codebook}$, at any level $i \leq \depth$ and any input $\inputseq$,
    $$
    \transfunc_{\vardict_i}\circ\dots\circ \transfunc_{\vardict_{1}}\rbr{\inputseq} = \mathcal{C}_k^{-1}\left(\transfunc_{\vardict_i}\circ\dots\circ \transfunc_{\vardict_{1}} \circ \mathcal{C}_k\rbr{\inputseq}\right).
    $$    
\end{lemma}

\begin{proof}
    Denote $\mathcal{C}_{k}$ as a circular operation that takes a sequence $\inputseq$ and returns a shifted sequence $\circular{\inputseq}$, i.e. if $\circular{\inputseq} = \mathcal{C}_k(\inputseq)$ then for all $j$, $\circular{\intermseq}_{1, j} = \intermseq_{1, (j+2k) \% \inputlength }$. Note that, we can also define an inverse function $\mathcal{C}_{k}^{-1}$ and $\inputseq = \mathcal{C}_k^{-1}(\mathcal{C}_k(\inputseq))$.

    Recall that for any level $i$, $\transfunc_{\vardict_{i}}$ denotes the translation step at level $i$ that includes \circularshift{}, and \translate{} with $\vardictionary_i$. The argument will again follow by an induction step.

    \paragraph{Base condition: $i=1$:} We need to show that 
    $$
        \transfunc_{\vardict_{1}}\rbr{\inputseq} = \mathcal{C}_k^{-1}\left(\transfunc_{\vardict_{1}} \circ \mathcal{C}_k\rbr{\inputseq}\right).
    $$
    To do so, we will look at the behavior of the first $2$ characters. Argument for others can be extended. If $\circular{\inputseq} = \mathcal{C}_k\rbr{\inputseq}$, $\intermseq_2 = \transfunc_{\vardict_{1}} (\inputseq)$, $\circular{\intermseq_2} = \transfunc_{\vardict_{1}} (\circular{\inputseq})$ , then by \cref{lem:dependencies},
    \begin{align*}
        &(\intermseq_{2, 1}, \intermseq_{2, 2}) = \dictionary_1 ( (\intermseq_{1, 2}, \intermseq_{1, 3}) ) \\
        &(\circular{\intermseq}_{2, 1}, \circular{\intermseq}_{2, 2}) = \dictionary_1 ( (\circular{\intermseq}_{1, 2}, \circular{\intermseq}_{1, 3}) ).
    \end{align*}

    But by definition of $\mathcal{C}_k$,
    \begin{align*}
        (\circular{\intermseq}_{1, 2}, \circular{\intermseq}_{1, 3}) = (\intermseq_{1, (2+2k) \% \inputlength}, \intermseq_{1, (3+2k) \% \inputlength}).
    \end{align*}

    On the other hand, \cref{lem:dependencies} can be adapted to give
    \begin{align*}
        \rbr{\intermseq_{2, (1+2k) \% \inputlength}, \intermseq_{2, (2+2k) \% \inputlength}} = \dictionary_1 (\intermseq_{1, (2+2k) \% \inputlength}, \intermseq_{1, (3+2k) \% \inputlength}).
    \end{align*}

    Thus, we can show by combining the above $3$ steps that
    \begin{align*}
        (\circular{\intermseq}_{2, 1}, \circular{\intermseq}_{2, 2}) &= \dictionary_1 ( (\circular{\intermseq}_{1, 2}, \circular{\intermseq}_{1, 3}) ) \\
        &= \dictionary_1(\intermseq_{1, (2+2k) \% \inputlength}, \intermseq_{1, (3+2k) \% \inputlength}) \\
        &= \rbr{\intermseq_{2, (1+2k) \% \inputlength}, \intermseq_{2, (2+2k) \% \inputlength}} .
    \end{align*}

    We can extend the above argument to show that for any position $j$,
    \begin{align*}
        \circular{\intermseq}_{2, j} = \intermseq_{2, (j+2k) \% \inputlength},
    \end{align*}
    which by definition of $\mathcal{C}_k$, implies
    \begin{align*}
        \circular{\intermseq}_{2} = \mathcal{C}_k ( \intermseq_2 ) \quad \text{ or } \intermseq_2 = \mathcal{C}_k^{-1} \rbr{\circular{\intermseq}_{2}}, 
    \end{align*}
    which can be further simplified (using the notations: $\circular{\inputseq} = \mathcal{C}_k\rbr{\inputseq}$, $\intermseq_2 = \transfunc_{\vardict_{1}} (\inputseq)$, $\circular{\intermseq_2} = \transfunc_{\vardict_{1}} (\circular{\inputseq})$)
    \begin{align*}
        \intermseq_2 
        &= \mathcal{C}_k^{-1} \rbr{\circular{\intermseq}_{2}} \\&= \mathcal{C}_k^{-1} \rbr{ \transfunc_{\vardictionary_1} \rbr{ \circular{\intermseq}_{1} } } \\&
        =\mathcal{C}_k^{-1} \rbr{ \transfunc_{\vardictionary_1} \rbr{ \mathcal{C}_k \rbr{ \inputseq } } } := \mathcal{C}_k^{-1}\left(\transfunc_{\vardict_{1}} \circ \mathcal{C}_k\rbr{\inputseq}\right).
    \end{align*}

    \paragraph{General argument for $i$:} Suppose the induction condition holds true for all layers $\ell < i$. Then,
    \begin{align*}
        \transfunc_{\vardict_i}\circ\dots\circ \transfunc_{\vardict_{1}}\rbr{\inputseq} = \transfunc_{\vardict_i} \rbr{ \mathcal{C}_k^{-1}\left(\transfunc_{\vardict_{i-1}} \dots\circ \transfunc_{\vardict_{1}} \circ \mathcal{C}_k\rbr{\inputseq}\right) }.
    \end{align*}
    We can then follow the same argument as the base condition, and show that
    \begin{align*}
         \transfunc_{\vardict_i} \rbr{ \mathcal{C}_k^{-1}\left(\transfunc_{\vardict_{i-1}} \dots\circ \transfunc_{\vardict_{1}} \circ \mathcal{C}_k\rbr{\inputseq}\right) } = 
         \mathcal{C}_k^{-1} \rbr{ \transfunc_{\vardict_{i}} \dots\circ \transfunc_{\vardict_{1}} \circ \mathcal{C}_k\rbr{\inputseq} }.
    \end{align*}
\end{proof}

\begin{lemma}\label{lem:cross_corr}
    For any $1 \leq i \leq \depth$, if $\intermseq_{i, j}^{\alpha}, \intermseq_{i, j}^{\beta}$ denote the output after $i-1$st translation step  for a random sequence $\inputseq \sim \mathcal{U}(\{0, 1\}^{\inputlength})$ at any position  $1 \le j \le \inputlength$ under the two sets of random \strdicts{} $\codebook_{:i}^{\alpha} := \{\vardictionary^{\alpha}_{\ell}\}_{\ell=1}^{i-1}$ and $\codebook_{:i}^{\beta} := \{\vardictionary^{\beta}_{\ell}\}_{\ell=1}^{i-1}$, the following holds true for all positions $j$: 
    \begin{align*}
        &\Pr_{\codebook_{:i}^{\alpha}, \codebook_{:i}^{\beta} } \left[ \corr\rbr{ \intermseq_{i, j}^{\alpha}, \intermseq_{i, j+1}^{\beta}, \mathcal{U}(\{0, 1\}^{\inputlength} )} = 1\right] \\& \leq  \frac{1}{2} \rbr{\Pr_{\codebook_{:i}^{\alpha}, \codebook_{:i}^{\beta} } \left[ \corr\rbr{ \intermseq_{i, j}^{\alpha}, \intermseq_{i, j}^{\beta}, \mathcal{U}(\{0, 1\}^{\inputlength}  )} = 1 \right] + \Pr_{\codebook_{:i}^{\alpha}, \codebook_{:i}^{\beta} } \left[\corr\rbr{ \intermseq_{i, j+1}^{\alpha}, \intermseq_{i, j+1}^{\beta}, \mathcal{U}(\{0, 1\}^{\inputlength}  )} = 1 \right]}.
    \end{align*}
\end{lemma}

\begin{proof}
    We will prove the required result with a counting argument. We will create a family of \strdicts{}: let $\family(j)$ and $\family(j+1)$ denote two sets of \strdicts{}, such that for every \strdict{} $\codebook_{:i}^{\alpha}$ in   $\family(j)$, there exists at least one \strdict{} $\codebook_{:i}^{\beta}$ in $\family(j+1)$ such that
    \begin{align*}
         &\corr\rbr{ \intermseq_{i, j}^{\alpha}, \intermseq_{i, j+1}^{\beta}, \mathcal{U}(\{0, 1\}^{\inputlength} )} = 1, \\& \quad \quad \text{(Equiv. to saying translations for $\codebook_{:i}^{\alpha}$ and $\codebook_{:i}^{\beta}$ have correlations at position $j$ and $j+1$)}
    \end{align*}
    where $\intermseq_{i, j}^{\alpha}, \intermseq_{i, j+1}^{\beta}$ are outputs on a random sequence $\inputseq \sim \mathcal{U}(\{0, 1\}^{\inputlength})$ corresponding to using $\codebook_{:i}^{\alpha}$ and $\codebook_{:i}^{\beta}$ respectively.

    We then apply a grouping algorithm $\textsc{Group}$ to group correlated \strdicts{} together in each family. That is, in $\family(j)$, we create groups of \strdicts{} $\{ S_1, S_2, \cdots \}$ such that for any two \strdicts{} $\codebook_{:i}^{\alpha}$ and $\codebook_{:i}^{\alpha'}$ that belong to a group $S$, 
    \begin{align*}
         &\corr\rbr{ \intermseq_{i, j}^{\alpha}, \intermseq_{i, j}^{\alpha'}, \mathcal{U}(\{0, 1\}^{\inputlength} )} = 1,\\& \quad \quad \text{(Equiv. to saying translations for $\codebook_{:i}^{\alpha}$ and $\codebook_{:i}^{\beta}$ have correlations at position $j$)}
    \end{align*}
    where $\intermseq_{i, j}^{\alpha}, \intermseq_{i, j+1}^{\alpha'}$ are outputs on a random sequence $\inputseq \sim \mathcal{U}(\{0, 1\}^{\inputlength})$ corresponding to using $\codebook_{:i}^{\alpha'}$ and $\codebook_{:i}^{\alpha'}$respectively. We call the resulting output of this operation as $\textsc{Group}(\family(j))$. Similarly, we compute $\textsc{Group}(\family(j+1))$.

    We can observe the following two characteristics of $\textsc{Group}(\family(j))$ and $\textsc{Group}(\family(j+1))$:
    \begin{enumerate}
        \item For every set $S_1 \in \textsc{Group}(\family(j))$ there will exist one set $S_2 \in \textsc{Group}(\family(j+1))$, such that for all \strdicts{} $\codebook_{:i}^{\alpha} \in S_1$ and $\codebook_{:i}^{\beta} \in S_2$,  correlation will be $1$ for output at positions $j$ and $j+1$.
       
        \item For every set $S_1 \in \textsc{Group}(\family(j))$ there can't exist two sets $S_2, S_{2}^{*} \in \textsc{Group}(\family(j+1))$, such that the \strdicts{} in $S_1$ are correlated to \strdicts{} from both $S_2, S_{2}^{*}$ for output at positions $j$ and $j+1$ respectively. Otherwise, we could have merged $S_2$ and $S_2^{*}$ under the \textsc{Group} operation.
       
    \end{enumerate}
    Let $\textsc{Correlation-map}$ denote the map between $\textsc{Group}(\family(j))$ and $\textsc{Group}(\family(j+1))$, which connects sets $S_1 \in \textsc{Group}(\family(j))$ to a set $S_2 \in \textsc{Group}(\family(j+1))$ such that any two \strdicts{} in $S_1$ and $S_2$ have  correlations for output at positions $j$ and $j+1$ respectively.

    The result will then follow from a counting argument. The number of possible pairs (can be identical \strdicts{}) that can give correlations for output at position $j$ are given by: $\sum_{S_1 \in \textsc{Group}(\family(j))} |S_1|^2$. Similarly, the number of possible pairs that can give correlations for output at position $j+1$ are given by: $\sum_{S_2 \in \textsc{Group}(\family(j+1))} |S_2|^2$. On the other hand, the number of possible pairs that can give correlations for output at position $j$ and $j+1$ respectively are given by: $\sum_{S_1 \in \textsc{Group}(j); S_2 = \textsc{Correlation-map}(S_1)} |S_1| |  S_2 |$. Applying the AM-GM inequality, we can show that the average of the number of pairs for which correlation is $1$ for output characters at either position $j$ or $j+1$ is higher than the number of pairs for which correlation is $1$ for output at positions $j$ and $j+1$.
    
\end{proof}

\subsection{Proof for Statistical query lower bound for general $\numchar$}
\label{sec:proof_sqn}
We present the main theorem statement again for readability.

\begin{thmbox}
\sqn
\end{thmbox}

\begin{proof}
    We adapt the SQ-dimension proof for $\translationtask(\depth, 2)$ to show the SQ-dimension proof for $\translationtask(\depth, \numchar)$. We will design a family of \strcodebook{} $\codebook = \{\vardictionary_i : \{0, 1, \cdots, \numchar-1\}^2 \to  \{0, 1, \cdots, \numchar-1\}^2  \}_{i=1}^{\depth}$, where \strdicts{} in $\codebook$ are built on top of a translation task in $\translationtask(\log_2 \numchar, 2)$.
    
    For a character $a \in \{0, 1, \cdots, \numchar-1\}$, suppose $\textsc{bit}(a) \in \{0, 1\}^{\log_2 \numchar}$ indicates its binary representation, and $\textsc{numeric}$ represents the map from  binary representation to its corresponding numeric representation. Then, we design each \strdict{} $\vardictionary$ using a random translation task $\nu \in \translationtask(\log_2 \numchar, 2)$. For any tuple $(a, b) \in  \{0, 1, \cdots, \numchar-1\}^2$, output of $\vardictionary$ is given as
    
    \begin{align*}
        &\vardictionary(a, b) = (o_1, o_2), \text{ where } \\ 
        &o_1 = \textsc{numeric} \rbr{ \nu  \rbr{\{ \Tilde{\va}_i \oplus \Tilde{\vb}_i \}_{i=1}^{\log_2 \numchar}} } \\
        &o_2 =  \textsc{numeric} \rbr{\nu  \rbr{ \Tilde{\vb} } } \\
        &\Tilde{\va} = \textsc{bit}(a) \\
        &\Tilde{\vb} = \textsc{bit}(b)
    \end{align*}

    Primarily, the \strdicts{} are defined as follows:
    \begin{enumerate}
          
        \item On a $2$-tuple of characters $(a, b)$, we first compute their binary representations $(\textsc{bit}(a), \textsc{bit}(b))$. We then compute two intermediate outputs, one where a \xorop{} operation is applied on $\textsc{bit}(a), \textsc{bit}(b)$ at each bit, and another where $\textsc{bit}(b)$ is simply copied. This operation is equivalent to applying a deterministic map on $2$-tuples of binary characters (identical to $\specialdictionary_6$ from \cref{tab:truthtable_permmap}), where $2$-tuples are created by pairing binary bits in binary representation of $a$ and $b$.
       
        \item We then apply a random \translationtask{} task of depth $\log_2 \numchar$ on each of the intermediate outputs. This applies a random bijective map on the sequence of bits in the intermediate outputs, using a translation task in $\translationtask(\log_2 \numchar, 2)$ (\cref{lemma:setting-MLTinvertability}).

        \item The final tuple of characters is returned by applying a \textsc{numeric} operation on the binary representations.
    \end{enumerate}

    Thus, we have narrowed our focus on a special group of tasks from $\translationtask(\depth, \numchar)$ that applies $\translationtask(\log_2 \numchar, 2)$ on the binary representations of the characters at each level.
    By \cref{lem:sqn2}, at any level $1 \leq i \leq \depth$, we can create $2^{\Omega(\log_2 \numchar)}$ \strdicts{}, the output of which are pairwise uncorrelated .
     Now, we can compose these uncorrelated \strdicts{} to give multiple \strcodebook{} that are uncorrelated. That is, following a similar proof as \cref{lem:sqn2}, we can show that we can create $\rbr{2^{\Omega(\log_2 \numchar)}}^{\Omega(\depth)} = \numchar^{\Omega(\depth)}$ \strcodebook{}, using the above restriction, that are pairwise uncorrelated on any bit in the binary representation of the output of their corresponding translation tasks. This will translate to the output of the translation task in numeric form as well, as the mapping between binary representation and numeric form of a digit is bijective. 
     
\end{proof}

\subsection{Proofs of Useful lemmas}
\label{sec:Proofs_of_Useful_lemmas}

Here we give the proofs for the useful lemmas necessary to prove \cref{lem:sqn2}. We repeat the lemma statements for easier readability.

\mapformntwo

\begin{table}[!ht]
    \centering
    \begin{tabular}{c|cccc|cc}
         Map name ($\vardictionary$) & \multicolumn{4}{c|}{Output for corresponding input tuple $\vardictionary(a, b)$} & \multicolumn{2}{c}{General formulation on output for map $\vardictionary$}\\
         \toprule
         & $(0, 0)$ & $(0, 1)$ & $(1, 0)$ & $(1, 1)$ & $\vardictionary(a, b)_1$ & $\vardictionary(a, b)_2$ \\ 
         \midrule
         $\specialdictionary_1$ & $(0, 0)$ & $(0, 1)$ & $(1, 0)$ & $(1, 1)$ & $\copyop_1(a, b)$ & $\copyop_2(a, b)$ \\
         $\specialdictionary_2$ & $(0, 0)$ & $(0, 1)$ & $(0, 1)$ & $(1, 0)$ & $\copyop_1(a, b)$ & $\xorop(a, b)$  \\
         $\specialdictionary_3$ & $(0, 0)$ & $(1, 0)$ & $(0, 1)$ & $(1, 1)$ & $\copyop_2(a, b)$ & $\copyop_1(a, b)$   \\
         $\specialdictionary_4$ & $(0, 0)$ & $(1, 1)$ & $(0, 1)$ & $(1, 0)$ & $\copyop_2(a, b)$ & $\xorop(a, b)$ \\
         $\specialdictionary_5$ & $(0, 0)$ & $(1, 0)$ & $(1, 1)$ & $(0, 1)$ & $\xorop(a, b)$ & $\copyop_1(a, b)$\\
         $\specialdictionary_6$ & $(0, 0)$ & $(1, 1)$ & $(1, 0)$ & $(0, 1)$ & $\xorop(a, b)$ & $\copyop_2(a, b)$ \\
         \bottomrule
    \end{tabular}
    \caption{The table captures the definition of $6$  bijective maps (\strdicts{}) on $2$-tuples $\{0, 1\}^2 \to \{0, 1\}^2$, whose output characters can be defined in terms of $\copyop{}$ and $\xorop{}$ operations on the input characters. Any other bijective map can be shown to belong to $\mirror$ of one of these maps.}
    \label{tab:truthtable_permmap}
\end{table}

\begin{proof}
    There are $4$ possible tuples $(0, 0), (0, 1), (1, 0), (1, 1)$. By \cref{lem:num_bijective}, the number of possible bijective maps (\strdicts{}) that connect $2$-tuples are $4! = 24$. We will show that the output tuple of each map can be represented by $3$ operations.

    \paragraph{Operation \notop{} creates mirror maps:} For each map $\vardictionary$, there exists $3$ alternative maps $\vardictionary_{(1)}, \vardictionary_{(2)}, \vardictionary_{(3)}$ such that if $\vardictionary(a, b)_{i}$ represents the $i$th character in the output tuple for input tuple $(a, b)$:
    \begin{itemize}
        \item $\vardictionary_{(1)}$ selectively applies the \notop{} operation to the first character in the output tuple of $\vardictionary$ on any input tuple, i.e. $$\vardictionary_{(1)} (a, b) = (\notop{} (\vardictionary(a, b)_1), \vardictionary(a, b))$$  for all tuples $(a, b) \in \{0, 1\}^2$.
        \item $\vardictionary_{(2)}$ selectively applies the \notop{} operation to the second character in the output tuple of $\vardictionary$ on any input tuple, i.e. $$\vardictionary_{(2)} (a, b) =  (\vardictionary(a, b)_1, \notop{} (\vardictionary(a, b)))$$ for all tuples $(a, b) \in \{0, 1\}^2$.
        
        \item $\vardictionary_{(3)}$ selectively applies the \notop{} operation to both characters in the output tuple of $\vardictionary$ on any input tuple, i.e. $$\vardictionary_{(3)} (a, b) =  (\notop{} (\vardictionary(a, b)_1), \notop{} (\vardictionary(a, b)))$$ for all tuples $(a, b) \in \{0, 1\}^2$.
    \end{itemize}
    Thus, for each map $\vardictionary$, there exist $3$ other alternative maps that simply modify the output of map $\vardictionary$ with the \notop{} operation.

    \paragraph{After removing the mirror maps:} We now show that there $6$ possible maps that apply either a \xorop{} or a \copyop{} on the input characters to get the output characters. We name them $\specialdictionary_1, \specialdictionary_2, \cdots, \specialdictionary_6$. We give the output of each map on the $4$ tuples in \cref{tab:truthtable_permmap} and show that the each character in the output tuple can be represented using \copyop{} and \xorop{} operations.

\end{proof}

\corrntwo

\begin{proof}
    We prove the lemma for case 1, cases 2 and 3 can be similarly proved. From \cref{lem:num_bijective}, $\vardictionary^{\alpha}$ and $\vardictionary^{\beta}$ can be randomly selected from a set of $24$ possible candidates. On the other hand, \cref{lem:map_form_n2} shows that there are $6$ maps $\{\specialdictionary_1, \cdots, \specialdictionary_6\}$ whose output can be defined in terms of \copyop{} and \xorop{} operations of characters in the input tuple. For each $\vardictionary$ in this set, there are mirror maps $\vardictionary_{(1)}, \vardictionary_{(2)}, \vardictionary_{(3)}$ whose output are defined by selective \notop{} operations on the output of $\vardictionary$, and the set of $4$ maps is represented by $\mirror(\vardictionary)$.

    The proof will follow from $2$ steps: 
    first, we argue about correlations when $\vardictionary^{\alpha}$ and $\vardictionary^{\beta}$ are selected from $\{\specialdictionary_1, \cdots, \specialdictionary_6\}$, and then we argue about the general case when $\vardictionary^{\alpha}$ and $\vardictionary^{\beta}$ are selected from the general set of bijective maps.
    
    \begin{itemize}
        \item In \cref{lem:nonmirror_corr}, we show that for two maps that are randomly selected from $\{\specialdictionary_1, \cdots, \specialdictionary_6\}$, the correlation is $1$ with probability $1/3$.
        
        \item The remaining possibility is when $\vardictionary^{\alpha}$ and $\vardictionary^{\beta}$ belong to $\mirror(\vardictionary_i)$ and $\mirror(\vardictionary_j)$ for some $\vardictionary_i, \vardictionary_j \in \{\specialdictionary_1, \cdots, \specialdictionary_6\}$. We show in \cref{lem:corr_distinctmirrormap}, correlation of $\vardictionary^{\alpha}$ and $\vardictionary^{\beta}$ will be equal to correlation of $\vardictionary_i, \vardictionary_j$.
    \end{itemize}
    
    Thus, we can combine all the observations to show that for two random maps $\vardictionary^{\alpha}, \vardictionary^{\beta}$ that belong to $\mirror(\vardictionary_i)$ and $\mirror(\vardictionary_j)$ for some $\vardictionary_i, \vardictionary_j \in \{\specialdictionary_1, \cdots, \specialdictionary_6\}$,
    \begin{align*}
        &\corr(\vardictionary^{\alpha}(\cdot)_1, \vardictionary^{\beta}(\cdot)_1, \uniform(\{0, 1\}^2)) \\& = \corr(\vardictionary_{i}(\cdot)_1, \vardictionary_{j}(\cdot)_1, \uniform(\{0, 1\}^2)) = \begin{cases}
        1, \quad \text{w.p. } 1/3 \text{ w.r.t. randomness in } \vardictionary_i, \vardictionary_j\\
        0, \quad \text{otherwise.}\\
        \end{cases}
    \end{align*}

\end{proof}

\begin{lemma}\label{lem:corr_distinctmirrormap}
    The following holds true for any maps $\vardictionary^{\alpha}$ and $\vardictionary^{\beta}$ with $\vardictionary^{\alpha}, \vardictionary^{\beta} \in \{\specialdictionary_1, \cdots, \specialdictionary_6\}$ and for all $\Tilde{\vardictionary}^{\alpha} \in \mirror(\vardictionary^{\alpha}) , \Tilde{\vardictionary}^{\beta} \in \mirror(\vardictionary^{\beta})$:
    $$
        \corr(\vardictionary^{\alpha}(\cdot)_1, \vardictionary^{\beta}(\cdot)_1, \uniform(\{0, 1\}^2)) = \corr(\Tilde{\vardictionary}^{\alpha}(\cdot)_1, \Tilde{\vardictionary}^{\beta}(\cdot)_1, \uniform(\{0, 1\}^2)),
    $$
\end{lemma}

\begin{proof}
    We prove as follows:
    \begin{align}
        \corr(\vardictionary^{\alpha}(\cdot)_1, \vardictionary^{\beta}(\cdot)_1, \uniform(\{0, 1\}^2)) &=  \abs{  \Pr_{\var \sim \uniform(\{0, 1\}^2)} [\vardictionary^{\alpha}(\var)_1 \ne  \vardictionary^{\beta}(\var)_1] - \Pr_{\var \sim \uniform(\{0, 1\}^2)} [\vardictionary^{\alpha}(\var)_1 =  \vardictionary^{\beta}(\var)_1]} \nonumber\\
        & =  \abs{  2 \Pr_{\var \sim \uniform(\{0, 1\}^2)} [\vardictionary^{\alpha}(\var)_1 \ne  \vardictionary^{\beta}(\var)_1] - 1 } \\
        & = \begin{cases}
           \abs{  2 \Pr_{\var \sim \uniform(\{0, 1\}^2)} [\Tilde{\vardictionary}^{\alpha}(\var)_1 \ne  \Tilde{\vardictionary}^{\beta}(\var)_1] - 1 }, \text{if condition ``c1'' is true} \\
           \abs{  1 - 2 \Pr_{\var \sim \uniform(\{0, 1\}^2)} [\Tilde{\vardictionary}^{\alpha}(\var)_1 =  \Tilde{\vardictionary}^{\beta}(\var)_1] }, \text{if condition ``c2'' is true}\\
        \end{cases} \\
        & = \abs{  \Pr_{\var \sim \uniform(\{0, 1\}^2)} [\Tilde{\vardictionary}^{\alpha}(\var)_1 \ne  \Tilde{\vardictionary}^{\beta}(\var)_1] - \Pr_{\var \sim \uniform(\{0, 1\}^2)} [\Tilde{\vardictionary}^{\alpha}(\var)_1 =  \Tilde{\vardictionary}^{\beta}(\var)_1] } \nonumber\\
        & = \corr(\Tilde{\vardictionary}^{\alpha}(\cdot)_1, \Tilde{\vardictionary}^{\beta}(\cdot)_1, \uniform(\{0, 1\}^2)),
    \end{align}
    where the second and the penultimate steps follow from the law of total probability.
    Here, condition ``c1'' holds when either case is true,
    \begin{align*}
        \Tilde{\vardictionary}^{\alpha}(x)_1 = \notop(\vardictionary^{\alpha}(x)_1),& \quad \Tilde{\vardictionary}^{\beta}(x)_1 = \notop(\vardictionary^{\beta}(x)_1), \quad \text{for all } x \in \{0, 1\}^2 \\
        \Tilde{\vardictionary}^{\alpha}(x)_1 = \vardictionary^{\alpha}(x)_1,& \quad \Tilde{\vardictionary}^{\beta}(x)_1 = \vardictionary^{\beta}(x)_1, \quad \text{for all } x \in \{0, 1\}^2.
    \end{align*}
    and condition ``c2'' holds when either case is true,
    \begin{align*}
        \Tilde{\vardictionary}^{\alpha}(x)_1 = \vardictionary^{\alpha}(x)_1,& \quad \Tilde{\vardictionary}^{\beta}(x)_1 = \notop(\vardictionary^{\beta}(x)_1), \quad \text{for all } x \in \{0, 1\}^2 \\
        \Tilde{\vardictionary}^{\alpha}(x)_1 = \notop(\vardictionary^{\alpha}(x)_1),& \quad \Tilde{\vardictionary}^{\beta}(x)_1 = \vardictionary^{\beta}(x)_1, \quad \text{for all } x \in \{0, 1\}^2.
    \end{align*}
    One of condition ``c1'' or condition ``c2'' is true because $\Tilde{\vardictionary}^{\alpha}$ and $\vardictionary^{\alpha}$ (similarly, $\Tilde{\vardictionary}^{\beta}$ and $\vardictionary^{\beta}$) are mirror maps.
    
\end{proof}

\begin{lemma}\label{lem:nonmirror_corr}

    The following holds true for two randomly selected maps $\vardictionary^{\alpha}$ and $\vardictionary^{\beta}$ with $\vardictionary^{\alpha}, \vardictionary^{\beta} \in \{\specialdictionary_1, \cdots, \specialdictionary_6\}$ :
    \begin{itemize}
        \item With probability $1/3$ w.r.t. random selection, 
        \begin{align*}
            \corr(\vardictionary^{\alpha}(\cdot)_1, \vardictionary^{\beta}(\cdot)_1, \uniform(\{0, 1\}^2)) > 0 \quad (=1).
        \end{align*}
        \item With probability $1/3$ w.r.t. random selection, 
        \begin{align*}
            \corr(\vardictionary^{\alpha}(\cdot)_2, \vardictionary^{\beta}(\cdot)_2, \uniform(\{0, 1\}^2)) > 0 \quad (=1).
        \end{align*}
    \end{itemize}
\end{lemma}

\begin{proof}
    We prove for case 1, proof for case 2 is analogous.

    Among $\{\specialdictionary_1, \cdots, \specialdictionary_6\}$, we can create three family of maps: $F_{\copyop{}_1}: \{\specialdictionary_1, \specialdictionary_2\}$, $F_{\copyop{}_2}: \{\specialdictionary_3, \specialdictionary_4\}$, $F_{\xorop{}}: \{\specialdictionary_5, \specialdictionary_6\}$ that are identical in operation ($\copyop{}_1, \copyop{}_2, \xorop{}$ respectively) at the first character in output tuple. This would imply, if the $\vardictionary^{\alpha}$ and $\vardictionary^{\beta}$ both belong to one of these families, $\corr(\vardictionary^{\alpha}(\cdot)_1, \vardictionary^{\beta}(\cdot)_1, \uniform(\{0, 1\}^2))$ will be $1$.

    On the other hand, one can show that for any operation $f_1, f_2 \in \{\copyop{}_1, \copyop{}_2, \xorop{}\}$ with $f_1 \ne f_2$ will have $\corr(f_1, f_2, \uniform(\{0, 1\}^2)) = 0$. That would then suggest that if $\vardictionary^{\alpha}$ and $\vardictionary^{\beta}$ belong to different families among $F_{\copyop{}_1}, F_{\copyop{}_2}, F_{\xorop{}}$, then $\corr(\vardictionary^{\alpha}(\cdot)_1, \vardictionary^{\beta}(\cdot)_1, \uniform(\{0, 1\}^2))$ will be $0$.

    By a simple counting argument, with probability $\frac{1}{3}$, two maps $\vardictionary^{\alpha}$ and $\vardictionary^{\beta}$ randomly selected from $\{\specialdictionary_1, \cdots, \specialdictionary_6\}$ will have non-zero correlation. 
\end{proof}


\newpage

\section{Upper Bound: Context-Enhanced Learning of $\translationtask(\depth, \numchar)$ with Simple Surrogate Model}
\label{sec:apdx-surrogateModel}

\subsection{Setup of Surrogate Model with In-context Capability}
\label{sec:apdx-surrogateModel-surrogateModelSetup}

Given $\depth+1$ alphabets $\alphabetset_1, \dots, \alphabetset_{\depth+1}$ of size $\numchar$ and $\depth$ bijective \strdicts{} $\vardictionary_i: \alphabetset_i^2\to \alphabetset_{i+1}^2$ 
The input of the translation process is an even-length sequence in the first alphabet, which we denote as $\inputseq \in \alphabetset_1^{\inputlength}$ where $\inputlength$ is the sequence length. The translation process modifies the input string recursively from $\intermseq_i$ to $\intermseq_{i+1}$ through the following 2 sub-processes:

\begin{enumerate}
    \item  \circularshift{}: The characters in $\intermseq_i \in \alphabetset_i^{\inputlength}$ are shifted by $1$ character leftward (and wrapped around to the end if necessary) to give sequence $\shiftintermseq_i \in \alphabetset_i^{\inputlength}$. Formally, for each $j \in [1, \inputlength]$ we have $\shiftintermseq_{i, j} = \intermseq_{i, (j+1) \% \inputlength}$.

    \item \translate{}: Using the \emph\strdict{} $\vardictionary_i: \alphabetset_i^2\to \alphabetset_{i+1}^2$, we translate 2-tuples (bigrams) of consecutive characters in sequence $\shiftintermseq_i$ to create $\intermseq_{i+1}$. That is, for every odd $j  \in [1, \inputlength]$, $(\intermseq_{i+1, j}, \intermseq_{i+1, j+1}) = \vardict_i (\shiftintermseq_{i, j}, \shiftintermseq_{i, j+1})$. 
    
\end{enumerate}

Now let us revisit the surrogate model introduced in \cref{sec:theory-surrogateModel}. Without loss of generality let $\alphabetset_1 = \alphabetset_2 = \cdots =\alphabetset_{d+1} := \alphabetset = \cbr{1, 2, \dots, \numchar}$. For any single character $a\in \alphabetset$, let its vector representation be a one-hot vector $\svecrep_a\in\R^\numchar$ such that $\rbr{\svecrep_a}_a = 1$. For any 2-tuple $(a,b)\in \alphabetset^2$, let its vector representation be a $\numchar^2$-dimensional vector $\vecrep(a,b) \triangleq \svecrep_a \otimes \svecrep_b$. Note that $\vecrep(a,b)$ is also a one-hot vector where \begin{align*}
    \vecrep(a,b)_i = \begin{cases}
        1 & \text{if}\ i = a\numchar + b\\
        0 & \text{elsewhere}
    \end{cases}.
\end{align*}

We use the notation $\lonehot_{a}$ (long one-hot) to denote a one-hot vector in $\R^{\numchar^2}$ with $a$-th position being 1 to avoid confusion.

\begin{defnbox}
\begin{definition}[Matrix Representation of Sequence]\emph{}\\
\label{defn:apdx-seq-matrep}
For a length-$L$ input sequence $\intermseq_i = \srbr{s_{i,1}, \dots, s_{i,\inputlength}}$, let its matrix representation be $\matricize(\intermseq_i) \triangleq 
 \matrep_i\in \R^{\numchar^2\times \inputlength/2}$ that
\begin{align*}
    \matrep_i = \begin{bmatrix}
        \vert & \vert & \cdots & \vert\\
        \vecrep\rbr{s_{i,1}, s_{i,2}} & \vecrep\rbr{s_{i,3}, s_{i,4}}&  \cdots & \vecrep\rbr{s_{i,\inputlength-1}, s_{i,\inputlength}}\\\vert& \vert& \cdots & \vert
    \end{bmatrix}
\end{align*}

For each $j\in [\inputlength/2]$, we use $\matrep_i^{(j)}$ to denote the $j$-th column of $\matrep_i$. 
We also denote the above conversion from a sequence $\intermseq_i$ to its matrix form as $\matrep_i = \matricize(\intermseq_i)$ and assume that $\matrep_1$ serves as the input to the surrogate model.
\end{definition}
\end{defnbox}

Note that the matricization operation is invertible by construction: for each column $\V{i}{j}$, let $x = \argmax{\V{i}{j}}$, we may read off the two characters in the original alphabet by computing $\matricize^{-1}(\V{i}{j}) = (\ceil{x/\numchar}, x\%\numchar)$.

At each level of translation, we assume the surrogate model will perform the following operations to $\matrep_i$:

\textbf{(a) Circular shift from $\matrep_i$ to $\shiftmatrep_i$}

\begin{defnbox}
\begin{definition}[Circular Shifting Operator $\shiftfunc$]\emph{}
\label{defn:apdx-surrogate-shift-operator}

    Given a matrix representation $\matrep\in \R^{\numchar^2\times \inputlength/2}$ of a sequence $\intermseq$, the circular shifting operator $\shiftfunc$ acts on $\matrep$ as $\shiftfunc(\matrep):=\shiftmatrep\in \R^{\numchar^2\times \inputlength/2}$ where for all $j\in [\inputlength/2]$,
\begin{align*}
    \shiftmatrep^{(j)} = \Q\matrep^{(j)}\odot \Q^\T\matrep^{((j+1)\% \inputlength)}
\end{align*}
where $\shiftmat = \rbr{I_\numchar\otimes \1_\numchar}\rbr{\1_\numchar\otimes I_\numchar}^\top$
, $\1_n \in \R^{\numchar\times 1}$ is the all-ones vector, and $\odot$ is the Hadamard product.
\end{definition}
\end{defnbox}

\begin{thmbox}
\begin{lemma}[Equivalence of $\shiftfunc$ and circular shift]
\label{lemma:shift-equiv}\emph{}

    For any sequence $\intermseq\in\alphabetset^\inputlength$, let $\shiftintermseq$ be the circular shifted $\intermseq$, then $$\emph{\matricize}\rbr{\shiftintermseq} = \shiftfunc\rbr{\emph{\matricize}\rbr{\intermseq}}.$$
\end{lemma}
\end{thmbox}

\begin{proof}[Proof of \cref{lemma:shift-equiv}]

    In this proof we will show that $\emph{\matricize}\rbr{\shiftintermseq}$ and $\shiftfunc\rbr{\emph{\matricize}\rbr{\intermseq}}$ agrees on every column.

    Fix a column $j\in[\numchar/2]$,
    without loss of generality let the input sequence $\intermseq$ be $(a,b,c,d)\in\alphabetset$ starting from the $(2j-1)$-th position to the $(2j+2)$-th position (wrapped around when necessary).
    By construction each output column of $\shiftmatrep^{(j)}$ is dependent on at most $\matrep^{(j)}$ and $\matrep^{(j+1)}$ which corresponds to 4 characters in the sequence $\intermseq$.
    
    By definition of $\matrep$ we then have $\matrep^{(j)} = \vecrep(a,b) = \svecrep_a\otimes \svecrep_b$ and $\matrep^{(j+1)\% L} = \vecrep(c,d) = \svecrep_c\otimes \svecrep_d$. It follows that
    
    \begin{align*}
        \shiftmatrep^{(j)} &= \Q\matrep^{(j)}\odot \Q^\T\matrep^{((j+1)\% \inputlength)}\\
        &= \rbr{\rbr{I_\numchar\otimes \1_\numchar}\rbr{\1_\numchar^\top\otimes I_\numchar^\top} \rbr{\svecrep_a\otimes \svecrep_b}} \odot \rbr{\rbr{\1_\numchar\otimes I_\numchar} \rbr{I_\numchar^\top\otimes \1_\numchar^\top} \rbr{\svecrep_c\otimes \svecrep_d}}\\
        &= \rbr{\rbr{I_\numchar\otimes \1_\numchar}\rbr{\1_\numchar^\top\svecrep_a\otimes I_\numchar^\top \svecrep_b}} \odot \rbr{\rbr{\1_\numchar\otimes I_\numchar}\rbr{I_\numchar^\top\svecrep_c\otimes \1_\numchar^\top \svecrep_d}}\\
        &= \rbr{\rbr{I_\numchar\otimes \1_\numchar}\rbr{1\otimes  \svecrep_b}} \odot \rbr{\rbr{\1_\numchar\otimes I_\numchar}\rbr{\svecrep_c\otimes 1}}\tag{$\svecrep_a, \svecrep_d$ are one-hot, $\1_\numchar^\top \svecrep_a=\1_\numchar^\top \svecrep_d=1$}\\
        &= \rbr{\rbr{I_\numchar\otimes \1_\numchar}\rbr{\svecrep_b\otimes 1}} \odot \rbr{\rbr{\1_\numchar\otimes I_\numchar}\rbr{1\otimes\svecrep_c}}\\
        &=\rbr{\svecrep_b\otimes \1_\numchar} \odot \rbr{\1_\numchar \otimes \svecrep_c}\\
        &= \svecrep_b \otimes \svecrep_c. \tag{by definition of Kronecker product}
    \end{align*}
    Note that $\svecrep_b \otimes \svecrep_c$ is just $\matricize(\shiftintermseq)^{(j)}$ since the $(2j-1)$-th and the $2j$-th character of the shifted sequence $\shiftintermseq$ is now $(b,c)$.

    This concludes the proof.
\end{proof}

\textbf{(b) Translation from $\shiftmatrep_i$ to $\matrep_{i+1}$}

With $\shiftfunc$ effectively completing \merge{}, \circularshift{}, and \split{}, what remains is the translation leveraging $\pi_i$. Since $\alphabetset_i = \alphabetset_{i+1} = \alphabetset = [\numchar]$, there is a natural bijection between the space of binary column-stochastic matrix and the space of all possible (not necessarily bijective) mappings between 2-tuples from $\alphabetset$. Concretely

\begin{defnbox}
\begin{definition}[Column-Stochastic Matrix Representation of \strdict{}]\emph{}

\label{def:apdx-surrogate-translationMatrixRepresentation}
    Given \strdict{} $\dictionary:\alphabetset^2\to\alphabetset^2$, its matrix representation is defined to be $\transmat(\pi)\in\R^{\numchar^2\times\numchar^2}$ that for $i,j\in [\numchar^2]$, \begin{align*}
        \transmat(\pi)_{j,i} = \begin{cases}
            1 & \text{if}\ \pi(\ceil{i/\numchar}, i\%\numchar) = (\ceil{j/\numchar}, j\%\numchar)\\
            0 & \text{elsewhere}
        \end{cases}.
    \end{align*}
\end{definition}
\end{defnbox}

Let $\matrep_{i+1} = \transmat\shiftmatrep_{i}$ where $\transmat$ is a binary column-stochastic matrix, we can show that, a complete translation process for one level can be formally expressed as follows: 

\begin{thmbox}
\begin{lemma}[Equivalence of $\transmat$ and \translate{}]
\label{lemma:translate-equiv}\emph{}

    For any sequence $\intermseq_i\in\alphabetset^\inputlength$, let $\pi_i$ be a bijective \strdict{} $\alphabetset^2\to \alphabetset^2$ defined in \cref{sec:setup-translationTask}, then 
    \begin{align*} \matrep_{i+1} = \transmat{(\dictionary_i)}\shiftfunc\rbr{\emph{\matricize}\rbr{\intermseq_i}} = \emph{\matricize}\rbr{\transfunc_{\dictionary_i}(\intermseq_i)} = \emph{\matricize}(\intermseq_{i+1}).
    \end{align*}
\end{lemma}
\end{thmbox}

\begin{proof}[Proof of \cref{lemma:translate-equiv}] Fix any column $j\in[\inputlength/2]$, let $(a,b)$ be the $(2j-1)$-th and the $2j$-th character of the shifted sequence $\shiftintermseq_i$. Let $(c,d)=\pi_i(a,b)$, by the translation construction we know the $(2j-1)$-th and the $2j$-th character of $\intermseq_{i+1}$ is just $c$ and $d$. Hence $\matricize(\intermseq_{i+1})^{(j)} = v(c,d)$.

From \cref{lemma:shift-equiv} we know that $\shiftfunc(\matricize(\intermseq_i))^{(j)} = v(a,b)$. By construction of $\transmat{(\pi_i)}$ above, we know that $\transmat{(\pi_i)}\shiftfunc(\matricize(\intermseq_i))^{(j)}$ is a one-hot vector at the $(c\numchar + d)$-th position, i.e. $\matrep_{i+1}^{(j)} = \transmat{(\pi_i)}\shiftfunc(\matricize(\intermseq_i))^{(j)} = \svecrep_c\otimes\svecrep_c = \vecrep(c,d)$. Since $\matricize(\intermseq_{i+1})^{(j)} = \matrep_{i+1}^{(j)}$ for all $j$, we have $\matrep_{i+1} = \matricize(\intermseq_{i+1})$.
\end{proof}

\paragraph{Parameterization of the Translation Matrix $\P$.}\emph{}

With the two key operations in place, now we can introduce the surrogate model in its full detail. Since the multi-level translation task is a naturally sequential operation, we model the surrogate operations as a multi-layer network as well (which also matches with the solution found by Llama-based models when trained on real data as in \cref{sec:mechanistic}).

For each level, the surrogate model needs to present both in-context learning capability at the initialization and in-weight capability toward the end of \ICM{} on a certain \strcodebook{} $\targetCB$. The in-weight capability requires certain parameter to store $\targetCB$ on its own that is independent of the context.

To capture both capabilities at the same time, we parameterize the translation matrix $\P_i$ as a combination of in-context information $\contextmat_i\in\R^{\numchar^2\times\numchar^2}$ and in-weight memory $\weightmat_i\in\R^{\numchar^2\times\numchar^2}$.

\begin{defnbox}
\begin{definition}[Effective Translation Matrix]\label{defn:apdx-surrogate-effectivetransmat}
\emph{}\\
For in-context information $\contextmat_i\in\R^{\numchar^2\times\numchar^2}$ and in-weight memory $\weightmat_i\in\R^{\numchar^2\times\numchar^2}$ at level $i$, the corresponding effective translation matrix $\P_i$ is defined as
    \begin{align*}
    \P_i = \maxfunc\rbr{\contextmat_i + \weightmat_i}
\end{align*}
where $\maxfunc$ is the column-wise hard-max function converting $\contextmat_i + \weightmat_i$ to a binary column stochastic matrix.
\end{definition}
\end{defnbox}

Note that column $k$ in $\P_i$, which we denote as $\P_i^{(k)}$, is equal to the one-hot vector at $\argmax(\contextmat_i^{(k)} + \weightmat_i^{(k)})$.

\paragraph{Surrogate Model with In-Context Capability}\emph{}

Now we can formally introduce the surrogate model.

\begin{defnbox}
\begin{definition}[Surrogate Model for \translationtask ]\emph{}

\label{defn:apdx-surrogate-surrogateModel}
    The surrogate model for \translationtask$(\depth, \numchar)$ can be represented by the recursive expression 
\begin{align}
\label{eqn:theory-surrogateRep-recurrance-apdx}
    \matrep_{i+1} = \maxfunc(\contextmat_i + \weightmat_i)\shiftfunc(\matrep_i)
\end{align}
The learnable parameters for the surrogate model are the weight matrices $\surrogateWeights:=\cbr{\weightmat_1,\weightmat_2,\dots,\weightmat_\depth}$. 
We denote the surrogate model parameterized by $\surrogateWeights$ as $\surrogateModel_\surrogateWeights(\cdot)$ which maps the input and the in-context information to the output as 
\begin{align}
\begin{split}
    \matrep_{\depth+1} &= \surrogateModel_\surrogateWeights(\contextmat_1, \contextmat_2,\dots, \contextmat_\depth, \matrep_1)
    \\
    &\triangleq \maxfunc(\contextmat_\depth + \weightmat_\depth)\\
    &\qquad \shiftfunc\rbr{\maxfunc(\contextmat_{\depth - 1} + \weightmat_{\depth-1})\shiftfunc\rbr{\cdots \maxfunc(\contextmat_1 + \weightmat_1)\shiftfunc(\matrep_1) \cdots}}.
\end{split}
\end{align}
\end{definition}
\end{defnbox}

In this surrogate model, the in-context descriptive text  $\taskdesc$ is just $\scbr{\contextmat_1,\dots,\contextmat_\depth}$, where each matrix is directly being passed into the corresponding layer. When providing information about a \strcodebook{} $\codebook = \scbr{\dictionary_i}_{i=1}^\depth$, we have $\contextmat_i = \transmat{(\dictionary_i)}$
where $\transmat{(\dictionary_i)}$ is the column-stochastic matrix representation of $\dictionary_i$ as defined in \cref{def:apdx-surrogate-translationMatrixRepresentation}.

When partial \strdict{} information is provided (corresponding to dropping certain \texttt{ab->CD} entries in the language model context), we zero-out the corresponding column in $\contextmat_i$. When no in-context information is provided for level $i$, we just have $\contextmat_i$ be the all-zero matrix containing no information (assuming zero as prior).

For simplicity of presentation, in \cref{defn:apdx-surrogate-surrogateModel} the in-context information $\contextmat_i$'s are provided directly to the corresponding layers. We note that the equivalent operations can be exactly re-parameterized such that $\contextmat_i$’s are provided in-context (in concatenation with $\matrep_1$). The reparameterization of the surrogate model is provided below:

\begin{defnbox}
\begin{definition}[Context-Augmented Surrogate Model for \translationtask ]\emph{}

\label{defn:apdx-surrogate-surrogateModel-context-augmented}
 
Let the context-augmented input be $$\mathbf{X}_1 = [\contextmat_1,\dots, \contextmat_d, \matrep_1]\in\R^{n^2\times (dn^2+L)},$$ then we can rewrite the same surrogate model as 

\begin{align}
\begin{split}
\mathbf{X}_{i+1} &= 
\mathbf{X}_i\begin{bmatrix} I_{dn^2} & 0\\0 & 0_{L\times L} \end{bmatrix} + \rbr{\rbr{\mathbf{X}_i
\begin{bmatrix}e_i\otimes I_{n^2}\\0_{L\times n^2}\end{bmatrix}
+ \weightmat_i}\texttt{Shift}\rbr{\mathbf{X}_i
\begin{bmatrix}0_{n^2d\times L}\\I_L\end{bmatrix}
}\begin{bmatrix}0_{dn^2\times dn^2} & 0\\0 & I_{L}\end{bmatrix}}\\
&=[\contextmat_1,\dots, \contextmat_d, 0_{n^2\times L}] + [0,\dots, 0, \rbr{\contextmat_i + \weightmat_i}\texttt{Shift}\rbr{\matrep_i}]\\
&=[\contextmat_1,\dots, \contextmat_d, \matrep_{i+1}]
\end{split}
\end{align}

With the reparameterization, the model is capable of generating $[\contextmat_1,\dots, \contextmat_d, \matrep_{d+1}]$ with input $[\contextmat_1,\dots, \contextmat_d, \matrep_1]$, with all information provided in-context in the input.

\end{definition}
\end{defnbox}

Now let us check what does \cref{def:icl-capable-model} (ICL-capable) and \cref{def:task-capable-model} (specific task-capable) mean in the context of the surrogate model. To make things more rigorous we introduce two stronger notions of capabilities:

\def \surrogateICLcapable {\translationtask$(\depth, \numchar)$-ICL-capable}
\def \surrogatecapable {$\translationtask_\targetCB$-capable}

\begin{defnbox}
\begin{definition}[Strongly \surrogateICLcapable{} surrogate model]
\label{def:icl-capable-model-exact}
\emph{}\\
    We say a surrogate model $\surrogateModel_\surrogateWeights(\cdot)$ is strongly \surrogateICLcapable{} if for any \strcodebook{} $\codebook = \cbr{\dictionary_i}_{i=1}^\depth$ in \translationtask$(n,d)$, for any input sequence $\inputseq\in\alphabetset^\inputlength$ where $L$ is even, we have $$\surrogateModel_\surrogateWeights(\transmat{(\dictionary_1)}, \dots, \transmat{(\dictionary_\depth)}, \matricize(\inputseq)) = \matricize(\translationtask_\codebook(\inputseq)).$$
\end{definition}
\end{defnbox}

\begin{defnbox}
\begin{definition}[Strongly \surrogatecapable{} surrogate model]
\label{def:task-capable-model-exact}
\emph{}\\
    For a fixed \strcodebook{} $\targetCB = \cbr{\dictionary^*_i}_{i=1}^\depth$ in \translationtask$(n,d)$, we say a surrogate model $\surrogateModel_\surrogateWeights(\cdot)$ is strongly \surrogatecapable{} if for any input sequence $\inputseq\in\alphabetset^\inputlength$ where $L$ is even, we have $$\surrogateModel_\surrogateWeights(\zeromat, \dots, \zeromat, \matricize(\inputseq)) = \matricize(\translationtask_\targetCB(\inputseq)).$$
\end{definition}
\end{defnbox}

Now we can show the following properties of the surrogate model $\surrogateModel_\surrogateWeights(\cdot)$:

\begin{thmbox}
\begin{lemma}
\label{lemma:surrogateICLCapable-apdx}
    When $\norm{\weightmat_i}_0 < \frac12$ for all $i\in[\depth]$, $\emph{\surrogateModel}_\surrogateWeights$ is strongly $\emph{\translationtask}(\depth,\numchar)$-ICL-capable.
\end{lemma}
\end{thmbox}

\begin{proof}
    Fix any \strcodebook{} $\codebook = \cbr{\dictionary_i}_{i=1}^\depth$ and its corresponding matrix representations $\scbr{\transmat{(\dictionary_i)}}_{i=1}^d$. Since $\norm{\weightmat_i}_0 < \frac12$, for any column $k$, no entries in $\weightmat_i^{(k)}$ can flip the argmax of $\transmat(\dictionary_i)^{(k)}+\weightmat_i^{(k)}$ away from being $\argmax\transmat(\dictionary_i)^{(k)}$. Therefore we have $\P_i = \maxfunc(\transmat{(\dictionary_i)} + \weightmat_i) = \transmat{(\dictionary_i)}$ for all layers $i\in[\depth]$.     
    
    By \cref{lemma:translate-equiv}, for all $i\in[\depth]$ we have $\P_i = \transmat{(\dictionary_i)}$ recovering $\transfunc_{\dictionary_i}$. Hence for any input sequence $\inputseq\in\alphabetset^\inputlength$, we have $\surrogateModel_\surrogateWeights(\cbr{(\dictionary_1)}, \dots, \cbr{(\dictionary_\depth)}, \matricize(\inputseq)) = \matricize(\translationtask_\codebook(\inputseq))$.
\end{proof}

\begin{thmbox}
\begin{lemma}
\label{lemma:surrogateCapable-apdx}
    Fix a target \strcodebook{} $\targetCB = \cbr{\dictionary^*_i}_{i=1}^\depth$, when $\emph{\maxfunc}\rbr{\weightmat_i} = \transmat{(\dictionary^*_i)}$ for all  $i\in[\depth]$, $\emph{\surrogateModel}_\surrogateWeights(\cdot)$ is strongly $\emph{\translationtask}_\targetCB$-capable.
\end{lemma}
\end{thmbox}

\begin{proof}
    By \cref{lemma:translate-equiv}, for all $i\in[\depth]$ we have $\P_i = \maxfunc(\weightmat_i + \zeromat) = \transmat{(\dictionary^*_i)}$ recovering $\transfunc_{\dictionary^*_i}$. Hence for any input sequence $\inputseq\in\alphabetset^\inputlength$, we have $\surrogateModel_\surrogateWeights(\zeromat, \dots, \zeromat, \matricize(\inputseq)) = \matricize(\translationtask_\targetCB(\inputseq))$.
\end{proof}

\cref{lemma:surrogateICLCapable-apdx} suggests that when the weight matrices have small initializations, the model has perfect ICL capability. Meanwhile \cref{lemma:surrogateCapable-apdx} suggests that when the weights $\weightmat_i$ recover $\transmat{(\dictionary^*_i)}$ in the column-wise hard-max sense, then the surrogate model can perform $\translationtask_\targetCB$ when no context is being provided ($\contextmat_i = \zeromat$).

\subsection{Learning $\targetCB$ in $\translationtask(\depth,\numchar)$ with Heuristics Search}
\label{sec:apdx-surrogateModel-learningGeneral}

In this section, we provide a brute-force algorithm that can learn any target \strcodebook{} $\targetCB$ in $\translationtask(\depth,\numchar)$ using a single ``short'' sequence whose length is not exponentially dependent on $\depth$.

Before proceeding to the details, let us first investigate more on the nature of $\translationtask$ and the surrogate model. For simplicity of notations, given $\targetCB = \cbr{\dictionary^*_1,\dots,\dictionary^*_d}$, we denote the general translation operator $\transmat$ as $\P$, and denote the translation operator $\transmat^{(\dictionary^*_1)}$ as $\W_i^*$. Also, with slight abuse of notations we use ${\translationtask}_\targetCB(\matrep_1)$ to denote $\matricize\rbr{{\translationtask}_\targetCB(\matricize^{-1}(\matrep_1))}.$

First, we characterize the input sequence that is good for providing learning signals.

\begin{defnbox}
\begin{definition}[$\targetCB$-coverable input]\emph{}

    Fix a target \strcodebook{} $\targetCB = \cbr{\dictionary^*_1,\dots,\dictionary^*_d}$ in $\translationtask(\depth,\numchar)$ and an input matrix $\matrep_1\in\R^{\numchar^2\times \inputlength}$, let $\shiftmatrep^*_1, \matrep^*_2, \shiftmatrep^*_2,\dots, \shiftmatrep^*_\depth, \matrep^*_{\depth + 1}\in\R^{\numchar^2\times \inputlength}$ be the intermediate outputs when applying $\translationtask_\targetCB$ on $\matrep_1$. We say $\matrep_1$ is $\targetCB$-coverable if for all levels $i\in[d]$, $\shiftmatrep^*_i$ is of rank-$\numchar^2$.
\end{definition}
\end{defnbox}

Note that as a matrix with only one-hot columns, $\shiftmatrep^*_i$ being rank-$\numchar^2$ suggests that for all $k\in[\numchar^2]$, there exists some column $j\in[\inputlength]$ such that $\shiftmatrep_i^{*(j)} = \onehot_k$. In the context of the translation process, it means that the correct translation process of a $\targetCB$-coverable input $\matrep_1$ would require all entries of all \strdicts{} in $\targetCB$.

Next we will show that if we use the context to condition all translation operators $\P_l$ of the surrogate model to be $\P^{(\pi^*_l)}$ for all but one level $l\in[d]\backslash\cbr{i}$ ($i$ as the unconditioned level), then for the surrogate model to correctly perform $\translationtask_\targetCB$, the operator for the unconditioned level $i$ must also be equal to $\P^{(\pi^*_i)}$. 

\begin{lemma}[Uniqueness of a single $\P_i$ when conditioning all other levels]\emph{}
\label{lemma:apdx-surrogate-heuristics-PiUniqueness}

Fix a target \strcodebook{} $\targetCB = \cbr{\dictionary^*_1,\dots,\dictionary^*_d}$ in $\emph{\translationtask}(\depth,\numchar)$ and a $\targetCB$-coverable input $\matrep_1$. For any level $i\in[\depth]$, consider a surrogate model $\emph{\surrogateModel}_\surrogateWeights(\contextmat_1, \contextmat_2,\dots, \contextmat_\depth, \cdot)$ with certain context $\contextmat_1, \contextmat_2,\dots, \contextmat_\depth$ such that $\P_l = \W_l^*$ for all $l\in[\depth]$ except $l=i$. Then $\emph{\surrogateModel}_\surrogateWeights(\contextmat_1, \contextmat_2,\dots, \contextmat_\depth, \matrep_1) = \emph{\translationtask}_\targetCB(\matrep_1)$ if and only if $$\P_i = \emph{\maxfunc}\rbr{\contextmat_i + \W_i} = \W_i^*.$$
\end{lemma}

\begin{proof}
    This lemma is a direct consequence of the bijective property of the translation process shown in \cref{lemma:setting-MLTinvertability}. Let $\shiftmatrep^*_1, \matrep^*_2, \shiftmatrep^*_2,\dots, \shiftmatrep^*_\depth, \matrep^*_{\depth + 1}\in\R^{\numchar^2\times \inputlength}$ be the intermediate outputs when applying $\translationtask_\targetCB$ on $\matrep_1$, and let $\shiftmatrep_1, \matrep_2, \shiftmatrep_2,\dots, \shiftmatrep_\depth, \matrep_{\depth + 1}\in\R^{\numchar^2\times \inputlength}$ be the intermediate outputs when applying $\surrogateModel_\surrogateWeights(\contextmat_1, \contextmat_2,\dots, \contextmat_\depth, \cdot)$ as described. Since we assume $\P_l = \P^{(\pi^*_l)}$ for all $l < i$, we have $\matrep_{i} = \matrep^*_{i}$ and therefore $\shiftmatrep_{i} = \shiftmatrep^*_{i}$.

    On the other end, since $\surrogateModel_\surrogateWeights(\contextmat_1, \contextmat_2,\dots, \contextmat_\depth, \matrep_1) = \translationtask_\targetCB(\matrep_1) = \matrep^*_{\depth+1}$ and $\P_l = \P^{(\pi^*_l)}$ for all $l > i$, by the invertible property of the translation process (\cref{lemma:setting-MLTinvertability}) we must have $\shiftmatrep_{i+1} = \shiftmatrep^*_{i+1}$ and thus $\matrep_{i+1} = \matrep^*_{i+1}$. Combining both ends we know that \begin{align}
        \P_{i}\shiftmatrep_{i} = \matrep_{i+1} = \matrep^*_{i+1} =\W_i^*\shiftmatrep^*_{i} = \W_i^*\shiftmatrep_{i}.
    \end{align}
    Since $\shiftmatrep_{i} = \shiftmatrep^*_{i}$ is rank $\numchar^2$ by the $\targetCB$-coverable assumption and $\P_{i}\in\R^{\numchar^2\times\numchar^2}$, it must be so that $\P_{i}=\W^*_i$.
\end{proof}

If we further condition on the held-out level $\P_i$ such that we only leave one column of $\P_i^{(k)}$ not necessarily equal to $\W_i^{*(k)}$, we have the following corollary
\begin{corollary}[Uniqueness of a single $\P_i$ column when conditioning everything else]\emph{}
\label{cor:apdx-surrogate-holdingOutCol}

    Fix a target \strcodebook{} $\targetCB = \cbr{\dictionary^*_1,\dots,\dictionary^*_d}$ in $\emph{\translationtask}(\depth,\numchar)$ and a $\targetCB$-coverable input $\matrep_1$. For any level $i\in[\depth]$ and translation entry $k\in[\numchar^2]$, consider a surrogate model $\emph{\surrogateModel}_\surrogateWeights(\contextmat_1, \contextmat_2,\dots, \contextmat_\depth, \cdot)$ with certain context $\contextmat_1, \contextmat_2,\dots, \contextmat_\depth$ such that $\P_l^{(j)} = \W_l^{*(j)}$ for all $(l,j)\in[\depth]\times[\numchar^2]\backslash\cbr{(i,k)}$. Then $\emph{\surrogateModel}_\surrogateWeights(\contextmat_1, \contextmat_2,\dots, \contextmat_\depth, \matrep_1) = \emph{\translationtask}_\targetCB(\matrep_1)$ if and only if $$\P_i^{(k)} = \emph{\maxfunc}\rbr{\contextmat_i^{(j)} + \W_i^{(j)}} = \W^{*(k)}_i.$$
\end{corollary}

This suggests that if we condition everything else except for one column of the translation operator, then to match the final output on a $\targetCB$-coverable sequence, the model must recover the held-out column to be the same as the ground truth in the \strcodebook{}.

Now we can introduce the search algorithm, which simply enumerate over all translation columns $\W_{i}^{(j)}$ as learning target, generate a contextual information that only leaves that column unconditioned, and search over all possible one-hot vectors for $\W_{i}^{(j)}$ until the output matches with $\translationtask_\targetCB$. Once the output matches, by \cref{cor:apdx-surrogate-holdingOutCol} we know $\maxfunc\srbr{\W_{i}^{(j)}}$ recovers $\W_{i}^{*(j)}$ and we move on to the next learning target. The algorithm can be formalized as follows:

\begin{algorithm}[H]
   \caption{Context-Enhanced Searching Algorithm for \translationtask$(\depth,\numchar)$}
   \label{alg:mlt-surrogate-search}
\begin{algorithmic}[1]
    \STATE {\bfseries Input:}
    \STATE input $\matrep_1\in\R^{n^2\times L}$, label  $\matrep_{\depth+1}^*\in\R^{n^2\times L}$, descriptive text $\gtW_1,\dots,\gtW_{\depth}\in\R^{n^2\times n^2}$ 
    \STATE
    \STATE \textbf{Initialize} $\W_1,\dots, W_\depth\gets \zeromat$ \COMMENT{Start with zero initialization}
    \FOR{$i=1$ {\bfseries to} $d$} 
    \FOR{$k=1$ {\bfseries to} $n^2$} 
    \STATE \textbf{Initialize} $\CM_{i(k)}\gets \gtW_i\rbr{\mI_{\numchar^2} - \diag(\lonehot_k)}$ \COMMENT{Create masked context matrix}
    \STATE \text{\# Search Loop} \label{line:search-loop}
    \FOR{$a=1$ {\bfseries to} $\numchar^2$}
        \STATE $\W_{i}^{(k)}\gets \lonehot_a$ \COMMENT{Search over one-hot columns}
        \STATE $\matrep_{\depth+1} \gets \surrogateModel_\surrogateWeights(\W^*_1\dots,\W^*_{i-1}, \CM_{i(k)},\W^*_{i+1}\dots,\W^*_{\depth}, \matrep_1)$
        \IF{$\matrep_{\depth+1} = \matrep^*_{\depth+1}$}
            \STATE \textbf{break} \COMMENT{Break when found the right column}
        \ENDIF
    \ENDFOR
    \ENDFOR
    \ENDFOR
    \STATE \textbf{Return} $\W_1,\dots,\W_\depth$.
\end{algorithmic}
\end{algorithm}

\begin{thmbox}
\begin{theorem}[Learning $\targetCB$ with context-enhanced search with $\targetCB$-coverable input]\emph{}
\label{thm:apdx-surrogate-search-learning-coverable}

For any target \strcodebook{} $\targetCB = \cbr{\dictionary^*_1,\dots,\dictionary^*_d}$ in $\emph{\translationtask}(\depth,\numchar)$, given an $\targetCB$-coverable input $\matrep_1$ and the corresponding ground truth label $\matrep^*_{d+1} = \emph{\translationtask}_\targetCB(\matrep_1)$, \cref{alg:mlt-surrogate-search} terminates with $\W_i = \W_i^*$ for all $i\in[\depth]$ with $O(n^4d)$ forward passes through the surrogate model.
\end{theorem}
\end{thmbox}

\begin{proof}
    The statement can be proven by a simple induction.
    
    Let the inductive hypothesis be such that when the enumeration goes to $k$-th column of the $i$-th layer, if $\maxfunc\srbr{\W_l^{(j)} + \W_l^{*(j)}} = \W^{*(j)}_l$ for all $(l,j)\in[\depth]\times[\numchar^2]$ and $\W_l^{(j)} = \W_l^{*(j)}$ for all $(l,j)$ such that $l<i$ or $l=i\land j<k$, then the search loop (\cref{line:search-loop}: line 13) breaks with $\W_i^{(k)} = \W_i^{*(k)}$ while $\maxfunc\srbr{\W_l^{(j)} + \W_l^{*(j)}} = \W^{*(j)}_l$ for all $(l,j)\in[\depth]\times[\numchar^2]$ is preserved.

    The base case is satisfied as with zero initialization, we have $\maxfunc\srbr{\W_l^{(j)} + \W_l^{*(j)}} = \maxfunc\srbr{\zeromat + \W_l^{*(j)}} = \W^{*(j)}_l$ for all $(l,j)\in[\depth]\times[\numchar^2]$ and there are no requirements for $\W_i^{(k)} = \W_i^{*(k)}$ yet.

    For the induction step, we note that with the condition of $\maxfunc\srbr{\W_l^{(j)} + \W_l^{*(j)}} = \W^{*(j)}_l$ for all $(l,j)\in[\depth]\times[\numchar^2]$, $\W^*_1\dots,\W^*_{i-1}, \CM_{i(k)},\W^*_{i+1}\dots,\W^*_{\depth}$ will correctly condition all columns of $\P$'s except for the $\P_{i}^{(k)}$ since 
    \begin{align}
        \CM_{i(k)}^{(k)} = \gtW_i\rbr{\mI_{\numchar^2} - \diag(\lonehot_k)}^{(k)} = \zeromat.
    \end{align}
    Thus by \cref{cor:apdx-surrogate-holdingOutCol}, we know that the search loop will terminate when it finds $\W_{i}^{(k)} = \W_{i}^{*(k)}$. The newly added column provides the correct inductive hypothesis on $\W_l^{(j)} = \W_l^{*(j)}$ for the next enumeration step.

    By induction to $i=\depth$ and $k=\numchar^2$, we will be able to recover $\W_i = \W_i^*$ for all $i\in[\depth]$.
\end{proof}

Given that we can learn $\W_i$ effectively with $\targetCB$-coverable input, how should we construct such inputs? It turned out that with high probability, short random strings suffices.

\begin{lemma}[Distribution of intermediate sequences]
\label{lemma:apdx-surrogate-sequence-distn}
    Fix a target \strcodebook{} $\targetCB = \cbr{\dictionary^*_1,\dots,\dictionary^*_d}$ in $\emph{\translationtask}(\depth,\numchar)$. Let $\rmatrep_1\in\R^{\numchar^2\times \inputlength}$ be a random input matrix to $\emph{\translationtask}(\depth,\numchar)$ such that each column $\rmatrep_1^{(j)}$ is i.i.d. sampled from $\cU\srbr{\cbr{\lonehot_k}_{k=1}^{\numchar^2}}$ (the uniform distribution over one-hot vectors $\cbr{\lonehot_k}_{k=1}^{\numchar^2}$), the columns of intermediate random sequences $\rshiftmatrep^*_1, \rmatrep^*_2, \rshiftmatrep^*_2,\dots, \rshiftmatrep^*_\depth, \rmatrep^*_{\depth + 1}\in\R^{\numchar^2\times \inputlength}$ obtained by passing the input $\rmatrep_1$ through $\emph{\translationtask}_\targetCB$ also follow the same i.i.d. uniform distribution.
\end{lemma}

\begin{proof}
    We will prove the claim by induction on depth $i$. Let the inductive hypothesis be that columns in $\rmatrep_i$ independently follow an uniform distribution over the one-hot vectors $\cbr{\lonehot_k}_{k=1}^{\numchar^2}$. Note that the base case is just the assumption.

    Now we prove for the inductive step. For any $j\in[\inputlength]$, we can write $\RV{i}{j} = \onehot_{\rai{2j-1}}\otimes\onehot_{\rai{2j}}$ where $\rai{i}$'s follows i.i.d. $\cU([\numchar])$. Intuitively this means each 2-tuple in the random input sequence is formed from two i.i.d. uniformly random characters, which is straightforward by construction.
    
    Now by \cref{lemma:shift-equiv} we have $\RVS{i}{j} = \onehot_{\rai{2j}}\otimes\onehot_{\rai{2j+1}}$ also following $\cU\srbr{\cbr{\lonehot_k}_{k=1}^{\numchar^2}}$. Note that there is total independency of $\RVS{i}{j}$ with respect to the set of any other columns of $\rshiftmatrep_i$ since $\rai{2j}$ and $\rai{2j+1}$ are independent from the generative process of any other columns in $\rshiftmatrep_i$.

    With columns in $\rshiftmatrep_i$ i.i.d. following $\cU\srbr{\cbr{\lonehot_k}_{k=1}^{\numchar^2}}$, permuting the indices via $\P^{(\dictionary^*_i)}$ does not change the distribution by symmetry of the uniform distribution. Therefore we have columns of $\rmatrep_{i+1} = \P^{(\dictionary^*_i)}\rshiftmatrep_i$ also i.i.d. distributed as  $\cU\srbr{\cbr{\lonehot_k}_{k=1}^{\numchar^2}}$ and this complete the inductive step. By induction on $i$ we have the desired statement proved.
\end{proof}

Since each column in $\rshiftmatrep_i$ i.i.d. follows $\cU\srbr{\cbr{\lonehot_k}_{k=1}^{\numchar^2}}$, sampling $\rshiftmatrep_i$ to be rank $\numchar^2$ becomes identical to the classic coupon collection problem (see \cref{lemma:apdx-surrogate-coupon} from \citet{motwani1995randomized}). Thus we have the following bound:

\begin{thmbox}
\begin{lemma}[Short $\targetCB$-coverable random sequence]\emph{}

    A random sequence $\rmatrep_1$ of length $L \geq 2\numchar^2\log\frac{\numchar\depth}{\delta}$ is $\targetCB$-coverable with probability at least $1-\delta$. 
\end{lemma}
\end{thmbox}

\begin{proof}
    Let the event $A_i$ denote that $\rshiftmatrep_i$ is not rank $n^2$. By \cref{lemma:apdx-surrogate-sequence-distn}, each column of $\rshiftmatrep_i$ is i.i.d. distributed following $\cU\srbr{\cbr{\lonehot_k}_{k=1}^{\numchar^2}}$. Thus making $\rshiftmatrep_i$ being rank $n^2$ is equivalent to a coupon collecting problem \citep{motwani1995randomized} with set size $n^2$. By \cref{lemma:apdx-surrogate-coupon} we know that with $L = 2\numchar^2\log\frac{\numchar\depth}{\delta}$, $\pr{A_i}\leq \frac{\delta}{\depth}$. Thus by a simple union bound the probability that $\rshiftmatrep_i$ being not $\targetCB$-coverable is 
    \begin{align}
        \pr{\bigcup_{i=1}^d A_i} \leq \sum_{i=1}^d\pr{A_i} \leq \depth\frac
        \delta\depth = \delta.
    \end{align}
\end{proof}

Now we can apply the above result and extend \cref{thm:apdx-surrogate-search-learning-coverable}.

\begin{thmbox}
\begin{corollary}[Learning $\targetCB$ with random input using heuristics search]\emph{}
\label{cor:apdx-surrogate-search-learning-random}

For any target \strcodebook{} $\targetCB = \cbr{\dictionary^*_1,\dots,\dictionary^*_d}$ in $\emph{\translationtask}(\depth,\numchar)$, with probability at least $1-\delta$ over a uniformly random input $\matrep_1$ of length $L = 2\numchar^2\log\frac{\numchar\depth}{\delta}$, \cref{alg:mlt-surrogate-search} provided with ground truth label $\matrep^*_{d+1} = \emph{\translationtask}_\targetCB(\matrep_1)$ terminates with $\W_i = \W_i^*$ for all $i\in[\depth]$ with $O(n^4d)$ forward passes through the surrogate model.
\end{corollary}
\end{thmbox}
\subsection{Learning $\targetCB$ in $\translationtask(2,\numchar)$ with Surrogate Gradient Descent}
\label{sec:apdx-surrogateModel-learningD2}

In this section we take the analysis one step beyond the heuristics searching regime. We will show that any \strcodebook{} $\targetCB = \cbr{\dictionary^*_1, \dictionary^*_2}$ can be sample-efficiently learned by a gradient-descent based algorithm. In this particular case, the surrogate model is parameterized by
\begin{align}
    \matrep_{3} &= \surrogateModel_\surrogateWeights(\CM_1, \CM_2, \matrep_1)
    \triangleq \maxfunc(\CM_2 + \weightmat_2)\shiftfunc\rbr{\maxfunc(\CM_{1} + \weightmat_{1})\shiftfunc\rbr{\matrep_1}}.
\end{align}

We start with any weight initializations $\W_1^{(0)}, \W_2^{(0)} \in\R^{n^2\times n^2}$ satisfying $\snorm{\W_1^{(0)}}_1 < \frac12, \snorm{\W_2^{(0)}}_1 < \frac12$, by \cref{lemma:surrogateICLCapable-apdx} the initialization is strongly \surrogateICLcapable{}.

For simplicity, we denote the ground truth permutation matrix induced by $\dictionary^*_1$ as $\gtW_1\triangleq \P^{(\dictionary^*_1)}$ and similarly the ground truth permutation matrix induced by $\dictionary^*_2$ as $\gtW_2\triangleq \P^{(\dictionary^*_2)}$.
From \cref{lemma:surrogateCapable-apdx} we know that the learning is successful if we have $\maxfunc\rbr{\W_1} = \gtW_1$ and $\maxfunc\rbr{\W_2} = \gtW_2$ \footnote{Note that this is not the unique solution to attain strongly $\translationtask_\targetCB$-capable model} (i.e. the maximum index of each weight column agrees with that of the ground truth).

We employ a layer-wise gradient descent algorithm for the learning process. The algorithm takes in a single fixed sequence $\inputseq$ with matrix representation $\matrep_1$ and its corresponding ground truth label
\begin{align}
\matrep_3^* \triangleq \matricize(\translationtask_\targetCB(\inputseq)) = \surrogateModel_{\srbr{\W_1^{(0)}, \W_2^{(0)}}}(\W^*_1, \W^*_2, \matrep_1).
\end{align}

Given the input and label, we employ the following gradient descent based algorithm to update the weights:

The training happens in a layer-wise and column-wise fashion:
We first freeze $\W_2$ and set $\W_1$ as the trainable parameter. For each entry $k\in [\numchar^2]$, we create a context matrix $\CM_{1(k)}\triangleq \gtW_1\rbr{\mI_{\numchar^2} - \diag(\ve_k)}$ which essentially creates a copy of $\gtW_1$ except of setting the $k$-th column to be zero. Then we take a forward pass through the surrogate model with the one-column dropped-out context and get output $\matrep_{3(1,k)} \triangleq \surrogateModel_{\srbr{\W_1, \W_2}}(\CM_{1(k)}, \W^*_2, \matrep_1)$. Here the subscript $\cdot_{(1,k)}$ denotes the final output when the $i$-th column of the first context matrix is being dropped.

The weight update follows a surrogate gradient update scheme where we use the MSE loss: $\gdloss = \snorm{\matrep_{3(1,k)} - \matrep_{3}^*}_2^2$. Since it is difficult to take gradient through the hardmax function, we instead compute the gradient of the loss with respect to the translation matrix $\P_1 = \maxfunc\rbr{\W_1 + \CM_{1(k)}}$ and apply the update $\W_1^{(k)}\gets \W_1^{(k)} - \pdr{\gdloss}{\P_1^{(k)}}$. We apply such gradient update \emph{twice} for each dropped column $k$.

For the second layer, we freeze the first layer $\W_1$ and apply one-column dropouts to the $\CM_2$. Similarly we apply the surrogate gradient update $\W_2\gets \W_2 - \pdr{\gdloss}{\P_2}$ but we only need one gradient step per column.
\begin{algorithm}[tb]
   \caption{Context-Enhanced Layerwise Gradient Descent}
   \label{alg:mlt-d2-GD}
\begin{algorithmic}[1]
    \STATE {\bfseries Input:} input $\matrep_1\in\R^{n^2\times L}$, label  $\matrep_3^*\in\R^{n^2\times L}$, descriptive text $\gtW_1, \gtW_2\in\R^{n^2\times n^2}$, init $\W_1^{(0)}, \W_2^{(0)} \in\R^{n^2\times n^2}$ 
    \STATE
    \STATE \# Train the first layer
    \FOR{$k=1$ {\bfseries to} $n^2$} 
    \STATE $\CM_{1(k)}\triangleq \gtW_1\rbr{\mI_{\numchar^2} - \diag(\ve_k)}$ \COMMENT{Create context matrix with $k$-th column dropped.}
    \FOR{$t=1$ {\bfseries to} $2$}
        \STATE $\matrep_{3(1,k)} \gets \surrogateModel_{\srbr{\W_1, \W_2}}(\CM_{1(k)}, \W^*_2, \matrep_1)$\COMMENT{Forward pass}
        \STATE $\cL\gets \snorm{\matrep_{3(1,k)} - \matrep_{3}^*}_2^2$
        \STATE $\W_1^{(k)}\gets \W_1^{(k)} - \pdr{l}{\P_1^{(k)}}$ \COMMENT{Surrogate gradient update}
    \ENDFOR
    \ENDFOR
    \STATE
    \STATE \# Train the second layer
    \FOR{$k=1$ {\bfseries to} $n^2$} 
    \STATE $\CM_{2(k)}\triangleq \gtW_1\rbr{\mI_{\numchar^2} - \diag(\ve_k)}$ \COMMENT{Create context matrix with $k$-th column dropped.}
    \STATE $\matrep_{3(2,k)} \gets \surrogateModel_{\srbr{\W_1, \W_2}}(W^*_1, \CM_{2(k)}, \matrep_1)$ \COMMENT{Forward pass}
    \STATE $\cL\gets \snorm{\matrep_{3(2,k)} - \matrep_{3}^*}_2^2$
    \STATE $\W_2^{(k)}\gets \W_2^{(k)} - \pdr{l}{\P_2^{(k)}}$ \COMMENT{Surrogate gradient update}
    \ENDFOR
    \STATE Return $\W_1, \W_2$
\end{algorithmic}
\end{algorithm}

We claim the surrogate gradient descent update can correctly recover $\P_1^*$ and $\P_2^*$ similar to the heuristics search case.

\begin{thmbox}
\begin{restatable*}[Learning $\targetCB$ with context-enhanced surrogate GD with $\targetCB$-coverable input]{theorem}{thmsurrogdcoverable}
\label{thm:apdx-surrogate-GD-learning-coverable}
\emph{}

For any initialization $\W_{1(0)},\W_{2(0)}\in\R^{\numchar^2\times\numchar^2}$ such that $\snorm{\W_{1(0)}}_0\leq \frac12$ and $\snorm{\W_{2(0)}}_0\leq \frac12$,
for any target \strcodebook{} $\targetCB = \cbr{\dictionary^*_1,\dictionary^*_2}$ in $\emph{\translationtask}(2,\numchar)$, given an $\targetCB$-coverable input $\matrep_1$ and the corresponding ground truth label $\matrep^*_{3} = \emph{\translationtask}_\targetCB(\matrep_1)$, \cref{alg:mlt-surrogate-search} terminates with $\emph{\maxfunc}\rbr{\W_1} = \W_1^*$ and $\emph{\maxfunc}\rbr{\W_2} = \W_2^*$.
\end{restatable*}
\end{thmbox}

To prove for \cref{thm:apdx-surrogate-GD-learning-coverable}, we will carefully analyze the learning of the first layer and second layer respectively, and provide a similar induction argument as in the proof for the heuristics search case.

\subsubsection{Learning the First Layer}

To study the learning process we first need to compute the closed-form gradient $\pdr{\gdloss}{\P_1^{(k)}}$, which requires the following lemma:

\begin{lemma}[Gradient with respect to incorrect column in $\P_1$]\emph{}

\label{lemma:apdx-surrogate-grad-key-W1}
    When only the $k$-th column of translation matrix $\P_1^{(k)}$ is not equal to $\P_1^{*(k)}$, let $\P_1^{(k)} = \onehot_{a}\otimes \onehot_{b}$ and $\P_1^{*(k)} = \onehot_{a^*}\otimes \onehot_{b^*}$ for some $a,b,a^*,b^*\in[\numchar]$. If $\P_1^{(k)}$ is used in the forward pass, there exists $\alpha\in\mathbb{Z}_+$ and $\beta\in\NN$ such that the gradient of $\gdloss$ with respect to $\P_1^{(k)}$ is of the form
    \begin{align*}
        \pdr{\gdloss}{\P_1^{(k)}} = \begin{cases}
        (2\alpha + 2\beta) \rbr{\1_\numchar\otimes\onehot_b}  - (2\alpha + 2\beta) \rbr{\1_\numchar\otimes\onehot_{b^*}} + 2\beta\rbr{\onehot_{a^*}\otimes\1_\numchar}  & \text{if } a^*=a, b^*\neq b\\
        (2\alpha + 2\beta) \rbr{\onehot_{a}\otimes\1_\numchar}  - (2\alpha + 2\beta) \rbr{\onehot_{a^*}\otimes\1_\numchar} + 2\beta\rbr{\1_\numchar\otimes\onehot_{b^*}} & \text{if } a^*\neq a, b^*= b\\
        (2\alpha + 2\beta) \rbr{\onehot_{a}\otimes\1_\numchar}  + (2\alpha + 2\beta) \rbr{\1_\numchar\otimes\onehot_b} -2\alpha\rbr{\onehot_{a^*}\otimes\1_\numchar}  - 2\alpha\rbr{\1_\numchar\otimes\onehot_{b^*}} & \text{if } a^*\neq a, b^*\neq b.
        \end{cases}
    \end{align*}
\end{lemma}

\begin{proof}[Proof of \cref{lemma:apdx-surrogate-grad-key-W1}]
In this proof we use $\lonehot_a$ (long one-hot) to denote the one-hot vector in $\R^{\numchar^2}$ with 1 on the $a$-th index and use $\onehot_a$ (short one-hot) to denote the one-hot vector in $\R^{\numchar}$ with 1 on the $a$-th index.
We use $\tilde\matrep_1, \matrep_2, \tilde \matrep_2$ and $\matrep_3$ to denote the intermediate sequences attained with translation matrix $\P_1^{(k)}$ and $\tilde\matrep_1^*, \matrep_2^*, \tilde \matrep_2^*$ and $\matrep_3^*$ to denote the counterfactual intermediate sequences should the forward pass is done with the ground truth translations $\P_1^{*(k)} = \gtW_1$.

Now we can proceed to the gradient calculations.
First note that with $\gdloss(\P_1) = \snorm{\matrep_3 - \matrep^*_3}_2^2$, by chain rule we have
\begin{align}
\begin{split}
    \pdr{\gdloss}{\P_1} &= 
    \sum_{j=1}^{\len}\rbr{\pdr{\V{3}{j}}{\P_1}}^\T\rbr{\pdr{\snorm{\V{3}{j} - \VT{3}{j}}_2^2}{\V{3}{j}}}
    = 2\sum_{j=1}^{\len}\rbr{\pdr{\V{3}{j}}{\P_1}}^\T\rbr{\V{3}{j} - \VT{3}{j}}.
\end{split}
\end{align}
Specifically for each column $l\in[\numchar^2]$ of $\P_1$ we have
\begin{align}
    \label{eqn:apdx-surrogate-gd-gradexp}
    \pdr{\gdloss}{\P_1^{(l)}} &= 2\sum_{j=1}^{\len}\rbr{\pdr{\V{3}{j}}{\P_1^{(l)}}}^\T\rbr{\V{3}{j} - \VT{3}{j}}.
\end{align}
Let us fix a particular column $j\in[\len]$ in the output $\matrep_{3}$ and compute the gradients. With $\mM^{(j)}$ as the $j$-th column of the matrix $\mM$, the computation graph for the forward pass is of the form:
\begin{align}
    \begin{NiceArray}{ccccccccc}[columns-width=1em]
        \V{1}{j} & \rightarrow & \SV{1}{j} & \transarrow{\P_1} & \V{2}{j} & \rightarrow & \SV{2}{j} & \transarrow{\P_2} & \V{3}{j}\\
        &&&&&&&&\\
        & \nearrow & & & & \nearrow & & & \\
        &&&&&&&&\\
        \V{1}{j+1} & \rightarrow & \SV{1}{j+1} & \transarrow{\P_1} & \V{2}{j+1} & &&&\\
        &&&&&&&&\\
        & \nearrow & & & &  & & & \\
        &&&&&&&&\\
        \V{1}{j+2} & &&&&&&&\\
    \end{NiceArray}
\end{align}
We can see that $\V{3}{j}$ is only affected by $\P_1$ through $\V{2}{j}$ and $\V{2}{j+1}$.

Let us first compute $\partial{\V{3}{j}}/\partial{\P_1}$. Assume $\V{2}{j} = \lonehot_p$ and $\V{2}{j+1} = \lonehot_q$ for some $p,q\in [\numchar^2]$, $\V{3}{j}$ is computed as 
\begin{align}
    \V{3}{j} = \W_2^* \SV{2}{j} = \W_2^*\rbr{\Q\V{2}{j}\odot \QT\V{2}{j+1}} = \W_2^*\rbr{\Q\P_1^{(p)}\odot \QT\P_1^{(q)}}.
\end{align}
We may express the Hadamard product in the following two ways:
\begin{align}
\begin{split}
    \V{3}{j} &= \W_2^*\rbr{\Q\P_1^{(p)}\odot \QT\P_1^{(q)}} = \W_2^*\diag\rbr{\Q\P_1^{(p)}}\QT\P_1^{(q)};\\
     &= \W_2^*\rbr{\QT\P_1^{(q)} \odot \Q\P_1^{(p)}} = \W_2^*\diag\rbr{\QT\P_1^{(q)}}\Q\P_1^{(p)}.
\end{split}
\end{align}
Therefore when $p\neq q$ we have \begin{align}
\label{eqn:apdx-surrogate-gd-gradexpr-neq}
    \pdr{\V{3}{j}}{\P_1^{(p)}} = \W_2^*\diag\rbr{\QT\P_1^{(q)}}\Q;\quad
    \pdr{\V{3}{j}}{\P_1^{(q)}} = \W_2^*\diag\rbr{\Q\P_1^{(p)}}\QT;\quad \forall l\not\in\cbr{p,q}: \pdr{\V{3}{j}}{\P_1^{(l)}} = \zeromat.
\end{align}
When $p = q$, the expression would be \begin{align}
\label{eqn:apdx-surrogate-gd-gradexpr-eq}
    \pdr{\V{3}{j}}{\P_1^{(p)}} = \W_2^*\diag\rbr{\QT\P_1^{(p)}}\Q +  \W_2^*\diag\rbr{\Q\P_1^{(p)}}\QT;\quad \forall l\neq p: \pdr{\V{3}{j}}{\P_1^{(l)}} = \zeromat.
\end{align}

Now we move on to compute $\srbr{\V{3}{j} - \VT{3}{j}}$. There are in total four cases to consider: depending on whether $\P_1^{(k)}$ is being used when computing $\SV{2}{j}$ and $\SV{2}{j+1}$. Let us go over these cases one-by-one.

\begin{enumerate}
\item \textbf{Case 1:} $\SV{1}{j}\neq \lonehot_k$ and $\SV{1}{j+1}\neq \lonehot_k$.

Assume $\SV{1}{j} = \lonehot_p$ and $\SV{1}{j+1} = \lonehot_q$ for some $p,q\neq k$. Since $\V{2}{j} = \P_1\SV{1}{j}$ while $\P_1$ equals $\P_1^*$ for all columns not equal to $k$ by assumption, we have $\V{2}{j} = \P_1\lonehot_p = \P_1^{(p)} = \P_1^{*(p)} = \P_1^*\lonehot_p = \VT{2}{j}$. With an identical argument we have $\V{2}{j+1} = \VT{2}{j+1}$. Thus in this case $\V{3}{j} = \VT{3}{j}$ and hence \begin{align}
    \rbr{\pdr{\V{3}{j}}{\P_1}}^\T\rbr{\V{3}{j} - \VT{3}{j}} = \zeromat.
\end{align}

\item \textbf{Case 2:} $\SV{1}{j} = \lonehot_k$ and $\SV{1}{j+1}\neq \lonehot_k$.

Assume $\SV{1}{j+1} = \lonehot_q$ for some $q\neq k$, from case 1 we know $\V{2}{j+1} = \VT{2}{j+1}$. However since $\SV{1}{j} = \lonehot_k$ we have $\V{2}{j} = \P_1\lonehot_k = \P_1^{(k)}$, which is not equal to $\VT{2}{j} = \P_1^{*(k)}$.
Therefore \begin{align}
\label{eqn:apdx-surrogate-case2}
\begin{split}
    \V{3}{j} - \VT{3}{j} &= \W_2^*\rbr{\Q\P_1^{(k)}\odot \QT\P_1^{(q)}} - \W_2^*\rbr{\Q\P_1^{*(k)}\odot \QT\P_1^{(q)}}\\
    &= \W_2^*\rbr{ \QT\P_1^{(q)} \odot \Q\rbr{\P_1^{(k)} - \P_1^{*(k)}}}\\
    &= \W_2^*\diag\rbr{\QT\P_1^{(q)}}\Q\rbr{\P_1^{(k)} - \P_1^{*(k)}}.
\end{split}
\end{align}

For the gradient with respect to $\P_1^{(k)}$, combining \cref{eqn:apdx-surrogate-case2} with \cref{eqn:apdx-surrogate-gd-gradexpr-neq} (substituting $p=k$) we have \begin{align}
\begin{split}
    \rbr{\pdr{\V{3}{j}}{\P_1^{(k)}}}^\T\rbr{\V{3}{j} - \VT{3}{j}} &= \rbr{\W_2^*\diag\rbr{\QT\P_1^{(q)}}\Q}^\T\rbr{\W_2^*\diag\rbr{\QT\P_1^{(q)}}\Q\rbr{\P_1^{(k)} - \P_1^{*(k)}}}\\
    &= \QT\diag\rbr{\QT\P_1^{(q)}}\W_2^{*\T}\W_2^*\diag\rbr{\QT\P_1^{(q)}}\Q\rbr{\P_1^{(k)} - \P_1^{*(k)}}\\
    &= \QT\diag\rbr{\QT\P_1^{(q)}}\diag\rbr{\QT\P_1^{(q)}}\Q\rbr{\P_1^{(k)} - \P_1^{*(k)}}\\
    &= \rbr{\srbr{\1_\numchar \1_\numchar^\T}\otimes I_\numchar}\rbr{\P_1^{(k)} - \P_1^{*(k)}}\qquad\qquad\qquad\qquad\rbr{\text{by \cref{lemma:apdx-surrogate-gd-aux1}}}
\end{split}
\end{align}
where we dropped $\W_2^{*\T}\W_2^*$ in the third step since $\W_2^*$ is a permutation matrix and $\W_2^{*\T}\W_2^*=I_\numchar$.

Now plugging in $\P_1^{(k)} = \onehot_{a}\otimes \onehot_{b}$ and $\P_1^{*(k)} = \onehot_{a^*}\otimes \onehot_{b^*}$ we have
\begin{align}
\begin{split}
    \rbr{\pdr{\V{3}{j}}{\P_1^{(k)}}}^\T\rbr{\V{3}{j} - \VT{3}{j}} &= \rbr{\srbr{\1_\numchar \1_\numchar^\T}\otimes I_\numchar}\rbr{\onehot_{a}\otimes \onehot_{b}} - \rbr{\srbr{\1_\numchar \1_\numchar^\T}\otimes I_\numchar}\rbr{\onehot_{a^*}\otimes \onehot_{b^*}}\\
    &= \rbr{\srbr{\1_\numchar \1_\numchar^\T\onehot_{a}}\otimes I_\numchar\onehot_{b}} - \rbr{\srbr{\1_\numchar \1_\numchar^\T\onehot_{a^*}}\otimes I_\numchar\onehot_{b^*}}\\
    &= \1_\numchar\otimes\onehot_b - \1_\numchar\otimes\onehot_{b^*}\\
    &= \1_\numchar\otimes\rbr{\onehot_b - \onehot_{b^*}}.
\end{split}
\end{align}

\item \textbf{Case 3:} $\SV{1}{j} \neq \lonehot_k$ and $\SV{1}{j+1}= \lonehot_k$.

This is a symmetric case with respect to case 2. Assume $\SV{1}{j} = \lonehot_p$ for some $p\neq k$, from case 1 we know $\V{2}{j} = \VT{2}{j}$. However since $\SV{1}{j} = \lonehot_k$ we have $\V{2}{j+1} = \P_1\lonehot_k = \P_1^{(k)}$, which is not equal to $\VT{2}{j+1} = \P_1^{*(k)}$.
Therefore similar to case 2 we have \begin{align}
\label{eqn:apdx-surrogate-case3}
\begin{split}
    \V{3}{j} - \VT{3}{j} &= \W_2^*\diag\rbr{\Q\P_1^{(p)}}\QT\rbr{\P_1^{(k)} - \P_1^{*(k)}}.
\end{split}
\end{align}
Combining \cref{eqn:apdx-surrogate-case3} with \cref{eqn:apdx-surrogate-gd-gradexpr-eq} gives \begin{align}
\begin{split}
    \rbr{\pdr{\V{3}{j}}{\P_1^{(k)}}}^\T\rbr{\V{3}{j} - \VT{3}{j}} &= \rbr{\W_2^*\diag\rbr{\Q\P_1^{(p)}}\QT}^\T\rbr{\W_2^*\diag\rbr{\Q\P_1^{(p)}}\QT\rbr{\P_1^{(k)} - \P_1^{*(k)}}}\\
    &= \Q\diag\rbr{\Q\P_1^{(p)}}\W_2^{*\T}\W_2^*\diag\rbr{\Q\P_1^{(p)}}\QT\rbr{\P_1^{(k)} - \P_1^{*(k)}}\\
    &=\Q\diag\rbr{\Q\P_1^{(p)}}\diag\rbr{\Q\P_1^{(p)}}\QT\rbr{\P_1^{(k)} - \P_1^{*(k)}}\\
    &= \rbr{I_\numchar\otimes \srbr{\1_\numchar \1_\numchar^\T}}\rbr{\P_1^{(k)} - \P_1^{*(k)}}\qquad\qquad\qquad\qquad\rbr{\text{by \cref{lemma:apdx-surrogate-gd-aux2}}}
\end{split}
\end{align}
Now plugging in $\P_1^{(k)} = \onehot_{a}\otimes \onehot_{b}$ and $\P_1^{*(k)} = \onehot_{a^*}\otimes \onehot_{b^*}$ we have
\begin{align}
\begin{split}
    \rbr{\pdr{\V{3}{j}}{\P_1^{(k)}}}^\T\rbr{\V{3}{j} - \VT{3}{j}} &= \rbr{I_\numchar\otimes \srbr{\1_\numchar \1_\numchar^\T}}\rbr{\onehot_{a}\otimes \onehot_{b}} - \rbr{I_\numchar\otimes \srbr{\1_\numchar \1_\numchar^\T}}\rbr{\onehot_{a^*}\otimes \onehot_{b^*}}\\
    &= \rbr{I_\numchar\onehot_{a}\otimes \srbr{\1_\numchar \1_\numchar^\T\onehot_{b}}} - \rbr{I_\numchar\onehot_{a^*}\otimes \srbr{\1_\numchar \1_\numchar^\T\onehot_{b^*}}}\\
    &= \onehot_a \otimes \1_\numchar - \onehot_{a^*} \otimes \1_\numchar\\
    &= \srbr{\onehot_a - \onehot_{a^*}} \otimes \1_\numchar.
\end{split}
\end{align}

\item \textbf{Case 4:} $\SV{1}{j} = \SV{1}{j+1}= \lonehot_k$.

This is the most complicated case since the loss is contributed by two different paths. We can first decompose the negative residual as \begin{align}
\begin{split}
    \V{3}{j} - \VT{3}{j} = &\ \W_2^*\rbr{\Q\P_1^{(k)}\odot \QT\P_1^{(k)}} - \W_2^*\rbr{\Q\P_1^{*(k)}\odot \QT\P_1^{*(k)}}\\
    =&\ \W_2^*\rbr{\Q\P_1^{(k)}\odot \QT\P_1^{(k)}} - \W_2^*\rbr{\Q\P_1^{(k)}\odot \QT\P_1^{*(k)}}\\&+ \W_2^*\rbr{\Q\P_1^{(k)}\odot \QT\P_1^{*(k)}} - \W_2^*\rbr{\Q\P_1^{*(k)}\odot \QT\P_1^{*(k)}}\\
    =&\ \W_2^*\diag\rbr{\Q\P_1^{(k)}}\QT\rbr{\P_1^{(k)} - \P_1^{*(k)}}+ \W_2^*\diag\rbr{\QT\P_1^{*(p)}}\Q\rbr{\P_1^{(k)} - \P_1^{*(k)}}.
\end{split}
\end{align}
Combining with \cref{eqn:apdx-surrogate-gd-gradexpr-eq}, we have \begin{align*}
        &\ \rbr{\pdr{\V{3}{j}}{\P_1}}^\T\rbr{\V{3}{j} - \VT{3}{j}}\\
        =&\ \rbr{\W_2^*\diag\rbr{\QT\P_1^{(k)}}\Q +  \W_2^*\diag\rbr{\Q\P_1^{(k)}}\QT}^\T\\
        &\ \qquad\rbr{\W_2^*\diag\rbr{\Q\P_1^{(k)}}\QT\rbr{\P_1^{(k)} - \P_1^{*(k)}}+ \W_2^*\diag\rbr{\QT\P_1^{*(k)}}\Q\rbr{\P_1^{(k)} - \P_1^{*(k)}}}\\
        =&\ \QT\diag\rbr{\QT\P_1^{(k)}}\diag\rbr{\Q\P_1^{(k)}}\QT\rbr{\P_1^{(k)} - \P_1^{*(k)}}\tag{a}\\
        &\ + \Q\diag\rbr{\Q\P_1^{(k)}}\diag\rbr{\Q\P_1^{(k)}}\QT\rbr{\P_1^{(k)} - \P_1^{*(k)}}\tag{b}\\
        &\ + \QT\diag\rbr{\QT\P_1^{(k)}}\diag\rbr{\QT\P_1^{*(k)}}\Q\rbr{\P_1^{(k)} - \P_1^{*(k)}}\tag{c}\\
        &\ +
        \Q\diag\rbr{\Q\P_1^{(k)}}\diag\rbr{\QT\P_1^{*(k)}}\Q\rbr{\P_1^{(k)} - \P_1^{*(k)}}\tag{d}.
\end{align*}
Now let us analyze the four cross terms term-by-term.
\begin{enumerate}
    \item With $\P_1^{(k)} = \onehot_{a}\otimes \onehot_{b}$, by \cref{lemma:apdx-surrogate-gd-aux-QvQTv} we have $\QT\P_1^{(k)} = \1_\numchar\otimes \onehot_a$ and $\Q\P_1^{(k)} = \onehot_b\otimes \1_\numchar.$ Therefore $\QT\P_1^{(k)}\odot \Q\P_1^{(k)} = \onehot_b\otimes\onehot_a$ and hence
    \begin{align}
        \diag\rbr{\QT\P_1^{(k)}}\diag\rbr{\Q\P_1^{(k)}} = \diag\rbr{\onehot_b\otimes\onehot_a} = \diag\rbr{\onehot_b}\otimes\diag\rbr{\onehot_a}
    \end{align}
    It follows that\begin{align}
    \begin{split}
        &\quad\QT\diag\rbr{\QT\P_1^{(k)}}\diag\rbr{\Q\P_1^{(k)}}\QT\\
        &= \rbr{\1_\numchar\otimes I_\numchar} \rbr{I_\numchar^\T\otimes \1_\numchar^\T}\rbr{\diag\rbr{\onehot_b}\otimes\diag\rbr{\onehot_a}}\rbr{\1_\numchar\otimes I_\numchar} \rbr{I_\numchar^\T\otimes \1_\numchar^\T}\\
        &= \rbr{\1_\numchar\otimes I_\numchar} \rbr{I_\numchar^\T\otimes \1_\numchar^\T}\rbr{\onehot_b\otimes\diag\rbr{\onehot_a}}\rbr{I_\numchar^\T\otimes \1_\numchar^\T}\\
        &= \rbr{\1_\numchar\otimes I_\numchar} \rbr{I_\numchar^\T\otimes \1_\numchar^\T}\rbr{\onehot_b\otimes\onehot_a^\T}\\
        &= \rbr{\1_\numchar\otimes I_\numchar} \rbr{I_\numchar^\T\otimes \1_\numchar^\T}\rbr{\onehot_b\onehot_a^\T\otimes 1}\\
        &= \rbr{\1_\numchar\otimes I_\numchar} \rbr{\onehot_b\onehot_a^\T\otimes \1_\numchar^\T}\\
    \end{split}
    \end{align}
    Plugging into (a) we have
    \begin{align}
    \begin{split}
        \text{(a)} &= \QT\diag\rbr{\QT\P_1^{(k)}}\diag\rbr{\Q\P_1^{(k)}}\QT\rbr{\P_1^{(k)} - \P_1^{*(k)}}\\
        &= \rbr{\1_\numchar\otimes I_\numchar} \rbr{\onehot_b\onehot_a^\T\otimes \1_\numchar^\T}\rbr{\onehot_a\otimes \onehot_b} - \rbr{\1_\numchar\otimes I_\numchar} \rbr{\onehot_b\onehot_a^\T\otimes \1_\numchar^\T}\rbr{\onehot_{a^*}\otimes \onehot_{b^*}}\\
        &= \rbr{\1_\numchar\otimes I_\numchar} \rbr{\onehot_b\onehot_a^\T\onehot_a\otimes 1} - \rbr{\1_\numchar\otimes I_\numchar} \rbr{\onehot_b\onehot_a^\T\onehot_{a^*}\otimes 1}\\
        &= \rbr{\1_\numchar\otimes I_\numchar} \rbr{1\otimes \onehot_b\onehot_a^\T\onehot_a} - \rbr{\1_\numchar\otimes I_\numchar} \rbr{1\otimes \onehot_b\onehot_a^\T\onehot_{a^*}}\\
        &= \begin{cases}
            \1_\numchar \otimes\onehot_b & \text{when } a\neq a^*\\
            \zeromat & \text{otherwise  }
        \end{cases}
    \end{split}
    \end{align}
    \item We have seen the same term as in case 3, by \cref{lemma:apdx-surrogate-gd-aux2} we have \begin{align}
        \text{(b)} = \Q\diag\rbr{\Q\P_1^{(k)}}\diag\rbr{\Q\P_1^{(k)}}\QT\rbr{\P_1^{(k)} - \P_1^{*(k)}} = \srbr{\onehot_a - \onehot_{a^*}} \otimes \1_\numchar.
    \end{align}
    \item With $\P_1^{(k)} = \onehot_{a}\otimes \onehot_{b}$ and $\P_1^{*(k)} = \onehot_{a^*}\otimes \onehot_{b^*}$, we have \begin{align}
        \diag\rbr{\QT\P_1^{(k)}}\diag\rbr{\QT\P_1^{*(k)}} = \diag\rbr{\1_\numchar \otimes\onehot_a}\diag\rbr{\1_\numchar \otimes\onehot_{a^*}} = \begin{cases}
            \diag\rbr{\1_\numchar \otimes\onehot_a} & \text{if } a=a^*\\
            \zeromat & \text{otherwise}
        \end{cases}
    \end{align}
    Thus by \cref{lemma:apdx-surrogate-gd-aux1} we have\begin{align}
        \QT\diag\rbr{\QT\P_1^{(k)}}\diag\rbr{\QT\P_1^{*(k)}}\Q = \begin{cases}
            \rbr{\1_\numchar\1_\numchar^\T}\otimes I_\numchar  & \text{if } a=a^*\\
            \zeromat & \text{otherwise}
        \end{cases}
    \end{align}
    When $a=a^*$, we then have \begin{align}
    \begin{split}
        &\quad \QT\diag\rbr{\QT\P_1^{(k)}}\diag\rbr{\QT\P_1^{*(k)}}\Q\rbr{\P_1^{(k)} - \P_1^{*(k)}}\\
        &= \rbr{\rbr{\1_\numchar\1_\numchar^\T}\otimes I_\numchar}\rbr{\onehot_a\otimes \onehot_b} - \rbr{\rbr{\1_\numchar\1_\numchar^\T}\otimes I_\numchar}\rbr{\onehot_{a^*}\otimes \onehot_{b^*}}\\
        &= \1_\numchar\otimes\rbr{\onehot_{b} - \onehot_{b^*}}.
    \end{split}
    \end{align}
    Thus in summary \begin{align}
        \text{(c)} = \QT\diag\rbr{\QT\P_1^{(k)}}\diag\rbr{\QT\P_1^{*(k)}}\Q\rbr{\P_1^{(k)} - \P_1^{*(k)}} = \begin{cases}
            \1_\numchar\otimes\rbr{\onehot_{b} - \onehot_{b^*}}  & \text{if } a=a^*\\
            \zeromat & \text{otherwise}.
        \end{cases}
    \end{align}
    \item Now for the last term, note that
    \begin{align}
    \begin{split}
        &\quad\
        \Q\diag\rbr{\Q\P_1^{(k)}}\diag\rbr{\QT\P_1^{*(k)}}\Q\\
        &= \Q\diag\rbr{\onehot_b\otimes\1_\numchar}\diag\rbr{\1_\numchar\otimes \onehot_{a^*}}\Q\\
        &= \Q\diag\rbr{\onehot_b\otimes\onehot_{a^*}}\Q\\
        &= \rbr{I_\numchar\otimes \1_\numchar}\rbr{\1_\numchar^\T\otimes I_\numchar^\T}\rbr{\diag\rbr{\onehot_{b}}\otimes\diag\rbr{\onehot_{a^*}}}\rbr{I_\numchar\otimes \1_\numchar}\rbr{\1_\numchar^\T\otimes I_\numchar^\T}\\
        &= \rbr{I_\numchar\otimes \1_\numchar}\rbr{\1_\numchar^\T\otimes I_\numchar^\T}\rbr{\diag\rbr{\onehot_{b}}\otimes\onehot_{a^*}}\rbr{\1_\numchar^\T\otimes I_\numchar^\T}\\
        &= \rbr{I_\numchar\otimes \1_\numchar}\rbr{\onehot_{b}^\T\otimes\onehot_{a^*}}\rbr{\1_\numchar^\T\otimes I_\numchar^\T}\\
        &= \rbr{I_\numchar\otimes \1_\numchar}\rbr{\onehot_{a^*}\onehot_{b}^\T\otimes 1}\rbr{\1_\numchar^\T\otimes I_\numchar^\T}\\
        &= \rbr{\onehot_{a^*}\onehot_{b}^\T\otimes \1_\numchar}\rbr{\1_\numchar^\T\otimes I_\numchar^\T}.
    \end{split}
    \end{align}
    Plugging in $\srbr{\P_1^{(k)} - \P_1^{*(k)}}$, we have that
    \begin{align}
    \begin{split}
        &\ \Q\diag\rbr{\Q\P_1^{(k)}}\diag\rbr{\QT\P_1^{*(k)}}\Q\rbr{\P_1^{(k)} - \P_1^{*(k)}}\\
        =&\ \rbr{\onehot_{a^*}\onehot_{b}^\T\otimes \1_\numchar}\rbr{\1_\numchar^\T\otimes I_\numchar^\T}\rbr{\onehot_a\otimes \onehot_b} - \rbr{\onehot_{a^*}\onehot_{b}^\T\otimes \1_\numchar}\rbr{\1_\numchar^\T\otimes I_\numchar^\T}\rbr{\onehot_{a^*}\otimes \onehot_{b^*}}\\
        =&\ \rbr{\onehot_{a^*}\onehot_{b}^\T\otimes \1_\numchar}\rbr{1\otimes \onehot_b} - \rbr{\onehot_{a^*}\onehot_{b}^\T\otimes \1_\numchar}\rbr{1\otimes \onehot_{b^*}}\\
        =&\ \rbr{\onehot_{a^*}\onehot_{b}^\T\otimes \1_\numchar}\rbr{\onehot_b\otimes 1} - \rbr{\onehot_{a^*}\onehot_{b}^\T\otimes \1_\numchar}\rbr{\onehot_{b^*}\otimes 1}\\
        =&\ \rbr{\onehot_{a^*}\onehot_{b}^\T\onehot_b\otimes \1_\numchar} - \rbr{\onehot_{a^*}\onehot_{b}^\T\onehot_{b^*}\otimes \1_\numchar}\\
        =&\begin{cases}
            \onehot_{a^*}\otimes \1_\numchar & \text{if}\ b\neq b^*\\
            \zeromat &\text{otherwise}.
        \end{cases}
    \end{split}
    \end{align}
\end{enumerate}
Summing the four terms together, we then have\begin{align}
        \rbr{\pdr{\V{3}{j}}{\P_1}}^\T\rbr{\V{3}{j} - \VT{3}{j}} = 
        \begin{cases}
            0 & \text{when } a^*=a, b^*=b\\
            \1_\numchar\otimes\rbr{\onehot_b - \onehot_{b^*}} + \onehot_{a^*}\otimes\1_\numchar & \text{when } a^*=a, b^*\neq b\\
            \1_\numchar\otimes\onehot_{b^*} + \rbr{\onehot_a - \onehot_{a^*}}\otimes\1_\numchar & \text{when } a^*\neq a, b^*= b\\
            \1_\numchar\otimes\onehot_{b} + \onehot_a \otimes\1_\numchar & \text{when } a^*\neq a, b^*\neq b.
        \end{cases}
    \end{align}
\end{enumerate}
Now we are ready to provide the gradient expression for the loss over the entire sequence. Observe that for every consecutive sequence of $m$ columns $\scbr{\VT{1}{j},\VT{1}{j+1},\dots,\VT{1}{j+m-1}}$ that all equals to $\lonehot_k$, it will result in one incorrect column $\V{3}{j-1}$ in case 2, one incorrect column $\V{3}{j+m-1}$ in case 3, and $m-1$ incorrect columns ($\V{3}{j},\dots,\V{3}{j+m-2}$) in case 4.

\begin{figure}[h]
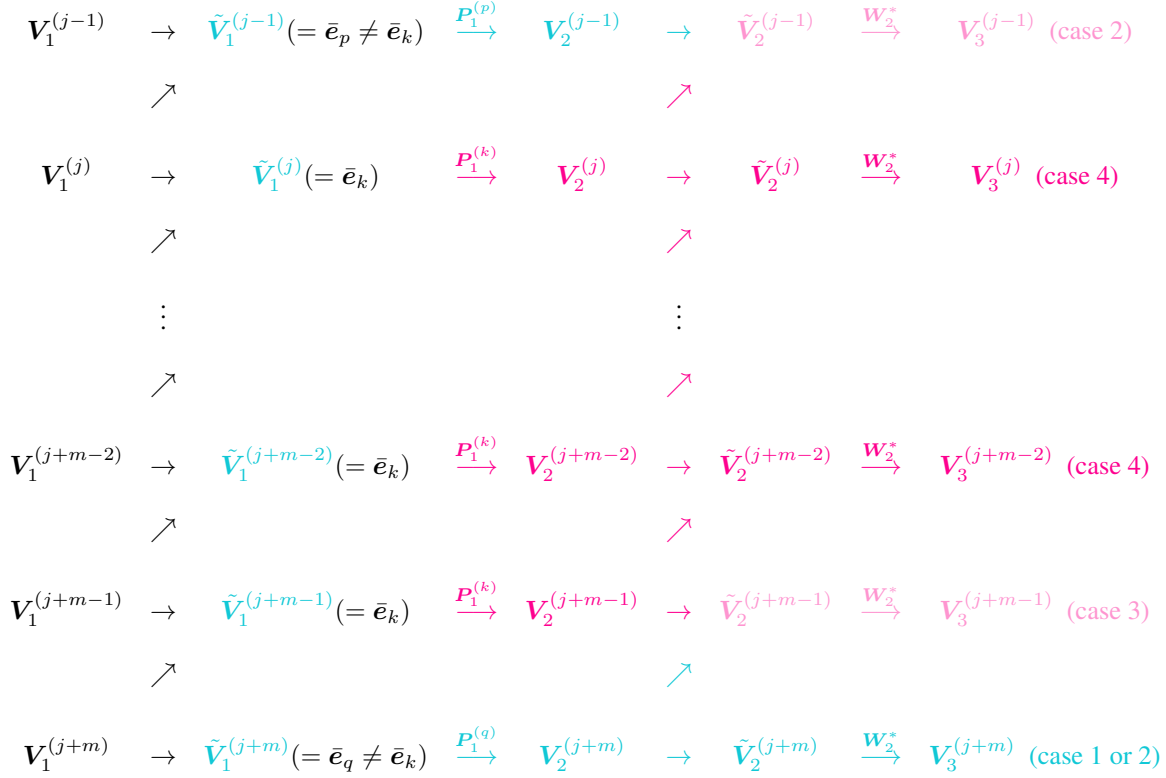

    \centering
    \begin{align*}
    \begin{NiceArray}{ccccccccc}[columns-width=1em]
        \V{1}{j-1} & \rightarrow & \green{\SV{1}{j-1}} (=\lonehot_p\neq\lonehot_k) & \green{\transarrow{\P_1^{(p)}}} & \green{\V{2}{j-1}} & \green{\rightarrow} & \orange{\SV{2}{j-1}} & \orange{\transarrow{\W^*_2}} & \orange{\V{3}{j-1}\text{\ \ (case 2)}}\\
        &&&&&&&&\\
        & \nearrow & & & & \red{\nearrow} & & & \\
        &&&&&&&&\\
        \V{1}{j} & \rightarrow & \green{\SV{1}{j}} (=\lonehot_k) & \red{\transarrow{\P_1^{(k)}}} & \red{\V{2}{j}} & \red{\rightarrow} & \red{\SV{2}{j}} & \red{\transarrow{\W^*_2}} & \red{\V{3}{j}\text{\ \ (case 4)}}\\
        &&&&&&&&\\
        & \nearrow & & & & \red{\nearrow} & & & \\
        &&&&&&&&\\
        & \vdots & & & & \vdots & & & \\
        &&&&&&&&\\
        & \nearrow & & & & \red{\nearrow} & & & \\
        &&&&&&&&\\
        \V{1}{j+m-2} & \rightarrow & \green{\SV{1}{j+m-2} }(=\lonehot_k) & \red{\transarrow{\P_1^{(k)}}} & \red{\V{2}{j+m-2}} & \red{ \rightarrow} & \red{\SV{2}{j+m-2}} & \red{\transarrow{\W^*_2}} & \red{\V{3}{j+m-2}\text{\ \ (case 4)}}\\
        &&&&&&&&\\
        & \nearrow & & & & \red{\nearrow}  & & & \\
        &&&&&&&&\\
        \V{1}{j+m-1} & \rightarrow & \green{\SV{1}{j+m-1} }(=\lonehot_k) & \red{\transarrow{\P_1^{(k)}}} & \red{\V{2}{j+m-1}} & \red{\rightarrow} & \orange{\SV{2}{j+m-1}} & \orange{\transarrow{\W^*_2}} & \orange{\V{3}{j+m-1}\text{\ \ (case 3)}}\\
        &&&&&&&&\\
        & \nearrow & & & & \green{\nearrow}  & & & \\
        &&&&&&&&\\
        \V{1}{j+m} & \rightarrow & \green{\SV{1}{j+m}} (=\lonehot_q\neq\lonehot_k) & \green{\transarrow{\P_1^{(q)}}} & \green{\V{2}{j+m}} & \green{\rightarrow} & \green{\SV{2}{j+m}} & \green{\transarrow{\W^*_2}} & \green{\V{3}{j+m}\text{\ \ (case 1 or 2)}}\\
    \end{NiceArray}
\end{align*}
    \caption{Error propagation of $\P_1^{(k)}$}
    \label{fig:apdx-surrogate-error-prop}
\end{figure}

For illustration, one can refer to the computation graph in \cref{fig:apdx-surrogate-error-prop}. In the graph, \green{green} entries agrees with the counterfactual values with correct $\P_1^{*(k)}$, \red{Red} and \orange{pink} entries are incorrect entries where \red{red} entries are consequence solely dependent on $\P_1^{(k)}$ (case 4) and \orange{pink} entries depend on other correct columns (case 2,3).

Assume that in total there are $\alpha$ columns in $\matrep_3$ under case 2, $\alpha$ columns in $\matrep_3$ under case 3, and $\beta$ columns in $\matrep_3$ under case 4, then by \cref{eqn:apdx-surrogate-gd-gradexp} the total gradient can be expressed as follows:
\begin{itemize}
    \item When $a=a^*, b\neq b^*$:\begin{align}
    \begin{split}
        \pdr{\gdloss}{\P_1^{(k)}} &= 2\alpha \rbr{\1_\numchar\otimes\rbr{\onehot_b - \onehot_{b^*}} + 2\alpha\rbr{\onehot_a - \onehot_{a^*}} \otimes \1_\numchar} + 2\beta\rbr{\1_\numchar\otimes\rbr{\onehot_b - \onehot_{b^*}} + \onehot_{a^*}\otimes\1_\numchar}\\
        &= (2\alpha + 2\beta) \rbr{\1_\numchar\otimes\onehot_b}  - (2\alpha + 2\beta) \rbr{\1_\numchar\otimes\onehot_{b^*}} + 2\beta\rbr{\onehot_{a^*}\otimes\1_\numchar}.
    \end{split}
    \end{align}
    \item When $a\neq a^*, b= b^*$:\begin{align}
    \begin{split}
        \pdr{\gdloss}{\P_1^{(k)}} &= 2\alpha \rbr{\1_\numchar\otimes\rbr{\onehot_b - \onehot_{b^*}}} + 2\alpha\rbr{\rbr{\onehot_a - \onehot_{a^*}} \otimes \1_\numchar} + 2\beta\rbr{\1_\numchar\otimes\onehot_{b^*} + \rbr{\onehot_a - \onehot_{a^*}}\otimes\1_\numchar}\\
        &= (2\alpha + 2\beta) \rbr{\onehot_{a}\otimes\1_\numchar}  - (2\alpha + 2\beta) \rbr{\onehot_{a^*}\otimes\1_\numchar} + 2\beta\rbr{\1_\numchar\otimes\onehot_{b^*}}.
    \end{split}
    \end{align}
    \item When $a\neq a^*, b\neq b^*$:\begin{align}
    \begin{split}
        \pdr{\gdloss}{\P_1^{(k)}} &= 2\alpha \rbr{\1_\numchar\otimes\rbr{\onehot_b - \onehot_{b^*}}} + 2\alpha\rbr{\rbr{\onehot_a - \onehot_{a^*}} \otimes \1_\numchar} + 2\beta\rbr{\1_\numchar\otimes\onehot_{b} + \onehot_a \otimes\1_\numchar}\\
        &= (2\alpha + 2\beta) \rbr{\onehot_{a}\otimes\1_\numchar}  + (2\alpha + 2\beta) \rbr{\1_\numchar\otimes\onehot_b} -2\alpha\rbr{\onehot_{a^*}\otimes\1_\numchar}  - 2\alpha\rbr{\1_\numchar\otimes\onehot_{b^*}}.
    \end{split}
    \end{align}
\end{itemize}
This gives the desired expression of gradient.⁄⁄
\end{proof}

Now with the gradient expression, we are ready to prove for the learning of a single missing column.

\begin{thmbox}
\begin{lemma}[Learning Column of $\weightmat_1$]\emph{}

\label{lemma:apdx-surrogate-learnW1}
    Fix an input sequence $\matrep_1\in\R^{\numchar^2\times \inputlength}$ and any column index $k\in[\numchar^2]$, if $\emph{\maxfunc}\rbr{\contextmat_{2} + \weightmat_2} = \W_2^*$ and $\emph{\maxfunc}\rbr{\contextmat_{1(k)} + \weightmat_1}$ equals to $\W_1^*$ everywhere except for the $k$-th column and if there exists a non-empty subset of indices $\cJ\subset[\inputlength]$ such that $\VT{1}{j} = \lonehot_k$ for all $j\in\cJ$,
    for any initialization ${\W_{1(0)}^{(k)}}\in\R^{\numchar^2}$ such that $\snorm{\W_{1(0)}^{(k)}}_0\leq \frac12$. Taking two surrogate gradient updates on $\W_{1(0)}^{(k)}$ as described in \cref{alg:mlt-d2-GD} gives $\W_{1(2)}^{(k)}$ such that $\emph{\maxfunc}\srbr{\W_{1(2)}^{(k)}} = \W_1^{*(k)}.$
\end{lemma}
\end{thmbox}

\begin{proof}[Proof of \cref{lemma:apdx-surrogate-learnW1}]\emph{}

    Without loss of generality, assume at the initialization $\P_{1(0)}^{(k)} = \onehot_{a_{(0)}}\otimes \onehot_{b_{(0)}}$ and $\P_1^{*(k)} = \onehot_{a^*}\otimes \onehot_{b^*}$ for some $a_{(0)},b_{(0)},a^*,b^*\in[\numchar]$.
    We first note that the conditions specified in the lemma meets the assumptions required by \cref{lemma:apdx-surrogate-grad-key-W1}, namely there is only one incorrect column in $\W_1$ missing and that column is being used in the forward pass at least one time (since $\cJ$ is non-empty).

    To prove $\maxfunc\srbr{\W_{1(2)}^{(k)}} = \W_1^{*(k)}$, it is sufficient to show that $\argmax\srbr{\W_{1(2)}^{(k)}} = a^*\numchar + b^*$.
    
    Now we will leverage the gradient expressions in \cref{lemma:apdx-surrogate-grad-key-W1}. We will dive into three different cases:
    \begin{itemize}
        \item \textbf{When $\ai{0} \neq a^*$ and $\bi{0} \neq b^*$}.

        By \cref{lemma:apdx-surrogate-grad-key-W1} we have for some $\alpha\in\mathbb{Z}_+$ and $\beta\in\NN$ that \begin{align}
        \label{eqn:apdx-surrogate-W1-aneqbneq-grad}
            \pdr{\gdloss}{\P_{1(0)}^{(k)}}=(2\alpha + 2\beta) \rbr{\onehot_{\ai{0}}\otimes\1_\numchar}  + (2\alpha + 2\beta) \rbr{\1_\numchar\otimes\onehot_{\bi{0}}} -2\alpha\rbr{\onehot_{a^*}\otimes\1_\numchar}  - 2\alpha\rbr{\1_\numchar\otimes\onehot_{b^*}}.
        \end{align}
        It is not hard to verify that the unique smallest entry is of index $a^*\numchar+b^*$ with value $-4\alpha$. This entry is contributed by the intersection of $-2\alpha\rbr{\onehot_{a^*}\otimes\1_\numchar}  - 2\alpha\rbr{\1_\numchar\otimes\onehot_{b^*}}$, the remaining smaller entries are of value $-2\alpha$ contributed by non-intersecting entries in the same expression above.
        
        Thus we know $\argmin{\partial\gdloss/\partial\P_{1(0)}^{(k)}} = a^*\numchar+b^*$ with a margin of at least $-2$ (since $\alpha\geq 1$).

        Now we can apply the gradient step to the weight initialization. Since $\snorm{\W_{1(0)}^{(k)}}_0\leq \frac12$, the margin of $a^*\numchar+b^*$ dominates the largest margin in the initialization (which is 1), we have 
        \begin{align}
            \argmax\rbr{\W_{1(1)}^{(k)}} = \argmax\rbr{\W_{1(0)}^{(k)} - \pdr{\gdloss}{\P_{1(0)}^{(k)}}} = a^*\numchar+b^*.
        \end{align}
        Therefore $\P_{1(1)}^{(k)} = \P_1^{*(k)}$ with the first step. We will reach zero loss after the first step and hence the second step is static, so we have shown $\argmax\srbr{\W_{1(2)}^{(k)}} = a^*\numchar+b^*$ as desired.

        \item \textbf{When $\ai{0} = a^*$ and $\bi{0} \neq b^*$}.
        
        By \cref{lemma:apdx-surrogate-grad-key-W1} we have for some $\alpha\in\mathbb{Z}_+$ and $\beta\in\NN$ that 
        \begin{align}
            \pdr{\gdloss}{\P_{1(0)}^{(k)}} &= (2\alpha + 2\beta) \rbr{\1_\numchar\otimes\onehot_{\bi{0}}}  - (2\alpha + 2\beta) \rbr{\1_\numchar\otimes\onehot_{b^*}} + 2\beta\rbr{\onehot_{a^*}\otimes\1_\numchar}.
        \end{align}
        In this case we no longer have an unique smallest entry, the set of negative entries is of index $\numchar x+b^*$ where $x\in [\numchar]$. The values are $-2\alpha$ for the case of $x = a^*$ and $-2(\alpha + \beta)$ for other $x$'s. These negative entries are contributed by the $- (2\alpha + 2\beta) \rbr{\1_\numchar\otimes\onehot_{b^*}}$, and all other entries are at least 0.
        
        Thus we also have a negative margin of at least $-2$ since $\alpha\geq 1$. Therefore after taking the first gradient step, we know that there exists some $x\in [\numchar]$ such that 
        \begin{align}
            \argmax\rbr{\W_{1(1)}^{(k)}} = \argmax\rbr{\W_{1(0)}^{(k)} - \pdr{\gdloss}{\P_{1(0)}^{(k)}}} = \numchar x+b^*.
        \end{align}
        Let $P_{1(1)}^{(k)} = \argmax\srbr{\W_{1(1)}^{(k)}} = \onehot_{a_{(1)}}\otimes \onehot_{b_{(1)}}$, then now we are in the case of $b_{(1)} = b^*$ since the negative margin of $-2$ dominates any entry-wise difference in the initialization as $\snorm{\W_{1(0)}^{(k)}}_0\leq \frac12$. If $\beta=0$ and it happens that $a_{(1)} = a^*$, then we have zero loss after the first step and we are done as the second step will be static. If $\bi{1}\neq b^*$, then by \cref{lemma:apdx-surrogate-grad-key-W1} the second step gradient is of the form
        \begin{align}
            \pdr{\gdloss}{\P_{1(1)}^{(k)}} &= (2\alpha + 2\beta) \rbr{\onehot_{\ai{1}}\otimes\1_\numchar}  - (2\alpha + 2\beta) \rbr{\onehot_{a^*}\otimes\1_\numchar} + 2\beta\rbr{\1_\numchar\otimes\onehot_{b^*}}.
        \end{align}
        This is a bit tricky to analyze directly since we no longer have the small entry-wise difference from the initialization in the weights, but one may note that the sum of the two update steps is of the form
        \begin{align}
        \begin{split}
            &\ \pdr{\gdloss}{\P_{1(1)}^{(k)}} + \pdr{\gdloss}{\P_{1(0)}^{(k)}}\\
            =&\ (2\alpha + 2\beta) \rbr{\onehot_{\ai{1}}\otimes\1_\numchar}  - (2\alpha + 2\beta) \rbr{\onehot_{a^*}\otimes\1_\numchar} + 2\beta\rbr{\1_\numchar\otimes\onehot_{b^*}}\\&\ + (2\alpha + 2\beta)\rbr{\1_\numchar\otimes\onehot_{\bi{0}}}  - (2\alpha + 2\beta) \rbr{\1_\numchar\otimes\onehot_{b^*}} + 2\beta\rbr{\onehot_{a^*}\otimes\1_\numchar}\\
            =&\ (2\alpha + 2\beta) \rbr{\onehot_{\ai{1}}\otimes\1_\numchar}  + (2\alpha + 2\beta) \rbr{\1_\numchar\otimes\onehot_{\bi{0}}} -2\alpha\rbr{\onehot_{a^*}\otimes\1_\numchar}  - 2\alpha\rbr{\1_\numchar\otimes\onehot_{b^*}}.
        \end{split}
        \end{align}
        This is identical to the single step gradient as in \cref{eqn:apdx-surrogate-W1-aneqbneq-grad} in the first case, so follow the identical argument we have 
        \begin{align}
            \argmax\rbr{\W_{1(2)}^{(k)}} = \argmax\rbr{\W_{1(0)}^{(k)} - \rbr{\pdr{\gdloss}{\P_{1(0)}^{(k)}} + \pdr{\gdloss}{\P_{1(1)}^{(k)}}}} = a^*\numchar+b^*
        \end{align}
        as desired.
        \item \textbf{When $\ai{0} \neq a^*$ and $\bi{0} = b^*$}

        Note that all expressions are symmetric with respect to $a$ and $b$ with an additional swap of the Kronecker products, so we may follow the exact same argument as in the case of $\ai{0} = a^*$ and $\bi{0} \neq b^*$ and arrive at the same conclusion.
    \end{itemize}
    Thus for any initializations, after two surrogate gradient steps on $\W_1^{(k)}$, we have $\maxfunc\srbr{\W_{1(2)}^{(k)}} = \W_1^{*(k)}.$
\end{proof}

\subsubsection{Learning the Second Layer}

Now for the second layer, we will similarly first derive the gradient (which is much simpler) and show that with only one gradient step, one can learn the correct entry.

\begin{lemma}[Gradient with respect to incorrect column in $\P_2$]\emph{}

\label{lemma:apdx-surrogate-grad-key-W2}
    When only the $k$-th column of translation matrix $\P_2^{(k)}$ is not equal to $\P_2^{*(k)}$. If $\P_2^{(k)}$ is used in the forward pass, there exists $\alpha\in\mathbb{Z}_+$ such that the gradient of $\gdloss$ with respect to $\P_2^{(k)}$ is of the form
    \begin{align*}
        \pdr{\gdloss}{\P_2^{(k)}} = 
        2\alpha \rbr{\P_2^{(k)} - \P_2^{*(k)}}.
    \end{align*}
\end{lemma}

\begin{proof}
    We follow the same set of notations as used in the proof for \cref{lemma:apdx-surrogate-grad-key-W1}. In particular, we use $\tilde\matrep_1, \matrep_2, \tilde \matrep_2$ and $\matrep_3$ to denote the intermediate sequences attained with translation matrix $\P_2^{(k)}$ and $\tilde\matrep_1^*, \matrep_2^*, \tilde \matrep_2^*$ and $\matrep_3^*$ to denote the counterfactual intermediate sequences should the forward pass is done with the ground truth translations $\P_2^{*(k)} = \gtW_1$. Since we assume that $\P_1 = \P_1^*$, we have $\tilde\matrep_2 = \tilde\matrep_2^*$. Therefore for all column $j$ such that $\V{3}{j}\neq\VT{3}{j}$, it must be so that $\SVT{2}{j}=\lonehot_k$, and the residual can be written as \begin{align}
        \V{3}{j} - \VT{3}{j} = \P_2\SVT{2}{j} - \P_2^*\SVT{2}{j} = \P_2\lonehot_k - \P_2^*\lonehot_k = \P_2^{(k)} -\P_2^{*(k)}.
    \end{align}
    Assume that $\P_2^{(k)}$ has been used $\alpha$ times in the forward pass, follow the chain rule we then have 
    \begin{align}
        \pdr{\gdloss}{\P_2^{(k)}} &= 2\sum_{j=1}^{\len}\rbr{\pdr{\V{3}{j}}{\P_2^{(k)}}}^\T\rbr{\V{3}{j} - \VT{3}{j}} = 2\alpha\rbr{\P_2^{(k)} -\P_2^{*(k)}}.
    \end{align}
\end{proof}

\begin{thmbox}
\begin{lemma}[Learning $\weightmat_2$]\emph{}

\label{lemma:apdx-surrogate-learnW2}
    Fix an input sequence $\matrep_1\in\R^{\numchar^2\times \inputlength}$ and any column index $k\in[\numchar^2]$, if $\maxfunc\rbr{\contextmat_{1} + \weightmat_1} = \W_1^*$ and $\maxfunc\rbr{\contextmat_{2(k)} + \weightmat_2}$ equals to $\W_2^*$ everywhere except for the $k$-th column and if there exists a non-empty subset of indices $\cJ\subset[\inputlength]$ such that $\VT{2}{j} = \lonehot_k$ for all $j\in\cJ$,
    for any initialization ${\W_{2(0)}^{(k)}}\in\R^{\numchar^2}$ such that $\snorm{\W_{2(0)}^{(k)}}_0\leq \frac12$. Taking one surrogate gradient updates on $\W_{2(0)}^{(k)}$ as described in \cref{alg:mlt-d2-GD} gives $\W_{2(1)}^{(k)}$ such that $\maxfunc\srbr{\W_{2(1)}^{(k)}} = \W_2^{*(k)}.$
\end{lemma}
\end{thmbox}

\begin{proof}
    Without loss of generality, assume at the initialization $\P_{2(0)}^{(k)} = \onehot_{a_{(0)}}\otimes \onehot_{b_{(0)}}$ and $\P_2^{*(k)} = \onehot_{a^*}\otimes \onehot_{b^*}$ for some $a_{(0)},b_{(0)},a^*,b^*\in[\numchar]$.
    We first note that the conditions specified in the lemma meets the assumptions required by \cref{lemma:apdx-surrogate-grad-key-W2}, namely there is only one incorrect column in $\W_2$ missing and that column is being used in the forward pass at least one time (since $\cJ$ is non-empty).
    To prove $\maxfunc\srbr{\W_{2(1)}^{(k)}} = \W_2^{*(k)}$, it is sufficient to show that $\argmax\srbr{\W_{2(1)}^{(k)}} = a^*\numchar + b^*$.
    By \cref{lemma:apdx-surrogate-grad-key-W2}, we have \begin{align}
        \pdr{\gdloss}{\P_2^{(k)}} &= 2\sum_{j=1}^{\len}\rbr{\pdr{\V{3}{j}}{\P_2^{(k)}}}^\T\rbr{\V{3}{j} - \VT{3}{j}} = 2\alpha\rbr{\P_2^{(k)} -\P_2^{*(k)}} = 2\alpha\onehot_{\ai{0}}\otimes\onehot_{\bi{0}} - 2\alpha\onehot_{a^*}\otimes\onehot_{b^*}.
    \end{align}
    Since $\alpha\geq 1$, the negative margin of the $a^*\numchar + b^*$-th entry dominates the initial difference in the initialization which is bounded by $\snorm{\W_{2(0)}^{(k)}}_0\leq \frac12$. Thus we have \begin{align}
        \argmax\rbr{\W_{2(1)}^{(k)}} = \argmax\rbr{\W_{2(0)}^{(k)} - \pdr{\gdloss}{\P_{2(0)}^{(k)}}} = a^*\numchar+b^*
    \end{align}
    as desired.
\end{proof}

\subsubsection{Proof for \cref{thm:apdx-surrogate-GD-learning-coverable}}

Now we are ready to prove for the main theorem restated below:
\begin{thmbox}
\thmsurrogdcoverable
\end{thmbox}

\begin{proof}
The statement can be proven by a similar induction as in the proof for \cref{thm:apdx-surrogate-search-learning-coverable}.
    
    Let the induction hypothesis be such that when the enumeration goes to $k$-th column of the $i$-th layer, if $\maxfunc\srbr{\W_l^{(j)} + \W_l^{*(j)}} = \W^{*(j)}_l$ for all $(l,j)\in[2]\times[\numchar^2]$ and $\maxfunc\rbr{\W_l^{(j)}} = \W_l^{*(j)}$ for all $(l,j)$ such that $l<i$ or $l=i\land j<k$, then the gradient update on the $k$-th column of the $i$-th layer ends with $\W_i^{(k)} = \W_i^{*(k)}$ while $\maxfunc\srbr{\W_l^{(j)} + \W_l^{*(j)}} = \W^{*(j)}_l$ for all $(l,j)\in[2]\times[\numchar^2]$ is preserved.

    The base case is satisfied as with initialization of $\snorm{\W_{1(0)}}_0\leq \frac12$ and $\snorm{\W_{2(0)}}_0\leq \frac12$, by \cref{lemma:surrogateICLCapable-apdx}, we have $\maxfunc\srbr{\W_l^{(j)} + \W_l^{*(j)}} = \maxfunc\srbr{\zeromat + \W_l^{*(j)}} = \W^{*(j)}_l$ for all $(l,j)\in[\depth]\times[\numchar^2]$ and there are no requirements for $\W_i^{(k)} = \W_i^{*(k)}$ yet.

    For the induction step, we note that with the inductive hypothesis of $\maxfunc\srbr{\W_l^{(j)} + \W_l^{*(j)}} = \W^{*(j)}_l$ for all $(l,j)\in[2]\times[\numchar^2]$, $\srbr{\CM_{1(k)}, \W^*_2}$ (when $i=1$) or $\srbr{\W^*_1, \CM_{2(k)}}$ (when $i=2$) will correctly condition all columns of $\P$'s except for the $\P_{i}^{(k)}$ since 
    \begin{align}
        \CM_{i(k)}^{(k)} = \gtW_i\rbr{\mI_{\numchar^2} - \diag(\lonehot_k)}^{(k)} = \zeromat.
    \end{align}
    Thus by \cref{lemma:apdx-surrogate-learnW1} (when $i=1$) or \cref{lemma:apdx-surrogate-learnW2} (when $i=2$), we know that after updating  $\W_{i}^{(k)}$, we have $\maxfunc\srbr{\W_{i}^{(k)}} = \W_{i}^{*(k)}$. The newly added column provides the correct inductive hypothesis on $\maxfunc\srbr{\W_l^{(j)}} = \W_l^{*(j)}$ for the next enumeration step.

    By induction to $i=2$ and $k=\numchar^2$, we will be able to recover $\maxfunc\srbr{\W_i} = \W_i^*$ for all $i\in[\depth]$.
\end{proof}

Similarly we may use the coupon collecting argument to generalize the input to uniformly random strings as follows:

\begin{thmbox}
\begin{corollary}[Learning $\targetCB$ in $\translationtask(2,\numchar)$ with context-enhanced surrogate GD with random input]
\label{thm:apdx-surrogate-GD-learning-random}
\emph{}

For any initialization $\W_{1(0)},\W_{2(0)}\in\R^{\numchar^2\times\numchar^2}$ such that $\snorm{\W_{1(0)}}_0\leq \frac12$ and $\snorm{\W_{2(0)}}_0\leq \frac12$,
for any target \strcodebook{} $\targetCB = \cbr{\dictionary^*_1,\dictionary^*_2}$ in $\emph{\translationtask}(2,\numchar)$, with probability at least $1-\delta$ over a uniformly random input $\matrep_1$ of length $L = 2\numchar^2\log\frac{2\numchar}{\delta}$, \cref{alg:mlt-surrogate-search} provided with the ground truth label $\matrep^*_{3} = \emph{\translationtask}_\targetCB(\matrep_1)$ terminates with $\emph{\maxfunc}\rbr{\W_1} = \W_1^*$ and $\emph{\maxfunc}\rbr{\W_2} = \W_2^*$.
\end{corollary}
\end{thmbox}

\subsection{Auxiliary Lemmas for Learning Surrogate Models}

\begin{lemma}
\label{lemma:apdx-surrogate-gd-aux-QvQTv}
    For any long one-hot vector $\vv = \onehot_a\otimes\onehot_b$, with $\shiftmat = \rbr{I_\numchar\otimes \1_\numchar}\rbr{\1_\numchar\otimes I_\numchar}^\top$ we have \begin{align}
            \Q\vv 
            &= \ve_b\otimes \1_\numchar;\qquad
            \Q^\T\vv 
            = \1_\numchar\otimes \ve_a.
    \end{align}
\end{lemma}

\begin{proof} Note that with  $\vv = \onehot_a\otimes\onehot_b$, we have
    \begin{align}
        \begin{split}
            \Q\vv 
            &= \rbr{I_\numchar\otimes \1_\numchar}\rbr{\1_\numchar^\T\otimes I_\numchar^\T}\rbr{\ve_a\otimes \ve_b}
            = \rbr{I_\numchar\otimes \1_\numchar}\rbr{1\otimes \ve_b}
            = \rbr{I_\numchar\otimes \1_\numchar}\rbr{\ve_b \otimes 1}
            = \ve_b \otimes \1_\numchar;\\
            \Q^\T\vv 
            &= \rbr{\1_\numchar\otimes I_\numchar}\rbr{I_\numchar^\T\otimes \1_\numchar^\T}\rbr{\ve_a\otimes \ve_b}
            = \rbr{\1_\numchar\otimes I_\numchar}\rbr{\ve_a\otimes 1}
            = \rbr{\1_\numchar\otimes I_\numchar}\rbr{1\otimes\ve_a}
            = \1_\numchar\otimes \ve_a.
        \end{split}
    \end{align}
\end{proof}

\begin{lemma}
\label{lemma:apdx-surrogate-gd-aux1}
    For any one-hot vector  $\vv = \onehot_a\otimes\onehot_b\in\R^{\numchar^2}$, \begin{align*}
        \Q^\T\diag\rbr{\Q^\T\vv}\diag\rbr{\Q^\T\vv}\Q = \srbr{\1_\numchar \1_\numchar^\T}\otimes I_\numchar.
    \end{align*}
\end{lemma}

\begin{proof}
    By \cref{lemma:apdx-surrogate-gd-aux-QvQTv}, $\Q^\T\vv= \1_\numchar\otimes \ve_a.$
    Therefore $\diag\rbr{\Q^\T\vv} = \diag\rbr{\1_\numchar}\otimes\diag\rbr{\ve_a} = I_\numchar\otimes \diag\rbr{\ve_a}.$ Thus
    \begin{align}
    \begin{split}
        \diag\rbr{\Q^\T\vv}\Q = \rbr{I_\numchar\otimes \diag\rbr{\ve_a}}\rbr{I_\numchar\otimes \1_\numchar}\srbr{\1_\numchar^\T\otimes I_\numchar^\T} = \rbr{I_\numchar\otimes \ve_a}\srbr{\1_\numchar^\T\otimes I_\numchar^\T}
    \end{split}
    \end{align}
    and therefore we have\begin{align}
    \begin{split}
        \Q^\T\diag\rbr{\Q^\T\vv}\diag\rbr{\Q^\T\vv}\Q &= \srbr{\1_\numchar\otimes I_\numchar}\rbr{I_\numchar\otimes \ve_a^\T}\rbr{I_\numchar\otimes \ve_a}\srbr{\1_\numchar^\T\otimes I_\numchar^\T}\\
        &= \srbr{\1_\numchar\otimes I_\numchar}\rbr{I_\numchar\otimes 1}\srbr{\1_\numchar^\T\otimes I_\numchar^\T}\qquad\qquad\rbr{\text{since }\ve_a^\T\ve_a=1}\\
        &= \srbr{\1_\numchar\otimes I_\numchar}\rbr{1\otimes I_\numchar}\srbr{\1_\numchar^\T\otimes I_\numchar^\T}\\
        &= \srbr{\1_\numchar\otimes I_\numchar}\srbr{\1_\numchar^\T\otimes I_\numchar^\T}\\
        &= \srbr{\1_\numchar \1_\numchar^\T}\otimes I_\numchar
    \end{split}
    \end{align}
\end{proof}

\begin{lemma}
\label{lemma:apdx-surrogate-gd-aux2}
    For any one-hot vector  $\vv = \onehot_a\otimes\onehot_b\in\R^{\numchar^2}$, \begin{align*}
        \Q\diag\rbr{\Q\vv}\diag\rbr{\Q\vv}\Q^\T = I_\numchar\otimes \srbr{\1_\numchar \1_\numchar^\T}.
    \end{align*}
\end{lemma}

\begin{proof}
    This proof is very similar to the proof for \cref{lemma:apdx-surrogate-gd-aux1}. By \cref{lemma:apdx-surrogate-gd-aux-QvQTv}, $\Q^\T\vv= \1_\numchar\otimes \ve_a.$
    Therefore $\diag\rbr{\Q\vv} = \diag\rbr{\ve_b}\otimes\diag\rbr{\1_\numchar} =  \diag\rbr{\ve_b}\otimes I_\numchar.$ Thus
    \begin{align}
    \begin{split}
        \diag\rbr{\Q\vv}\Q^\T = \rbr{\diag\rbr{\ve_b}\otimes I_\numchar}\rbr{\1_\numchar\otimes I_\numchar} \rbr{I_\numchar^\top\otimes \1_\numchar^\top} = \rbr{\ve_b\otimes I_\numchar}\rbr{I_\numchar^\top\otimes \1_\numchar^\top}
    \end{split}
    \end{align}
    and therefore we have\begin{align}
    \begin{split}
        \Q\diag\rbr{\Q\vv}\diag\rbr{\Q\vv}\Q^\T &= \srbr{I_\numchar\otimes \1_\numchar}\rbr{\ve_b^\T\otimes I_\numchar}\rbr{\ve_b\otimes I_\numchar}\srbr{I_\numchar^\T\otimes \1_\numchar^\T}\\
        &= \srbr{I_\numchar\otimes \1_\numchar}\rbr{1\otimes I_\numchar}\srbr{I_\numchar^\T\otimes \1_\numchar^\T}\qquad\qquad\rbr{\text{since }\ve_b^\T\ve_b=1}\\
        &= \srbr{I_\numchar\otimes \1_\numchar}\rbr{I_\numchar\otimes 1}\srbr{I_\numchar^\T\otimes \1_\numchar^\T}\\
        &= \srbr{I_\numchar\otimes \1_\numchar}\srbr{I_\numchar^\T\otimes \1_\numchar^\T}\\
        &= I_\numchar \otimes \srbr{\1_\numchar \1_\numchar^\T}.
    \end{split}
    \end{align}
\end{proof}

\begin{thmbox}
\begin{lemma}[Tail Bound for Coupon Collector Problem \citep{motwani1995randomized}]\emph{}

\label{lemma:apdx-surrogate-coupon}
    For a set $S$ of size $n$, with probability at least $1-\delta$ one can cover all unique elements of $S$ in $n\log\frac n\delta$ independent uniformly random sampling trials from $S$.
\end{lemma}
\end{thmbox}

\subsection{Learning $\targetCB$ in $\translationtask_\targetCB$ with Gradient Descent (Empirical Evidence)}
\label{sec:apdx-experiments-linear-surrogate}
\label{sec:apdx-surrogateModel-GDexp}
In this section we provide more details on empirically optimizing the simple surrogate model $\surrogateModel_\surrogateWeights$, which was only briefly discussed in the main text by the end of \cref{sec:theory-surrogateModel}. We will first introduce the approximations we made to the surrogate model to make gradient-based optimization easy and stable, then we will present empirical results on the model learning target sets of phrasebooks $\translationtask_\targetCB$ in $\translationtask(5,10)$, $\translationtask(10,10)$, and even $\translationtask(20,10)$.

\subsubsection{Approximated Latent Model for GD}
\label{sec:apdx-surrogateModel-GDexp-apprGD}
Recall that with input sequence represented by $\matrep_1\in\R^{\numchar^2\times L}$, the surrogate model for a depth-$d$ translation is being recursively defined by the translation + shifting operations
\begin{align}
    \matrep_{i+1} = \maxfunc(\contextmat_i + \weightmat_i)\shiftfunc(\matrep_i)
\end{align}
until we reach $\matrep_{d+1}$. While this model captures the essence of transition from ICL capability to memorization of specific set of phrasebooks, $\maxfunc$ is making it not directly differentiable and hard to optimize. To address this issue, we approximate it with an column-wise softmax function with very low temperature ($T=1/25$). The recursive definition in the approximated model is then
\begin{align}
    \tilde\matrep_{i+1} = \softmaxfunc(25 (\contextmat_i + \weightmat_i))\shiftfunc(\tilde\matrep_i).
\end{align}
We denote the recursive surrogate model with the softmax substitution as $\tilde\surrogateModel_\surrogateWeights(\contextmat_1, \contextmat_2,\dots, \contextmat_\depth, \matrep_1)$ where
\begin{align}
\begin{split}
    \tilde\matrep_{\depth+1} &= \tilde\surrogateModel_\surrogateWeights(\contextmat_1, \contextmat_2,\dots, \contextmat_\depth, \matrep_1)
    \\
    &\triangleq \softmaxfunc(25\contextmat_\depth + 25\weightmat_\depth)\\
    &\qquad \shiftfunc\rbr{\softmaxfunc(25\contextmat_{\depth - 1} + 25\weightmat_{\depth-1})\shiftfunc\rbr{\cdots \softmaxfunc(25\contextmat_1 + 25\weightmat_1)\shiftfunc(\tilde\matrep_1) \cdots}}.
\end{split}
\end{align}

We define the objective function as the column-wise cross-entropy loss between the final output and the input. Namely for input $\matrep_1$ with prediction $\tilde\matrep_{\depth+1}$ and ground truth label $\matrep_{\depth + 1}^*$, the loss is computed as \begin{align}
    \cL &=\sum_{k=1}^L\text{CrossEntropy}(\tilde\matrep^{(k)}_{\depth+1}, \matrep^{*(k)}_{\depth+1}).
\end{align}
We follow the same masking (dropout) curriculum as described in \cref{sec:apdx-surrogateModel-learningGeneral} and \cref{sec:apdx-surrogateModel-learningD2}, that at each step we zero-out a single column from a single context matrix $\contextmat_i$. We experiment on two gradient update schemes:
\begin{itemize}
    \item Layer-wise Training: at each step, if we are masking a column on $\contextmat_i$, we only compute the gradient with respect to $\weightmat_i$ and update it. This training is more akin to the theoretical analysis described in \cref{sec:apdx-surrogateModel-learningD2}.

    \item Full Parameter Training: at any step, we compute the gradient with respect to each of the weight matrices and update all parameters. This is more akin to the real gradient-based training as we do not have the heuristics for localized update.
\end{itemize}

To allow for fast and stable training, we adopt a very large learning rate of $\eta = 100$ and apply parameter clipping between $[0,1]$ after each update. The complete algorithm is described as follows:

\begin{algorithm}[H]
   \caption{Layerwise Gradient Descent with Context-Enhanced Learning For Optimizing $\surrogateModel$}
   \label{alg:mlt-surrogate-layerwisegd}
\begin{algorithmic}[1]
    \STATE {\bfseries Input:}
    \STATE input $\matrep_1\in\R^{n^2\times L}$, label  $\matrep_{\depth+1}^*\in\R^{n^2\times L}$, descriptive text $\gtW_1,\dots,\gtW_{\depth}\in\R^{n^2\times n^2}$, learning rate $\eta$, total steps $T$
    \STATE
    \STATE \textbf{Initialize} $\W_1,\dots, W_\depth\gets \zeromat$ \COMMENT{Start with zero initialization}
    \FOR{$t=1$ {\bfseries to} $T$}
    \STATE $i\gets \floor{(t-1)/n^2}\% d + 1$ \COMMENT{Get the layer to be masked}
    \STATE $k\gets ((t-1)\%n^2) + 1$ \COMMENT{Get the column index to be masked}
    \STATE \textbf{Initialize} $\CM_{i(k)}\gets \gtW_i\rbr{\mI_{\numchar^2} - \diag(\lonehot_k)}$ \COMMENT{Create masked context matrix}
    \STATE $\tilde\matrep_{\depth+1} \gets \surrogateModel_\surrogateWeights(\W^*_1\dots,\W^*_{i-1}, \CM_{i(k)},\W^*_{i+1}\dots,\W^*_{\depth}, \matrep_1)$
    \STATE $\cL\gets \text{CrossEntropy}(\tilde\matrep_{\depth+1}, \matrep^{*}_{\depth+1})$
    \STATE $\W_i \gets \W_i - \eta\nabla_{\W_i}\cL$ \COMMENT{Update the weight for the layer with mask}
    \ENDFOR
    \STATE \textbf{Return} $\W_1,\dots,\W_\depth$.
\end{algorithmic}
\end{algorithm}

\begin{algorithm}[H]
   \caption{Full Parameter Gradient Descent with Context-Enhanced Learning For Optimizing $\surrogateModel$}
   \label{alg:mlt-surrogate-fullparamgd}
\begin{algorithmic}[1]
    \STATE {\bfseries Input:}
    \STATE input $\matrep_1\in\R^{n^2\times L}$, label  $\matrep_{\depth+1}^*\in\R^{n^2\times L}$, descriptive text $\gtW_1,\dots,\gtW_{\depth}\in\R^{n^2\times n^2}$, learning rate $\eta$, total steps $T$
    \STATE
    \STATE \textbf{Initialize} $\W_1,\dots, W_\depth\gets \zeromat$ \COMMENT{Start with zero initialization}
    \FOR{$t=1$ {\bfseries to} $T$}
    \STATE $i\gets \floor{(t-1)/n^2}\% d + 1$ \COMMENT{Get the layer to be masked}
    \STATE $k\gets ((t-1)\%n^2) + 1$ \COMMENT{Get the column index to be masked}
    \STATE \textbf{Initialize} $\CM_{i(k)}\gets \gtW_i\rbr{\mI_{\numchar^2} - \diag(\lonehot_k)}$ \COMMENT{Create masked context matrix}
    \STATE $\tilde\matrep_{\depth+1} \gets \surrogateModel_\surrogateWeights(\W^*_1\dots,\W^*_{i-1}, \CM_{i(k)},\W^*_{i+1}\dots,\W^*_{\depth}, \matrep_1)$
    \STATE $\cL\gets \text{CrossEntropy}(\tilde\matrep_{\depth+1}, \matrep^{*}_{\depth+1})$
    \FOR{$l=1$ {\bfseries to} $d$}
    \STATE $\W_l \gets \W_l - \eta\nabla_{\W_l}\cL$ \COMMENT{Update the weight for all layers}
    \ENDFOR
    \ENDFOR
    \STATE \textbf{Return} $\W_1,\dots,\W_\depth$.
\end{algorithmic}
\end{algorithm}

\newpage

\begin{figure}[H]
     \centering
     \begin{subfigure}[b]{0.7\linewidth}
         \centering
         \includegraphics[width=\textwidth]{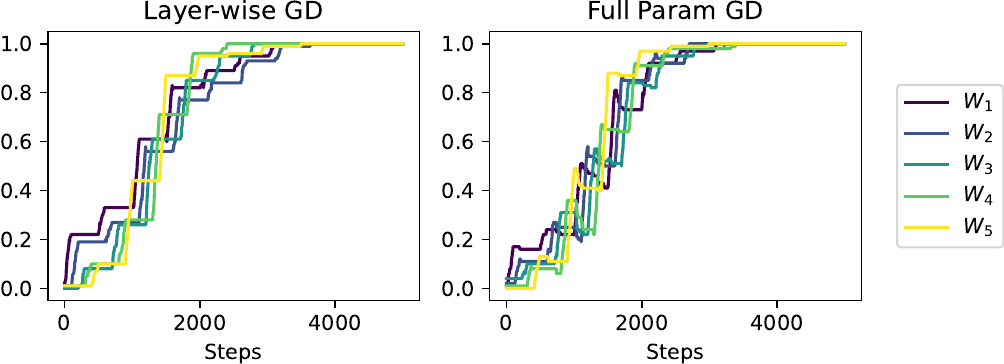}
         \caption{$\translationtask(5,10)$}
     \end{subfigure}
     \vspace{0.2in}
     
     \begin{subfigure}[b]{0.7\linewidth}
         \centering
         \includegraphics[width=\textwidth]{files/Appendix/experiments_apdx/figs/surrogate_linear/GD/N10D10_layerwise_vs_fullparam.pdf}
         \caption{$\translationtask(10,10)$}
     \end{subfigure}
     \vspace{0.2in}
     
     \begin{subfigure}[b]{0.7\linewidth}
         \centering
         \includegraphics[width=\textwidth]{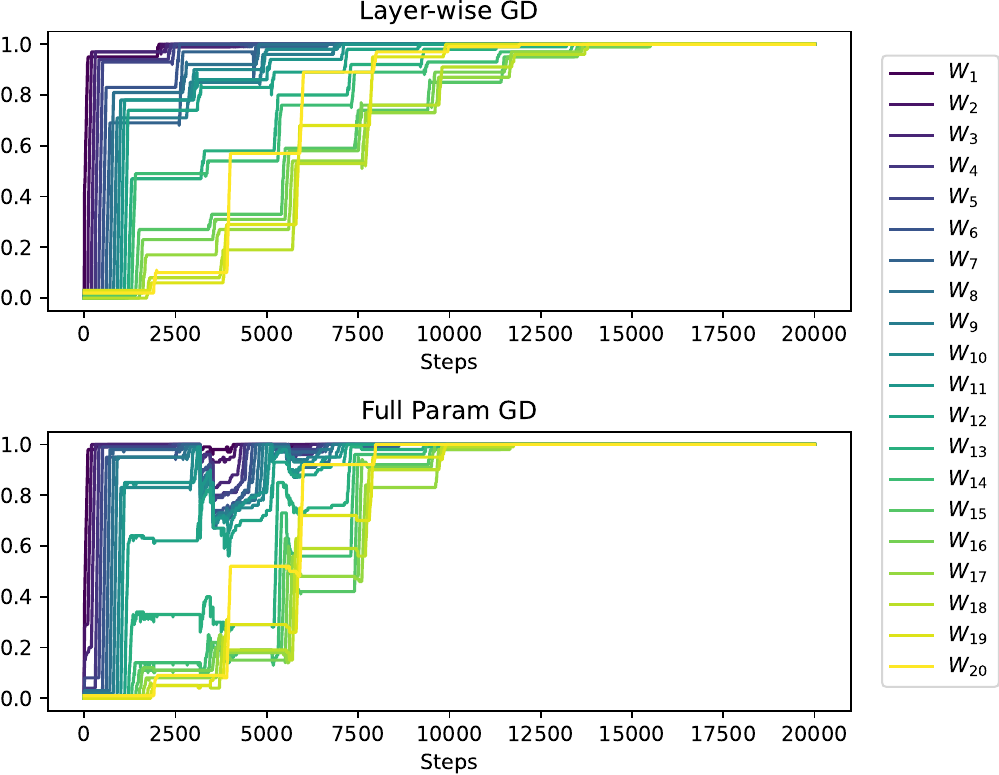}
         \caption{$\translationtask(20,10)$}
     \end{subfigure}
        \caption{We perform Layer-wise and Full Parameter Training with Gradient Descent (defined in \cref{sec:apdx-surrogateModel-GDexp-apprGD}) on the $\surrogateModel$ (\cref{defn:apdx-surrogate-surrogateModel}) designed for $\translationtask_\targetCB$ in $\translationtask(5,10), \translationtask(10,10), \translationtask(20, 10)$ respectively (alternately, depth $\depth=5,10,20$ respectively, while number of characters is fixed at $\numchar=10$). Here, we report the portion of columns from the trainable parameters, which after $\maxfunc$ application $\{ \maxfunc\rbr{\weightmat_i}\}_{i=1}^{\depth}$,  align with the corresponding stochastic matrices of the \strdicts{} $\{\transmat\rbr{{\vardict^*_i}}\}_{i=1}^{\depth}$.We observe that under both algorithms, the trainable parameters quickly learn the relevant stochastic matrices. }
        \label{fig:GD_latent}
\end{figure}

\begin{figure}[H]
    \centering
    \includegraphics[width=0.75\linewidth]{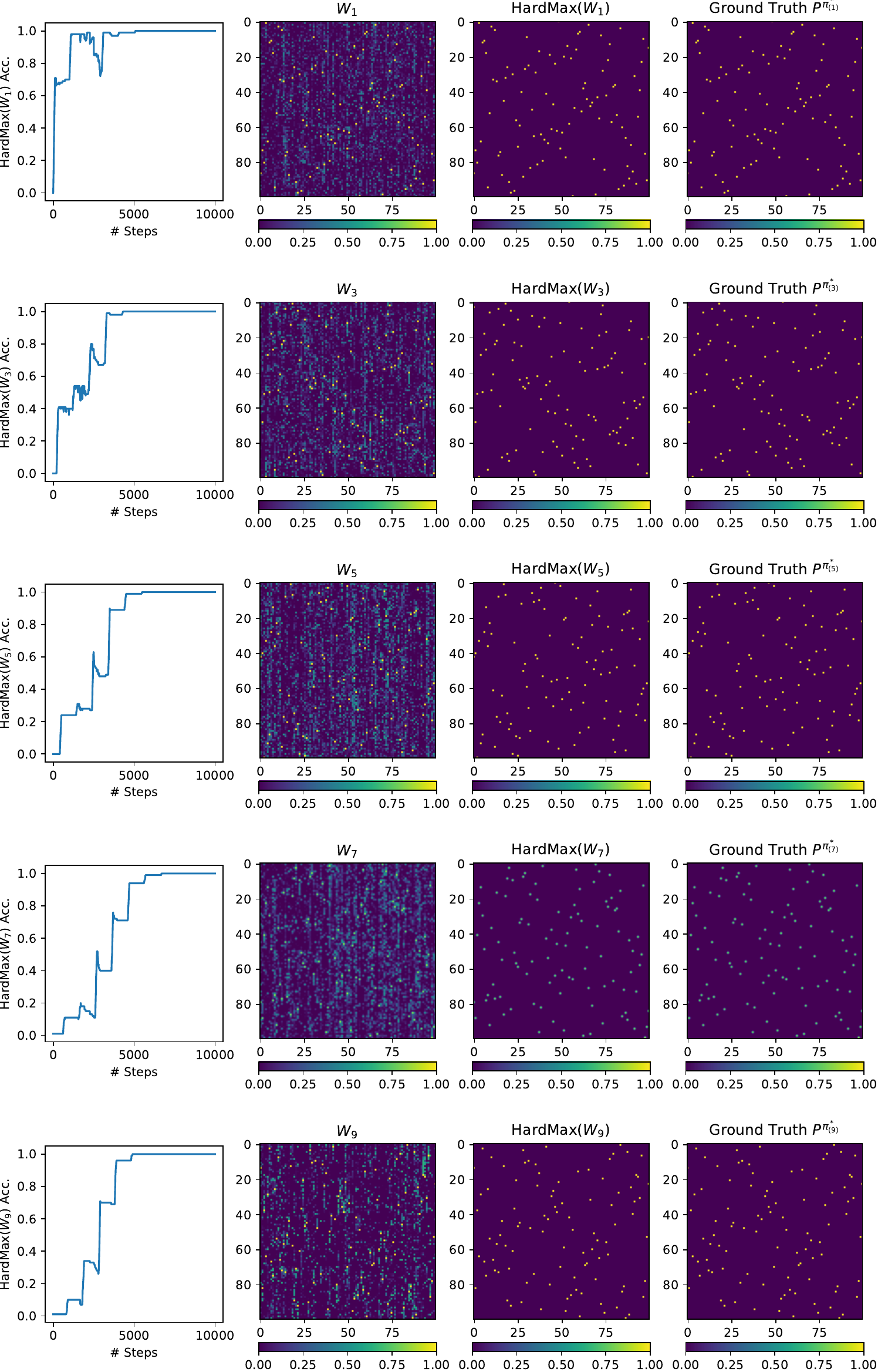}
    \caption{Detailed analysis on the Full parameter training behavior of the trainable parameters in $\surrogateModel$ from \cref{fig:GD_latent} for $\translationtask_\targetCB$ in $\translationtask(10, 10)$ (i.e. depth $\depth=10$ and number of characters $\numchar=10$). We report the behavior of all odd-index parameters $\weightmat_1, \weightmat_3, \cdots, \weightmat_9$. (left to right) first, we show the number of columns of the trainable parameter, which after $\maxfunc\rbr{\weightmat_i}$ align with the corresponding columns of $\transmat\rbr{{\vardict^*_i}}.$ Second, third and fourth visualize the matrices, $\weightmat_i$, $\maxfunc\rbr{\weightmat_i}$, and $\transmat\rbr{{\vardict^*_i}}$ respectively. All matrices learn to match $\transmat\rbr{{\vardict^*_i}}$ at the end of training with $\maxfunc$ operation. }
    \label{fig:enter-label}
\end{figure}

\newpage
\section{Construction of a Transformer that can Simulate the Latent Model} \label{app:construction_transformer}

\subsection{Useful definitions and lemmas}
\begin{definition}[Relative self-attention with $1$ head] \label{def:rel_attn}
    For a set of matrices $\{\mW_{query}, \mW_{key}, \mW_{value}\}$ with each matrix $\in \mathbb{R}^{k \times k}$ for some $k > 0$ and a set of $(t+1)$ biases $\{b_i\}_{t \leq i \leq 0}$,  the self-attention computation on an input sequence $\vx_1, \cdots, \vx_{\inputlength}$ with each $\vx_{p_2} \in \mathbb{R}^{k}$ is given by the output sequence $\vy_1, \cdots, \vy_{\inputlength}$, where for all $p_1 \in [1, \inputlength]$
    \begin{align*}
        &\vy_{p_1} = \sum_{p_2=1}^{\inputlength} a_{p_1, p_2} \vo_j, \text{ where } a_{{p_1}{p_2}} = \frac{e^{\vq_{p_1}^{\top}\vk_{p_2} + b_{{p_2}-{p_1}}}}{\sum_{{p_2}'\leq {p_1}}e^{\vq_{p_1}^{\top}\vk_{{p_2}'}+ b_{{p_2}'-{p_1}}}} \text{ if } {p_2} \leq {p_1}, 0 \text{ otherwise} \\
        &\vq_{p_2} = \mW_{query} \vx_{p_2}, \quad \vk_{p_2} = \mW_{key} \vx_{p_2}, \quad \vo_{p_2} = \mW_{value} \vx_{p_2}, \quad \text{ for all } p_2  \in [1, \inputlength]. 
    \end{align*}
\end{definition}

For a relative self-attention with $H$ heads, we will simply add the output of the $H$ heads as the final output.

\begin{definition}[MLP]\label{def:mlp}
    For a set of matrices $\mW_{outer} \in \mathbb{R}^{k\times H}, \mW_{inner} \in \mathbb{R}^{H\times k}$ for some $k, H > 0$ and an activation function $\sigma$, the output of the MLP layer on an input sequence $\vx_1, \cdots, \vx_{\inputlength}$ with each $\vx_i \in \mathbb{R}^{k}$ is given by the output sequence $\vy_1, \cdots, \vy_{\inputlength}$, where for all $i$
    \begin{align*}
        \vy_i = \mW_{outer} \sigma ( \mW_{inner} \vx_i ).
    \end{align*}
\end{definition}

\begin{lemma}\label{lem:gelu_mul}
    For GELU \citep{hendrycks2016gaussian} activation function, which takes $x \in \mathbb{R}$ as input and returns $x \Phi(x)$ as output, with $\Phi(x)$ representing the standard Gaussian cumulative distribution function, for any two variables $x, y \in \mathbb{R}$, the following holds true:
    \begin{align*}
        \sqrt{\pi/2}\left(\text{GELU}(x+y) - \text{GELU}(x) - \text{GELU}(y)\right) = xy + \mathcal{O}(x^3 y^3).
    \end{align*}
\end{lemma}
The above lemma has been taken from \citet{akyurek2023what}.

\subsection{Transformer construction}

Recall that a translation task in $\translationtask(\depth, \numchar)$ involves two primary operations at each step : \circularshift{} and \translate{}. We will refer to the surrogate model to use notations for different operations in the translation task. Recall from \cref{eqn:theory-surrogateRep-recurrance-apdx}, the surrogate model for \translationtask$(\depth, \numchar)$, denoted by $\surrogateModel_\surrogateWeights(\cdot)$ with trainable parameters $\surrogateWeights = \{\weightmat_1, \cdots, \weightmat_{\depth}\}$, can be represented by the following recursive expression 
\begin{align}
    \matrep_{i+1} = \maxfunc(\contextmat_i + \weightmat_i)\shiftfunc(\matrep_i) := \maxfunc(\contextmat_i + \weightmat_i)\shiftmatrep_i,
\end{align}
for $1 \le i \le \depth$. Here $\shiftfunc$ represents \circularshift{} operation, and $\maxfunc(\contextmat_i + \weightmat_i)$ represents \translate{} operation, where the operation can be done either using the relevant in-context information $\contextmat_i$ or the in-weights memory parameters $\weightmat_i$ when in-context information  isn't provided in form of $\contextmat_i$. $\maxfunc$ is the column-wise hard-max function converting $\contextmat_i + \weightmat_i$ to a binary column stochastic matrix.

\begin{tcolorbox}
\begin{lemma}\label{lem:construction}
    For the family of translation tasks $\translationtask(\depth, \numchar)$, there exists a transformer model with embedding size $2\numchar^2 + 2\depth + 4$, $2\depth$ relative self-attention layers (containing either $1$ or $3$ heads), and $2\depth$ MLP layers that can simulate the surrogate model, $\surrogateModel_\surrogateWeights(\cdot)$. 

    \vspace{1 em}
    For each input sequence $\inputseq$ of length $\inputlength$, the input sequence of embeddings to the transformer will be of length $\numchar^2 \depth + \inputlength/2 + \depth$, where the last $\depth$ embeddings are padding tokens ($\pad$) and given as $0$s, and the length of output sequence of the transformer model will be $\numchar^2 \depth + \depth + \inputlength/2$, where the last $\inputlength/2$ output embeddings will be used for loss computation. The middle $\depth$ embeddings in output are represented by $\think$ tokens.
    \footnote{In our experiments, we used output sequence length as $\numchar^2 \depth +  \depth + \depth \inputlength/2$, where $(\depth+1) \inputlength/2$ tokens in the output sequence were represented by $\think$ tokens. Instead, we only use $\depth$ $\think$ tokens in our output construction for simplicity, the proof can be easily modified to align with the experiments. }
    
\end{lemma}
\end{tcolorbox}

\paragraph{Outline of the construction:} 
Our argument will be for any general $\translationtask(\depth, \numchar)$.
To create the transformer, we will create similar transformer modules that handle 
$\maxfunc(\contextmat_i + \weightmat_i)\shiftfunc(\matrep_i)$ for each $i$.
We will refer to the constructed modules as $\constructmodule$ and will mention the specific step $i$ as an argument when we attempt to use the module to perform translation step $\maxfunc(\contextmat_i + \weightmat_i)\shiftfunc(\matrep_i)$. W.l.o.g., we will assume we are building a $\constructmodule$ to perform translation step $\maxfunc(\contextmat_i + \weightmat_i)\shiftfunc(\matrep_i)$. \constructmodule{} will contain two modules, \constructcircularshiftmodule{} and \constructtranslatemodule{}, which simulate \circularshift{} and \translate{} operations, and have been outlined in \cref{alg:circularshift,alg:translate}.

Next, we explain the structure of the embeddings in the transformer architecture. Our embeddings will be built on the context representation matrices $\{\contextmat_1, \cdots, \contextmat_{\depth}\}$ and the matrix representation $\matrep_1$ from input sequence $\inputseq := ( \intermseq_{1, 1}, \cdots, \intermseq_{1, \inputlength} )$ and subsequent intermediate representation of the translation task. Furthermore, our embeddings will contain additional information like indices of the context matrices when utilizing them in-context, segment indicators that represent whether embeddings represent
context matrices, or the input sequence tokens, and start and end indicators that indicate the start and the end embeddings representing the input sequence. This information can be extracted from the input sequence using a few input processing layers, though we do not delve into the specifics.

\subsection{Structure of input embeddings to $\constructmodule$} 

Our embeddings will be split into $4$ components: token, context matrix indicator, start and end indicator, and segment indicator. To maintain simplicity in our discussion, we will present them as $4$ separate embeddings.

For a module that will represent $\maxfunc(\contextmat_i + \weightmat_i)\shiftfunc(\matrep_i)$, we assume the input will be provided as $2$ major segments. 
\begin{enumerate}
    \item \textbf{First segment: in-context matrices} 
    We will give the in-context information $\{ \contextmat_{p_1} \}_{p_1=1}^{\depth}$ as follows.: each context matrix $\contextmat_{p_1}$ will be fed as
    $\numchar^2$ token embeddings, $\{ [\onehot_j; \contextmat_{p_1} \onehot_j] \in \mathbb{R}^{2\numchar^2}\}_{j=1}^{\numchar^2}$, where $\onehot_j$ represents a one-hot $\numchar^2$ dimensional vector that contains $1$ in position $j$ and $[\onehot_j; \contextmat_{p_1} \onehot_j]$ represents a concatenation of $\onehot_j$ and $\contextmat_i \onehot_j$.
    Thus, the in-context information will look as follows
    \begin{align*}
        \{[\onehot_j; \contextmat_1 \onehot_j]\}_{j=1}^{\numchar^2}, \cdots, \{[\onehot_j; \contextmat_\depth \onehot_j]\}_{j=1}^{\numchar^2}
    \end{align*}

    \paragraph{Additional embedding: context matrix level indicator} In order to differentiate the different context matrices, we will use an additional $\depth$ dimensions in the embeddings to represent a one-hot vector that indicates the index of the corresponding step they will be used for. For simplicity, we will represent these dimensions separately as separate embeddings: $\levelembedding_{{p_1}} \in \mathbb{R}^{\depth}$, which are one-hot vectors that contain $1$ in position $p_1$ and $0$ otherwise.  

    \item  \textbf{Second segment: input query} 
    For a length-$L$ input sequence $\intermseq_i = \srbr{\intermseq_{i,1}, \intermseq_{i,2}, \dots, \intermseq_{i,\inputlength}}$, we will use the sequence of columns of its matrix representation $\matrep_i \in \R^{\numchar^2\times \inputlength/2}$, appended by $0$s to match embedding sizes, as token embeddings for the input query. A padding embedding $\pad$, containing $0$s, follows this sequence as an end of sequence embedding.

    \paragraph{Additional embedding: start and end indicator embedding} We will have $2$ additional dimensions representing whether a token represents the start or the end of the input query sequence (first or second dimension activated respectively). Start of the input query sequence is determined by the first embedding in the input query, while end of the input query sequence is determined by the first padding embedding containing $0$s after the input query sequence.
    We will represent these dimensions separately as separate embeddings: $\{\startendembedding_{1}, \startendembedding_{2}, \mathbf{0} \in \mathbb{R}^{2}\}_{i \in \inputlength}$, which are one-hot vectors. $\startendembedding_1$ contains $1$ in dimension $1$ if the embeddings represents start of the sequence,  $\startendembedding_2$ contains $1$ in dimension $2$ if the embeddings represent end of the sequence. Other embeddings have $\mathbf{0}$s.  
    
\end{enumerate}

\paragraph{Additional embedding: segment embedding} We will differentiate the input embeddings in the two segments using $2$ dimensions that represent one-hot vectors indicating segment indices. We will represent these dimensions separately as separate embeddings: $\segment_{1}, \segment_{2} \in \mathbb{R}^{2}$, where both are one-hot vectors, with $\segment_1$ containing $1$ in dimension $1$ and $\segment_2$ containing $1$ in dimension $2$.

All the notations have been summarized in \cref{tab:input_transformer_handconstr}.

\paragraph{Additional optional inputs:} There might be additional input embeddings, represented as $\think$ in the first segment, which are null inputs and are ignored during self-attention computation. As discussed next, we will right shift the sequence by $1$ at each step, in order to handle \circularshift{} operation with causal masking in transformers.

\subsection{Structure of the output embeddings from $\constructmodule$} All the embeddings in the first segment are kept intact. In the second segment, the module outputs a null output $\think$, followed by $\inputlength/2$ output embeddings that represent the columns of $\matrep_i$. We will require $\think$ to represent \circularshift{} with causal self-attention in transformers. 
$\think$ will be ignored in self-attention computation and so, we will ignore their discussion for simplicity of presentation. Other embedding values are kept intact for the generated output, except start and end indicator embeddings $\startendembedding_1, \startendembedding_2$, which need to be right shifted at each step. The right shift operation can be handled similar to our computations on the token embeddings and so, we ignore them in our construction below. We summarize these in \cref{tab:output_transformer_handconstr}.

\begin{table}[!ht]
\scalebox{0.9}{
    \centering
    \begin{tabular}{c|c|c|c}
    \toprule
    Embedding Name & Dimension size & First segment values & Second segment values\\ 
    & & (In-context information) & (Input sequence embeddings) \\ 
    \midrule
       Token   &  $2\numchar^2$ & 
        $\{[\onehot_j; \contextmat_1 \onehot_j]\}_{j=1}^{\numchar^2}, \cdots, \{[\onehot_j; \contextmat_\depth \onehot_j]\}_{j=1}^{\numchar^2}$ & $[\matrep_{i-1}^{(1)}; \mathbf{0}], \cdots, [\matrep_{i-1}^{(\inputlength/2)}; \mathbf{0}]$, $\pad$ \\
       Context matrix index indicator   & $\depth$ & $\{\levelembedding_{1}\}_{[1, \numchar^2]}, \{\levelembedding_{2}\}_{[1, \numchar^2]}, \cdots \{\levelembedding_{\depth}\}_{[1, \numchar^2]}$ & $\mathbf{0}, \cdots, \mathbf{0}$ \\
       Start and End indicator   & $2$ & $\mathbf{0}, \cdots, \mathbf{0}$ & $\startendembedding_1$, $\mathbf{0}, \cdots, \mathbf{0}$, $\startendembedding_2$ \\
       Segment & $2$ & $\segment_1, \cdots, \segment_1$ & $\segment_2, \cdots, \segment_2$\\
    \bottomrule
    \end{tabular}
    }
    \caption{Input embeddings to \constructmodule{} that simulates  $\maxfunc(\contextmat_i + \weightmat_i)\shiftfunc(\matrep_i)$.  $\onehot_j$ indicates a one-hot $\numchar^2$ dimensional vector that contains $1$ in dimension $j$.}
    \label{tab:input_transformer_handconstr}
\end{table}

\begin{table}[!ht]
\scalebox{0.9}{
    \centering
    \begin{tabular}{c|c|c|c}
    \toprule
    Embedding Name & Dimension size & First segment values & Second segment values\\  
    & & (In-context information) & (Input sequence embeddings) \\ 
    \midrule
       Token   &  $2\numchar^2$ & 
       $\{[\onehot_j; \contextmat_1 \onehot_j]\}_{j=1}^{\numchar^2}, \cdots, \{[\onehot_j; \contextmat_\depth \onehot_j]\}_{j=1}^{\numchar^2}$ & $\think$,  $[\matrep_{i}^{(1)}; \mathbf{0}], \cdots, [\matrep_{i}^{(\inputlength/2)}; \mathbf{0}]$  \\
       Context matrix index indicator   & $\depth$ & $\{\levelembedding_{1}\}_{[1, \numchar^2]}, \{\levelembedding_{2}\}_{[1, \numchar^2]}, \cdots \{\levelembedding_{\depth}\}_{[1, \numchar^2]}$ & $\mathbf{0}, \mathbf{0}, \cdots, \mathbf{0}$ \\
       Start and End indicator   & $2$ & $\mathbf{0}, \cdots, \mathbf{0}$ & $\mathbf{0}$, $\startendembedding_1$, $\mathbf{0}, \cdots, \mathbf{0}$, $\startendembedding_2$ \\
       Segment & $2$ & $\segment_1, \cdots, \segment_1$ & $\mathbf{0}$, $\segment_2, \cdots, \segment_2$\\
    \bottomrule
    \end{tabular}
    }
    \caption{Output of the transformer module \constructmodule{} that simulates $\maxfunc(\contextmat_i + \weightmat_i)\shiftfunc(\matrep_i)$ . $\think$ represents a null output and won't be attended to in the future modules. We ignore this symbol for simplicity, when analyzing any module. $\onehot_j$ indicates a one-hot $\numchar^2$ dimensional vector that contains $1$ in dimension $j$.}
    \label{tab:output_transformer_handconstr}
\end{table}

\paragraph{Constructing the {\constructmodule}:} The {\constructmodule} consists of $2$ self-attention layers and $2$ MLP layers. We use one self-attention layer and an MLP layer to represent \circularshift{} operation, one self-attention layer to represent $\contextmat_i \shiftfunc(\matrep_i)$, and one MLP layer to represent $\maxfunc(\contextmat_i + \weightmat_i) \shiftfunc(\matrep_i)$. We name the two modules for \circularshift{} and \translate{} as \constructcircularshiftmodule{} and \constructtranslatemodule{} respectively. We have outlined their constructions in \cref{alg:circularshift,alg:translate}.

\subsection{Step 1 (\constructcircularshiftmodule{}): Represent \circularshift{} using a self-attention and an MLP layer} As \circularshift{} only focuses on the input query sequence and not the in-context matrices, we will simply focus the module's operation on embeddings in the second segment. The effect of the operation on embeddings in the first segment can be removed using a gated residual connection. From \cref{lemma:shift-equiv}, we have that the output of the \circularshift{} operation on any sequence, represented by its matrix representation $\matrep_i$, can be written as 
\begin{align*}
    \shiftmatrep_i^{(j)} = Q\matrep^{(j)}_i\odot Q^\T\matrep^{((j+1)\% \inputlength)}_i, \text{ for all } 1 \leq j \leq \inputlength/2. 
\end{align*}
where $Q = \rbr{I_\numchar\otimes \1_\numchar}\rbr{\1_\numchar\otimes I_\numchar}^\top$
, $\1_n \in \R^{\numchar\times 1}$ is the all-ones vector, and $\odot$ is the Hadamard product.

In order to represent the operation, we will first use a self-attention layer to compute $\rbr{\1_\numchar\otimes I_\numchar}^\top \matrep^{(j)}_i$ and $\rbr{I_\numchar\otimes \1_\numchar} \matrep^{(j+1\% \inputlength)}_i$ at each column $j$. Because the computation of $\shiftmatrep_i^{(j)}$ requires the model to look forward to  $\matrep_i^{(j+1)}$, we need to shift the computation of $\shiftmatrep^{(j)}_i$ to position $j+1$, as a causal attention mask is involved in self-attention computation. 

\paragraph{After right shift operation,} we will represent the output of the self-attention computation as $\think$, $[\mathbf{o}_2; \mathbf{0}], [\mathbf{o}_3; \mathbf{0}], \cdots, [\mathbf{o}_{\inputlength/2+1}; \mathbf{0}]$, and we will ignore the $\think$ embedding. Note that the second half of the output embeddings will still contain $0$s and we will ignore them in the current computation. Then, the above computation can be rephrased as
\begin{align*}
    &\mathbf{o}_j = Q\matrep^{(j-1)}_i\odot Q^\T\matrep^{(j)}_i, \text{ for all } 2 \leq j \leq \inputlength/2. \\
    &\mathbf{o}_{\inputlength/2+1} = Q\matrep^{(\inputlength/2)}_i\odot Q^\T\matrep^{(1)}_i
\end{align*}

\textbf{Self-attention layer:} The computation of $\mathbf{o}_j$, for $2 \leq j \leq \inputlength/2$, requires the computation of  $\rbr{\1_\numchar\otimes I_\numchar}^\top \matrep^{(j-1)}_i$ and $\rbr{I_\numchar\otimes \1_\numchar} \matrep^{(j)}_i$.
This will require $2$ attention heads, one head that attends to itself, and another that attends to previous embedding at each position. We will outline both below. We will require one additional head, as computing $\vo_{\inputlength/2+1}$ will require the model to compute $\rbr{I_\numchar\otimes \1_\numchar} \matrep^{(1)}_i$. 
\begin{enumerate}
    \item Attention Head 1 computes $\rbr{\1_\numchar\otimes I_\numchar}^\top \matrep^{(j-1)}_i$ at position $j$ for all $2 \leq j \leq \inputlength/2+1$. This can be done using a self-attention head (\cref{def:rel_attn}) that sets query and key matrices $\mW_{query}$, $\mW_{key}$, and biases $\{b_i\}_{t \leq i \leq 0}$ such that the attention score between embeddings at any two positions $p_1, p_2$ is given as follows:
    \begin{align*}
        a_{p_1, p_2} = 1, \text{if } p_2 - p_1=-1, \quad 0 \text{ otherwise}
    \end{align*}
    $\mW_{value}$ is set such that for any input $\vx$, the output of $\mW_{value} \vx$ is given by
    \begin{align*}
        (\mW_{value} \vx)_{p_1} = \sum_{j=0}^{\numchar-1} x_{ \numchar \cdot j + {p_1} }.  
    \end{align*}
    In simple words, this operation simply adds up the values in dimensions ${p_1}, {p_1}+\numchar, {p_1}+2\numchar, \cdots$ and stores them at position $p_1$ for all $1 \leq {p_1} \leq \numchar$.
    
    \item Attention Head 2 computes $\rbr{I_\numchar\otimes \1_\numchar} \matrep^{(j)}_i$ at position $j$ for all $2 \leq j \leq \inputlength/2$. This can be done using a self-attention head (\cref{def:rel_attn}) that sets $\mW_{query}$, $\mW_{key}$, $\mW_{value}$  and biases $\{b_i\}_{t \leq i \leq 0}$ such that the attention score between embeddings at any two positions $p_1, p_2$ is given as follows:
    \begin{align*}
        a_{p_1, p_2} = 1, \text{if } p_2 - p_1=0, \quad 0 \text{ otherwise}
    \end{align*}
    $\mW_{value}$ is set such that for any input $\vx$, the output of $\mW_{value} \vx$ is given by
    \begin{align*}
        (\mW_{value} \vx)_{{p_1} + \numchar} = \sum_{j=1}^{\numchar} x_{ j + {p_1} \numchar - \numchar }.  
    \end{align*}
    In simple words, this operation simply adds up the values in dimensions  $1+({p_1}-1) \numchar, 2+({p_1}-1) \numchar, \cdots$  and stores them at dimension ${p_1} + \numchar$ for all $1 \leq {p_1} \leq \numchar$.
    
    \item Attention head 3 will compute $\rbr{I_\numchar\otimes \1_\numchar} \matrep^{(1)}_i$ and store in $\mathbf{o}_{\inputlength/2+1}$. This can be done by a self-attention layer which activates only between positions $p_1$ and $p_2$ that represent the start and the end tokens of the sequence, i.e.  contain $\startendembedding_1$ and $\startendembedding_2$ as start and end indicator embeddings, and is $0$ otherwise. $\mW_{value}$ is set same as attention head $2$.
    
\end{enumerate}
The output of the three heads are simply added up. Hence, at each position $2 \leq j \leq \inputlength/2+1$, the output $\vo_j$ has $\rbr{\1_\numchar\otimes I_\numchar}^\top \matrep^{(j-1)}_i$ in $[1, \numchar]$ dimensions and $\rbr{I_\numchar\otimes \1_\numchar} \matrep^{(j\%\inputlength)}_i$ in $[\numchar+1, 2\numchar]$ dimensions. 
 
\textbf{MLP layer:} The objective with the MLP layer (\cref{def:mlp}) will be to multiply $\rbr{\1_\numchar\otimes I_\numchar}^\top \matrep^{(j-1)}_i$ present in $[1, \numchar]$ dimensions and $\rbr{I_\numchar\otimes \1_\numchar} \matrep^{(j\%\inputlength)}_i$ present in $[\numchar+1, 2\numchar]$ dimensions in each position $j$. This can be done by using an MLP layer with GELU activation by using \cref{lem:gelu_mul}. The weights of the MLP
layer are set as follows: $\mW_{inner}$ is set such that for all input $\vx$, we have
\begin{align*}
    (\mW_{inner} \vx)_i = \frac{1}{N} x_i + \frac{1}{N} x_{\numchar + i}, &\quad \text{ for all } 1 \leq i \leq \numchar, \\
    (\mW_{inner} \vx)_{i + \numchar} = \frac{1}{N} x_i, &\quad \text{ for all } 1 \leq i \leq \numchar, \\
    (\mW_{inner} \vx)_{i + 2\numchar} = \frac{1}{N} x_{\numchar + i}, &\quad \text{ for all } 1 \leq i \leq \numchar.
\end{align*}
$\mW_{outer}$ is set such that for all $\vx \in \mathbb{R}^{3\numchar}$
\begin{align*}
    (\mW_{outer} \vx)_i = N^2 (x_{i} - x_{i+\numchar} - x_{i+2\numchar}), &\quad \text{ for all } 1 \leq i \leq \numchar.
\end{align*}
All other coordinates in these matrices are set as $0$s. $N$ is set as a large number (say $100$). By \cref{lem:gelu_mul}, the output of the MLP layer will be $\think, \mathbf{o}_2, \cdots, \mathbf{o}_{\inputlength/2+1}$, with $\mathbf{o}_{j}$ containing $$ \shiftmatrep_i^{(j-1)} + \mathcal{O}(N^{-4}) :=  \rbr{\1_\numchar\otimes I_\numchar}^\top \matrep^{(j-1)}_i \odot \rbr{I_\numchar\otimes \1_\numchar} \matrep^{(j\%\inputlength)}_i + \mathcal{O}(N^{-4}) $$ at each position $2 \leq j \leq \inputlength/2$.

\begin{table}[!ht]
    \centering
    \scalebox{0.9}{
    \begin{tabular}{c|c|c|c}
    \toprule
    Embedding Name & Dimension size & First segment values & Second segment values\\  
    & & (In-context information) & (Input query sequence embeddings) \\ 
    \midrule
       Token   &  $2\numchar^2$ & 
       $\{[\onehot_j; \contextmat_1 \onehot_j]\}_{j=1}^{\numchar^2}, \cdots, \{[\onehot_j; \contextmat_\depth \onehot_j]\}_{j=1}^{\numchar^2}$ & $\think$,  $[\shiftmatrep_i^{(1)}; \mathbf{0}], \cdots, [\shiftmatrep_i^{(\inputlength/2)}; \mathbf{0}]$  \\
       Context matrix index indicator   & $\depth$ & $\{\levelembedding_{1}\}_{[1, \numchar^2]}, \{\levelembedding_{2}\}_{[1, \numchar^2]}, \cdots \{\levelembedding_{\depth}\}_{[1, \numchar^2]}$ & $\mathbf{0}, \mathbf{0}, \cdots, \mathbf{0}$ \\
       Start and End indicator   & $2$ & $\mathbf{0}, \cdots, \mathbf{0}$ & $\mathbf{0}$, $\startendembedding_1$, $\mathbf{0}, \cdots, \mathbf{0}$, $\startendembedding_2$ \\
       Segment & $2$ & $\segment_1, \cdots, \segment_1$ & $\mathbf{0}$, $\segment_2, \cdots, \segment_2$\\
    \bottomrule
    \end{tabular}
    }
    \caption{Output of \constructcircularshiftmodule{} in \constructmodule{} that simulates \circularshift{}, i.e. computes $\shiftfunc(\matrep_i)$ at second segment token embeddings. $\think$ represents a null output and won't be attended to in the future modules. We ignore this symbol for simplicity, when analyzing any module. $\onehot_j$ indicates a one-hot $\numchar^2$ dimensional vector that contains $1$ in dimension $j$.}
    \label{tab:output_attention_circularshift}
\end{table}

\subsection{Step 2 (\constructtranslatemodule{}): \translate{} as a module containing a self-attention and an MLP layer} Our current token embeddings are given as $\think$, $[\mathbf{o}_2; \mathbf{0}], [\mathbf{o}_3; \mathbf{0}], \cdots, [\mathbf{o}_{\inputlength/2+1}; \mathbf{0}]$, where each $\mathbf{o}_{j}$ contain $\shiftmatrep_i^{(j-1)}$. Other embeddings have been kept intact. The in-context information are given as  $\{[\onehot_j; \contextmat_1 \onehot_j]\}_{j=1}^{\numchar^2}, \cdots, \{[\onehot_j; \contextmat_\depth \onehot_j]\}_{j=1}^{\numchar^2}$. We will first use a self-attention layer to compute $[\contextmat_i \shiftmatrep_i^{(j-1)}; \shiftmatrep_i^{(j-1)}]$ at position $2 \leq j \leq \inputlength/2 + 1$. We then use an MLP layer to represent $\maxfunc(\contextmat_i + \weightmat_i)  \shiftmatrep_i^{(j-1)}.$

\paragraph{Self-attention layer to represent $[\contextmat_i \shiftmatrep_i^{(j-1)}; \shiftmatrep_i^{(j-1)}]$:} We will use two attention heads. 
\begin{enumerate}
    \item \textbf{The first attention head computes $\contextmat_i \shiftmatrep_i^{(j-1)}$:} Matrices $\mW_{query}$, $\mW_{key}$ are set such that the attention score between an token embedding in first segment $[\onehot_{r}; \contextmat_{\ell} \onehot_{r}]$ and a token embedding in second segment  $\mathbf{o}_j$ is given by
    \begin{align*}
        \langle \onehot_{r}, \shiftmatrep_i^{(j-1)} \rangle, \text{ if } \ell = i, \text{ and } 0 \text{ otherwise.}
    \end{align*}
    for any $j \in [2, \inputlength/2+1], r \in [1, \depth]$. The condition requires the model to attend to $\contextmat_i$ and ignore other in-context information. The condition can be set using the Context matrix index indicator vectors $\levelembedding_{\ell}$ which is present in each in-context information embedding.

    The attention between any two input sequence embedding  $\mathbf{o}_j$ and $\mathbf{o}_{j'}$ is computed as $0$s. The distinction between the attention scores of pairs of embeddings in second segment, $\mathbf{o}_j$ and $\mathbf{o}_{j'}$, v/s attention scores between a token embedding in second segment and a token embedding in first segment, $\mathbf{o}_j$ and $[\onehot_{r}; \contextmat_{\ell} \onehot_{r}]$, can be done by using the segment indicator embeddings $\segment_1$ and $\segment_2$ used to differentiate token embeddings in first segment and the input sequence embedding vectors.

    Matrix $\mW_{value}$ is set such that the columns of each $\contextmat_{\ell}$s are picked from the token embeddings in the first segment: $\{\{[\onehot_{r}; \contextmat_{\ell} \onehot_{r}]\}_{r=1}^{\numchar^2}\}_{\ell=1}^{\depth}.$

    \item \textbf{The second attention head simply copies the input $\shiftmatrep_i^{(j-1)}$:} This can be done with an attention head that attends to itself at each position $j$ and copies  $\shiftmatrep_i^{(j-1)}$ to output. 
\end{enumerate}

The output of the two attention heads are simply added up. The output embeddings will now look as follows: $\think$, $\{ [\contextmat_i \shiftmatrep_i^{(j-1)}; \shiftmatrep_i^{(j-1)}] \}_{j=1}^{\inputlength/2}$.

\paragraph{MLP to represent $\maxfunc(\contextmat_i + \weightmat_i)  \shiftmatrep_i^{(j-1)}$:}  Our current token embeddings at any position $j$ contain both $\contextmat_i \shiftmatrep_i^{(j-1)}$ and  $\shiftmatrep_i^{(j-1)}$. The first layer of MLP can be used to compute $(\contextmat_i + \weightmat_i) \shiftmatrep_i^{(j-1)}$ by setting the weights of the layer using $\weightmat_i$.
We simulate $\maxfunc$ operation as follows: 
\begin{align*}
    &\Tilde{\mathbf{o}}_j / \norm{ \Tilde{\mathbf{o}}_j }_2, \text{ where }    
    \Tilde{\mathbf{o}}_j = \text{GELU} ((\contextmat_i + \weightmat_i) \shiftmatrep_i^{(j-1)})
\end{align*}
The $\ell_2$ normalization is equivalent to RMSnorm operation \citep{zhang2019root}. This is an approximation of the $\maxfunc$ function, which are equivalent only under the following conditions: for each column $j$
\begin{enumerate}
    \item either $\contextmat_i^{(j)}$ or $\weightmat_i^{(j)}$ are all $0$s.
    \item $\contextmat_i^{(j)}$ and $\weightmat_i^{(j)}$ are both one-hot vectors and they match at the corresponding activated dimension.
\end{enumerate}

\newpage

\begin{algorithm}[H]
    \small
    \caption{ \constructcircularshiftmodule{}: Self-attention and MLP layers for \circularshift{}}\label{alg:circularshift}
    \begin{algorithmic} 
        \REQUIRE Input embeddings (Token, Context matrix index indicator, Start and End Indicator, and Segment embeddings split into $2$ segments) (\cref{tab:input_transformer_handconstr}). Important ones (for the current module) are
        \begin{enumerate}
                \item Token embeddings: First segment contains $\{[\onehot_j; \contextmat_1 \onehot_j]\}_{j=1}^{\numchar^2}, \cdots, \{[\onehot_j; \contextmat_\depth \onehot_j]\}_{j=1}^{\numchar^2}$ and the second segment contains $[\shiftfunc(\matrep_{i-1}^{(1)}); \mathbf{0}], \cdots, [\shiftfunc(\matrep_{i-1}^{(\inputlength/2)}); \mathbf{0}], \mathbf{0}$
                \item Start and End indicator: First segment contains all $0$s and second segments contains $\startendembedding_1$ and $\startendembedding_2$ at first and last embedding, while containing all $0$s everywhere else.
            \end{enumerate}
        \STATE \textbf{Step a:} Using a self-attention layer with $3$ attention heads, change the token embeddings in second segment as $\think, \{\mathbf{o}_j\}_{j=2}^{\inputlength/2+1}$ s.t. $\vo_j$ has $\rbr{\1_\numchar\otimes I_\numchar}^\top \matrep^{(j-1)}_i$ in $[1, \numchar]$ dimensions and $\rbr{I_\numchar\otimes \1_\numchar} \matrep^{(j\%\inputlength)}_i$ in $[\numchar+1, 2\numchar]$ dimensions. Primarily,
        \begin{itemize}
            \item Attention head 1: Computes attention score between any two positions $p_1, p_2$ as  $a_{p_1, p_2} = 1$ iff $p_2 - p_1 = -1$ and $0$ otherwise. Value matrix $\mW_{value}$ is set such that for any input $\vx$, the output of $\mW_{value} \vx$ is given by (for all $p_1 \in [1, \numchar]$)
            \begin{align*}
                (\mW_{value} \vx)_{p_1} = \sum_{j=0}^{\numchar-1} x_{ \numchar \cdot j + {p_1} }.  
            \end{align*}
            \item Attention head 2: Computes attention score between any two positions $p_1, p_2$ as  $a_{p_1, p_2} = 1$ iff $p_2 - p_1 = 0$ and $0$ otherwise. $\mW_{value}$ is set such that for any input $\vx$, the output of $\mW_{value} \vx$ is given by (for all $p_1 \in [1, \numchar]$)
            \begin{align*}
                (\mW_{value} \vx)_{{p_1} + \numchar} = \sum_{j=1}^{\numchar} x_{ j + {p_1} \numchar - \numchar }.  
            \end{align*}
            \item Attention head 3: Computes attention score between any two positions $p_1, p_2$  as  $a_{p_1, p_2} = 1$  iff $\startendembedding_2$ and $\startendembedding_1$ are present as at positions $p_1$ and $p_2$ respectively and $0$ otherwise. $\mW_{value}$ is set such that for any input $\vx$, the output of $\mW_{value} \vx$ is given by (for all $p_1 \in [1, \numchar]$)
            \begin{align*}
                (\mW_{value} \vx)_{{p_1} + \numchar} = \sum_{j=1}^{\numchar} x_{ j + {p_1} \numchar - \numchar }.  
            \end{align*}
        \end{itemize}
        Sum the output of the three heads.
        \STATE \textbf{Step b:} Use MLP layer to change the token embeddings in second segment as $\think, \{\mathbf{o}_j\}_{j=2}^{\inputlength/2+1}$ s.t. $\vo_j$ has $\shiftmatrep_i^{(j-1)}$ with some small error. Primary computation at each position $j$ is given as (for a large $N$)
        \begin{align*}
            &\sqrt{2 /\pi} N^2 \left( \text{GELU}(\frac{1}{N}\rbr{\1_\numchar\otimes I_\numchar}^\top \matrep^{(j-1)}_i + \frac{1}{N}\rbr{I_\numchar\otimes \1_\numchar} \matrep^{(j\%\inputlength)}_i )\right.\\&\qquad - \left. \text{GELU}(\frac{1}{N}\rbr{\1_\numchar\otimes I_\numchar}^\top \matrep^{(j-1)}_i ) - \text{GELU}(\frac{1}{N} \rbr{I_\numchar\otimes \1_\numchar} \matrep^{(j\%\inputlength)}_i ) \right),
        \end{align*}
        which will return $$ \shiftmatrep_i^{(j-1)} + \mathcal{O}(N^{-4}) :=  \rbr{\1_\numchar\otimes I_\numchar}^\top \matrep^{(j-1)}_i \odot \rbr{I_\numchar\otimes \1_\numchar} \matrep^{(j\%\inputlength)}_i + \mathcal{O}(N^{-4}) $$
    \end{algorithmic}
    Return the output embeddings (as given in \cref{tab:output_attention_circularshift}).
\end{algorithm}

\begin{algorithm}[H]
    \small
    \caption{ \constructtranslatemodule{}: Self-attention and MLP layers for \translate{}}\label{alg:translate}
    \begin{algorithmic} 
        \REQUIRE We will require an index, and input embeddings as input:
        \begin{itemize}
            \item  Index $i$ (indicating index of the \constructmodule{} it is a part of), 
            \item  Embeddings (Token, Context matrix index indicator, Start and End Indicator, and Segment embeddings split into $2$ segments) from the output of its preceding \constructcircularshiftmodule{} (\cref{tab:output_attention_circularshift}). Important ones (for the current module) are
            \begin{enumerate}
                \item Token embeddings: First segment contains $\{[\onehot_j; \contextmat_1 \onehot_j]\}_{j=1}^{\numchar^2}, \cdots, \{[\onehot_j; \contextmat_\depth \onehot_j]\}_{j=1}^{\numchar^2}$ and the second segment contains $\think$,  $[\shiftmatrep_i^{(1)}; \mathbf{0}], \cdots, [\shiftmatrep_i^{(\inputlength/2)}; \mathbf{0}]$
                \item Context matrix index indicator: First segment contains  $\{\levelembedding_{1}\}_{[1, \numchar^2]}, \{\levelembedding_{2}\}_{[1, \numchar^2]}, \cdots \{\levelembedding_{\depth}\}_{[1, \numchar^2]}$ and second segments contains all $0$s vectors. 
            \end{enumerate}
        \end{itemize}
      
        \STATE \textbf{Step a:} Using a self-attention layer, change token embeddings in the second segment as $\think$, $\{ [\contextmat_i \shiftmatrep_i^{(j-1)}; \shiftmatrep_i^{(j-1)}] \}_{j=1}^{\inputlength/2}$. 
        \begin{itemize}
            \item Primarily, the self-attention score between embeddings that contain a second segment token embedding $[ \shiftmatrep_i^{(p_2)}; \mathbf{0} ]$ and a first segment token embedding $[\onehot_{p_1}; \contextmat_r \onehot_{p_1}]$ 
        (for any $r \in [1, \depth]$, $p_1 \in [1, \numchar^2], p_2 \in [1, \inputlength]$) is computed as
        \begin{align*}
            \langle \onehot_{p_1}, \shiftmatrep_i^{(p_2)} \rangle \cdot \langle \levelembedding_r, \levelembedding_{i} \rangle,
        \end{align*}
        where $\levelembedding_r$ is the corresponding context matrix index indicator for the first segment embedding under consideration and $\levelembedding_i$ is constructed using the index $i$.
        \item $\contextmat_r \onehot_{p_1}$ is used as value vector from each first segment embeddings.
        \end{itemize}

        \STATE \textbf{Step b:} Using an MLP layer, change token embeddings in the second segment to contain $\think$, $\{\mathbf{o}_j\}_{j=2}^{\inputlength/2+1}$, where
        \begin{align*}
            &\mathbf{o}_j = \Tilde{\mathbf{o}}_j / \norm{ \Tilde{\mathbf{o}}_j }_2, \text{ where }  \Tilde{\mathbf{o}}_j = \text{GELU} ((\contextmat_i + \weightmat_i) \shiftmatrep_i^{(j-1)})
        \end{align*}
        \STATE Return the output embeddings (as given in \cref{tab:output_transformer_handconstr}).
    \end{algorithmic}
\end{algorithm}

\newpage
\subsection{Trainability of the constructed transformer}
How does our constructed transformer perform on the $\translationtask$ task? To do so, we hand-construct the designed transformer and train the transformer model on a random translation task in $\translationtask(5, 10)$, that has depth $5$ and number of characters $10$ in each translation level. We train only parameters  $\{\weightmat_i\}_{i=1}^{\depth}$ using Layer-wise and Full Parameter optimization algorithms from \cref{alg:mlt-surrogate-layerwisegd,alg:mlt-surrogate-fullparamgd} respectively. 

However, we make two changes. First, we use Adam optimizer instead of SGD, as we found SGD to get stuck frequently at bad minimas. Next, instead of masking in-context representation of layers in a rotating curriculum fashion, where each layer $i$'s in-context representation were masked in intervals $[ j \cdot (i-1) \numchar^2, j \cdot i \numchar^2 ]$ for all integers $j \ge 0$, we mix in the masking of all layers together by randomly picking a layer $i$ and mask a random column from $\contextmat_i$. We found that mixing in masking of all layers helps train the model faster. 

We vary the peak learning rate as $\{10^{-2}, 10^{-3}, 10^{-4}\}$. We use a cosine decay learning rate schedule with warmup steps $100$ and minimum learning rate at end of training as $0.1 \times$ the peak learning rate. We train with cross entropy loss, and use a total of $51200$ training sequences of length $20$ during the course of training. We use a batch size of $32$ per gradient update step. We also use a small $\ell_1$ regularization, with strength $10^{-4}$, on the rows of the parameters to help optimization.

\textbf{Observations:} We report the performance of a model trained with both Layer-wise and Full Parameter optimization algorithm with peak learning rate $10^{-3}$ and $\ell_1$ regularization with strength $10^{-4}$ in \cref{fig:Mechtransformer_D5C10}. We observe that under both algorithms, the trainable parameters quickly learn the relevant matrices that represent the true \strcodebook{} to minimize the training loss.

\begin{figure}[!htbp]
    \centering
    \includegraphics[width=0.8\linewidth]{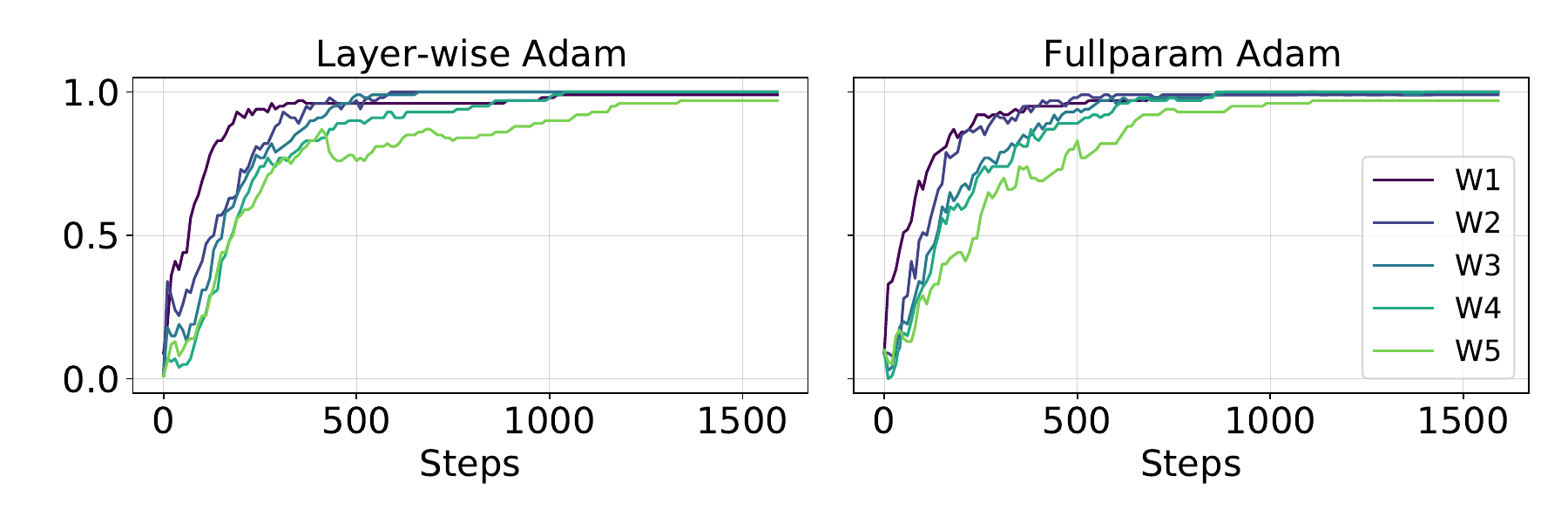}
    \caption{We perform Layer-wise and Full Parameter Training with Gradient Descent (defined in \cref{sec:apdx-surrogateModel-GDexp-apprGD}) on the handconstructed transformer (\cref{lem:construction}) that represents $\surrogateModel$ (\cref{defn:apdx-surrogate-surrogateModel}) designed for $\translationtask_\targetCB$ in $\translationtask(5,10)$ . Here, we report the portion of columns from the trainable parameters, which after $\maxfunc$ application $\{ \maxfunc\rbr{\weightmat_i}\}_{i=1}^{\depth}$,  align with the corresponding stochastic matrices of the \strdicts{} $\{\transmat\rbr{{\vardict^*_i}}\}_{i=1}^{\depth}$. We observe that under both algorithms, the trainable parameters quickly learn the relevant stochastic matrices.}
    \label{fig:Mechtransformer_D5C10}
\end{figure}

\newpage
\section{Additional Experiment Setups}
\label{sec:training_data_setups}

\subsection{Format of Training Text}
\label{sec:apdx-expconfig-inputformatting}

Here we present examples of training data used for experiments in \cref{sec:experiments}

\begin{figure}[H]
\begin{promptbox}
\input{files/Appendix/experiment_config/samples/cot}
\end{promptbox}
\caption{Input with complete context and explicit chain-of-thought. We use this data format at the begining of stage 1 training.}
\label{fig:apdx-expconfig-inputformat-explicitcot}
\end{figure}

\begin{figure}[H]
\begin{promptbox}
\input{files/Appendix/experiment_config/samples/internalized_cot}
\end{promptbox}
\caption{Input with complete context and internalized chain-of-thought. We use this data format by the end of stage 1 training as well as the base format for stage 2 training (before dropout)}
\label{fig:apdx-expconfig-inputformat-internalizedcot}
\end{figure}

\begin{figure}[H]
\begin{promptbox}
\input{files/Appendix/experiment_config/samples/masked}
\end{promptbox}
\caption{Input with completely masked context and internalized chain-of-thought. We use this data format to evaluate the model's capability on conducting translation without anything (useful information) in context.}
\label{fig:apdx-expconfig-inputformat-masked}
\end{figure}

\end{document}